%% file: main_vk.tex
\documentclass{article}

\usepackage{microtype}
\usepackage{graphicx}
\usepackage{subcaption}
\usepackage{booktabs} % for professional tables

\usepackage{xurl} % allow long bibliography URLs to break within margins
\usepackage{hyperref}
\usepackage{accsupp}

\let\origaddcontentsline\addcontentsline

\usepackage[accepted]{icml2026}

\usepackage{amsmath}
\usepackage{amssymb}
\usepackage{mathtools}
\usepackage{amsthm}

\usepackage{dblfloatfix} % better handling of wide floats in 2-col
\usepackage{placeins}    % gives \FloatBarrier
\usepackage{float}       % gives [H] for single-column figures
\usepackage{fltrace}
\usepackage{minitoc}                        % Appendix differentiated ToC
\usepackage{float}                          % Improved interface for floating objects
\usepackage{wrapfig}                        % For wrapping figures
\usepackage{amsmath, amsfonts, amssymb, bm} % Math environments and fonts
\usepackage{bbding}                         % Additional symbols   
\usepackage{mathtools,amsthm}               % Enhancements for amsmath and theorems
\usepackage{dsfont}                         % Double-struck math symbols
\usepackage{xifthen}                        % Conditional logic for commands
\usepackage{xfrac}                          % Fractional notation
\usepackage[all,cmtip]{xy}                  % Diagrams and arrows
\usepackage{enumitem}                       % Custom list formatting
\usepackage{graphicx}                       % Include graphics
\usepackage{booktabs}
\usepackage{tikz}
\usetikzlibrary{external}
\usepackage[nameinlink,capitalise,noabbrev]{cleveref}           % Smart referencing

\definecolor{darkblue}{rgb}{0.0, 0.0, 0.5}
\hypersetup{
	pdfsubject={Equivariant representation learning for regression, conditional probability estimation, and uncertainty quantification with statistical guarantees},
	pdfdisplaydoctitle=true,
	colorlinks=true,
	linkcolor=darkblue,
	citecolor=darkblue,
	urlcolor=darkblue
}
\newcommand{\FigureSummary}[2]{%
	\BeginAccSupp{method=escape,Alt={#1},ActualText={#1}}%
	#2%
	\EndAccSupp{}%
}
\Crefname{section}{Sec.}{Secs.}
\Crefname{subsection}{Subsec.}{Subsecs.}
\Crefname{subsubsection}{Subsubsec.}{Subsubsecs.}
\Crefname{figure}{Fig.}{Figs.}
\Crefname{table}{Tab.}{Tabs.}
\Crefname{equation}{Eq.}{Eqs.}
\Crefname{appendix}{App.}{Apps.}
\Crefname{theorem}{Thm.}{Thms.}
\Crefname{proposition}{Prop.}{Props.}
\Crefname{lemma}{Lemma}{Lemmas}
\Crefname{corollary}{Cor.}{Cors.}
\Crefname{definition}{Def.}{Defs.}
\makeatletter
\newtheoremstyle{customplain}%
{0.5em}%          % Space above (preskip)
{0.0em}%          % Space below (postskip)
{\itshape}%                         % Body font (italic for plain style)
{}%                                 % Indent amount (empty = no indent)
{\bfseries}%                        % Theorem head font (bold)
{.}%                                % Punctuation after theorem head
{.5em}%                             % Space after theorem head
{}%                                 % Theorem head spec (empty = normal)
\newtheoremstyle{customremark}%
{6pt plus 2pt minus 1pt}%           % Space above (preskip)
{6pt plus 2pt minus 1pt}%           % Space below (postskip)
{\normalfont}%                      % Body font (normal for remark style)
{}%                                 % Indent amount (empty = no indent)
{\bfseries}%                        % Theorem head font (bold)
{.}%                                % Punctuation after theorem head
{.5em}%                             % Space after theorem head
{}%                                 % Theorem head spec (empty = normal)

\makeatother

\theoremstyle{customplain}
\newtheorem{theorem}{Theorem}[section]
\newtheorem{proposition}[theorem]{Proposition}
\newtheorem{lemma}[theorem]{Lemma}
\newtheorem{corollary}[theorem]{Corollary}
\newtheorem{definition}[theorem]{Definition}
\newtheorem{assumption}[theorem]{Assumption}

\theoremstyle{customremark}
\newtheorem{remark}[theorem]{Remark}
\newtheorem{example}[theorem]{Example}

\input{acronyms.tex}

\input{notation.tex}

\input{ICLR_suggested_notation.tex}

\usepackage[textsize=tiny, linecolor=gray, bordercolor=gray]{todonotes} % TODO notes with controlled appearance
\newcommand{\customParagraph}[1]{{\bf #1}~~~}
\icmltitlerunning{
	Representation Learning for Equivariant Inference with Guarantees
}

\begin{document}

\twocolumn[
	\icmltitle{Representation Learning for Equivariant Inference with Guarantees}



	\icmlsetsymbol{equal}{*}

	\begin{icmlauthorlist}
		\icmlauthor{Daniel Ordoñez-Apraez}{iit,unige,equal}
		\icmlauthor{Vladimir R. Kostić}{iit,unins,equal}
		\icmlauthor{Alek Fröhlich}{iit,unige,equal} 
		\icmlauthor{Vivien Brandt}{ecole} \\
		\icmlauthor{Karim Lounici}{ecole,equal}
		\icmlauthor{Massimiliano Pontil}{iit,ucl,equal}
	\end{icmlauthorlist}

	\icmlaffiliation{iit}{CSML, Italian Institute of Technology, Genova, Italy}
	\icmlaffiliation{unige}{Università di Genova, Genova, Italy}
	\icmlaffiliation{unins}{University of Novi Sad, Novi Sad, Serbia}
	\icmlaffiliation{ecole}{CMAP, École Polytechnique, Palaiseau, France}
	\icmlaffiliation{ucl}{University College London, London, UK}

	\icmlcorrespondingauthor{Daniel Ordoñez-Apraez}{daniel.ordonez@iit.it}

	\icmlkeywords{Machine Learning, Geometric Deep Learning, Representation Learning, Equivariant Neural Networks, Conditional Probability Estimation, Uncertainty Quantification, Statistical Learning Theory}

	\vskip 0.3in
]


%\printAffiliationsAndNotice{}  % no special notice (required even if empty)
\printAffiliationsAndNotice{\icmlEqualContribution}

\newcommand{\fix}{\marginpar{FIX}}
\newcommand{\new}{\marginpar{NEW}}

% \iclrfinalcopy % Uncomment for camera-ready version, but NOT for submission.
%  Appendix TOC hack
\doparttoc % Tell to minitoc to generate a toc for the parts
\faketableofcontents % Run a fake tableofcontents command for the partocs
% \part{} % Start the document part

% \vspace{-0.1cm}
\begin{abstract}
	% In many real-world applications of regression, conditional probability estimation, and uncertainty quantification, exploiting symmetries rooted in physics or geometry can dramatically improve generalization and sample efficiency. While geometric deep learning has made significant empirical advances by incorporating group-theoretic structure, less attention has been given to statistical learning guarantees. In this paper, we introduce an equivariant representation learning framework that simultaneously addresses regression, conditional probability estimation, and uncertainty quantification while providing first-of-its-kind non-asymptotic statistical learning guarantees. Grounded in operator and group representation theory, our framework approximates the spectral decomposition of the conditional expectation operator, building representations that are both equivariant and disentangled along independent symmetry quotient groups. Empirical evaluations on synthetic datasets and real-world robotics applications confirm the potential of our approach, matching or outperforming existing equivariant baselines in regression while providing well-calibrated parametric uncertainty estimates.
	In many real-world applications of regression, conditional probability estimation, and uncertainty quantification, exploiting symmetries rooted in physics or geometry can dramatically improve generalization and sample efficiency. While geometric deep learning has made empirical advances by incorporating symmetry and geometry priors, less attention has been given to statistical learning guarantees. In this paper, we introduce an equivariant representation learning framework that simultaneously addresses regression, conditional probability estimation, and uncertainty quantification while providing first-of-its-kind non-asymptotic statistical learning guarantees. Grounded in operator and group representation theory, our framework approximates the spectral decomposition of the conditional expectation operator, building representations that are both equivariant and disentangled along independent symmetry quotient groups. Empirical evaluations on synthetic datasets and real-world robotics applications confirm the potential of our approach, matching or outperforming existing equivariant baselines in regression while providing well-calibrated uncertainty estimates.
	% \footnote{
	% 	\label{ref:code}
	% 	Code available at
	% 	\href{https://github.com/Danfoa/symm_rep_learn}{\textit{github.com/Danfoa/symm\_rep\_learn}}
	% }
\end{abstract}

\section{Introduction}

A central problem in machine learning is modeling conditional probabilities—understanding how the distribution of a target variable $\rvy$ changes with an observed variable $\rvx$. This underpins robust reasoning under uncertainty in critical applications such as medicine, finance, robotics, and physics \citep{Izbicki2025, Smith2013, Wasserman2007-ek}. However, estimating conditional distributions remains challenging in high-dimensional settings without strong inductive biases \citep{SCOTT1991, Nagler2016,izbicki2017converting}.

Symmetry priors, in the form of principled assumptions about invariance or equivariance in the underlying data-generating process, offer a compelling way to reduce sample complexity and improve generalization \citep{kashinath2021physics,wang2021incorporating,bronstein2021geometric,elesedy2025symmetrygeneralisationmachinelearning}. These priors naturally arise in inference tasks in chemistry and particle physics \citep{dresselhaus2007group_theory}, set-\&-graph structured data \citep{bronstein2021geometric}, computer graphics \citep{mitra2013symmetry,kovar2002motion}, and dynamical systems with group-invariant/equivariant laws of motion, which are ubiquitous in fields like physics \citep{dresselhaus2007group_theory}, fluid dynamics \citep{olver1993applications}, and robotics \citep{ordonez2024morphological,zhu2022sample}.

\begin{figure*}[t!]
	\centering
	\FigureSummary{Two-panel overview of the paper's robotics results: sample-efficient equivariant center-of-mass momentum regression and calibrated 90 percent uncertainty intervals for work and kinetic energy under disturbances.}{\includegraphics[width=\linewidth]{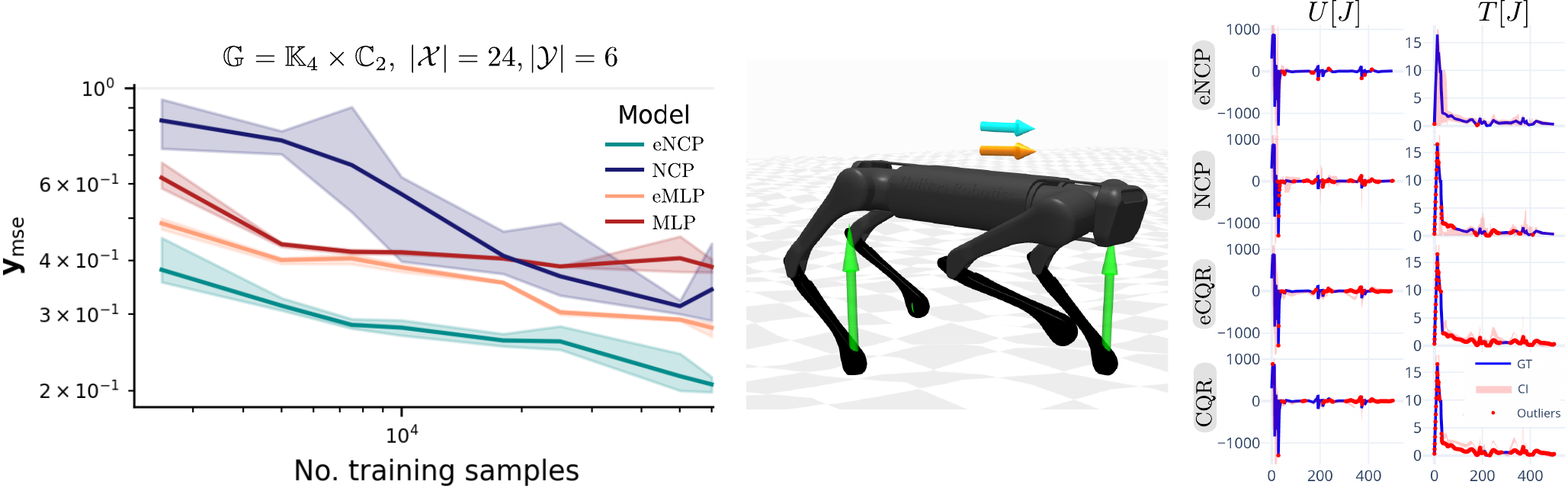}}
	%\label{fig:quadruped_uncertainty_quantification}
	% \vspace*{-0.7cm}
	\caption{
		\textbf{Left}: Test set sample efficiency for $\G$-equivariant regression (MSE vs. training samples) when predicting the $\G$-equivariant linear and angular momentum of a quadruped robot’s center of mass (\acrshort{com}) from noisy joint positions and velocities. \textbf{Right}: Uncertainty quantification via $\G$-equivariant prediction of $90\%$ \glspl{ci} (light-red area) for the robot's instantaneous work $U_t$ and kinetic energy $T_t$ during locomotion over rough terrain for our method (\acrshort{encp}) and competitors. The figure shows a trajectory with a strong initial disturbance, where blue markers denote samples within the predicted CI and red markers denote those outside. Note that only \acrshort{encp} is able to predict well-calibrated CI intervals that cover both the disturbance and recovery phases.
	}
	\label{fig:teaser}
	% \vspace*{-0.5cm}
\end{figure*}

Over the past few years, \gls{gdl} has produced a rich ecosystem of architectures that encode symmetries, achieving strong empirical performance across various supervised \citep{bronstein2021geometric,weiler2023EquivariantAndCoordinateIndependentCNNs,vanderpol2020mdp} and unsupervised tasks \citep{dangovski2022equivariant,keurti2023homomorphism,wang2024disentangled}. However, the field remains focused on application specific designs and architectural innovation, with limited understanding of how symmetry priors can be leveraged to \textit{learn representations with provable generalization guarantees}.

In this work, we take a different route: rather than proposing new architectures or solving specific inference tasks, we ask \textit{how to systematically learn symmetry-aware representations that best capture conditional structure in the data}. Specifically, how should equivariant networks be trained so that their learned features reveal conditional distributions, and how does the quality of these representations affect performance in downstream tasks such as regression, conditional probability estimation, and uncertainty quantification?

To answer these questions, we connect two fields rarely studied together: spectral contrastive learning \citep{Zou2013}, a self-supervised approach for learning deep representations via operator-theoretic models of conditional expectations \citep{tsai2021self,kostic2024learning}, and \gls{gdl} \citep{bronstein2021geometric}, which imposes symmetry and geometry priors through architectural constraints in \glspl{nn}. Our approach clarifies how symmetry constraints shape the representation space and improve generalization, opening new avenues for exchange between these fields. Concretely, we show that our method outperforms existing \gls{gdl} techniques on regression tasks and provides reliable uncertainty quantification for robot locomotion (\cref{fig:teaser}).

% \massi{Explain the fig a bit more}

% \begin{figure}[b!]
% 	\centering
% 	\begin{minipage}[t]{0.48\textwidth}
% 		\centering
% 		\includegraphics[width=\linewidth]{images/com_momentum_sample_eff.pdf}
% 		\label{fig:teaser}
% 	\end{minipage}%
% 	\hfill
% 	\begin{minipage}[t]{0.50\textwidth}
% 		\centering
% 		\includegraphics[width=\linewidth]{images/kin_work_uc.pdf}
% 		\label{fig:quadruped_uncertainty_quantification}
% 	\end{minipage}
% 	\vspace*{-0.5cm}
% 	\caption{
% 		\textbf{Left}: Test set sample efficiency for $\G$-equivariant regression (MSE vs. number of training samples) when predicting the $\G$-equivariant linear and angular momentum of a quadruped robot’s Center of Mass \gls{com} from noisy joint position %$\vq\in\R^{12}$ 
%         and velocity %$\dot{\vq}\in\R^{12}$ 
%         measurements. \textbf{Right}: Uncertainty quantification via $\G$-equivariant prediction of $90\%$ \glspl{ci} for the robot's instantaneous work $U_t$ and kinetic energy $T_t$ during locomotion over rough terrain. Red markers indicate samples out of the predicted CI. The two left columns show a trajectory with a strong disturbance, whereas the two right columns display steady-state stochastic dynamics.
% 	}
%     \label{fig:teaser}
% 	\vspace*{-0.2cm}
% \end{figure}

\customParagraph{Contributions}
{(1) Methodological framework}: We introduce \gls{encp}, the first framework to combine equivariant neural networks with operator-theoretic estimation of conditional distributions. {(2) Task-agnostic representation learning}: We show that any $\G$-equivariant architecture can be used to learn \textit{disentangled, symmetry-respecting representations} that generalize across diverse downstream inference tasks. {(3) Learning guarantees}: By linking the representation quality directly to sample complexity, we provide the \textit{first non-asymptotic statistical learning guarantees} for equivariant  regression and symmetry-aware conditional probability estimation and uncertainty quantification. {(4) Empirical results}:	\footnote{
		\label{ref:code}
		Code available at
		\href{https://github.com/Danfoa/symm_rep_learn}{\textit{github.com/Danfoa/symm\_rep\_learn}}
	} 
    On both synthetic and real-world robotics tasks, \gls{encp} consistently outperforms baselines, including contrastive methods \gls{ncp} \citep{kostic2024learning} and current equivariant models. In particular, \gls{encp} achieves state-of-the-art performance in the challenging robotics task of contact force estimation.

\customParagraph{Paper Structure}
\cref{sec:background} reviews modeling conditional probabilities with linear operators and \gls{ncp}. \cref{sec:problem_formulation} formally presents the symmetry priors we consider. \cref{sec:learning_cond_exp_operator} introduces our \gls{encp} learning framework. \cref{sec:learning_guarantees} outlines our theoretical learning guarantees. \cref{sec:experiments} showcases experiments on synthetic and real-world data. Furthermore, because the paper involves complex notation from probability, operator theory, and group theory, the appendices include a glossary of notation (\cref{sec:notation}) as well as detailed expositions on representation theory (\cref{sec:group_theory}), symmetric function spaces (\cref{sec:representation_theory_symmetric_function_spaces}), and equivariant linear operators (\cref{sec:equivariant_operators}). Finally, \cref{sec:related_works} offers an in-depth discussion of related work, contrasting our work with the literature across these rich fields.

%{\bf Notations}~~~We denote a random variable by $\rvx$, its realizations by $\vx\in\vsX$, its probability distribution by $\PP(\rvx)$ and measure by $\mux$. Expectations of functions of $\rvx$ as $\E_{\rvx}[f(\rvx)] = \int_{\vsX} f(\vx)\,P_\rvx(d\vx).$ 
%The same notation applies to other random variables such as $\rvy$. A linear operator is denoted by $\oE$, with its restriction in  finite dimensions as a matrix by $\mE$. Block-diagonal decomposition of a matrix as $\mE = \Oplus_{k=1}^n \mE_k$. The action of a symmetry transformation $g\in\G$ on a data point is given by $	g\Glact[\vsX]\vx := \rep[\vsX](g)\vx,$
%where $\rep[\vsX](g)$ is the group representation on $\vsX$ (\cref{def:left_group_action,def:group_representation}).

%%%%%%%%%%%%%%%%%%%%%%%%%%%%%%%%%%%%%%%%%%%%%%%%%%%%%%%%%%%%
%%%%%%%%%%%%%%%%%%%%%%%%%%%%%%%%%%%%%%%%%%%%%%%%%%%%%%%%%%%%
\section{Background} \label{sec:background}
\begin{figure*}[t!]
	% \begin{minipage}{0.49\textwidth}
	\centering
	\FigureSummary{Architecture comparison between the symmetry-agnostic NCP bilinear network and the equivariant ENCP bilinear network with equivariant feature maps and a block-diagonal conditional expectation operator.}{\includegraphics[width=\linewidth]{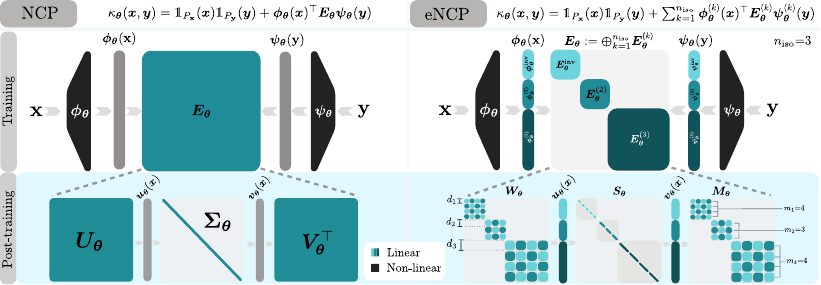}}
	% \vspace*{-0.6cm}
	\caption{
		\textbf{Left:} \gls{ncp}'s bilinear \gls{nn} architecture. \textbf{Right:} \gls{encp}'s $\G$-equivariant bilinear \gls{nn} architecture, featuring $\vbfxp$ and $\vbfyp$ as $\G$-equivariant \glspl{nn} and $\CEmp$ as a $\G$-equivariant block-diagonal matrix. Each block is equivariant to a quotient group $\G^{\isoIdxBrace} \subseteq \G$ and is constrained to have singular spaces of dimension at least $\irrepDim[\isoIdx]$—the minimal dimension for a representation of the action of $\G^{\isoIdxBrace}$.
	}
	% \vspace*{-0.4cm}
	\label{fig:diagrams}
\end{figure*}
%%%%%%%%%%%%%%%%%%%%%%%%%%%%%%%%%%%%%%%%%%%%%%%%%%%%%%%%%%%%
%%%%%%%%%%%%%%%%%%%%%%%%%%%%%%%%%%%%%%%%%%%%%%%%%%%%%%%%%%%%
We briefly review the operator-theoretic framework for modeling conditional probabilities, which underpins both \gls{ncp} and our proposed \gls{encp} method. We denote a random variable by $\rvx$, its realizations by $\vx \in \vsX$, its probability distribution by $\PP(\rvx)$, and measure by $\mux$. Expectations are written as $\E_{\rvx}[f(\rvx)] = \int_{\vsX} f(\vx)\,P_\rvx(d\vx)$. The same notation applies to other random variables such as $\rvy$.

\customParagraph{Modeling of Conditional Probabilities} \citet{kostic2024learning} proposed to model conditional probabilities by approximating the \textit{conditional expectation operator} \citep{Baker1973-ns,song2009hilbert,ryu2024operator}, $\CE \colon \LpyFull \to \LpxFull$, a linear integral operator acting on the Hilbert spaces  $\LpxFull:=\Lp[2]{\mux}(\vsX,\R)$ and $\LpyFull:=\Lp[2]{\muy}(\vsY,\R)$ of square-integrable functions of the random variables $\rvx$ and $\rvy$, respectively. The action of this operator on any function $h \in \LpyFull$ returns the function's conditional expectation:
\begin{equation}
	\label{eq:op_form_cond_mean_reg}
	\small
	\begin{aligned}
		[\CE h](\vx)
		 & =
		\E[h(\rvy) | \rvx \!=\! \vx]
		:=
		\int_{\vsY} h(\vy) P_{\rvy\vert\rvx}(d\vy\vert \vx)
		\\
		 & =
		\int_{\vsY} h(\vy) \tfrac{P_{\rvy\rvx}(d\vy,d\vx)}{P_{\rvx}(d\vx)}
		=
		\int_{\vsY} h(\vy) \pmd(\vx, \vy) \muy(d\vy),
	\end{aligned}
\end{equation}
%\karim{I am not comfortable with notation $P(dy)/Q(dx)$. Usually, we write $f(x) = (dP/dQ)(x)$.}
where $P_{\rvy\mid\rvx}$ is the conditional probability measure, and $\pmd(\vx,\vy) {:=} \frac{\muxy(d\vx,d\vy)}{\mux(d\vx)\,\muy(d\vy)}$ is the \gls{pmd} \citep{sugiyama2012density} kernel defining $\CE$, obtained as the Radon-Nikodym derivative of the joint measure to the product of marginal measures (see \cref{fig:symmetry_priors_long,sec:conditional_expectation_operator}).

The conditional expectation operator is significant because it provides an infinite-dimensional linear model—in a nonlinear representation space—for computing conditional probabilities and expectations. To see this, note that for any $\vx \in \vsX$ and any measurable set $\sB \subset \vsY$ we have that:
\begin{equation}
	\label{eq:uses_of_conditional_expectation_operator}
	\small
	\begin{split}
		\PP(\rvy \in \sB \vert \rvx {=} \vx)
		:=
		\int_{\vsY} \one_{\sB}(\vy) \muycondx(d\vy \vert \vx)
		=
		[\CE \one_{\sB}](\vx),
		\;\; \text{and}
		\\
		\E[\rvy | \rvx\! =\! \vx]
		:=
		[\CE \rvy](\vx). \qquad
	\end{split}
\end{equation}
Therefore, to estimate conditional probabilities and expectations, \gls{ncp} seeks the best finite-dimensional approximation of $\CE$. As we explain next, this gives rise to a representation learning problem \citep{oord2018representation}, in which the optimal representations of $\rvx$ and $\rvy$ are given by the top left and right singular functions of $\CE$ \citep{kostic2024neural,ryu2024operator}.

\customParagraph{Spectral Representation Learning}
The problem of approximating the conditional expectation operator $\CE$ as a rank-$r$ operator $\CEp$ with matrix representation $\CEmp \in \R^{r \times r}$ is defined as:
\begin{equation}
	\label{eq:learning_pmd}
	\small
	\begin{aligned}
		\argmin_{\nnParams}\; &
		\|\CE {-} \CEp \|_{\HS}^2
		=
		\E_{\rvx}\E_{\rvy} (\pmd(\rvx,\rvy) - \pmd_\nnParams(\rvx,\rvy))^2,
		\\
		% & 
		\text{s.t.}           & \;
		\E_{\rvx}\E_{\rvy}\pmd_\nnParams(\rvx,\rvy) \!=\!1
		\;\text{ and }\;
		\rank(\CEp)\leq r+1.
		%\;
		%\pmd_\nnParams(\cdot,\cdot) \geq 0
		% &\!=\! -2\E_{\rvx\rvy} \pmd_\nnParams(\rvx,\rvy) + \E_{\rvx}\E_{\rvy} \pmd_\nnParams(\rvx,\rvy)^2 + \gamma\;\Omega(\nnParams),
	\end{aligned}
\end{equation}
The optimal solution, denoted $\oE_{\star}$, is the $r$-truncated \gls{svd} of $\CE$ \citep{eckart1936approximation,weidmann2012linear,Baker1973-ns}, namely
\begin{equation}
	\label{eq:optimal_representations}
	\small
	\begin{aligned}
		[\oE_{\star}f](\vx)
		= &
		\textstyle{\sum_{i=0}^r} \sval_i\,\innerprod[\muy]{f,\rsfn_i} \lsfn_i(\vx),%\subset \LpyFull,
		\quad  \text{with}                                                \\
		  & \sval_i \lsfn_i(\vx) = [\CE \rsfn_i](\vx),\; \forall i\in[r],
	\end{aligned}
\end{equation}
% \newline 
where $(\sval_i, \lsfn_i, \rsfn_i)$ denotes the $i^{\text{th}}$ singular value and left/right singular functions of $\CE$, with $(\sval_0{=}1, \lsfn_0{=}\one_{\mux}, \rsfn_0{=}\one_{\muy})$ being the constant functions supported on $\mux$ and $\muy$.

Consequently,
% In \eqref{eq:learning_pmd},  % DRF methods parameterize differently we need to clarify.
\gls{ncp} parameterizes
$\CEp$ by a bilinear model
$
	\pmd_{\nnParams}(\vx,\vy) = \one_{\mux}(\vx)\one_{\muy}(\vy) + \vbfxp(\vx)^\top \CEmp \vbfyp(\vy)
$, composed of two encoder \glspl{nn} $\vbfxp \colon \vsX \to \R^r$ and $\vbfyp \colon \vsY \to \R^r$ that aim to approximate the \textit{span} of the top $r$ (non-constant) left and right singular functions of $\CE$. See \cref{fig:diagrams}-left. %and \citep{}.

As $\pmd$ is generally unavailable analytically, \eqref{eq:learning_pmd} is solved via the regularized \gls{clora} loss:

\begin{equation}
	\label{eq:contrastive_loss}
	\small
	\begin{split}
		\small
		\mathcal{L}_{\gamma}(\nnParams)	 &\,=\, -2\E_{\rvx\rvy} \pmd_\nnParams(\rvx,\rvy) + \E_{\rvx}\E_{\rvy} \pmd_\nnParams(\rvx,\rvy)^2
		\\
		&+2 \gamma
		\big(\| \E_{\rvx}\vbfxp(\rvx)\|_F^2 + \| \E_{\rvy}\vbfyp(\rvy)\|_F^2
		\\
		&\hspace{.5truecm} +\,\| \Cov(\vbfxp) - \Identity_r\|_F^2
		+ \| \Cov(\vbfyp) - \Identity_r\|_F^2\big),
	\end{split}
\end{equation}

where the first two regularization terms center the learned representations, ensuring that $\E_{\rvx}\E_{\rvy}\pmd_\nnParams(\rvx,\rvy){\approx} 1$ \citep{kostic2024learning}, while the last two enforce approximate orthonormality of the learned bases in $\Lpxp :=\operatorname{span}(\vbfxp)\subset\LpxFull$ and $\Lpyp := \operatorname{span}(\vbfyp) \subset\LpyFull$ \citep{izbicki2017converting}.
A key property of \gls{ncp} is that the learned representations  enables reliable regression and conditional probability estimation—and thus uncertainty quantification—via \eqref{eq:uses_of_conditional_expectation_operator}.
% A key property of \gls{ncp} is that, once representations are learned, it enables reliable regression, conditional probability estimation, and uncertainty quantification via \eqref{eq:uses_of_conditional_expectation_operator} \citep{kostic2024learning}.
% \massi{This was fast - if there is space recall briefly a bit how it enables one, e.g. " For instance conditional probability is computed as...}

%where the first two regularization terms center the learned representations—ensuring that $\E_{\rvx}\E_{\rvy}\pmd_\nnParams(\rvx,\rvy)\approx1$ (see \citep{kostic2024learning}), while the last two promote the formation of orthonormal bases for the learned representation spaces $\Lpxp :=\operatorname{span}(\vbfxp)\subset\LpxFull$ and $\Lpyp := \operatorname{span}(\vbfyp) \subset\LpyFull$. % \vspace{2cm}
%	{\bf Learning guarantees}~~~
%A crucial property of \gls{ncp} is that once representation learning is accomplished, one can reliably solve regression tasks as well as estimate conditional probabilities and uncertainty quantification via \eqref{eq:uses_of_conditional_expectation_operator}; see \citep{kostic2024learning}. 

%The learning bounds are shaped by quality of the learned representations---measured via the operator norm error
%$	\error^r	=	\|\CE - \CEmp\|_{\text{op}}
%\le
%\scalebox{0.85}{$\sqrt{\loss_{\gamma}(\nnParams) - \loss_{\gamma}^*}$}$ w.r.t. (under mild conditions) the unique minimizer $\CE^{\star,r}$, c.f.\citep{kostic2024learning}. \daniel{Star symbol never used/defined in body}

\section{Problem Formulation} \label{sec:problem_formulation}
% \textbf{Probabilistic symmetry priors}~~~ 
\begin{figure*}[t!]
	\centering
	\FigureSummary{Illustration of a reflection-symmetric joint distribution showing invariant marginals, invariant conditionals, and an invariant probability mass density under sign flips in x and y.}{\includegraphics[width=\textwidth]{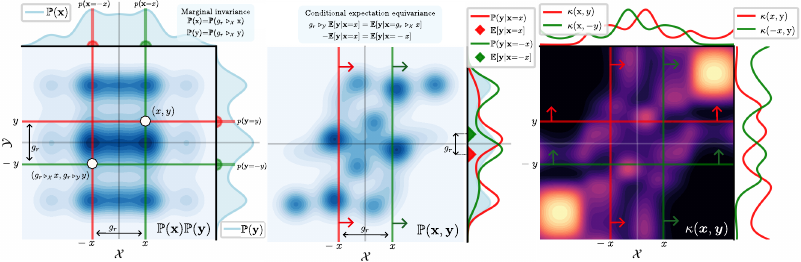}}
	\vspace*{-0.6cm}
	\caption{
	Example of symmetric random variables $(\rx, \ry)\sim \vsX\times\vsY\subset\R\times\R$, whose marginals $\PP(\rx)$ and $\PP(\ry)$; joint $\PP(\rx,\ry)$; and conditional $\PP(\ry \vert \rx)$ distributions are invariant to reflections of the data: $\g_r \Glact[\vsX] x=-x$ and $\g_r \Glact[\vsY] y=-y$, where $\g_r$ denotes the reflection element of the reflection symmetry group $\CyclicGroup[2]:=\{e,g_r \vert g_r^{\tiny2}=e\}$. Consequently, the \acrshort{pmd} $\pmd(x,y)$ %$:=\sfrac{\PP(x,y)}{\PP(x)\PP(y)}$ 
	is $\CyclicGroup[2]$-invariant.
	}
	\vspace*{-0.4cm}
	\label{fig:symmetry_priors_long}
\end{figure*}
This paper tackles the problem of estimating the  conditional expectation $\E[\rvy\vert\rvx \,{=}\, \cdot]$, and, %\massi{Why? I think the operator is more general, no?} 
more generally, conditional distribution $ \PP(\rvy \vert \rvx) $, for random variables $ \rvx \,{\in}\, \vsX $ and $ \rvy \,{\in}\, \vsY $, under the assumption that\footnote{
	Throughout, with some abuse of notation we denote by $\PP(\rvx)$ and $\PP(\rvy \vert \rvx)$ both the probability and conditional probability, respectively, as well as the corresponding densities, when they exist.
} $ \PP(\rvy \vert \rvx) $ and $ \PP(\rvx) $ are $\G$-invariant under \emph{known} symmetry transformations of the data (as depicted in \cref{fig:symmetry_priors_long}), i.e:
%
%there is prior knowledge about the invariance of both the conditional and marginal distributions under symmetry transformations of the data (see \cref{fig:symmetry_priors_long}). That is, if prior knowledge indicates that $ \PP(\rvy \vert \rvx) $ and $ \PP(\rvx) $ are $\G$-invariant:
%% Throughout this paper, we study \textit{symmetric random variables} $\rvx \in \vsX$ and $\rvy \in \vsY$, whose marginal and conditional distributions are invariant under the transformations of a \textit{finite symmetry group} $\G$ (see \cref{def:group}), as illustrated in \cref{fig:symmetry_priors,fig:symmetry_priors_long}. That is, whenever the following distributions' $\G$-invariance properties (\cref{def:equivariantMaps}) hold:
%
\begin{equation}
	\label{eq:main_assumtions}
	\small
	\PP(\rvy \vert \rvx)
	\!=\!
	\PP(\g \Glact[\vsY] \rvy \vert \g \Glact[\vsX] \rvx),
	% \\
	\quad
	\PP(\rvx) = \PP( \g \!\Glact[\vsX] \rvx),
	\quad
	\forall\; \g \in \G,
\end{equation}
where $\G$ denotes a finite symmetry group (\cref{def:group}) acting on the data spaces $\vsX$ and $\vsY$ via the group actions, $\Glact[\vsX]: \G \times \vsX \to \vsX,$ and $\Glact[\vsY]: \G \times \vsY \to \vsY$, with $\g \Glact[\vsX]\vx \in \vsX$ and $\g \Glact[\vsY]\vy \in \vsY$ denoting linear, invertible transformations of $\vx$ and $\vy$ defined by $\g \in \G$ (see \cref{fig:symmetry_priors_long,def:left_group_action}).

These priors imply the $\G$-invariance of the joint distribution $\PP(\rvx, \rvy)$ and of $\rvy$'s marginal distribution $\PP(\rvy)$, as well as the $\G$-equivariance of conditional expectations (see \cref{fig:symmetry_priors_long}-middle and \cref{prop:equivariant_conditional_expectation}):
\begin{equation}
	\label{eq:conditional_expectation_equivariance}
	\small
	% \PP(\rvy, \rvx)
	% {=} 
	% \PP(\g \Glact[\vsY] \rvy, \g \Glact[\vsX] \rvx),
	% \quad\;
	% \PP(\rvy)
	% {=}
	% \PP(\g \Glact[\vsY] \rvy),
	% \quad\;
	\g \Glact[\vsY] \E[\rvy \vert \rvx {=} \vx]
	{=}
	\E[\rvy \vert \rvx {=} \g \Glact[\vsX] \vx]
	\quad\;
	\forall\; \g \in \G, \vx \in \vsX.
\end{equation}
Note that \eqref{eq:conditional_expectation_equivariance} implies the $\G$-equivariance of the regression function
$\vx \mapsto \E[\rvy \vert \rvx \,{=} \,\vx]$ (see \cref{fig:symmetry_priors_long}-middle).
%$\vy = \vf(\cdot) {:=} \E[\rvy \vert \rvx \,{=} \,\cdot\;]$. 
Therefore, the symmetry priors
of \eqref{eq:main_assumtions} are satisfied whenever we approximate an equivariant/invariant function—that is, in virtually \emph{all applications of \gls{gdl}} \citep{bronstein2021geometric}.

The above symmetry priors represent a strong inductive bias for the conditional expectation operator \eqref{sec:background}, as they lead the \gls{pmd} kernel defining the operator to be $\G$-invariant (see \cref{fig:symmetry_priors_long}-right):
\begin{equation}
	\label{eq:pmd_invariance}
	\small
	% \tfrac{\muxy(d\vx,d\vy)}{\mux(d\vx)\,\muy(d\vy)} = 
	\pmd(\vx, \vy)
	=
	\pmd(\g \Glact[\vsX] \vx, \g \Glact[\vsY] \vy)
	% = \tfrac{\muxy(d\g \Glact[\vsX] \vx,d\g \Glact[\vsY] \vy)}{\mux(d\g \Glact[\vsX] \vx)\,\muy(d\g \Glact[\vsY] \vy)}
	\quad
	\forall\; \g \in \G, \vx \in \vsX, \vy \in \vsY.
\end{equation}
% 
%In \cref{sec:learning_cond_exp_operator}, we extend the \gls{ncp} framework \citep{kostic2024learning} by leveraging \eqref{eq:pmd_invariance} to incorporate the above symmetry priors. As we shall see, this enables efficient use of \gls{gdl} architectures to estimate the $\G$-invariant conditional probabilities in \eqref{eq:main_assumtions} and $\G$-equivariant regression in \eqref{eq:conditional_expectation_equivariance}, via \eqref{eq:uses_of_conditional_expectation_operator}, with strong learning guarantees.
% 
In \cref{sec:learning_cond_exp_operator}, we extend the \gls{ncp} framework \citep{kostic2024learning} by leveraging \eqref{eq:pmd_invariance} to incorporate symmetry priors. This enables efficient estimation of $\G$-invariant conditional probabilities \eqref{eq:main_assumtions} and $\G$-equivariant regression \eqref{eq:conditional_expectation_equivariance} using \gls{gdl} architectures, via \eqref{eq:uses_of_conditional_expectation_operator}, with strong learning guarantees.

%%%%%%%%%%%%%%%%%%%%%%%%%%%%%%%%%%%%%%%%%%%%%%%%%%%%%%%%%%%%%
%%%%%%%%%%%%%%%%%%%%%%%%%%%%%%%%%%%%%%%%%%%%%%%%%%%%%%%%%%%%%

\section{ENCP Method for Equivariant Representation Learning}\label{sec:learning_cond_exp_operator}
%%%%%%%%%%%%%%%%%%%%%%%%%%%%%%%%%%%%%%%%%%%%%%%%%%%%%%%%%%%%
%%%%%%%%%%%%%%%%%%%%%%%%%%%%%%%%%%%%%%%%%%%%%%%%%%%%%%%%%%%%
In this section, we show how to incorporate the symmetry priors \eqref{eq:main_assumtions} into \gls{ncp}'s representation learning framework. First, we analyze the symmetry constraints on the infinite-dimensional conditional expectation operator and prove that, for symmetric random variables $\rvx$ and $\rvy$, the optimal solution of \eqref{eq:learning_pmd} yields $\G$-equivariant representations $\vbfxp$ and $\vbfyp$ and approximates the operator with a $\G$-equivariant matrix $\CEmp$. Then, we explain how to embed these structural constraints into the bilinear neural network architecture of \gls{ncp} using \emph{any} type of equivariant \glspl{nn}.

% First, we characterize how the $\G$-invariance of the marginals $\mux$ and $\muy$ imprints symmetries on the function spaces $\LpxFull$ and $\LpyFull$. Next, we show that the $\G$-invariance of conditional distributions results in the conditional expectation operator $\CE$ is $\G$-equivariant. Finally, ...
%  we describe how this structure of $\CE$, $\LpxFull$, and $\LpyFull$ can be leveraged to solve the operator learning problem in \cref{eq:learning_pmd}.
% \begin{wrapfigure}{r}{0.5\textwidth}
% 	\centering
% 	\vspace*{-0.35cm}
% 	\includegraphics[width=\linewidth]{images/isotypic_decomposition.pdf}
% 	\vspace*{-0.6cm}
% 	\caption{
% 		TODO: Example group action and orthogonality decomposition of the function spaces $\LpxFull$ and $\LpyFull$ into isotypic subspaces $\LpxFullIso$ and $\LpyFullIso$. Image will be polished.
% 	}
% 	\vspace*{-0.5cm}
% 	\label{fig:ncp_diagram}
% \end{wrapfigure}

%\subsection{Symmetry constraints of the conditional expectation operator}
%\label{sec:equiv_conditional_expectation_operator}
\customParagraph{Symmetric Function Spaces}
The assumption of $\G$-invariance of the marginal probabilities (\cref{sec:problem_formulation}) implies that the function spaces $\LpxFull$ and $\LpyFull$ are symmetric Hilbert spaces of $\G$-equivariant functions, as these inherit unitary group actions
$
	\Glact[{\LpxFull}]\colon \G \times \LpxFull \to \LpxFull
$ and
$
	\Glact[{\LpyFull}]\colon \G \times \LpyFull \to \LpyFull
$
defined via the push-forward of symmetry transformations of the data spaces (see details in \cref{sec:representation_theory_symmetric_function_spaces} and in \cref{fig:symmetry_action_on_f_space}):

\begin{equation}
	\small
	\label{eq:group_action_funct_space_main}
	\begin{aligned}
		\g \Glact[{{\LpxFull}}] f (\cdot) & := f(\g^{-1} \Glact[\vsX] \cdot) \in \LpxFull,
		\;\;\text{ and }
		\\
		                                  & \g \Glact[{{\LpyFull}}] h (\cdot) := h(\g^{-1} \Glact[\vsY] \cdot) \in \LpyFull, \quad \forall \g \in \G.
	\end{aligned}
\end{equation}

A fundamental property of $\G$-symmetric Hilbert spaces is their orthogonal decomposition into $\isoNum \leq |\G|$ subspaces referred to as \textit{isotypic subspaces}:
$
	\LpxFull {=} \bm{\oplus}^{\perp}_{\isoIdx\in[1,\isoNum]} \LpxFullIso
$, and
$
	\LpyFull {=} \bm{\oplus}^{\perp}_{\isoIdx\in[1,\isoNum]} \LpyFullIso
$ (see \cref{thm:isotypic_decomposition}).
% \begin{equation}
% 	\small
% 	\label{eq:decomposition_lpx_lpy_isotypic_subspaces}
% 	\LpxFull = \bm{\oplus}^{\perp}_{\isoIdx\in[1,\isoNum]} \LpxFullIso,
% 	\quad
% 	\LpyFull = \bm{\oplus}^{\perp}_{\isoIdx\in[1,\isoNum]} \LpyFullIso,
% \end{equation}
Where each $\LpxFullIso$ and $\LpyFullIso$ denote the spaces of $\G^{\isoIdxBrace}$-equivariant functions of $\rvx$ and $\rvy$, with $\G^{\isoIdxBrace} := \G/\irrepKernel[\isoIdx]$ being a \textbf{quotient subgroup}, generated by a normal subgroup $\irrepKernel[\isoIdx]$ defined by the kernel of the group's $\isoIdx^{\text{th}}$ irreducible representation (see \cref{sec:isotypic_decomposition_and_disentangled_representations}).

This standard result from harmonic analysis \citep{mackey1980harmonic} enables us to express any $\G$-equivariant function as a sum of its projections onto the isotypic subspaces:

\begin{equation}
	\small
	\label{eq:decomposition_lpx_lpy_isotypic_subspaces}
	\begin{aligned}
		f(\cdot) = & f^{\inv}(\cdot) + \sum_{\isoIdx=2}^{\isoNum} f^{\isoIdxBrace}(\cdot),
		\quad
		h(\cdot) = h^{\inv}(\cdot) + \sum_{\isoIdx=2}^{\isoNum} h^{\isoIdxBrace}(\cdot),
		\\
		           & \text{with }\;
		f^{\isoIdxBrace} {\in} \LpxFullIso, h^{\isoIdxBrace} {\in} \LpyFullIso, \forall\; \isoIdx\in[\isoNum],
	\end{aligned}
\end{equation}

where $f^{\isoIdxBrace}$ and $h^{\isoIdxBrace}$ denote the $\G^{\isoIdxBrace}$-equivariant components of $f$ and $h$. Moreover, by convention, we associate the first subspace ($\isoIdx=1$) with the space of $\G$-invariant functions, i.e., $\G^{(1)}{=}\G^{\inv}{=}\{e\}$ (see \cref{ex:finite_dimensional_symmetric_function_space} in the Appendix).
% 
% , i.e., $\g \Glact[\LpxFull] f = f$ and $\g \Glact[\LpyFull] h = h$ for all $\g \notin \G^{\isoIdxBrace}$, $f \in \LpxFullIso$, and $h \in \LpyFullIso$ (see \cref{ex:finite_dimensional_symmetric_function_space}).

% This enable us to express any $\G$-equivariant function in terms of its projections onto the isotypic subspaces: $f = \sum_{\isoIdx=1}^{\isoNum} f^{\isoIdxBrace}$ and $h = \sum_{\isoIdx=1}^{\isoNum} h^{\isoIdxBrace}$, where $f^{\isoIdxBrace} \in \LpxFullIso$, and $h^{\isoIdxBrace} \in \LpyFullIso$. Note that, by convention, we associate the first subspace, $\isoIdx=1$, to the \textit{trivial} subgroup $\G^{(1)}=\{e\}$ and denote these spaces by $\LpxFullIso[\inv]$ and $\LpyFullIso[\inv]$ as these are composed of $\G$-invariant functions (see \cref{ex:finite_dimensional_symmetric_function_space}).
% ____________________

\customParagraph{Equivariant Conditional Expectation Operator}
The $\G$-invariance of the \gls{pmd} kernel \eqref{eq:pmd_invariance}, implies that $\CE$ is a $\G$-equivariant linear operator (see \cref{def:equiv_operator}). This means that $\CE$ commutes with the group action on the function spaces, and consequently due to Shur's lemma (\cref{lemma:schursLemma}), can be decomposed (disentangled) into a direct sum of operators acting on the corresponding isotypic subspaces (see details in \cref{sec:equivariant_operators}), that is:
\begin{equation}
	\label{eq:G_equiv_cond_exp_operator}
	\small
	\begin{aligned}
		\g \Glact[{{\LpxFull}}] [\CE h](\cdot)
		=\, &
		\CE [\g \Glact[{{\LpyFull}}] h](\cdot)
		\quad \forall\;
		h \in \LpyFull, \g \in \G,
		\;\;\text{ gives }
		\\
		    & [\CE h](\cdot)
		=
		\sum_{\isoIdx=1}^{\isoNum} [\CEIso h^{\isoIdxBrace}](\cdot)
	\end{aligned}
\end{equation}
where each $\CEIso \colon \LpyFullIso \to \LpxFullIso$ models the conditional expectation for $\G^{\isoIdxBrace}$-equivariant functions.

\customParagraph{Equivariant Representation Learning}
The $\G$-equivariant structure of $\CE$ and its disentanglement \eqref{eq:G_equiv_cond_exp_operator} into isotypic components suggests that computing the conditional expectation of a $\G$-equivariant function is equivalent to summing the conditional expectations of its $\G^{\isoIdxBrace}$-equivariant components for all $\isoIdx\in[\isoNum]$. Therefore, the loss function of problem \eqref{eq:learning_pmd}, where $\CE$ is approximated in finite dimensional spaces $\Lpxp$ and $\Lpyp$, decouples into $\isoNum$ independent (disentangled) components:
\begin{equation}
	\label{eq:learning_pmd_symm}
	\small
	\begin{aligned}
		\begin{aligned}
			\argmin_{\nnParams}\;
			\|\CE {-} \CEp \|_{\HS}^2
			 & =
			\textstyle{\sum_{\isoIdx=1}^{\isoNum}}%\sum_{\isoIdx=1}^{\isoNum} 
			\|\CEIso {-} \CEpIso \|_{\HS}^2
			\\
			 & =
			\E_{\rvx}\E_{\rvy}
			\textstyle{\sum_{\isoIdx=1}^{\isoNum}}
			(\pmd^\isoIdxBrace(\rvx,\rvy) {-} \pmd^\isoIdxBrace_\nnParams(\rvx,\rvy))^2,
		\end{aligned}
		\\
		\begin{aligned}
			% \quad 
			\text{s.t. } \;
			%  & 
			%  \pmd_\nnParams(\vx,\vy) = \textstyle{\sum_{\isoIdx=1}^{\isoNum}} \pmd^\isoIdxBrace_\nnParams(\vx,\vy), 
			%  \\
			 & \pmd_\nnParams(\g \Glact[\vsX] \vx, \g \Glact[\vsY] \vy) = \pmd_\nnParams(\vx,\vy),
			\\
			 & \E_{\rvx}\E_{\rvy}\pmd_\nnParams(\rvx,\rvy) = 1,
			\;\;
			\forall \g \in \G, (\vx,\vy) \in \vsX\times\vsY.
		\end{aligned}
	\end{aligned}
\end{equation}
Importantly, to satisfy the $\G$-invariance constraint on $\pmd_\nnParams$, the truncated operator $\CEp\colon \Lpyp {\to} \Lpxp$ must act on symmetric finite-dimensional Hilbert spaces that are stable under $\G$, i.e., $g \Glact[\LpxFull] f \in \Lpxp$ and $g \Glact[\LpyFull] h \in \Lpyp$ for all $f \in \Lpxp$ and $h \in \Lpyp$, and have a $\G$-equivariant matrix representation $\CEmp$ (see \cref{remark:invariant_kernel_implies_stable_basis}). Thus, as in the infinite-dimensional case, these finite-dimensional spaces decompose into $\isoNum$ isotypic subspaces
$\Lpxp {=} \Oplus_{\isoIdx=1}^{\isoNum} \LpxpIso$
and
$\Lpyp {=} \Oplus_{\isoIdx=1}^{\isoNum} \LpypIso$. Accordingly, the truncated operator matrix decomposes block-diagonally into $\isoNum$ blocks, i.e., $\CEmp = \Oplus_{\isoIdx=1}^{\isoNum} \CEmpIso$, where each $\isoIdx$-th block needs to be a $\G^{\isoIdxBrace}$-equivariant matrix approximating the restriction of the conditional expectation operator on its corresponding isotypic subspace  \citep{salova2019koopman,ordonez2024dynamics}, see \cref{fig:diagrams}-right.

Analogous to \eqref{eq:optimal_representations}, the optimal truncation of $\CEIso$ is given by its truncated \gls{svd} \citep{weidmann2012linear}. However, the $\G^{\isoIdxBrace}$-equivariance of each disentangled component imposes additional structure on this \gls{svd}: any symmetry-transformed singular function remains a singular function in the same singular space \citep{ordonez2024dynamics}. Consequently, each singular space must have dimension at least $\irrepDim[\isoIdx]$—the smallest real vector-space dimension in which $\G^{\isoIdxBrace}$ can be faithfully represented, that is, the dimension of the irreducible representation of type $\isoIdx$; see \cref{fig:diagrams}-right and \cref{prop:op_isospaces_singular_space_dimensionality}.

\customParagraph{Equivariant \gls{nn} Parametrization}
To solve $\G$-equivariant regression and estimate $\G$-invariant conditional probabilities via \eqref{eq:learning_pmd_symm}, we propose the disentangled bilinear \gls{nn} setup in \cref{fig:diagrams}-right. This approach supports any \gls{gdl} $\G$-equivariant backbone (e.g., \acrshort{mlp}, \acrshort{cnn}, \acrshort{gnn}, Transformer), adapting to diverse data modalities and tasks. Our framework can handle \emph{continuous} compact groups via group discretization, and \emph{non-compact} groups by selecting appropriate backbones, as it is done with $\G$-steerable \acrshortpl{cnn} for image/audio processing with \emph{non-compact translation} equivariance  (see details in \cref{app:continuous_noncompact_symmetry_groups} and \citet{cesa2022program}).

The representation functions $\vbfxp:\vsX \to \mathbb{R}^r$ and $\vbfyp:\vsY \to\mathbb{R}^r$ can be parameterized by \emph{any} $\G$-equivariant \gls{nn} architecture, with an additional non-learnable linear output layer that applies a change of basis to the isotypic basis, i.e.,
\begin{equation*}
	\small
	\begin{split}
		\vbfxp(\cdot)
		=
		\begin{bsmallmatrix}
			\vbfxp^\inv(\cdot) \\
			\vdots \\
			\vbfxp^{(\isoNum)}(\cdot)
		\end{bsmallmatrix}
		% [\vbfxp^\inv(\cdot)^\top, {\dots}, \vbfxp^{(\isoNum)}(\cdot)^\top]^\top 
		= \mQ_{\rvx}^{\top}
		(
		\nnX_{\nnParams}(\cdot) -  \EE_\rvx [\nnX_{\nnParams}(\rvx)]
		)
		\quad \text{and}
		\\
		\vbfyp(\cdot)
		=
		\begin{bsmallmatrix}
			\vbfyp^\inv(\cdot) \\
			\vdots \\
			\vbfyp^{(\isoNum)}(\cdot)
		\end{bsmallmatrix}
		= \mQ_{\rvy}^{\top}
		(
		\nnY_{\nnParams}(\cdot) -  \EE_\rvy [\nnY_{\nnParams}(\rvy)]
		).
		% [\vbfyp^\inv(\cdot)^\top, {\dots}, \vbfyp^{(\isoNum)}(\cdot)^\top]^\top.
	\end{split}
\end{equation*}

Here, $\nnX_{\nnParams} : \vsX \to \R^r$ and $\nnY_{\nnParams} : \vsY \to \R^r$ denote any $\G$-equivariant backbone architectures that output $r$-dimensional representations of $\rvx$ and $\rvy$. These representations are then centered to satisfy the centering constraint in \eqref{eq:learning_pmd_symm}. Moreover, $\mQ_{\rvx}$ and $\mQ_{\rvy}$ are orthogonal, non-learnable change-of-basis matrices that expose the isotypic subspaces of $\Lpxp$ and $\Lpyp$ (see details in \cref{remark:isotypic_decomposition}).

Such an orthogonal decomposition of the learned equivariant representations is referred to in the representation learning literature as a \textit{disentangled} representation \citep{higgins2018towards} (see \cref{def:disentangled_representation}). In our setting, disentanglement means that the representation decomposes into orthogonal components, each associated with a distinct group of symmetry transformations. Specifically, each component $\vbfxp^\isoIdxBrace: \vsX \to \R^{r_\isoIdx}$ and $\vbfyp^\isoIdxBrace: \vsY \to \R^{r_\isoIdx}$ spans the approximated isotypic subspace of $\G^{\isoIdxBrace}$-equivariant functions, i.e.,

\begin{equation*}
	\small
	\begin{split}
		\Lpxp = \bm{\oplus}^{\perp}_{\isoIdx\in[1,\isoNum]} \LpxpIso, \quad \text{with } \LpxpIso := \operatorname{span}(\vbfxp^{\isoIdxBrace}) \subset \LpxFullIso, \quad \text{and}
		\\
		\Lpyp = \bm{\oplus}^{\perp}_{\isoIdx\in[1,\isoNum]} \LpypIso, \quad \text{with } \LpypIso := \operatorname{span}(\vbfyp^{\isoIdxBrace}) \subset \LpyFullIso.
	\end{split}
\end{equation*}

Furthermore, the truncated operator's matrix is parameterized in block-diagonal form
$\CEmp = \Oplus_{\isoIdx=1}^{\isoNum} \CEmpIso$,
with each block an $r_{\isoIdx} {\times} r_{\isoIdx}$ $\G^{\isoIdxBrace}$-equivariant matrix \footnote{We chose  square matrices for notational convenience. Dimensions for $\rvy$ and $\rvx$ spaces need not match}. The corresponding approximated $\G$-invariant \gls{pmd} kernel is given by:
\begin{equation}
	\label{eq:pmd_model}
	\small
	\begin{aligned}
		\pmd_\nnParams(\vx,\vy)
		 & =
		\one_{\mux}(\vx) \one_{\muy}(\vy) +
		\textstyle{\sum_{\isoIdx=1}^{\isoNum}}\pmd_\nnParams^\isoIdxBrace(\vx,\vy),
		\;\;\; \text{ with }
		\\
		 & \pmd_\nnParams^\isoIdxBrace(\vx,\vy) := \vbfx_{\nnParams}^\isoIdxBrace(\vx)^\top \CEmpIso \vbfy_{\nnParams}^\isoIdxBrace(\vy), \quad \forall \isoIdx \in [1, \isoNum]
	\end{aligned}
\end{equation}
where $\one_{\mux}(\vx)\one_{\muy}(\vy)$ arises since the first singular functions of $\CE$ are constant, similarly as in \eqref{eq:optimal_representations}.

This parameterization inherently guarantees that the learned truncated operator $\CEmp$ satisfies the $\G$-equivariance constraint and that the learned representation spaces satisfy the symmetry constraints imposed on the singular spaces of each $\G^{\isoIdxBrace}$ isotypic subspace. These structural constraints enable sharper symmetry-aware statistical learning guarantees in theory and improved empirical performance in practice; see \cref{sec:learning_guarantees,sec:equiv_op_finite_rank_approximation,sec:experiments}.

\customParagraph{Disentangled Training Loss}
Having introduced the equivariance constraints on the truncated operator matrix, we decompose the contrastive loss \eqref{eq:contrastive_loss} to reflect the separability of the optimization arising from the operator's isotypic decomposition \eqref{eq:G_equiv_cond_exp_operator}:
\begin{equation}
	\label{eq:disentangled_contrastive_loss}
	\small
	\begin{aligned}
		%		\small
		\mathcal{L}_{\gamma}(\nnParams)
		%		 & 
		 & =
		\textstyle{\sum_{\isoIdx=1}^{\isoNum}}
		{-}2\E_{{\rvx\rvy}} \pmd_\nnParams^\isoIdxBrace(\rvx,\rvy)
		+
		\E_{\rvx}\E_{\rvy} \pmd_\nnParams^\isoIdxBrace(\rvx,\rvy)^2
		+
		\\
		 &
		\gamma\Omega^\isoIdxBrace(\nnParams)
		+
		2 \gamma
		\big(\| \E_{\rvx}\vbfxp^\inv(\rvx)\|_F^2 {+} \| \E_{\rvy}\vbfyp^\inv(\rvy)\|_F^2 \big).
	\end{aligned}
\end{equation}
This decomposes learning $\G$-equivariant representations of $\rvx$ and $\rvy$ into learning $\isoNum$ less constrained $\G^{\isoIdxBrace}$-equivariant representations for distinct quotient groups of $\G$.

\begin{table*}[!tb]
	\centering
	% \small
	\resizebox{\textwidth}{!}{%
		\begin{tabular*}{\textwidth}{p{2.5cm}c|c}
			\textbf{Task}
			&
			$ \vf(\vx) := \E_{\rvy}[\rvy\vert\rvx{=}\vx] \approx \hat{\vf}_{\nnParams}(\vx)$
			&
			$\PP[\rvy {\in} \sB \vert \rvx\in\sA] \approx \widehat{\PP}_{\nnParams}[\rvy \!\in\! \sB \vert \rvx \!\in\! \sA]$
			\\ \midrule
			\textbf{Estimate}
			&
			$\EE_{\rvy}[\rvy] {+} \vbfxp(\vx)^\top \CEmp \EE_{\rvy}[\vbfyp(\rvy) \otimes \rvy]$
			&
			$
				\EEY[\one_\sB] {+} \frac{\EEX[\one_\sA(\rvx)\otimes \vbfxp(\rvx)]^\top\CEmp \EEY[\one_\sB(\rvy) \otimes \vbfyp(\rvy)]}{\EEX[\one_\sA(\rvx)]}
			$
			\\ \midrule
			\textbf{Guarantees}
			&
			$\|\vf {-} \hat{\vf}_{\nnParams}\|_{\LpxFull}
				{\lesssim}
				\scalebox{0.7}{$\sqrt{\Var[\norm{\rvy}]}$}
				\left(
				\error^{r}
					{+}
				\tfrac{\ln (\sfrac{\color{blue}\isoDim}{\delta})}
					{
						({\color{blue}\isoDim}\samplesize)^{\frac{\rpar}{1+2\rpar}}
					}
				\right)
			$
			&
			$
				|\PP {-} \widehat{\PP}_{\nnParams}| {\lesssim}
				\sqrt{\tfrac{\sP[\rvy\in \sB]}{\color{blue}\sP[\rvx\in\G\Glact[\scalebox{0.5}{$\vsX$}]\sA]}}
				\left(
				\error^{r}{+}
				\tfrac{\ln(\sfrac{\color{blue}\isoDim}{\delta})}
					{
						({\color{blue}\isoDim}\samplesize)^{\frac{\rpar}{1+2\rpar}}
					}
				\right)
			$
		\end{tabular*}%
	}
	\vspace{0.2em}
	\caption{
	Statistical guarantees for \gls{encp}. The error bounds are governed by three factors: (i) the estimation error, impacted by the structure of the symmetry group $\G$---specifically, the sum of the minimum dimensionalities of the singular subspaces associated with each isotypic component, ${\color{blue}\isoDim = \sum_{\isoIdx=1}^{\isoNum}\irrepDim[\isoIdx]}$, equivalently the sum of the dimensions of the irreducible representations of the group (see \cref{fig:diagrams}), which increases the effective sample size; (ii) the representation learning error, measuring the quality of the learned representations, quantified by $\error^{r} = \|\CE - \CEp \|_{\text{op}} \leq \sqrt{\sfrac{\sigma_r}{\sigma_r - \sigma_{r+1 }}}\sqrt{\loss_{\gamma}(\nnParams)-\loss_{\gamma}(\star)}$; and (iii) the singular-value decay rate $\alpha>0$ of the operator. Here, $\G\Glact[\vsX]\sA := \cup_{g\in\G} \; g \Glact[\vsX] \sA$ denotes the group orbit of $\sA$ (\cref{def:group_orbit}).
	}
	\vspace{-0.5cm}
	\label{tab:encp_guarantees}
\end{table*}

Moreover, we improve the estimates of the regularization terms in \eqref{eq:contrastive_loss} by leveraging our symmetry priors to: (i) tighten the centering regularization \eqref{eq:disentangled_contrastive_loss} given that functions in $\LpxIso$ and $\LpyIso$ are centered by construction for $\isoIdx\neq\inv$ (see \cref{cor:functions_without_invariant_component_centered})—and (ii) exploit the orthogonality between isotypic subspaces \eqref{eq:decomposition_lpx_lpy_isotypic_subspaces} to independently regularize orthonormality for each isotypic subspace (see example in \cref{fig:basis_functions}), leading to better covariance estimates \citep{shah2012group_cov_estimation}:
\begin{equation}
	\label{eq:disentangled_orthonormal_reg}
	\small
	\Omega^\isoIdxBrace(\nnParams) \!:=\!
	\textstyle{\sum_{\isoIdx=1}^{\isoNum}}
	\|\Cov(\vbfx^\isoIdxBrace) \!-\! \Identity_{r_\isoIdx}\|^2_F \!+\! \|\Cov(\vbfy^\isoIdxBrace) \!-\! \Identity_{r_\isoIdx}\|^2_F.
\end{equation}
Given a batch $\{(\vx_n,\vy_n)\}_{n=1}^{N}$ and their corresponding embeddings $\{(\vbfxp(\vx_n),\vbfyp(\vy_n))\}_{n=1}^{N}$, the empirical unregularized loss is estimated via U-statistics, yielding an unbiased estimate with an effective sample size of $N^2$ \citep{wang2022spectral,tsai2020neuralPMD}.
\begin{equation}
	\label{eq:empirical_loss}
	\small
	\begin{split}
		\widehat{\mathcal{L}}_{0}(\nnParams)
		\!=\!
		\sum_{\isoIdx\in[\isoNum]}
		\Bigg[
			\tfrac{1}{N}
			\sum_{n\in[N]} \pmd_\nnParams^\isoIdxBrace(\vx_n,\vy_n)
			+ \\
			\tfrac{1}{N(N-1)}
			\sum_{a\in[N]} \sum_{b\in[N]\setminus\{a\}}\pmd_\nnParams^\isoIdxBrace(\vx_a,\vy_b)^2\Bigg].
	\end{split}
\end{equation}
Similarly, we use U-statistics to obtain unbiased estimates for orthonormal regularization in \eqref{eq:disentangled_orthonormal_reg}, achieving an effective sample size of $\irrepDim[\isoIdx]N^2$ per isotypic subspace (see \cref{app:symmetric_orthonormality}). Consequently, standard \gls{nn} optimization methods can be employed to learn equivariant representations via the approximate model of $\CE$, enabling downstream inference tasks described in the next section.

\section{Inference and Learning Guarantees}  \label{sec:learning_guarantees}
%%%%%%%%%%%%%%%%%%%%%%%%%%%%%%%%%%%%%%%%%%%%%%%%%%%%%%%%%%%%
%%%%%%%%%%%%%%%%%%%%%%%%%%%%%%%%%%%%%%%%%%%%%%%%%%%%%%%%%%%%

Once training is complete, the learned $\G$-invariant \gls{pmd} from \eqref{eq:pmd_model} can be used, via \eqref{eq:uses_of_conditional_expectation_operator}, for $\G$-equivariant regression and $\G$-invariant conditional probability estimation. These estimates are obtained using a \gls{nn} architecture composed of $\vbfxp$, $\CEmp$, and a final linear layer that delivers the basis expansion coefficients of the target variable in the $\rvy$ representation space $\Lpyp = \operatorname{span}(\vbfyp)$. For a summary of the estimates and their learning guarantees refer to \cref{tab:encp_guarantees}.

Both estimates are derived from the general problem of vector-valued regression of a target function $\vz\colon \vsX \to \vsZ$ defined by the conditional expectation of an observable $\vh\in\Lp[2]{\muy}{(\vsY,\vsZ)}$, that is, $\vz(\vx) := \E_{\rvy}[\vh(\rvy)\vert\rvx=\vx] = [\CE \vh](\vx)$, where $\vsZ$ is a symmetric vector space. Using the learned model, we estimate $\vz(\vx) := [\CE \vh](\vx)$ by:
\begin{equation}
	\label{eq:regression_estimate}
	\small
	\begin{aligned}
		\widehat{z}_{\nnParams}(\vx)
		:=
		\EEY[\vh(\rvy)]
		+
		\vbfxp(\vx)^\top \CEmp \,\EEY[\vbfyp(\rvy) \otimes \vh(\rvy)],
	\end{aligned}
\end{equation}
where $\EEY[\vbfyp(\rvy) \otimes \vh(\rvy)]$ represents the basis expansion coefficients of $\vh$ in the learned basis of $\Lpyp \subset \LpyFull$. Here, $\EE_{\rvx} \colon \LpxFull \to \R$ and $\EE_{\rvy} \colon  \LpyFull \to \R$ are the $\G$-invariant empirical expectations defined by:
\begin{equation}
	\label{eq:equiv_empirical_mean_main}
	\small
	\begin{split}
		\EE_{\rva}[f(\rva)] \!=\! \frac{1}{|\G|\samplesize} \sum_{g\in \G} \sum_{n{=}1}^{\samplesize} f(g\Glact[\vsA]\!\va_{n})
		\quad \text{for } \va \in \{\vx,\vy\}
		% 
		% \EE_{\rvx}[f(\rvx)] \!=\! \frac{1}{|\G|\samplesize} \sum_{g\in \G} \sum_{n{=}1}^{\samplesize} f(g\Glact[\vsX]\!\vx_{n})
		% \quad\text{ and }
		% \\
		% \EE_{\rvy}[h(\rvy)] \!=\! \frac{1}{|\G|\samplesize} \sum_{g\in \G} \sum_{n{=}1}^{\samplesize} h(g\Glact[\vsY]\!\vy_{n}).
	\end{split}
\end{equation}
Hence, our method learns representations of $\rvx$ and $\rvy$ that transform \textit{nonlinear regression of observables} into \textit{linear regression} in the learned representation space. For example, assuming $\rvy$ has bounded variance and setting $\vh(\vy)=\vy$, we recover the standard ($\G$-equivariant) regression solution (see \cref{tab:encp_guarantees}-left). Equally important, by letting $\vh=\one_{\sB}$—the indicator of a measurable set $\sB\subseteq\vsY$—the model estimates conditional probabilities (see \cref{tab:encp_guarantees}-right), thereby supporting both regression and uncertainty quantification (e.g., conditional quantiles, see \cref{sec:experiments} and \citet{kostic2024neural}). The following learning bounds cover this general setting.
\begin{theorem}\label{thm:main}
	Let the symmetry priors in \eqref{eq:main_assumtions} hold, $\CE$ be a $\sfrac{1}{\rpar}$-Schatten-class $\G$-equivariant conditional expectation operator, as in \eqref{eq:G_equiv_cond_exp_operator}, and $\CEp$ be a truncated approximation, defined in \eqref{eq:pmd_model} and trained via \eqref{eq:disentangled_contrastive_loss}, with rank $r$ and parameters $\nnParams\in\Theta$.
	Then, given an appropriate truncation dimension $r \asymp (\sfrac{\samplesize}{\isoDim^{\rpar}})^{\sfrac{1}{(1+2\rpar)}}$, the following results hold for any $\G$-equivariant/invariant $\vh\!\in\!\Lp[2]{\muy}{(\vsY,\vsZ)}$, measurable set $\sA\subset\vsX$, and $\G'\leq\G$; with probability at least $1-\delta$ w.r.t. an iid draw of $\sD_\samplesize=\{(\rvx_n,\rvy_n) \sim \muxy\}_{n=1}^{\samplesize}$:
	\begin{subequations}
		\begin{equation}\label{eq:main_thm_regression}
			\small
			\Vert \vz-\hat{\vz}_{\nnParams}\Vert_{\Lp[2]{\mux}(\vsX,\vsZ)}
			\!\lesssim\!  \sqrt{\Var[\norm{\vh(\rvy)}_{\vsZ}]}
			\;\cdot\; {\color{blue}\xi_{\nnParams}}
			\quad \text{and}
		\end{equation}
		% and
		\begin{equation}\label{eq:main_thm_regression_set}
			\small
			\begin{split}
				\Vert \vz(\sA)-\hat{\vz}_{\nnParams}(\sA) \Vert_{\vsZ}
				\!\lesssim\!
				\tfrac{
				\sqrt{1+ {\color{blue}{(|\G'|\,{-}\,1)\gamma_{\G'}(\sA)}}}\,
				\sqrt{\Var[\norm{\vh(\rvy)}_{\vsZ}]}
				}{
				{\color{blue}\sqrt{|\G'|\PP(\rvx\in\sA)}}
				}
				\;\cdot\;  {\color{blue}\xi_{\nnParams}}
				% \cdot
				% \\
				% \Big(
				% \error^{r}
				% +
				% \tfrac{\ln(\sfrac{{\color{blue}\isoNum}}{\delta})}{({\color{blue}\isoDim}\samplesize)^{\frac{\rpar}{1+2\rpar}}}
				% \Big).
			\end{split}
		\end{equation}
		where the modeling/approximation and finite-sample-estimation error are given by:
		\begin{equation*}
			\small
			{\color{blue}\xi_{\nnParams}} :=
			\error^{r}
			+
			\frac{\ln(\sfrac{{\color{blue}\isoDim}}{\delta})}{({\color{blue}\isoDim}\samplesize)^{\frac{\rpar}{1+2\rpar}}}
			.
		\end{equation*}
	\end{subequations}
	Here, $\error^r = \|\CE {-} \CEp\|_{\text{op}}$ denotes representation learning modeling error \eqref{eq:learning_pmd_symm}, $r = \sum_{\isoIdx=1}^{\isoNum} \irrepDim[\isoIdx] \irrepMultiplicity_\isoIdx$ defines the representation space dimension, with $(\irrepDim[\isoIdx], \irrepMultiplicity_\isoIdx)$ denoting dimension and multiplicity of the group's irreducible representation of type $\isoIdx$, and $\isoDim := \sum_{\isoIdx=1}^{\isoNum} \irrepDim[\isoIdx] \geq \isoNum$ (see \cref{fig:diagrams}).
\end{theorem}
\begin{proof}
	$\G$-invariance of $\mux$ and $\muy$ allows us to control both bias (\cref{thm:approx_error}) and variance (\cref{prop:empirical_inner_product}) of $\hat{\vz}_{\nnParams}$. A simple balancing of $r$ yields the final bound on the error.
\end{proof}
% 
% {\bf Conditional probability estimation}~~~	
We conclude by highlighting key theoretical and practical implications of \cref{thm:main}. Since pointwise guarantees are essential in many applications, \cref{eq:main_thm_regression_set} provides a learning bound conditioned on measurable sets $\sA \subseteq \vsX$, yielding the estimate
$
	\vz(\sA)
	:=
	\E_{\rvy}[\vh(\rvy)\vert\rvx\in\sA]\approx
	\sfrac{\EEX[\widehat{\vz}_{\nnParams}(\vx)]}{\EEX[\one_{\sA}(\rvx)]}
$. In this setting, exploiting symmetry helps mitigate the bottlenecks of rare-event estimation. To capture this effect, we introduce the \textit{symmetry index} of $\sA$ w.r.t. $\mux$, which quantifies the degree of symmetry of $\sA$ and appears in \eqref{eq:main_thm_regression_set}:

\begin{equation}
	\label{eq:setsym_bis_main}
	\small
	\begin{split}
		\setsym{\G}{\sA}
		=
		\frac{1}{|\G|-1}{
			\sum_{g\in\G\setminus\{e\}}}
		\frac{\PP(\rvx\in \sA \cap \; g \Glact[\scalebox{0.8}{$\vsX$}] \sA)}{\PP(\rvx\in \sA)},
		% \quad
		% \text{where}
		\\
		% \setsym{\G}{\sA} \in [0,1] \;\; \forall \sA\subseteq\vsX.
	\end{split}
\end{equation}

Observe that $\setsym{\G}{\sA} \in [0,1]$ for any $\sA\subseteq\vsX$, with
$\setsym{\G}{\sA}=1$ if $\sA$ is $\G$‑invariant (e.g., the vertical and horizontal reflection planes in \cref{fig:symmetry_priors_long}), while $\setsym{\G}{\sA}=0$ if $\sA$ is a $\G$‑asymmetric set, that is, if $g\Glact_{\vsX}\sA\cap\sA=\emptyset$ for all $g\in\G$ (e.g., any set disjoint from the reflection planes in \cref{fig:symmetry_priors_long}). In particular, the effective rarity of $\rvx\in\sA$ is captured by $\setsym{\G'}{\sA}$, yielding a maximal gain of $|\G|\PP[\rvx\in\sA]\gg\PP[\rvx\in\sA]$ when $\sA$ is asymmetric.

Equivariant disentangled representations increase the effective sample size according to
$\samplesize \ll \isoNum\samplesize \leq {\color{blue}\isoDim \samplesize} \leq |\G|\samplesize$.
Thus, they provide not only the expected $\isoNum$-fold gain from disentanglement, but also the stronger boost induced by
$\isoDim = \sum_{\isoIdx=1}^{\isoNum} \irrepDim[\isoIdx]$ (see \cref{fig:diagrams}, right). This improvement in the estimation-error term of the bound is achieved by exploiting the structural constraints of the singular spaces associated with the irreducible representations of the group that appear in the chosen representation space. Moreover, any non-trivial symmetry group sharpens the bound. In the absence of a symmetry prior, namely when $\G=\{e\}$ and $|\G|=\isoNum=\isoDim=1$, our framework recovers the symmetry-agnostic baseline results of \citep{kostic2024neural}.

The remaining term in the bound is the representation-learning error $\error^{r}$ inherited from \gls{ncp} \citep{kostic2024neural}. We do not provide a full statistical characterization of this term, as doing so would require additional architecture- and regularity-dependent assumptions without changing the main qualitative conclusions. Instead, we complement the theory with synthetic experiments in which the true operator is known, allowing us to empirically assess the impact of exploiting symmetry on $\error^{r}$ (see \cref{fig:synthetic_regression_results,fig:G_invariance_kappa}).
% A sharper statistical treatment of this term remains an important direction for future work.

\begin{figure*}[!t]
	\centering
	\FigureSummary{Sample-efficiency plots across four symmetry groups comparing test-set probability-mass-density mean squared error versus training set size for ENCP and baseline methods on synthetic conditional Gaussian mixture models.}{\includegraphics[width=\textwidth]{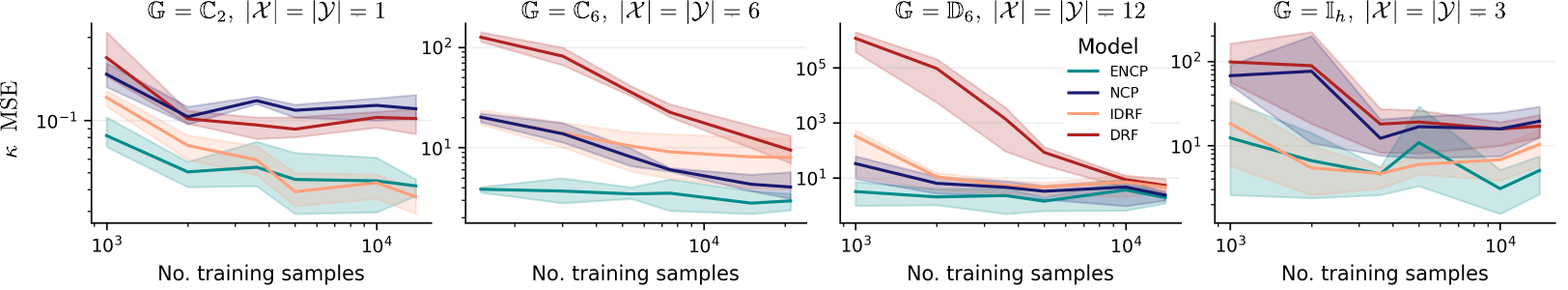}}
	\vspace*{-0.4cm}
	\caption{
		% Sample efficiency plots comparing the test set \gls{pmd} \gls{mse} $\pmd_{\text{mse}} := \E_\rvx \E_\rvy (\pmd(\rvx, \rvy) - \pmd_{\nnParams}(\rvx, \rvy))^2$ (first row) and \gls{pmd} $\G$-invariance constraint violation (second row) versus the number of training samples, in log scales. Each plot corresponds to a symmetric \gls{cgmm} with distinct symmetry groups and $(\rvx,\rvy)$ dimensionality. The tested groups are the cyclic groups $\CyclicGroup[2]$ and $\CyclicGroup[6]$, the Dihedral group $\DihedralGroup[6]$ (order 12), and the Icosahedral group $\sI_h$ (order 60).
		Sample-efficiency plots for the test-set \gls{pmd} \acrshort{mse}
		$
			\pmd_{\text{mse}} := \E_\rvx \E_\rvy (\pmd(\rvx, \rvy) - \pmd_{\nnParams}(\rvx, \rvy))^2
		$
		versus the number of training samples, on logarithmic scales.
		Each column corresponds to a symmetric \acrshort{cgmm} with a different symmetry group and $(\rvx,\rvy)$ dimensionality. The tested groups are the cyclic groups $\CyclicGroup[2]$ and $\CyclicGroup[6]$, the dihedral group $\DihedralGroup[6]$, and the icosahedral group $\sI_h$ (order 60).
	}
	\vspace*{-0.3cm}
	\label{fig:ablation_results}
\end{figure*}

Finally, the parameter $\rpar$ quantifies problem regularity through the decay rate of the operator's singular values, $\sum_{i\in\sN}\sval_{i}^{\sfrac{1}{\rpar}} < \infty$, with special cases including finite-rank operators ($\rpar=\infty$), compact operators ($\rpar=0$), trace class operators ($\rpar=1$), and Hilbert-Schmidt operators ($\rpar=1/2$, equivalent to $\pmd \in \Lp[2]{\mux\times\muy}(\vsX\!\times\!\vsY)$; see \cref{sec:statistical_learning_theory}).
Thus, our results cover learning rates ranging from arbitrarily slow rates as $\rpar\to0$ to fast rates $[\isoDim \samplesize]^{-1/2}$ as $\rpar\to\infty$.

% In contrast, exploiting symmetries yields substantial statistical gains: it amplifies the effective sample size and fundamentally mitigates the inherent bottlenecks of rare event estimation.

% \FloatBarrier
% 
%%%%%%%%%%%%%%%%%%%%%%%%%%%%%%%%%%%%%%%%%%%%%%%%%%%%%%%%%%%%%%
%%%%%%%%%%%%%%%%%%%%%%%%%%%%%%%%%%%%%%%%%%%%%%%%%%%%%%%%%%%%%%
\section{Experiments} \label{sec:experiments}
%%%%%%%%%%%%%%%%%%%%%%%%%%%%%%%%%%%%%%%%%%%%%%%%%%%%%%%%%%%%%%
%%%%%%%%%%%%%%%%%%%%%%%%%%%%%%%%%%%%%%%%%%%%%%%%%%%%%%%%%%%%%%

We present three experiments evaluating our method in (i) approximating the conditional expectation operator and the use of the learned operator for (ii) $\G$-equivariant regression and (iii) symmetry-aware uncertainty quantification. For further details and additional experiments on synthetic regression, uncertainty quantification, symmetry misspecification and train/inference computational costs refer to \cref{app:experimental_setup} and the code repository\footref{ref:code}

\customParagraph{Conditional Expectation Operator Learning}
This experiment quantifies the \gls{mse} of approximating $\CE$, i.e., $\pmd_{\text{mse}} := \E_\rvx \E_\rvy (\pmd(\rvx, \rvy) - \pmd_{\nnParams}(\rvx, \rvy))^2$. To achieve this, we extend the \gls{cgmm} of \citet{gilardi2002conditional} to parametrically construct symmetric vector-valued random variables $\rvx \in \vsX$ and $\rvy \in \vsY$ that satisfy the symmetry priors in \eqref{eq:main_assumtions} for arbitrary finite symmetry groups (see example in \cref{fig:symmetry_priors_long}). Each \gls{cgmm} possess an analytical \gls{pmd}, enabling direct $\pmd_{\text{mse}}$ estimation, usually impossible for real-world datasets.

% This experiment is relevant as it provides an analytical expression for the conditional expectation operator kernel, $\pmd(\vx, \vy) = \sfrac{p_{\rvx\rvy}(\vx,\vy)}{p_{\rvx}(\vx)p_{\rvy}(\vy)}$ (\cref{fig:symmetry_priors}-right). This allows us to directly estimate the approximation error of the conditional expectation operator \cref{eq:contrastive_loss} via the mean squared error between the true and learned density ratios: 
%
The results in \cref{fig:ablation_results} compare our \gls{encp} against \gls{ncp} \citep{kostic2024neural}, the baseline \gls{drf} \citep{tsai2020neuralPMD}, and our \gls{idrf} adaptation. Note that the baselines approximate $\pmd$ as a single \gls{nn}, $\pmd_{\nnParams}^{\text{drf}}\colon \vsX \times \vsY \to \R^+$, trained via the contrastive loss in \eqref{eq:contrastive_loss}. Since they do not enforce the separable structure in \eqref{eq:pmd_model}, these methods cannot be used for regression or conditional probability estimation.

We evaluate performance across \glspl{cgmm} with diverse symmetry groups and varying $(\rvx,\rvy)$ dimensions. Across all settings, \gls{encp} achieves lower \gls{mse} and better sample efficiency than \gls{idrf} and the symmetry-agnostic models \gls{ncp} and \gls{drf} (see \cref{fig:ablation_results}). Moreover, \Cref{fig:G_invariance_kappa} shows that the symmetry-agnostic models (\gls{ncp} and \gls{drf}) struggle to recover the $\G$-invariance of the true \gls{pmd} from data alone. In contrast, \gls{encp} and \gls{idrf} explicitly enforce the invariance constraint in \cref{eq:learning_pmd_symm} through architectural \gls{nn} constraints, preserving the desired $\G$-invariance throughout training up to numerical precision. These results show that \gls{encp} accurately approximates the conditional expectation operator, consistently achieving lower empirical representation-learning errors $\error^{r}$ than all other models. This underscores the critical role of symmetry in improving both accuracy and sample efficiency.

\customParagraph{$\G$-Equivariant Regression}
% =========================================================
To test our model's potential for performing $\G$-equivariant regression, we address the robot's \gls{com} momenta regression task of \citep{ordonez2024morphological}. The goal is to predict a quadruped robot's \gls{com} linear $\vl \in \R^3$ and angular momenta $\vk \in \R^3$ given the noisy observations of the robot's generalized positions $\vq \in \R^{12}$ and velocity coordinates $\dot{\vq} \in \R^{12}$, i.e.,
$[\vl^\top, \vk^\top]^\top = h_{\text{CoM}}(\vq + \epsilon_{\vq}, \dot{\vq} + \epsilon_{\dot{\vq}})$
(see details in \cref{app:exp_results_regression,fig:example_group_action} ). We compare \gls{encp} against \gls{ncp} and two baselines—a standard \gls{mlp} and an \gls{emlp}---all with equivalent architectural footprint.  The \gls{ncp} and \gls{encp} are trained using \eqref{eq:contrastive_loss} and \eqref{eq:disentangled_contrastive_loss}, respectively, while \gls{mlp} and \gls{emlp} are trained using standard \gls{mse}.

\begin{figure*}[!t]
	\centering
	\FigureSummary{Time-series uncertainty plots for front-leg ground-reaction forces of a quadruped robot, comparing predicted 90 percent confidence intervals from ENCP, NCP, ECQR, and CQR across x, y, and z force components.}{\includegraphics[width=\textwidth]{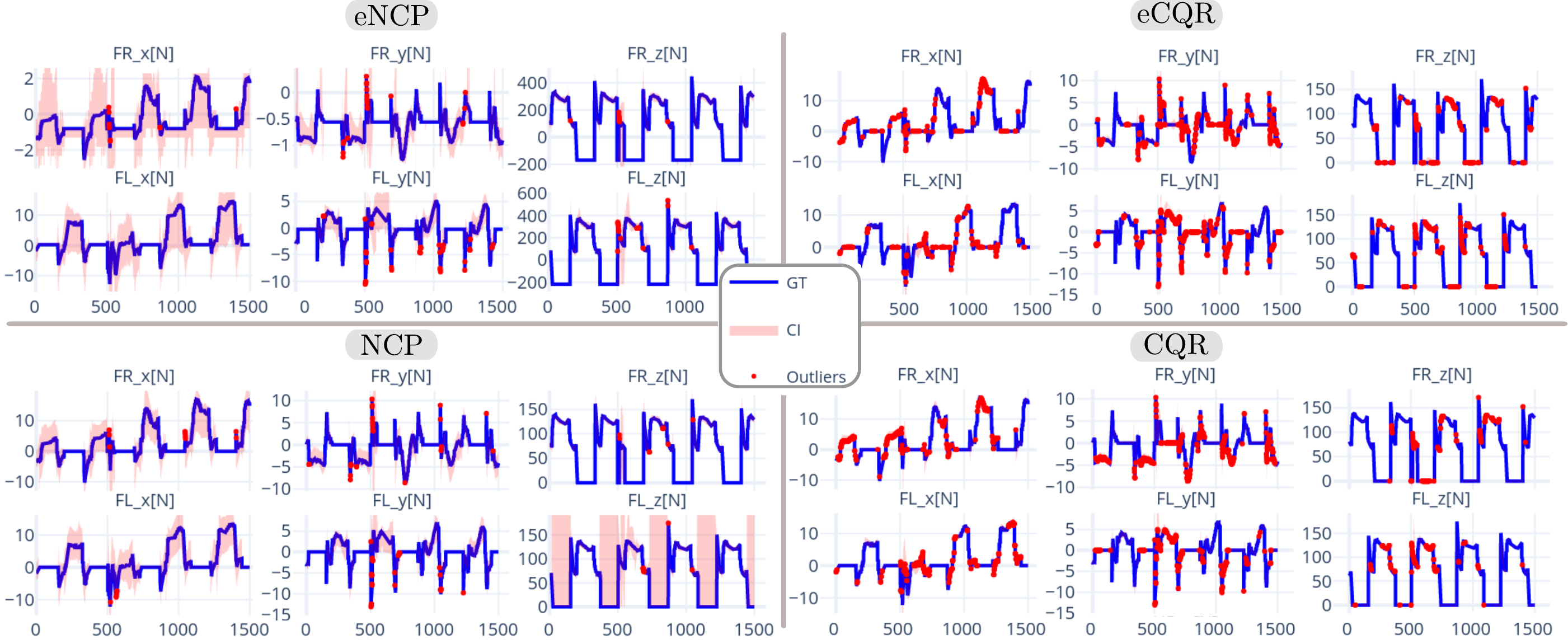}}
	% \vspace*{-0.5cm}
	\caption{
		Prediction of 90\% \glspl{ci} (light red areas) for ground-reaction forces $\vtau_{\text{grf}} \in \R^{12}$ of a quadruped robot on rough terrain with varying friction. We compare \acrshort{encp}, \acrshort{ncp}, \acrshort{ecqr}, and \acrshort{cqr} models. CIs are computed along $x$, $y$, $z$ compoents of the forces on the front-right (FR) and front-left (FL) legs (see all legs predictions in \cref{fig:quadruped_grf_uc}). The trajectory of ground reaction forces in time is shown in blue lines, while sensed values outside of the predicted \glspl{ci} are plotted in red markers.
		% Terrain variations cause significant variability in $x$/$y$ components due to surface orientation and friction differences, while $z$ components are influenced by local height changes affecting contact timing and producing short-duration high-impact forces.
	}
	% \vspace*{-0.5cm}
	\label{fig:quadruped_grf_uc_main}
\end{figure*}

The results in \cref{fig:teaser} demonstrate that our \gls{encp} model outperforms all other baselines in both performance and sample complexity. Consistent with \citep{kostic2024neural}, the \gls{ncp} model shows poorer sample complexity than \gls{mlp} and \gls{emlp} due to its indirect approach to regression, via approximation of $\CE$. However, by incorporating symmetry priors, \gls{encp} appear to mitigate this limitation.
% \footnote{
% Note that \gls{ncp} and \gls{encp} require nearly 10 times more epochs to converge in training loss than \gls{mlp} and \gls{emlp}. This extra cost, that unlocks the ability of uncertainty quantification, reflects the challenge of learning complex representations instead of direct regression.
% }

% =========================================================
\customParagraph{Symmetry-Aware Uncertainty Quantification}
% =========================================================
Finally, we demonstrate the practical impact of our approach on a core robotics problem: providing robust uncertainty quantification for unavailable yet crucial state observables for robot control and state estimation \citep{bledt2018contact,maravgakis2023probabilistic}. Specifically, we use proprioceptive sensor readings to provide $90\%$ \glspl{ci} for the robot's \gls{grf} $\vtau_{\text{grf}} \in \R^{12}$, the instantaneous work exerted or subtracted to the robot $U(\vq,\dot{\vq},\vtau) \in \R$, and the kinetic energy $T(\vq,\dot{\vq}) \in \R$, while the robot traverses rough terrain (see \cref{app:exp_results_uncertainty_quantification_robot}). Reliable probabilistic estimates of these quantities are of crucial relevance for optimal control \citep{bledt2018contact}, contact detection \citep{maravgakis2023probabilistic}, state estimation \citep{nistico2025muse}, and system identification \citep{gautier1997dynamic}.

% In general, to evaluate our model's ability of conditional probability estimation, we predict $(1{-}\alpha)$-\glspl{ci} for a random variable 
% $\rvy {=} %[\vtau_{\text{grf}}^\top,\,U(\vq,\dot{\vq},\vtau)]^\top{\in}\R^{13}%
% [\ry_i,\dots,\ry_{d_y}]
% $, conditioned on different values of $\rvx$. These intervals are obtained by regressing the lower-\&-upper conditional quantiles $(q_{\alpha/2}(\cdot), q_{1-\alpha/2}(\cdot))$ for each $\ry_i$ at any miscoverage level $\alpha \in [0,1]$ (see \cref{fig:synthetic_uc_depiction}). 
% 
% \begin{figure}[!h]
% 	\centering
% 	\includegraphics[width=\linewidth]{images/quadruped_grf_uc_eNCP_results.pdf}
% 	% \vspace*{-0.4cm}
% 	\caption{
% 		\gls{encp} prediction of 90\% \glspl{ci} (red areas) for ground-reaction forces $\vtau_{\text{grf}} \in \R^{12}$ (blue lines) of a quadruped robot on rough terrain. \glspl{ci} are computed for each leg—front-right (FR), front-left (FL), hind-right (HR), hind-left (HL)—along $x$, $y$, $z$ axes. Sensed values outside \glspl{ci} are painted red. See predictions for \gls{ncp}, \gls{cqr}, and \gls{ecqr} in \cref{fig:quadruped_grf_uc}
% 	}
% 	\label{fig:quadruped_grf_uc_main}
% 	% \vspace{-0.2cm}
% \end{figure}

% % 
\begin{table}[!b]
	\centering
	\small
	\vspace{-0.5cm}
	\begin{tabular}{l|ccc}
		\small
		                & r-Coverage $\uparrow$                & Coverage $\uparrow$                  & Set-size $\downarrow$                           \\
		\midrule
		\acrshort{encp} & $\textcolor{blue}{99.5 {\pm} 0.1\%}$ & $\textcolor{blue}{95.0 {\pm} 0.4\%}$ & $4.3 {\pm} 3.6{\times}10^{9}$                   \\
		\acrshort{ncp}  & $\textcolor{blue}{99.5 {\pm} 0.0\%}$ & $56.9 {\pm} 0.3\%$                   & $2.6 {\pm} 1.4{\times}10^{10}$                  \\
		\acrshort{ecqr} & $84.2 {\pm} 0.7\%$                   & $6.7 {\pm} 1.2\%$                    & $\textcolor{blue}{1.7 {\pm} 1.7{\times}10^{7}}$ \\
		\acrshort{cqr}  & $80.5 {\pm} 3.7\%$                   & $8.5 {\pm} 0.9\%$                    & $1.4 {\pm} 0.1{\times}10^{8}$                   \\
	\end{tabular}
	% \vspace{-0.2cm}
	\caption{
	Relaxed coverage \eqref{eq:relaxed_coverage}, coverage \eqref{eq:coverage}, and set-size \eqref{eq:setsize} for predicted \glspl{ci} on the test set of the quadruped locomotion uncertainty-estimation task with $\rvy = [\vtau_{\text{grf}}^\top,\, U,\, T]^\top$. The target coverage is $90\%$. See all metrics in \cref{tab:uc_robotics_coverage_set_size_full}.
	}
	\label{tab:uc_robotics_coverage_set_size}
	\vspace{-0.2cm}
\end{table}
% 
% \newline
This task tests our model's ability to learn conditional distributions from high-dimensional data, considering that for the \gls{encp} and \gls{ncp} models, quantile estimation is done by regressing the \gls{ccdf} for each dimension of $\rvy=[\ry_1,\dots]$ and then applying a linear search to extract quantiles (see \cref{fig:ccdf_regression,fig:synthetic_regression_results}). This is achieved by discretizing the range of each $\ry_i$ into $N_b$ bins and estimating
$
	\PP(\ry_i \in \sA_{i,n} \vert \rvx = \cdot) := [\CE \one_{\sA_{i,n}}](\cdot)
$
for all $n \in [N_b]$ (see \cref{sec:learning_guarantees}), where $\sA_{i,n}$ consists of the first $n$ bins. In practice, this means regressing $|\vsY|\times N_b$ conditional probabilities corresponding to sets of varying sizes in a \emph{single forward pass} (see details in \cref{app:exp_results_regression_synthetic}). By contrast, the baseline \gls{cqr} \citep{feldman2023calibrated} and its equivariant adaptation \gls{ecqr} directly regress quantiles for a fixed coverage level (i.e., the probability that an event lies within the predicted confidence interval) and need retraining for different coverage values.

The results, shown in \cref{tab:uc_robotics_coverage_set_size,fig:teaser} for work $U$ and kinetic energy $T$, and in \cref{fig:quadruped_grf_uc_main} for the ground reaction forces $\vtau_{\text{grf}}$, identify eNCP as the only model that provides robust uncertainty quantification under both transient disturbances and nominal locomotion conditions. It is also the only model whose empirical test-set coverage remains close to the target value, rendering the other models unreliable in practice. These results further highlight eNCP's potential for conditional probability estimation.

\section{Conclusions} \label{sec:conclusions}
We introduce a novel framework for equivariant contrastive representation learning that enables equivariant regression, symmetry-aware uncertainty quantification, and conditional probability estimation with non-asymptotic statistical learning guarantees. Building on a recent contrastive representation learning method for approximating the spectral decomposition of the conditional expectation operator, our approach incorporates symmetry priors to impose additional structural constraints on the operator approximation, yielding a disentangled representation that admits stronger symmetry-aware statistical guarantees. We demonstrate the framework's benefits through theoretical learning bounds and empirical evaluations on robotics applications. Notably, we provide the first learning guarantees for equivariant regression and uncertainty quantification using neural network features, bridging spectral representation learning and geometric deep learning.

%\textcolor{red}
\paragraph{Limitations and Future Work} Our method relies on fully specified symmetry priors, a natural assumption in \gls{gdl}; however, some applications may have only partial or misspecified symmetries. Future work could accommodate partial or uncertain symmetry information and statistically test for its presence in data, as explored in \cref{app:exp_results_symmetry_misspecification}.

\newpage
% \paragraph{Reproducibility statement} Detailed experimental setup is provided in \cref{app:experimental_setup}, and all experiments are reproducible using the code available at: \href{https://anonymous.4open.science/r/symm_rep_learn-2B3A/README.md}{\textit{symm\_rep\_learn}}

\section*{Impact Statement}
This paper presents work whose goal is to advance the field
of Machine Learning. There are many potential societal
consequences of our work, none which we feel must be
specifically highlighted here.

\section*{Acknowledgements}
The work of DOA, VRK, AF and MP was supported by the EU
Project ELIAS (grant No. 101120237), and by the European Union – NextGenerationEU and the Italian National
Recovery and Resilience Plan through the Ministry of University and Research (MUR), under Project PE0000013
CUP J53C22003010006. KL acknowledges support from ELIAS and
the French National Research Agency for the DECATTLON
project (ANR-24-CE40-3341). %VB only participated in the first discussions of the project, while all other authors contributed equally to conception, development, experimentation, and writing of this work.

% \paragraph{Author contributions}
%Vivien Brant has read the final version of the manuscript and requested to be included as an author. All other authors contributed equally to the conception, development, experimentation, and writing of this work.

%}

% REFERENCES ==========================================================
% ========================================================================
% \clearpage
\bibliographystyle{icml2026}
\bibliography{references}

% \subsubsection*{LLM use} Large language models were used to polish writting and grammar. Furthemore, agentic LLMs were used during coding to help via autocompletion, automatic generation of unit tests, debugging, and code refactoring.

% \subsubsection*{Author contributions}
% Detailed list of contributions per author to be added upon acceptance.

%%%%%%%%%%%%%%%%%%%%%%%%%%%%%%%%%%%%%%%%%%%%%%%%%%%%%%%%%%%%
%%%%%%%%%%%%%%%%%%%%%%%%%%%%%%%%%%%%%%%%%%%%%%%%%%%%%%%%%%%%

\newpage
\appendix
\let\addcontentsline\origaddcontentsline
\onecolumn

% Redefine part name to be empty to avoid "Part I" text
\renewcommand{\partname}{}
\renewcommand{\thepart}{}
\part{Appendix} % Start the appendix part

\setcounter{parttocdepth}{1}% only sections, nothing deeper
\noptcrule
\parttoc % Insert the appendix TOC

% \bigskip
The appendix is organized as follows.
%Below we give an overview of the structure of the supplementary material.
\begin{itemize}[left=1em]
	\item \cref{sec:notation} summarizes the notations used, while \cref{sec:acronyms} provides a glossary.
	\item \cref{sec:related_works} offers a detailed discussion of related work on contrastive learning, equivariant representations, and statistical learning theory with symmetry priors.
	\item \cref{sec:aux_results}-\cref{app:experimental_setup} detail our methodological contributions. In particular, \cref{app:equivariant_nn_architecture} and \cref{sec:loss_function} explain how our method leverages equivariant neural networks, and \cref{app:experimental_setup} complements \cref{sec:experiments} with additional experimental details and studies.
	\item We provide self-contained sections with unified notation covering essential preliminaries: group theory (\cref{sec:group_theory}), representation theory in function spaces (\cref{sec:representation_theory_symmetric_function_spaces}), equivariant linear operators (\cref{sec:equivariant_operators}), and the symmetries of covariance operators central to our work (\cref{sec:statistical_properties_symmetric_function_spaces}).
	\item Finally, \cref{sec:statistical_learning_theory} presents our theoretical contributions, summarized in Theorem \cref{thm:main}. We first establish approximation error bounds using operator theory and equivariant representations, then combine group theory with concentration inequalities to derive estimation error bounds that fully expose the benefits of symmetry priors.

\end{itemize}

\newpage

%\localtableofcontents
\input{appendix/math_glosary.tex}
\input{appendix/assumptions.tex}

\input{appendix/experimental_setup.tex}

\input{appendix/ncp_primer.tex}
\input{appendix/group_rep_theory.tex}
\input{appendix/rep_theory_fn_spaces.tex}
\input{appendix/equivariant_operators.tex}

\input{appendix/bounds.tex}

%}
%%%%%%%%%%%%%%%%%%%%%%%%%%%%%%%%%%%%%%%%%%%%%%%%%%%%%%%%%%%%

\end{document}

%% file: acronyms.tex
\usepackage[nonumberlist,numberedsection]{glossaries}  % for acronyms
\newacronym[
	description={
			Equivariant Neural Conditional Probability: Our proposed model integrating the symmetry priors \cref{eq:main_assumtions} into the NCP deep representation learning algorithm
		}
]{encp}{eNCP}{Equivariant Neural Conditional Probability}
\newacronym[
	description={
			Contrastive Low-Rank loss from \citep{kostic2024neural,ryu2024operator} for operator and representation learning. Used in density-ratio fitting \citep{sugiyama2012density}, representation learning \citep{wang2022spectral,chen2021provable}, and mutual information estimation \citep{tsai2020neuralPMD}
		}
]{clora}{CLoRa}{contrastive low-rank}
\newacronym[
	description={
			Neural Conditional Probability: A deep representation learning framework \citep{kostic2024neural} for conditional probability estimation and regression with statistical guarantees via operator theory \citep{Baker1973-ns}. This framework is symmetry-agnostic
		}
]{ncp}{NCP}{Neural Conditional Probability}
\newacronym{ncpaug}{NCPaug}{NCP with data augmentation}
\newacronym[
	description={
			Density Ratio Fitting \citep{tsai2020neuralPMD}: A density ratio \gls{nn} architecture that parameterizes the approximated \gls{pmd} $\pmd_\nnParams: \vsX \times \vsY \to \R_{+}$ as a single \gls{nn}. Consequently, this model cannot be used for downstream conditional probability estimation and regression—it is limited to estimating the mutual information between $\rvx$ and $\rvy$ \citep{tsai2020neuralPMD}
		}
]{drf}{DRF}{Density Ratio Fitting}

\newacronym[
	description={
			Invariant Density Ratio Fitting: This is a $\G$-invariant adaptation of the DRF model \citep{tsai2020neuralPMD} that parameterizes the approximated PMD $\pmd_\nnParams$ as a $\G$-invariant \gls{nn}
		}
]{idrf}{iDRF}{Invariant Density Ratio Fitting}

\newacronym[
	description={Geometric Deep Learning: A field of machine learning that incorporates geometric priors into deep learning models \citep{bronstein2021geometric}}
]{gdl}{GDL}{Geometric Deep Learning}

\newacronym[
	description={TODO}
]{ssl}{SSL}{Self-Supervised Learning}

\newacronym[
	description={
			Pointwise Mutual Dependency \cite{tsai2020neuralPMD}: A pointwise dependency measure between random variables $\rvx$ and $\rvy$, defined as $\pmd(\vx,\vy)=\frac{d\muxy(\vx,\vy)}{d\big(\mux(\vx)\times\muy(\vy)\big)}=\exp\big(\text{MI}(\vx,\vy)\big)$
		}
]{pmd}{PMD}{Pointwise Mutual Dependency}

\newacronym[
	description={
			Conditional Quantile Regression \citep{feldman2023calibrated}: A multivariate neural network approach that regresses upper and lower quantiles at a miscoverage level $\alpha$ using the pinball loss. Although CQR typically uses post-hoc conformal calibration, we omit it here for a fair comparison, as we evaluate the quality of the parametric estimates themselves, and conformal calibration is model-agnostic and applicable to eNCP as well
		}
]{cqr}{CQR}{Conditional Quantile Regression}
\newacronym[
	description={
			Version of eCQR where the upper and lower parametric quantile functions are parameterized by $\G$-equivariant NNs
		}
]{ecqr}{eCQR}{Equivariant CQR}

\newacronym{ccdf}{cCDF}{conditional Cumulative Distribution Function}
\newacronym{mlp}{MLP}{Multi-Layer Perceptron}
\newacronym{emlp}{eMLP}{Equivariant MLP}
\newacronym[
	description={
			Conditional Gaussian Mixture Model \citep{gilardi2002conditional}: A parametric model for benchmark conditional density estimation datasets. Generates random variables $\rvx$ and $\rvy$ of arbitrary dimensions with varying mutual information. Enables analytical computation of the \gls{pmd} density ratio (see \cref{sec:background}), unavailable in real-world datasets, allowing direct quantification of approximation error for the conditional expectation operator $\CE$ and its \gls{pmd} density ratio
		}
]{cgmm}{cGMM}{Conditional Gaussian Mixture Model}
\newacronym{mse}{MSE}{Mean Squared Error}
\newacronym{com}{CoM}{Center of Mass}
\newacronym{grf}{GRF}{Ground Reaction Forces}
\newacronym{svd}{SVD}{Singular Value Decomposition}
\newacronym{cnn}{CNN}{Convolutional Neural Network}
\newacronym{gnn}{GNN}{Graph Neural Network}

\newacronym[
	shortplural = CIs,
	longplural  = Confidence Intervals,
	first={Confidence Interval (CI)},
	firstplural={Confidence Intervals (CIs)},
]{ci}{CI}{Confidence Interval}

\newacronym[
	shortplural = NNs,
	longplural  = Neural Networks,
	first={Neural Network (NN)},
	firstplural={Neural Networks (NNs)},
]{nn}{NN}{Neural Network}

\newacronym[
	shortplural = DoF,                % plural of the short form
	longplural  = degrees of freedom,  % plural of the long form
	first={degree of freedom (DoF)},   % how the first *singular* appearance looks
	firstplural={degrees of freedom (DoF)}, % first *plural* appearance
]{dof}{DoF}{degree of freedom}
\makeglossaries

%% file: notation.tex
\usepackage{relsize}                        % Adjust sizes of math symbols

\DeclareMathOperator{\rank}{rank}
\newcommand{\norm}[1]{\left\|#1\right\|}
\newcommand{\abs}[1]{|#1|}

\DeclareRobustCommand{\symmlearning}{\href{https://danfoa.github.io/symmetric_learning/}{\texttt{symm\_learning}} }

\newcommand{\KleinFourGroup}{\mathbb{K}_{4}}                                % Klein Four Group
\newcommand{\CyclicGroup}[1][]{\mathbb{C}_{#1}}                             % Cyclic Group
\newcommand{\DihedralGroup}[1][]{\mathbb{D}_{#1}}                           % Dihedral Group
\newcommand{\UG}[1][]{\mathbb{U}(#1)}                                       % Unitary Group
\newcommand{\GLGroup}{\mathbb{GL}}                                          % General Linear Group
\newcommand{\SO}[1][\dimEvolution]{\mathbb{SO}_{#1}}                        % Special Orthogonal Group
\newcommand{\G}[1][]{\mathbb{G}_{\scalebox{0.6}{$#1$}}}                     % Symmetry Group
\newcommand{\g}{g}                                                          % Group member
\newcommand{\Glact}[1][]{
	\mathrel{
		\mathsmaller{
			\triangleright_{\scalebox{0.65}{$#1$}}
		}
	}
}                                                                            % Left group action    g \Glact x
\newcommand{\Gract}[1][]{\mathrel{\mathsmaller{\triangleleft}}}              % Right group action   x \Gract g
\newcommand{\Gconj}[1][]{\mathrel{\mathsmaller{\diamond}}}                   % Conjugate group action  g \Gconj := g \Glact X  \Gract \g^{-1}
\newcommand{\Gcomp}[1][]{\mathrel{\mathsmaller{\circ}}}                      % Group binary/composition operator 

\newcommand{\Endomorphism}[2][\G]{\text{End}_{#1}(#2)}
\newcommand{\EndomorphismC}[2][\G]{\text{End}^{\sC}_{#1}(#2)}
\newcommand{\homomorphism}[3][\G]{\text{Homo}_{#1}(#2,#3)}
\newcommand{\homomorphismDiag}[5]{
	\xymatrix{
		#1 \ar@{-}[r]^{#3}    \ar[d]^{#5} & #1 \ar[d]^{#5} \\
		#2 \ar@{-}[r]^{#4}                   & #2
	}
}
\newcommand{\invariantDiag}[5]{
	\xymatrix{
		#1 \ar@{-}[r]^{#3} \ar[dr]_{#4} & #1 \ar[d]^{#4} \\
		& #2 \ar@(r,d)[]^{#5}
	}
}
\newcommand{\isomorphism}[3][\G]{\text{Iso}_{#1}(#2,#3)}
\newcommand{\isomorphismDiag}[5]{
\xymatrix{
#1 \ar@{-}[r]^{#3}    \ar@<0.5ex>[d]^{#5} \ar@{<-}[d]_{#5^{-1}} & #1 \ar@<0.5ex>[d]^{#5} \ar@{<-}[d]_{#5^{-1}} \\
#2 \ar@{-}[r]^{#4}                   & #2
}
}

\newcommand{\mapping}[5]{ % Custom notation aux commands
	\begin{matrix}
		#1: & #2 & \longrightarrow & #3 \\
		    & #4 & \longrightarrow & #5
	\end{matrix}
}

\newcommand{\setsym}[2]{\gamma_{#1}(#2)}

\newcommand{\rep}[1][]{\bm{\rho}_{\scalebox{0.65}{$#1$}}}
\newcommand{\irrep}[1][]{\bar{\bm{\rho}}_{\scalebox{0.65}{$#1$}}}
\newcommand{\irrepKernel}[1][]{\sN_{\scalebox{0.65}{$#1$}}}
\newcommand{\irrepC}[1][]{\bar{\bm{\phi}}_{\scalebox{0.65}{$#1$}}}
\newcommand{\Oplus}{\bm{\oplus}}

\newcommand{\CE}{\oE_{\scalebox{0.65}{$\rvy \vert \rvx$}}}      % Conditional expectation operator L2 to  L2
\newcommand{\CEp}{\oE_{\scalebox{0.65}{$\nnParams$}}}      % Conditional expectation operator L2 to  L2
\newcommand{\CEIso}{\CE^{\scalebox{1.0}{$\isoIdxBrace$}}}      % Conditional expectation operator L2 to  L2
\newcommand{\CEpIso}{\CEp^{\scalebox{1.0}{$\isoIdxBrace$}}}      % Conditional expectation operator L2 to  L2
\newcommand{\CEm}{\mE}     % Conditional expectation operator L2 to  L2
\newcommand{\CEmp}{\CEm_{\scalebox{0.65}{$\nnParams$}}}     % Conditional expectation operator L2 to  L2
\newcommand{\CEmpIso}[1][\isoIdxBrace]{\CEmp^{\scalebox{1.0}{#1}}}     % Conditional expectation operator L2 to  L2
\newcommand{\PCE}{\CEp}           % Parametrized conditional expectation operator
\newcommand{\ECE}{\widehat{\oE}_{\scalebox{0.65}{$\nnParams$}}}  % Empirical conditional expectation operator 
\newcommand{\pmd}{\kappa}                    % Kernel of the conditional expectation operator
\newcommand{\HS}{\text{HS}}                  % Hilbert-Schmidt operator
\newcommand{\lowrankOp}[2][r]{#2_{#1}}  % Low-rank operator
\newcommand{\oTiso}[1][\isoIdxBrace]{\oT^{#1}}

\newcommand{\DCE}{\oD_{\scalebox{0.65}{$\rvy\vert \rvx$}}}  %Conditional expectation operator L2 to  L2
\newcommand{\isoDCE}{\DCE^{\isoIdxBrace}}  %Conditional expectation operator L2 to  L2
\newcommand{\Sigmap}{\mS_{\nnParams}}
\newcommand{\SigmapIso}{\mS_{\nnParams}^{\isoIdxBrace}}
\newcommand{\PDCE}{\oD_{\scalebox{0.65}{$\nnParams$}}}
\newcommand{\PDCEIso}{\oD_{\scalebox{0.65}{$\nnParams$}}^{\isoIdxBrace}}
\newcommand{\EDCE}{{\widehat{\oD}_{\scalebox{0.65}{$\nnParams$}}}}
\newcommand{\EDCEIso}{\EDCE^{\isoIdxBrace}}

\newcommand{\mux}{{P_{\rvx}}}    % Measure on X
\newcommand{\muy}{{P_{\rvy}}}    % Measure on Y
\newcommand{\muxy}{{P_{\rvx\rvy}}}    % Measure on X x Y
\newcommand{\muycondx}{{P_{\rvy\vert\rvx}}}    % Conditional measure on Y given X

\newcommand{\bSet}[1]{\sI_{\scalebox{0.8}{$#1$}}}  % Basis set
\newcommand{\Lp}[2][p]{\scalebox{0.85}{$\mathcal{L}^{#1}_{#2}$}}  % Lp space Args: 1: p, 2: measure, 3: set
\newcommand{\LpxFull}{\Lp[2]{\rvx}{}}  % L2 space on X
\newcommand{\LpyFull}{\Lp[2]{\rvy}{}} % L2 space on Y
\newcommand{\LpxFullIso}[1][\isoIdxBrace]{\Lp[2#1]{\rvx}{}}  % Isotypic subspace of Func space on X
\newcommand{\LpyFullIso}[1][\isoIdxBrace]{\Lp[2#1]{\rvy}{}}  % Isotypic subspace of Func space on Y

\newcommand{\Lpz}{\scalebox{0.9}{$\vsF$}}  % Func space on Z

\newcommand{\Lpx}{\scalebox{0.85}{$\vsF$}_{\scalebox{0.65}{$\rvx$}}}  % Truncated L2 space on X
\newcommand{\Lpy}{\scalebox{0.85}{$\vsF$}_{\scalebox{0.65}{$\rvy$}}}  % Truncated L2 space on Y
\newcommand{\Lpxp}{\scalebox{0.85}{$\vsF$}^{\scalebox{0.65}{$\nnParams$}}_{\scalebox{0.65}{$\rvx$}}}  % Func space on X
\newcommand{\Lpyp}{\scalebox{0.85}{$\vsF$}^{\scalebox{0.65}{$\nnParams$}}_{\scalebox{0.65}{$\rvy$}}}  % Func space on Y
\newcommand{\LpxIso}[1][\isoIdxBrace]{\Lpx^{#1}}  % Isotypic subspace of Func space on X
\newcommand{\LpyIso}[1][\isoIdxBrace]{\Lpy^{#1}}  % Isotypic subspace of Func space on Y
\newcommand{\LpxpIso}[1][\isoIdxBrace]{\Lpx^{\scalebox{0.65}{$\nnParams$}\scalebox{1.0}{#1}}}  % Isotypic subspace of Func space on X
\newcommand{\LpypIso}[1][\isoIdxBrace]{\Lpy^{\scalebox{0.65}{$\nnParams$}\scalebox{1.0}{#1}}}  % Isotypic subspace of Func space on Y

\newcommand{\iso}{{\text{\rm iso}}}    % Identifier for isotypic subspaces
\newcommand{\inv}{{\text{inv}}}    % Identifier for invariant subspaces
\newcommand{\regular}{\text{reg}}  % Identifier for regular representations
\newcommand{\isoNum}{{n_{\iso}}}   % Number of isotypic subspaces
\newcommand{\isoIdx}{{k}}          % Index for isotypic subspaces
\newcommand{\isoIrrepIdx}{{p}}     % Index for irreducible representations within isotypic subspaces
\newcommand{\irrepMultiplicity}{m} % Multiplicity of irreducible representations
\newcommand{\irrepDim}[1][]{d_{#1}}% Dimension of irreducible representations
\newcommand{\isoDim}{d_{\iso}}     % Dimension of isotypic subspaces
\newcommand{\isoIdxBrace}{{{\scalebox{0.6}{$(\isoIdx)$}}}}  % Isotypic index in brace

\newcommand{\obsFn}{f}                              % Observation function
\newcommand{\one}{\mathds{1}}
\newcommand{\bfx}{\phi}        % Basis function on X
\newcommand{\bfy}{\psi}        % Basis function on Y
\newcommand{\vbfx}{\bm{\phi}}  % Vector of basis functions on X
\newcommand{\vbfy}{\bm{\psi}}  % Vector of basis functions on Y
\newcommand{\vbfxp}{\bm{\phi}_{\nnParams}}  % Vector of basis functions on X
\newcommand{\vbfyp}{\bm{\psi}_{\nnParams}}  % Vector of basis functions on Y
\newcommand{\lsfn}{u}                % Left singular function on X
\newcommand{\rsfn}{v}                % Right singular function on Y
\newcommand{\lsfnp}{u^{\nnParams}}                % Left singular function on X
\newcommand{\rsfnp}{v^{\nnParams}}                % Right singular function on Y
\newcommand{\lsfnIso}{u^{\isoIdxBrace}}
\newcommand{\rsfnIso}{v^{\isoIdxBrace}}
\newcommand{\lsfnpIso}{u^{\nnParams\isoIdxBrace}}
\newcommand{\rsfnpIso}{v^{\nnParams\isoIdxBrace}}
\newcommand{\vlsfn}{\vu}             % Vector of left singular functions on X
\newcommand{\vrsfn}{\vv}             % Vector of right singular functions on Y
\newcommand{\vlsfnp}{\vlsfn_{\nnParams}}  % Vector of basis function on X
\newcommand{\vrsfnp}{\vrsfn_{\nnParams}}  % Vector of basis function on Y
\newcommand{\vlsfnpIso}[1][\isoIdxBrace]{{\vlsfn^{#1}_{\nnParams}}}             % Vector of left singular functions on X
\newcommand{\vrsfnpIso}[1][\isoIdxBrace]{{\vrsfn^{#1}_{\nnParams}}}             % Vector of right singular functions on Y

\newcommand{\sval}{\sigma}             % Singular value
\newcommand{\svalpIso}{\sigma^{\nnParams\isoIdxBrace}}

\newcommand{\bcoeff}{\valpha}   % Symbol for basis expansion coeff
\newcommand{\ebcoeff}{\alpha}   % Element of basis expansion coeff
\newcommand{\bcoeffx}{\valpha}  % Basis expansion coeff on X
\newcommand{\bcoeffy}{\vbeta}   % Basis expansion coeff on Y
\newcommand{\ebcoeffx}{\alpha}  % Element of basis expansion coeff on X
\newcommand{\ebcoeffy}{\beta}   % Element of basis expansion coeff on Y

\newcommand{\nnX}{\vf}  % Vector of basis function on X
\newcommand{\nnY}{\vh}  % Vector of basis function on Y
\newcommand{\nnParams}{{\boldsymbol{\theta}}}
\newcommand{\loss}{\mathcal{L}}

\newcommand{\rpar}{\alpha}

\newcommand{\error}{\mathcal{E}_\nnParams}
\newcommand{\errorIso}{\mathcal{E}_{\nnParams}^{\isoIdxBrace}}
\def\samplesize{N}
\newcommand{\EE}{\widehat{\mathbb{E}}}
\newcommand{\EEX}{\EE_{\rvx}}
\newcommand{\EEY}{\EE_{\rvy}}  
\newcommand{\PP}{\mathbb{P}}

\newcommand{\Identity}{\mI}
\newcommand{\scalarp}[1]{{\langle #1\rangle}}   % Scalar product
\newcommand{\innerprod}[2][]{{\langle #2 \rangle}_{#1}}

\newcommand{\C}{\mathbb{C}} %\C for komplekse tal
\newcommand{\N}{\mathbb{N}} %\N for naturlige tal
\def\dimtot{{r_{m}}}
\newcommand{\SVDdim}[1]{[\![#1]\!]_{r_m}}

%% file: ICLR_suggested_notation.tex
\usepackage{amsmath,amsfonts,bm}

\def\1{\bm{1}}

\def\rr{{\textnormal{r}}}

\def\rx{{\textnormal{x}}}
\def\ry{{\textnormal{y}}}
\def\rz{{\textnormal{z}}}

\def\rva{{\mathbf{a}}}

\def\rvx{{\mathbf{x}}}
\def\rvy{{\mathbf{y}}}
\def\rvz{{\mathbf{z}}}

\def\vmu{{\bm{\mu}}}
\def\valpha{{\bm{\alpha}}}
\def\vbeta{{\bm{\beta}}}

\def\vomega{{\bm{\omega}}}

\def\vtau{{\bm{\tau}}}
\def\va{{\bm{a}}}

\def\vc{{\bm{c}}}

\def\ve{{\bm{e}}}
\def\vf{{\bm{f}}}
\def\vg{{\bm{g}}}
\def\vh{{\bm{h}}}

\def\vk{{\bm{k}}}
\def\vl{{\bm{l}}}

\def\vp{{\bm{p}}}
\def\vq{{\bm{q}}}

\def\vu{{\bm{u}}}
\def\vv{{\bm{v}}}

\def\vx{{\bm{x}}}
\def\vy{{\bm{y}}}
\def\vz{{\bm{z}}}

\def\vsA{{\mathcal{A}}}

\def\vsF{{\mathcal{F}}}

\def\vsH{{\mathcal{H}}}

\def\vsV{{\mathcal{V}}}

\def\vsX{{\mathcal{X}}}
\def\vsY{{\mathcal{Y}}}
\def\vsZ{{\mathcal{Z}}}

\def\mA{{\bm{A}}}
\def\mB{{\bm{B}}}
\def\mC{{\bm{C}}}

\def\mE{{\bm{E}}}

\def\mI{{\bm{I}}}

\def\mK{{\bm{K}}}

\def\mM{{\bm{M}}}

\def\mO{{\bm{O}}}

\def\mQ{{\bm{Q}}}
\def\mR{{\bm{R}}}
\def\mS{{\bm{S}}}
\def\mT{{\bm{T}}}
\def\mU{{\bm{U}}}
\def\mV{{\bm{V}}}
\def\mW{{\bm{W}}}

\def\mSigma{{\bm{\Sigma}}}

\DeclareMathAlphabet{\mathsfit}{\encodingdefault}{\sfdefault}{m}{sl}
\SetMathAlphabet{\mathsfit}{bold}{\encodingdefault}{\sfdefault}{bx}{n}

\newcommand{\operator}[1]{\mathsf{#1}}
\def\oA{{\operator{A}}}
\def\oB{{\operator{B}}}
\def\oC{{\operator{C}}}
\def\oD{{\operator{D}}}
\def\oE{{\operator{E}}}

\def\oH{{\operator{H}}}

\def\oK{{\operator{K}}}

\def\oQ{{\operator{Q}}}

\def\oS{{\operator{S}}}
\def\oT{{\operator{T}}}
\def\oU{{\operator{U}}}
\def\oV{{\operator{V}}}

\def\sA{{\mathbb{A}}}
\def\sB{{\mathbb{B}}}
\def\sC{{\mathbb{C}}}
\def\sD{{\mathbb{D}}}
\def\sH{{\mathbb{H}}}
\def\sI{{\mathbb{I}}}

\def\sK{{\mathbb{K}}}

\def\sN{{\mathbb{N}}}

\def\sP{{\mathbb{P}}}

\def\sR{{\mathbb{R}}}

\def\sX{{\mathbb{X}}}

\newcommand{\E}{\mathbb{E}}

\newcommand{\R}{\mathbb{R}}

\newcommand{\Var}{\mathrm{Var}}

\newcommand{\Cov}{\mathrm{Cov}}

\DeclareMathOperator*{\argmin}{arg\,min}

%% file: appendix/math_glosary.tex
\section{Symbols and Notation}
\label{sec:notation}

\begin{tabular*}{\textwidth}{@{\extracolsep{\fill}} p{1.5in}p{\dimexpr\textwidth-1.5in-3\tabcolsep\relax}}
	\multicolumn{2}{c}{\textbf{Numbers and Arrays}} \\
	$x$ & A scalar, or scalar function $x(\cdot) $              \\
	${\vx}$ & A vector, or vector-valued function $\vx(\cdot)$  \\
	$\vx_1 \oplus \vx_2$ & Direct sum (stacking) of vectors, such that $\vx_1 \oplus \vx_2 := \begin{bsmallmatrix} \vx_1 \\ \vx_2 \end{bsmallmatrix}$\\
	$\mK$ & A matrix\\
	$\mA \oplus \mB$ & Direct sum of matrices, such that $\mA \oplus \mB := \begin{bsmallmatrix} \mA & \mO \\ \mO & \mB \end{bsmallmatrix}$\\
	$\oK$ & A linear operator \\
	% $\oA \oplus \oB$ & Direct sum of linear operators, such that the matrix form  of $\oA \oplus \oB$ can be expressed by $\begin{bsmallmatrix} \mA & \mO \\ \mO & \mB \end{bsmallmatrix}$ in an appropriate basis sets that expose the block-diagonal structure \\
	$\mI$ & Identity matrix\\
	% $\oI$ & \textcolor{red}{Identity operator}\\
	$\delta_{i,j}$ & The Kronecker function, equal to $1$ when $i = j$, and $0$ when $i \neq j$ \\
	% ·
	\multicolumn{2}{c}{} \\ % Empty row
	\multicolumn{2}{c}{\textbf{Sets, Vector Spaces, and Function Spaces}} \\
	$\vsX, \vsZ, \vsH, \vsF$ & A vector or Hilbert space \\
	$\bar{\vsX}, \bar{\vsZ}, \bar{\vsV}$ & An irreducible $\G$-stable space (\cref{def:G_stable_subspace})\\
	% $\bSet{\vsX}$ & \textcolor{red}{A basis set of the vector space} $\vsX$\\
	$\R, \C$ & The set of real and complex numbers \\
	$\vsX \oplus \vsY$ & Direct sum of vector spaces $\vsX$ and $\vsY$, such that if $\vx \in \vsX$ and $\vy \in \vsY$, then $\vx \oplus \vy \in \vsX \oplus \vsY$\\
	% $\vsF$ & \textcolor{red}{A function space} \\
	% $f_\valpha \in \vsF$ & \textcolor{red}{A function represented by its basis-expansion coefficient vector $\valpha$ with respect to a chosen basis $\bSet{\vsF} = \{\hat{f}_1, \dots \}$, such that $f_\valpha(\cdot) = \sum_{i=1}^{\dimModelSpace} \alpha_{i} \hat{f}_{i}(\cdot)$, where $\valpha = [\alpha_1, \dots]$} \\
	$\LpxFull:=\Lp[2]{P_\rvx}(\vsX,\R)$ & The space of square-integrable functions on $\vsX$. That is $\{ f \;\; \vert \int_{\vsX} |f(\vx)|^2 P_\rvx(d\vx) < \infty,\; f:\vsX \to \R \}$ \\
	$\innerprod[P_\rvx]{f, f'}$ & Inner product $\innerprod[\mux]{f, f'} := \int_{\vsX} f(\vx) f'(\vx) P_\rvx(d\vx)$ \\
	\multicolumn{2}{c}{} \\ % Empty row
	\multicolumn{2}{c}{\textbf{Group and Representation theory}} \\
	$\G$ & A symmetry group\\
	$\g, \g_1, \g_a$ & A symmetry group element\\
	$\g \Glact \vx$ & The (left) group action of $\g$ on $\vx$ defined by $\g \Glact \vx := \rep[\vsX](\g)\vsX$\\
	$\rep[\vsX]{} $ & A representation of the group $\G$ on the vector space $\vsX$, defined for a chosen basis of $\vsX$ \\
	$\irrep[\isoIdx]{} $ & An irreducible representation (\cref{def:irreducible_representation}) of the group $\G$ \\
	$\irrepDim[\isoIdx] := |\irrep[\isoIdx]{}| $ & Dimensionality of the irreducible representation $\irrep[\isoIdx]{}$ (see \cref{fig:diagrams}) \\
	$\irrepMultiplicity_{\isoIdx} $ & Multiplicity of the irreducible representation $\irrep[\isoIdx]{}$ in a given larger representation (see \cref{fig:diagrams})\\ 
	$\isoNum $ & Number of distinct irreps present in a given larger representation.\\
	$\rep[\vsX](\g)$ & Representation of the group element $\g$ on the vector space $\vsX$\\
	$\rep[\vsX]{} \oplus \rep[\vsY]{} $ & Direct sum of group representations,
	such that $\rep[\vsX](\g) \oplus \rep[\vsY](\g) := \begin{bsmallmatrix} \rep[\vsX](\g) &  \\  & \rep[\vsY](\g) \end{bsmallmatrix}$\\
	% $\rep[\vsX]{}=\mQ(\oplus_{i=1}^{\irrepMultiplicity_{\isoIdx}}\irrep[\isoIdx] )\mQ^\top $ & Decomposition of a representation into its irreducible representations \cref{def:decomposable_representation} \\
	$\G \vx$ & The group orbit of $\vx$, defined as $\G\vx := \{ \g \Glact \vx \;|\; \g \in \G\}$\\
	$\setsym{\G'}{A}$ & The symmetry index of a set $A\subseteq\vsX$ w.r.t. probability distribution on $\vsX$ and group elements $\G'\subseteq\G$\\
	$\G[a] \times \G[b]$ & Direct product of groups $\G[a]$ and $\G[b]$\\
	$\UG[\vsX]$ & Unitary group on the vector space $\vsX$\\
	$\GLGroup(\vsX)$ & General Linear group on the vector space $\vsX$, a.k.a the space of invertible matrices in $\R^{|\vsX|\times |\vsX|}$\\
	$\CyclicGroup[n]$ & Cyclic group of order $n$\\
	$\KleinFourGroup$ & Klein four-group\\
	\multicolumn{2}{c}{} \\ % Empty row
	\multicolumn{2}{c}{\textbf{Probability Theory}} \\
	% 
	% $\rvx$ & A random vector \\
	% $(\vsX, \Sigma_{\vsX}, \mux)$ & \textcolor{red}{A probability measure space, where $\vsX$ is the sample space, $\Sigma_{\vsX}$ is the sigma-algebra, and $\mux : \Sigma_{\vsX} \rightarrow [0,1]$ is the probability measure} \\
	% $P(\rx)$ & A probability distribution over a discrete variable\\
	% $p(\rx)$ & \textcolor{red}{A probability distribution over a continuous variable, or over a variable whose type has not been specified} \\
	$\rvx \sim \PP(\rvx)$ & Random vector $\rvx \in \vsX$ has distribution $\PP(\rvx)$\\
	$\mux$ & A probability measure on the space $\vsX$\\
	$\E_\rvx [ f(\rvx) ]$ & Expectation of $f(\rvx)$ with respect to $P_\rvx$ \\
	$\Cov(f(\rvx))$ & Variance of $f(\rvx)$ with respect to $P_\rvx$, define as $\E_{\rvx} (f(\rvx) - \E_{\rvx} f(\rvx))^2$ \\
	$\Cov(f(\rvx),h(\rvy))$ & Covariance of $f(\rvx)$ and $h(\rvy)$ with respect to the joint distribution $P_{\rvx\rvy}$, defined as $\E_{\rvx\rvy} (f(\rvx) - \E_{\rvx} f(\rvx))(h(\rvy) - \E_{\rvy} h(\rvy))$\\
	$\mathcal{N} ( \vx ; \vmu , \mSigma)$ & Gaussian distribution over $\vx$ with mean $\vmu$ and covariance $\mSigma$ \\
\end{tabular*}

\newpage

\setacronymstyle{long-short}
\glsadd{ncpaug}
\printglossary[type=\acronymtype, title={Acronyms}]
% \section{Acronyms}p
\label{sec:acronyms}

%% file: appendix/assumptions.tex
\section{Related Work}\label{sec:related_works}
% \textcolor{red}{Answer the following questions:
% \begin{itemize}
%     \item What is contrastive learning?
%     \item Why do people use contrastive learning?
%     \item Which guarantees have been proved for downstream tasks?
% \end{itemize}
% }

\subsection{Contrastive Representation Learning}
Contrastive representation learning obtains high-dimensional representations from unlabeled data by contrasting positive and negative sample pairs via a noise contrastive loss (similar to \cref{eq:contrastive_loss}) \citep{le2020contrastive,waida2023towards,bao2022surrogate}. Most works in this field aim to learn representations in a self-supervised fashion that transfer well to downstream classification tasks \cite{johnson2022contrastive,Cole_2022_CVPR,tosh2021contrastive,oord2018representation,tsai2021self,chen2021provable}. In contrast, our approach targets representations that effectively transfer to (equivariant) regression and uncertainty quantification, as in \cite{kostic2024learning}. Given a dataset $\sD = \{(\vx_n, \vy_n)\}_{n=1}^N$ from a target (stochastic) function $\rvy=\vf(\rvx)$, we treat positive pairs as drawn from the joint distribution $(\vx,\vy)\sim\PP(\rvx,\rvy)$ and negative pairs as drawn from the product of the marginals $(\vx,\vy)\sim\PP(\rvx)\PP(\rvy)$. In this setting, our contrastive loss aims to learn representations that approximate the \gls{pmd} ratio
$
	\pmd(\vx,\vy)=\frac{\PP(\vx,\vy)}{\PP(\vx)\PP(\vy)},
$ \cite{kostic2024learning}
or equivalently, the pointwise mutual information $\ln(\pmd(\vx,\vy))$ \citep{oord2018representation,lin2024mutual_skeleton,henaff2020data,ozair2019wasserstein}. Crucially, our work is the first study this problem when there is prior knowledge of the invariance of $\pmd$ under the action of a compact symmetry group, which occurs in most applications of \gls{gdl}.

% {\bf Contrastive losses and density ratio approximation.} \citet{tosh2021contrastive} demonstrated that various noise contrastive losses---including those in \citep{oord2018representation,chen2021provable,kostic2024learning,ozair2019wasserstein}---yield optimal representations that effectively approximate the \gls{pmd} kernel 
% $
% \pmd(\vx, \vy) = \frac{\PP(\vx, \vy)}{\PP(\vx)\PP(\vy)},
% $
% which defines a density ratio whenever the densities of $\rvx$ and $\rvy$ exist. A decade earlier, \citet{sugiyama2012density,Sugiyama2011} related different density ratio approximation losses (or noise contrastive losses) within a unified Bregman divergence framework, where each loss corresponds to a particular choice of distance function measuring the discrepancy between the true ratio $\pmd$ and its approximation $\pmd_\nnParams$ (see \citep[Ch.~7]{sugiyama2012density} for details).\footnote{Within this framework, \eqref{eq:contrastive_loss} is known as \textit{Least-Squares Importance Fitting} \citep{kanamori2009least}, while the widely used InfoNCE loss \cite{oord2018representation} is associated with the \textit{Kullback-Leibler importance estimation procedure}.} The aim of our work is similarly to learn representations that approximate $\pmd$. However, as in \citep{kostic2024learning}, we interpret this approximation from an operator-theoretic perspective within the context of the conditional dependence between $\rvx$ and $\rvy$.

\customParagraph{Linear Transferability} The goal of contrastive representation learning is to acquire representations that transfer to diverse downstream inference tasks \cite{bengio2013representation,waida2023towards}. While empirical studies demonstrate that contrastive learning can outperform supervised methods \cite{Cole_2022_CVPR,oord2018representation,henaff2020data}, theoretical works aim to establish \textit{linear separability/transferability} guarantees \citep{haochen2022beyond} \footnote{Also refeered to as linear evaluation protocol by \citet{chen2020simple}}. That is, showing that linear functionals of the (frozen) learned representations suffice for regression/classification inference.
% \textcolor{red}{AF: AFAIK, many of the top contrastive learning papers (e.g., SimCLR) use the "linear protocol" for evaluation of downstream tasks.}
%

In the context of \textbf{classification}, \cite{waida2023towards,bao2022surrogate,johnson2022contrastive} show that contrastive learning losses serve as surrogates for standard supervised classification losses (e.g., the cross-entropy). Where the gap between the surrogate and supervised loss diminishes with the number of negative samples \cite{bao2022surrogate} ($N^2$ for the loss in \cref{eq:empirical_loss}).
To provide these transferability guarantees, these work assume $\vsX = \vsY$, so that the \gls{pmd} ratio $\pmd$ becomes a positive definite kernel. Consequently, kernel method guarantees can be transferred to the classification task, even when the representations are parameterized by \glspl{nn} \citep{johnson2022contrastive,bao2022surrogate,haochen2022beyond}.

Considerably fewer works have studied contrastive representation learning in the context of downstream \textbf{regression} tasks \citep{yerxa2024contrastive,kostic2024learning}. Crucially, \citet{kostic2024learning} show that a contrastive learning loss serves as surrogate to the \gls{mse} regression loss (A summary of this method appears in \cref{sec:background} and in \cref{tab:ncp_guarantees}). While, to the best of our knowledge, \citep{yerxa2024contrastive} is the only work empirically studying contrastive learning for regression in the presence of symmetries.

\begin{table*}[h!]
	\centering
	% \small
	\resizebox{\textwidth}{!}{%
		\begin{tabular*}{\textwidth}{@{\extracolsep{\fill}}p{1cm}c|c@{}}
			\textbf{Task}
			&
			$ \vf(\vx) := \E_{\rvy}[\rvy\vert\rvx{=}\vx] \approx \hat{\vf}_{\nnParams}(\vx)$
			&
			$\PP[\rvy {\in} \sB \vert \rvx\in\sA] \approx \widehat{\PP}_{\nnParams}[\rvy \!\in\! \sB \vert \rvx \!\in\! \sA]$
			\\ \midrule
			\textbf{Estimate}
			&
			$\EE_{\rvy}[\rvy] + \vbfxp(\vx)^\top \CEmp \EE_{\rvy}[\vbfyp(\rvy) \otimes \rvy]$
			&
			$
				\EEY[\one_\sB] + \frac{
					\EEX[\one_\sA(\rvx) \otimes \vbfxp(\rvx)]^\top\CEmp \EEY[\one_\sB(\rvy) \otimes \vbfyp(\rvy)]
				}{
					\EEX[\one_\sA(\rvx)]
				}
			$
			\\ \midrule
			\textbf{Guarantees}
			&
			$\|\vf {-} \hat{\vf}_{\nnParams}\|_{\LpxFull}
				{\lesssim}
				\sqrt{\Var[\norm{\rvy}]}
				\left(
				\error^{r}
				{+}
				\frac{\ln (\sfrac{1}{\delta})}
				{
					\samplesize^{\frac{\rpar}{1+2\rpar}}
				}
				\right)
			$
			&
			$
				|\PP {-} \widehat{\PP}_{\nnParams}| {\lesssim}
				\sqrt{\frac{\sP[\rvy\in \sB]}{\sP[\rvx\in \sA]}}
				\left(
				\error^{r}{+}
				\frac{\ln (\sfrac{1}{\delta})}
				{
					\samplesize^{\frac{\rpar}{1+2\rpar}}
				}
				\right)
			$
		\end{tabular*}%
	}
	\vspace*{0.2cm}
	\caption{Statistical learning guarantees of \gls{ncp} \citep{kostic2024learning} for regression and conditional probability estimation. The bounds are shaped by
	the quality of the learned representations $\error^{r} = \|\CE - \CEp \|_{\text{op}} \leq \sqrt{\loss_{\gamma}(\nnParams)-\loss_{\gamma}(\star)}$ (see \eqref{eq:contrastive_loss}), the sample size $N$, and the decay rate of $\CE$ singular-values $\alpha>0$, which quantifies the difficulty of the problem.
	}
	\label{tab:ncp_guarantees}
\end{table*}

% \textcolor{red}{AF: I think the most important thing is to clearly explain what we do different: compared to standard contrastive learning, we don't have unlabeled data, so we are not doing unsupervised/self-supervised learning. We are doing supervised learning, albeit with a nonstandard task (PMD regression / conditional mean embedding) with some transferability goals/guarantees. Finally, AFAIK, our theory is much nicer than standard contrastive learning, as they impose artificial assumptions (e.g., existence of K latent classes) to get what we already have for free in our setting: X and Y, marginals and joint, a dependence structure in the form of the PMD, which can be estimated if the associated operator is compact and is subject to all the nice literature of density ratio estimation from the Japanese guys.}

% \paragraph{Representation learning via \gls{pmd} approximation} \citet{oord2018representation} first proposed \textit{contrastive predictive coding}, a self-supervised representation learning approach that models conditional relationships between high-dimensional variables using a contrastive loss similar to \cref{eq:contrastive_loss}, known as \textit{InfoNCE}. This approach is closely linked to density ratio fitting frameworks \citep{tsai2020neuralPMD,sugiyama2012density,sugiyama2010conditional,Sugiyama2011}, where the density ratio of interest is the \gls{pmd}. Although \citet{oord2018representation,tsai2021self,chen2021provable,ozair2019wasserstein} 

\subsection{Equivariant Representation Learning}
Equivariant contrastive representation learning \cite{dangovski2022equivariant,wang2024understanding} aims to learn representations that are equivariant—instead of invariant—to data transformations. For example, \citet{marchetti2023equivariant,gupta2023structuring,lin2024mutual_skeleton} provide empirical evidence that representations of 3D scenes, images, and human body poses that are equivariant to translations, rotations, or reflections yield improved performance in \textit{classification} tasks. Additionally, \citet{yerxa2024contrastive} show that rotation- and reflection-aware image representations enhance the \textit{regression} of neural responses in the macaque inferior temporal cortex, while also providing theoretical justification that such equivariant representations mirror the known structure of animal visual perception. By introducing these transformations via data-augmentation of the training set, these methods inherently enforce symmetries in the data distributions, which are the fundamental priors assumed in \cref{sec:problem_formulation}.

\customParagraph{Disentangled Representations} In equivariant representation learning, disentangled representations have been extensively studied \citep{wang2024disentangled}. Initially, \citep{bengio2013representation} defined disentanglement as decomposing representations into components that capture distinct, independently varying factors. Later, using group theory, \citet{higgins2018towards} formalized that a representation is disentangled if its space decomposes into orthogonal subspaces reflecting a symmetry group decomposition, with each subspace influenced exclusively by one quotient group (see \cref{def:disentangled_representation}). As discussed in \cref{sec:group_theory}, this aligns with the isotypic decomposition of a Hilbert space \citep{mackey1980harmonic}:
$
	\vsH = \Oplus_{\isoIdx=1}^{\perp} \vsH^\isoIdxBrace
$---known in dynamical systems \citep{golubitsky2012singularities_groups_bifurcation}---when the symmetry group decomposes as
$
	\G = \prod_{\isoIdx=1}^{\isoNum} \G^\isoIdxBrace.
$
Orthogonality between subspaces follows from Schur's orthogonality relations via Cartan's and Peter-Weyl's theorems \citep{cartan1952theorie}. This symmetric structure is the cause of the achitectural constraints imposed in the \gls{encp} architecture \cref{fig:diagrams}.

Several empirical works have explored disentanglement in representation learning. For instance, \citet{keurti2023homomorphism} proposed an autoencoder-based method to learn disentangled equivariant representations by using loss regularization to enforce latent space equivariance and sparsity for separating latent group actions. Unlike our approach, their method does not assume prior knowledge of the symmetry group and relies entirely on loss regularization rather than architectural constraints. Similarly, works such as \citet{yang2023latent,dangovski2022equivariant} have investigated various symmetry priors in latent space by examining the emergence of disentangled structures and enforcing algebraic constraints. Notably, in fields like molecular dynamics, physics, computer graphics, and robotics, symmetry priors are intrinsic to the task or system \citep{ordonez2024morphological,lin2024mutual_skeleton,marchetti2023equivariant}, making them natural assumptions. In a similar spirit to our work, \citet{marchetti2023equivariant} leverage the known $\SO[3]$ symmetries of the 3D world to learn $\SO[3]$-disentangled equivariant representations using contrastive learning, thereby demonstrating the empirical advantages of symmetry-aware, disentangled representations for object classification.

\subsection{Symmetry-Aware Statistical Learning Theory} Existing literature on symmetry-aware learning focuses on group-invariant regression via kernel methods  \citep{mroueh2015learning,tahmasebi2023exact,elesedy2021provablystrict,elesedy2021provablystrict2,elesedy2025symmetrygeneralisationmachinelearning,pal2017maxmargininvariantfeatures,mei2021learning,bietti2021sample,donhauser2021rotational}. Most of these methods cannot be directly transferred to modern \gls{gdl} architectures.

In contrast, in deep learning and \gls{gdl}, while many works offer a group-theoretical analysis and empirical evidence of the benefits of incorporating symmetry priors \citep{kashinath2021physics,wang2020incorporating,wang2022approximately,brandstetter2022clifford}, none, to our knowledge, provide statistical learning guarantees for equivariant \textbf{regression}. The only exception is \citep{behboodi2022pac}, which derives generalization bounds for a \gls{mlp} architecture in the context of  $n$-class \textbf{classification} task using a margin loss. In contrast, our work provides statistical learning guarantees for equivariant \textbf{regression} and symmetry-aware \textbf{uncertainty quantification}, both as corollaries of \cref{thm:main}.

%{\bf Approximation of equivariant operators}~~~ Previous works on equivariant linear operators have focused on dynamical systems with symmetry \citep{salova2019koopman,ordonez2024dynamics}, where the conditional expectation operator is known as the Koopman/Transfer operator. Notably, \citep{salova2019koopman} proved its decomposition into isotypic components, which \citep{ordonez2024dynamics} leveraged for dynamics modeling.

\section{Symmetry Constraints on Conditional Expectations}\label{sec:aux_results}

% In our work, we rely on key assumptions required to interprete the
% \gls{ncp} bilinear model of \cref{fig:ncp_diagram} with the operator-theoretic lense needed to produce the statistical learning guarantees of the method.
% for the approximation of the conditional expectation operator $\oE_{\rvy\mid\rvx}: \LpyFull \mapsto \LpxFull$ (\cref{eq:conditional_expectation_op}) in finite dimensions.

% \begin{assumption}[Compactness of the conditional expectation operator \citep{kostic2024neural}]
% 	\label{ass:compactness}
% 	\daniel{This needs rework and better explanation. We want to intuitively explain the absolutely continuous constraint assumed $\muxy \ll \mux \times \muy$}
% 	The conditional expectation operator $\oE_{\rvy\mid\rvx}$ is compact whenever the conditional probability density $P(\rvy \mid \vx)$ has the same support as the marginal $\muy$ for any likely conditioning event $\rvx = \vx \in \vsX$, such that $\mux(\vx) > 0$. In other words, the operator is compact if $\muy(\vy) = 0$, then $p(\vy \mid \vx) = 0$ for all $\vy \in \vsY$.

% 	Crucially, this assumption breaks for purely deterministic regression tasks of the form $\vy = \vf(\vx)$ where $\vf$ is a deterministic function.
% \end{assumption}

Under the assumed symmetry priors in \eqref{eq:main_assumtions} the conditional expectation of $\rvy$ is a $\G$-equivariant function/map. This property is depicted in \cref{fig:symmetry_priors_long}-center and proved in the following proposition.

\begin{proposition}[$\G$-equivariant conditional expectations]
	\label{prop:equivariant_conditional_expectation}
	Let $\rvx \in \vsX$ and $\rvy \in \vsY$ be two vector valued random variables satisfying the symmetry priors of \cref{eq:main_assumtions}. Then, the conditional expectation of $\rvy$ given $\rvx$ is $\G$-equivariant, since, for every $\g \in \G, \vx \in \vsX$,
	\begin{align*}
		\E[\rvy \vert \rvx = \g \Glact[\vsX] \vx]
		 & =  \g \Glact[\vsY] \E[\rvy \vert \rvx = \vx]                                                               \\
		 & = \int_{\vsY} \g \Glact[\vsY] \vy \; \muycondx(d\vy \vert \vx)
		%  \\
		%  & 
		= \int_{\g^{-1}\Glact[\vsY]\vsY} \vy \; \muycondx\Bigl(\g^{-1}\Glact[\vsY]d\vy \vert \vx\Bigr)                \\
		 & = \int_{\vsY} \vy \; \muycondx(d\vy \vert \g \Glact[\vsX] \vx) \quad \text{(by \cref{eq:main_assumtions})} \\
		 & = \E[\rvy \vert \rvx = \g \Glact[\vsX] \vx].
	\end{align*}
\end{proposition}

\section[G-Equivariant Bilinear NN Architecture]{$\G$-Equivariant Bilinear \gls{nn} Architecture}
\label{app:equivariant_nn_architecture}

This section outlines how to construct a $\G$-equivariant disentangled representation for the random variables $\rvx$ and $\rvy$ using \textbf{any} type of $\G$-equivariant \gls{nn} architecture backbone, such as \gls{mlp}, CNNs, Transformers, and others.

Let $\nnX_{\nnParams}:\vsX \mapsto \mathbb{R}^r$ and $\nnY_{\nnParams}:\vsY \mapsto \mathbb{R}^r$ be two $\G$-equivariant NNs, whose outputs will be interpreted as the basis functions of the truncated symmetric function spaces $\Lpx \subset \LpxFull$ and $\Lpy \subset \LpyFull$. Assume, the group representations on $\Lpx$ and $\Lpy$ are constructed from multiplicities of the group's regular representation, $\rep[\Lpx] = \bigoplus_{n=1}^{\sfrac{r}{|\G|}} \rep[\regular]$ and $\rep[\Lpy] = \bigoplus_{n=1}^{\sfrac{r}{|\G|}} \rep[\regular]$---as done usually in practice \citep{weiler2023EquivariantAndCoordinateIndependentCNNs,cesa2022program,kondor2018clebschgordannets}. Since for (most) finite groups,  the decomposition of $\rep[\regular]$ into \textsl{irreps} is known or can be computed \citep{babai1990computing,unger2006computing}, we can safely assume access to the analytical change of basis $\mQ_{\rvx}: \Lpx \mapsto \Lpx$ and $\mQ_{\rvy}: \Lpy \mapsto \Lpy$ to transition to the isotypic basis. Consequently, we can directly parameterize the representations of the random variables in disentangled form as:
\begin{equation}
	\label{eq:disentangled_representations}
	\small
	\begin{aligned}
		\vbfx_{\nnParams}(\cdot) = \mQ_{\rvx}^{\top}
		(
		\nnX_{\nnParams}(\cdot) -  \E_\rvx [\nnX_{\nnParams}(\rvx)]
		),
		\quad
		\vbfy_{\nnParams}(\cdot) = \mQ_{\rvy}^{\top}(
		\nnY_{\nnParams}(\cdot) - \E_\rvy [\nnY_{\nnParams}(\rvy)]
		).
	\end{aligned}
\end{equation}
% \paragraph{Spectral basis}
Given that during training these representations are not orthogonal, the truncated operator is parameterized as the trainable $\G$-equivariant matrix $\CEmp = \Oplus_{\isoIdx}^{\isoNum} \CEmpIso = \Oplus_{\isoIdx}^{\isoNum} \sum_{b \in \sB} \bm{\Theta}_{b}^{\isoIdxBrace} \otimes \Psi(b)$, where $\sB$ is the endomorphism's basis of the irreducible representation $\irrep[\isoIdx]$, $\Psi: \sB \to \R^{\irrepDim[\isoIdx]\times \irrepDim[\isoIdx]}$ is a mapping from the endomorphism's basis to the endomorphism space, and $\bm{\Theta}_{b}^{\isoIdxBrace} \in \R^{\irrepMultiplicity^{y}_{\isoIdx} \times \irrepMultiplicity^{x}_{\isoIdx}}$ represent the block's trainable parameters for each $b \in \sB$, see details in \cref{def:equivLinearMapsLong}.

Note that after training, the \gls{svd} of the learned operator can be computed by exploiting the constraints imposed by the operator's $\G$-equivariance (see \cref{thm:iso_svd,fig:diagrams}).
% Importantly, once changed to the spectral basis, the group action on the approximated spectral basis matches that on the isotypic basis (see \cref{cor:representation_on_singular_basis}).

\subsection{Continuous and Non-Compact Symmetry Groups}
\label{app:continuous_noncompact_symmetry_groups}

\paragraph{Compatibility with Continuous Compact Symmetry Groups}
Formally, Peter–Weyl/Cartan’s theorem—the engine behind our isotypic decomposition—applies to every compact symmetry group, covering both compact Lie groups and finite groups. The challenge with compact continuous groups is that they possess infinitely many irreducible representations (i.e., $\isoNum \to \infty$), so a finite dimensional space cannot be partitioned into infinitely many isotypic subspaces.

Traditional \gls{gdl} techniques overcome this by discretising the group. For $\SO[3]$‑equivariance, for instance, architectures are often restricted to a finite subgroup (e.g., icosahedral or octahedral); see \citep{cesa2022program,weiler2023EquivariantAndCoordinateIndependentCNNs} for details. Crucially, such discretization produces a finite symmetry subgroup, to which our formalism applies unchanged.

\paragraph{Compatibility with Non-Compact Symmetry Groups}
The proposed representation disentanglement applies to \textbf{compact} continuous/finite symmetry groups, given that compactness is a fundamental requirement of the isotypic decomposition. However, equivariance to non-compact symmetry groups can still be achieved by choosing an appropriate \gls{nn} architecture for parameterizing the embedding functions $\vbfxp: \vsX \to \R^r$ and $\vbfyp: \vsY \to \R^r$.

For instance, as is commonly done for image processing in \gls{gdl}---where images are assumed to possess 2D rotational \emph{and translational symmetries} $\G = \SO[2] \rtimes \mathbb{T}_2$---equivariance to the non-compact translation group $\mathbb{T}_2$ is naturally encoded via a \gls{cnn} architecture, which is by construction $\mathbb{T}_2$-equivariant. Constraining the filters of the \gls{cnn} then ensures the compact $\SO[2]$-equivariance. These architectures are usually referred to as $\G$-steerable \glspl{cnn} ; see \citet{cesa2022program,weiler2023EquivariantAndCoordinateIndependentCNNs} for details.

Given that our method only assumes the embedding functions $\vbfxp: \vsX \to \R^r$ and $\vbfyp: \vsY \to \R^r$ are equivariant to a compact symmetry group, we can use \emph{any} \gls{nn} backbone architecture which could encode equivariance to non-compact groups, like $\mathbb{T}_2$ in image/video processing or $\mathbb{T}_1$ in time series data. That is, our method can also be used to train $\G$-steerable \glspl{cnn} representations.

\section{Symmetry-Aware Orthonormalization of Disentangled Representations}
\label{sec:loss_function}

This section covers how to compute unbiased empirical estimates of the orthonormalization and centering regularization terms in \cref{eq:disentangled_contrastive_loss} in the presence of symmetries.

Let $\oE_{\rvy\mid\rvx} : \LpyFull \mapsto \LpxFull$ be the conditional expectation operator and $\CEp: \Lpy \mapsto \Lpx$ be its $r$-rank approximation on the spaces $\Lpx = \text{span}(\{\bfx_i\}_{i=1}^r)$ and $\Lpy = \text{span}(\{\bfy_i\}_{i=1}^r)$.
Denote by $
	\pmd(\vx,\vy):= \frac{P_{\rvx\rvy}(\vx,\vy)}{\mux(\vx)\muy(\vy)}
$ and $
	\pmd_{\nnParams} (\vx,\vy):= \sum_{i,j=1}^r [\CEmp]_{i,j} \bfx_i(\vx) \bfy_j(\vy) = \vbfx(\vx)^\top \CEmp \vbfy(\vy)
$ the kernel functions of the full and restricted operator, respectively.
Then we have that:
\begin{subequations}
	\label{eq:theoretical_loss_fn_long}
	\small
	\begin{align}
		\small
		\|\oE_{\rvy\mid\rvx} - \CEp\|_{\HS}^2
		 & \leq
		-2\innerprod[\HS]{\oE_{\rvy\mid\rvx},\CEp}
		+
		\|\CEp\|_{\HS}^2,
		\label{eq:theoretical_loss_fn_long_a}
		\\
		 & \leq
		-2
		\int_{\vsX\times\vsY} \pmd(\vx,\vy) \pmd_{\nnParams} (\vx,\vy) \mux(d\vx)\muy(d\vy)
		+
		\int_{\vsX\times\vsY} \pmd_{\nnParams} (\vx,\vy)^2 \mux(d\vx)\muy(d\vy)
		\nonumber
		\\
		 & \leq
		-2
		\int_{\vsX\times\vsY} \pmd_{\nnParams} (\vx,\vy) \muxy(d\vx,d\vy)
		+
		\int_{\vsX\times\vsY} \pmd_{\nnParams} (\vx,\vy)^2 \mux(d\vx)\muy(d\vy)
		\nonumber
		\\
		 &
		\leq -2\E_{\rvx\rvy} \pmd_{\nnParams} (\rvx,\rvy) + \E_{\rvx}\E_{\rvy} \pmd_{\nnParams} (\rvx,\rvy)^2.
		\label{eq:theoretical_loss_fn_long_d}
		% \\
		%  & \leq -2 \sum_{i,j=1}^r \CEm_{i,j} \underbrace{\int_{\vsX\times\vsY} \bfx_i(\vx) \bfy_j(\vy) P_{\rvx\rvy}(\vx,\vy) d\vx d\vy}_{\CEm_{i,j}}
		% +
		% \E_{\mux\muy} \pmd_{\nnParams} (\rvx,\rvy)^2,
		% \nonumber
		% \\
		%  & \leq -2 \sum_{i,j=1}^r [\CEm]^2_{i,j}
		% +
		% \int_{\vsX\times\vsY} \bigg(\sum_{i,j=1}^r \CEm_{i,j} \bfx_i(\vx) \bfy_j(\vy)\bigg)^2 \mux(\vx)\muy(\vy) d\vx d\vy,
		% \nonumber
		% \\
		%  &
		% \leq -2 \|\widetilde{\mE}_{\rvy\mid\rvx}\|_{\text{F }}^2
		% +
		% \sum_{i,j,i',j'=1}^r \CEm_{i,j} [\CEm]_{i',j'}
		% \underbrace{\int_{\vsX} \bfx_i(\vx) \bfx_{i'}(\vx) \mux(\vx) d\vx}_{\innerprod[\rvx]{\bfx_i}{\bfx_{i'}}}
		% \underbrace{\int_{\vsY} \bfy_j(\vy) \bfy_{j'}(\vy) \muy(\vy) d\vy}_{\innerprod[\rvy]{\bfy_j}{\bfy_{j'}}}
		% \nonumber
		% \\
		%  &
		% \leq -2 \|\widetilde{\mE}_{\rvy\mid\rvx}\|_{\text{F }}^2
		% +
		% \mathrm{tr}
		% (\widetilde{\mE}_{\rvy\mid\rvx} \mV_\rvy \widetilde{\mE}_{\rvy\mid\rvx}^\top \mV_\rvx),
		% \quad
		% \text{s.t. }
		% \quad
		% [\mV_\rvx]_{i,i'} = \innerprod[\rvx]{\bfx_i}{\bfx_{i'}}, \;
		% [\mV_\rvy]_{j,j'} = \innerprod[\rvy]{\bfy_j}{\bfy_{j'}},
		% \label{eq:theoretical_loss_fn_long_h}
		% \\
		%  &
		% \leq -2\|\widetilde{\mE}_{\rvy\mid\rvx}\|_{\text{F }}^2 + \|\widetilde{\mE}_{\rvy\mid\rvx}\|_{\text{F }}^2 =
		% \leq -\|\widetilde{\mE}_{\rvy\mid\rvx}\|_{\text{F }}^2
		% \qquad
		% \text{s.t. }
		% \quad
		% \mV_\rvx = \mV_\rvy = \mI_r
		% \label{eq:theoretical_loss_fn_long_i}
	\end{align}
\end{subequations}
% Where \cref{eq:theoretical_loss_fn_long_d} presents a contrastive loss function, introduced in \citet{kostic2024learning,wang2022spectral}, which can be estimated from samples of the joint distribution $P_{\rvx\rvy}$ and the marginals $\mux$ and $\muy$. \Cref{eq:theoretical_loss_fn_long_h} translates this loss when the truncated function spaces possess arbitrary basis sets, and \cref{eq:theoretical_loss_fn_long_i} when the basis sets are orthonormal. This suggests that in the orthonormal scenario, the loss reduces to the maximization of the spectral norm of the truncated operator $\|\widetilde{\mE}_{\rvy\mid\rvx}\|_{\text{F }}^2 = \sum_{i=1}^r \sigma_i^2$, where $\sigma_i$ are the operator's singular values, as predicted by the Eckart-Young-Mirsky theorem \citep{eckart1936approximation}.
For the purpose of our representation learning problem, we consider the scenario in which the chosen basis sets include the constant function, and all other basis functions are centered by construction. That is, $\bSet{\Lpx} = \{\one_{\mux}\} \cup \{\bfx_i \mid \innerprod[\rvx]{\bfx_i, \one_{\mux}} = 0\}_{i=1}^{r}$ and $\bSet{\Lpy} = \{\one_{\muy}\} \cup \{\bfy_i \mid \innerprod[\rvy]{\bfy_i, \one_{\muy}} = 0\}_{i=1}^{r}$. This results in the $(r+1)$-dimensional matrices:

\begin{equation}
	% \widetilde{\mE}_{\rvy\mid\rvx} :=
	% \begin{bsmallmatrix}
	% 	1 & \bm{0} \\
	% 	\bm{0} & \mC_{\rvx\rvy}
	% \end{bsmallmatrix},
	% \quad
	\mV_\rvx :=
	\begin{bsmallmatrix}
		1 & \bm{0} \\
		\bm{0} & \mC_{\rvx}
	\end{bsmallmatrix},
	\quad
	\mV_\rvy :=
	\begin{bsmallmatrix}
		1 & \bm{0} \\
		\bm{0} & \mC_{\rvy}
	\end{bsmallmatrix},
\end{equation}

where $\mC_{\rvx} = \Cov(\vbfx(\rvx), \vbfx(\rvx)) \in \R^{r\times r}$, $\mC_{\rvy} = \Cov(\vbfy(\rvy), \vbfy(\rvy)) \in \R^{r\times r}$ denote the matrix forms of the truncated covariance operators $\oC_{\rvx} : \Lpx \mapsto \Lpx$ and $\oC_{\rvy} : \Lpy \mapsto \Lpy$ (see \cref{def:cross_covariance_operator}), respectively. Then the orthonormality regularization of \cref{eq:contrastive_loss} becomes:

\begin{equation}
	\small
	\| \mV_\rvx - \mI \|_{\text{F}}^2 = \| \mC_\rvx - \mI_r \|_{\text{F}}^2 + 2\|\E_{\mux} \vbfx(\rvx)\|^2
	\quad
	\| \mV_\rvy - \mI \|_{\text{F}}^2  = \| \mC_\rvy - \mI_r \|_{\text{F}}^2 + 2\|\E_{\muy} \vbfy(\rvy)\|^2.
\end{equation}

Since $\|\mC_{\rvx}\|_{\text{F}}^2 = \text{tr}(\mC_{\rvx\rvy}^2)$ involves products of covariance matrices, we compute its empirical value using unbiased estimators. For generality, we present the unbiased estimator for the cross-covariance.

\paragraph{Unbiased Estimation of Frobenious Norm of Cross-Covariance Operators}
Since $\|\mC_{\rvx\rvy}\|_{F}^2 = \text{tr}(\mC_{\rvx\rvy}^2)$ involves products of covariance matrices, we obtain unbiased estimates from finite samples by computing the metric using two independent sampling sets from $\muxy$, by:
% \massi{There appear to be some errors in the sampling notation in (29) and paragraph below.}

\begin{equation}
	\label{eq:unbiased_estimation_Cxy}
	\small
	\begin{aligned}
		\|\mC_{\rvx\rvy}\|_{\text{F }}^2 & = \text{tr}(\mC_{\rvx\rvy}^2)
		= \sum_{i=1}^r [\mC_{\rvx\rvy}^2]_{i,i}
		= \sum_{i=1}^r \sum_{j=1}^r [\mC_{\rvx\rvy}]_{i,j} [\mC_{\rvx\rvy}]_{j,i}
		\\
		                                 & = \sum_{i=1}^r \sum_{j=1}^r
		\E_{(\rvx,\rvy)\sim P_{\rvx\rvy}} [\bfx_{c,i}(\rvx) \bfy_{c,j}(\rvy)] \E_{(\rvx',\rvy') \sim P_{\rvx\rvy}} [ \bfx_{c,j}(\rvx')\bfy_{c,i}(\rvy')]
		\\
		                                 & = \E_{(\rvx,\rvy,\rvx',\rvy') \sim P_{\rvx\rvy}}
		[
			\sum_{i=1}^r
			\bfx_{c,i}(\rvx) \bfy_{c,i}(\rvy')  \sum_{j=1}^r \bfx_{c,j}(\rvx') \bfy_{c,j}(\rvy)
		]
		\\
		                                 & = \E_{(\rvx,\rvy,\rvx',\rvy') \sim P_{\rvx\rvy}}
		[
			(\vbfx_c(\rvx)^\top \vbfy_c(\rvy'))(\vbfx_c(\rvx')^\top \vbfy_c(\rvy))
		]
		\\
		                                 & \approx \frac{1}{N^2} \sum_{n=1}^N \sum_{m=1}^N (\vbfx_c(\vx_n)^\top \vbfy_c(\vy'_m))(\vbfx_c(\vx'_m)^\top \vbfy_c(\vy_n)),
	\end{aligned}
\end{equation}

where
$\vbfx_c(\rvx)=\vbfx(\rvx)-\E_{\mux}\vbfx(\rvx)$
denotes the centered basis functions, and
$((\vx,\vy),(\vx',\vy'))\sim \muxy$
indicates two independent sampling sets from $P_{\rvx\rvy}$ used for the unbiased estimation of $\|\mC_{\rvx}\|_{F}^2$. The final equation then provides the unbiased empirical estimator computed on a dataset
$\sD=\{(\vx_n,\vy_n)\sim P_{\rvx\rvy}\}_{n=1}^N$
and any random permutation of it, denoted as
$\sD'=\{(\vx'_n,\vy'_n)\sim P_{\rvx\rvy}\}_{n=1}^N$.

\subsection{Unbiased Estimation of Orthonormal Regularization}

The regularization term for optimizing the loss \eqref{eq:contrastive_loss} involves encouraging the basis sets to be orthonormal. The metric quantifying the orthogonality of the basis sets is defined by:
\begin{equation}
	\label{sec:regularization_term}
	\small
	\begin{aligned}
		\| \mV_\rvx - \mI \|_{\text{F}}^2
		 & =
		\| \mC_\rvx - \mI_r \|_{\text{F}}^2 + 2\|\E_{\mux} \vbfx(\rvx)\|^2
		=
		\text{tr}(\mC_\rvx^2) - 2\text{tr}(\mC_\rvx) + r + 2\|\E_{\mux} \vbfx(\rvx)\|^2,
		\\
		\| \mV_\rvy - \mI \|_{\text{F}}^2
		 & =
		\| \mC_\rvy - \mI_r \|_{\text{F}}^2 + 2\|\E_{\muy} \vbfy(\rvy)\|^2
		=
		\text{tr}(\mC_\rvy^2) - 2\text{tr}(\mC_\rvy) + r + 2\|\E_{\muy} \vbfy(\rvy)\|^2.
	\end{aligned}
\end{equation}
% 
% To estimate the terms $\|\mC_\rvx \|_\text{F}^2 = \text{tr}(\mC_\rvx^2)$ and $\|\mC_\rvy \|_\text{F}^2 = \text{tr}(\mC_\rvy^2)$ from finite samples we rely on estimators of the form of \cref{eq:unbiased_estimation_Cxy}. involve the product of covariance matrices, we rely on of from finite samples we need to appropriatedly handle the estimation of the product of the theoretical covariances by computing the metric from two distinct sampling sets from the marginal distribution $\mux$ and $\muy$, respectively. This is achieved for the $\rvx$ case by:
% % 
% \begin{equation}
% 	\label{eq:unbiased_estimation_Cx_F}
% 	\small
% 	\begin{aligned}
% 		\| \mC_{\rvx} \|_{\text{F}}^2
% 		 & = \text{tr}(\mC_{\rvx}^2)
% 		= \sum_{i=1}^r [\mC_{\rvx}^2]_{i,i}
% 		= \sum_{i=1}^r \sum_{j=1}^r [\mC_{\rvx}]_{i,j} [\mC_{\rvx}]_{j,i}
% 		\\
% 		 & = \sum_{i=1}^r \sum_{j=1}^r
% 		\E_{\rvx\sim \mux} [\bfx_{c,i}(\rvx) \bfx_{c,j}(\rvx)]
% 		\E_{\rvx'\sim \mux} [\bfx_{c,j}(\rvx') \bfx_{c,i}(\rvx')]
% 		= \E_{(\rvx,\rvx') \sim \mux} [(\vbfx_c(\rvx)^\top \vbfx_c(\rvx'))^2]
% 	\end{aligned}
% \end{equation}
% % 
% Where $\vbfx_c(\rvx) = \vbfx(\rvx) - \E_{\mux} \vbfx(\rvx)$ denotes the centered basis functions, and $(\rvx,\rvx') \sim \mux$ denotes two independent sampling sets from the marginal $\mux$ required for the unbiased estimation of $\| \mC_{\rvx} \|_{\text{F}}^2$. 
Hence given a dataset of samples $\sD = \{(\vx_n, \vy_n) \sim P_{\rvx\rvy} \}_{n=1}^N$, and any random permutation of the dataset order $\sD' = \{(\vx'_n, \vy'_n) \sim P_{\rvx\rvy} \}_{n=1}^N$  we can derive unbiased empirical estimates of \eqref{sec:regularization_term} as:
\begin{equation}
	\label{eq:unbiased_estimation_Cx_F_empirical}
	\small
	\begin{aligned}
		\| \mV_\rvx - \mI \|_{\text{F}}^2
		 & \approx
		\widehat{\E}_{(\rvx,\rvx') \sim \mux}
		[(\vbfx_c(\rvx)^\top \vbfx_c(\rvx'))^2]
		- 2 \widehat{\E}_{\mux} [\vbfx_c(\rvx)^\top \vbfx_c(\rvx)]
		+ r
		+ 2\|\widehat{E}_{\mux} \vbfx(\rvx)\|^2
		\\
		 & \approx
		\frac{1}{N^2} \sum_{n=1}^N \sum_{m=1}^N (\vbfx_c(\vx_n)^\top \vbfx_c(\vx'_m))^2
		- 2 \frac{1}{N} \sum_{n=1}^N \vbfx_c(\vx_n)^\top \vbfx_c(\vx_n)
		+ r
		+ 2 \| \frac{1}{N} \sum_{n=1}^N \vbfx(\vx_n) \|^2,
		\\
		\| \mV_\rvy - \mI \|_{\text{F}}^2
		 & \approx
		\widehat{\E}_{(\rvy,\rvy') \sim \muy}
		[(\vbfy_c(\rvy)^\top \vbfy_c(\rvy'))^2]
		- 2 \widehat{\E}_{\muy} [\vbfy_c(\rvy)^\top \vbfy_c(\rvy)]
		+ r
		+ 2\|\widehat{E}_{\muy} \vbfx(\rvy)\|^2
		\\
		 & \approx
		\frac{1}{N^2} \sum_{n=1}^N \sum_{m=1}^N (\vbfy_c(\vy_n)^\top \vbfy_c(\vy'_m))^2
		- 2 \frac{1}{N} \sum_{n=1}^N \vbfy_c(\vy_n)^\top \vbfy_c(\vy_n)
		+ r
		+ 2 \| \frac{1}{N} \sum_{n=1}^N \vbfy(\vy_n) \|^2.
	\end{aligned}
\end{equation}
%%%%%%%%%%%%%%%%%%%%%%%%%%%%%%%%%%%%%%%%%%%%%%%%%%%%%%%%%%%%%%%%%%%%
%%%%%%%%%%%%%%%%%%%%%%%%%%%%%%%%%%%%%%%%%%%%%%%%%%%%%%%%%%%%%%%%%%%%
\subsection{Orthonormal Regularization of Symmetric Hilbert Spaces} \label{app:symmetric_orthonormality}
Since the covariance operators $\oC_\rvx: \LpxFull \mapsto \LpxFull$ and $\oC_\rvy: \LpyFull \mapsto \LpyFull$ are $\G$-equivariant (see \cref{prop:G_equivariant_cross_covariance_operator}), their matrix representations in the isotypic basis are constrained to be block-diagonal, with each block being composed of irrep endomorphisms (\cref{def:equivLinearMapsLong}). Hence \eqref{sec:regularization_term} becomes:

\begin{equation*}
	\label{sec:regularization_term_iso}
	\small
	\begin{aligned}
		\| \mV_\rvx - \mI \|_{\text{F}}^2 & = \| \mC_\rvx - \mI_r \|_{\text{F}}^2 + 2\|\E_{\mux} \vbfx(\rvx)\|^2                                                                                                                                                                                                                                                        \\
		                                  & = \| \Oplus_{\isoIdx=1}^{\isoNum} \mC_\rvx^{\isoIdxBrace} - \mI_r \|_{\text{F}}^2 + 2\|\E_{\mux} \vbfx^{\inv}(\rvx)\|^2,
		\\
		                                  & = \sum_{\isoIdx=1}^{\isoNum} \| \mC_\rvx^{\isoIdxBrace} - \mI_r^{\isoIdxBrace} \|_{\text{F}}^2 + 2\|\E_{\mux} \vbfx^{\inv}(\rvx)\|^2
		\\
		                                  & = \sum_{\isoIdx=1}^{\isoNum} \left(\|\mC_\rvx^\isoIdxBrace\|_{F}^2 - 2\text{tr}(\mC_\rvx^\isoIdxBrace) + r_\isoIdx \right) + 2\|\E_{\mux} \vbfx(\rvx)\|^2
		\\
		                                  & = \sum_{\isoIdx=1}^{\isoNum} \left(\| \sum_{b \in \sB} \Theta_b^{\isoIdxBrace} \otimes \Psi_\isoIdx(b) \|_{F}^2 - 2\text{tr}(\sum_{b \in \sB} \Theta_b^{\isoIdxBrace} \otimes \Psi_\isoIdx(b)) + r_\isoIdx \right) + 2\|\E_{\mux} \vbfx(\rvx)\|^2, \quad \text{by \eqref{eq:interwiner_endomorphism_blocks_tensor_product}} \\
		                                  & =
		2\|\E_{\mux} \vbfx(\rvx)\|^2 +
		\sum_{\isoIdx=1}^{\isoNum} \| \sum_{b \in \sB} \Theta_b^{\isoIdxBrace} \otimes \Psi_\isoIdx(b) \|_{F}^2
		-
		2\sum_{b \in \sB} \text{tr}(\Theta_b^{\isoIdxBrace})\text{tr}(\Psi_\isoIdx(b)) + r_\isoIdx
		\\
		                                  & =
		2\|\E_{\mux} \vbfx(\rvx)\|^2 +
		\sum_{\isoIdx=1}^{\isoNum} \| \sum_{b \in \sB} \Theta_b^{\isoIdxBrace} \otimes \Psi_\isoIdx(b) \|_{F}^2
		-
		2 \irrepDim[\isoIdx] \sum_{b \in \sB} \text{tr}(\Theta_b^{\isoIdxBrace}) + r_\isoIdx
		\qquad
		\text{by \eqref{eq:group_representation_endomorphism_basis}}
		\\
		                                  & =
		2\|\E_{\mux} \vbfx(\rvx)\|^2 +
		\sum_{\isoIdx=1}^{\isoNum}
		\sum_{b \in \sB} \| \Theta_b^{\isoIdxBrace}\|_{F}^2  \|\Psi_\isoIdx(b)\|_{F}^2
		-
		2 \irrepDim[\isoIdx] \sum_{b \in \sB} \text{tr}(\Theta_b^{\isoIdxBrace}) + r_\isoIdx
		\qquad
		\text{by \eqref{eq:group_representation_endomorphism_basis}}
		\\
		                                  & =
		2\|\E_{\mux} \vbfx(\rvx)\|^2 +
		\sum_{\isoIdx=1}^{\isoNum} \irrepDim[\isoIdx]
		\sum_{b \in \sB} \| \Theta_b^{\isoIdxBrace}\|_{F}^2
		-
		2 \irrepDim[\isoIdx] \sum_{b \in \sB} \text{tr}(\Theta_b^{\isoIdxBrace}) + r_\isoIdx,
		\qquad
		\text{s.t }
		\Psi_\isoIdx(b)\Psi_\isoIdx(b)^\top = \Identity_{\irrepDim[\isoIdx]}
		\\
		                                  & =
		2\|\E_{\mux} \vbfx(\rvx)\|^2
		+
		\irrepDim[\isoIdx]
		\sum_{\isoIdx=1}^{\isoNum}
		\sum_{b \in \sB} \left( \| \Theta_b^{\isoIdxBrace}\|_{F}^2
		-
		2 \text{tr}(\Theta_b^{\isoIdxBrace})\right) + r_\isoIdx,
	\end{aligned}
\end{equation*}

Where the set of matrices $\{\Theta_b^{\isoIdxBrace} \in \R^{m_\isoIdx \times m_\isoIdx}\}_{b \in \sB}$ define the free degrees of freedom of the covariance operator (see \cref{def:equivLinearMapsLong}), and $\{\Psi_\isoIdx(b)\}_{b \in \sB}$ denote the basis of endomorphisms of the irreducible representation $\irrep[\isoIdx]$ (see \eqref{eq:group_representation_endomorphism_basis}). Crucially, $\| \Theta_b^{\isoIdxBrace}\|_{F}^2$ features an unbiased U-statistic estimator as in \eqref{eq:unbiased_estimation_Cxy}.

%% file: appendix/experimental_setup.tex
\section{Experimental Setup}
\label{app:experimental_setup}

In this section we provide details on the experimental setup. We first describe general design choices and hyperparameters and then provide details for each experiment.

\paragraph{Experiments Overview}
\begin{itemize}[left=1em]
	\item \cref{app:exp_results_conditional_expectation_operator}: conditional expectation operator approximation on symmetric \glspl{cgmm}.
	\item \cref{app:exp_results_regression}: $\G$-equivariant regression of robot \gls{com} momenta from noisy proprioception.
	\item \cref{app:exp_results_uncertainty_quantification}: synthetic conditional quantile regression for symmetry-aware uncertainty quantification.
	\item \cref{app:exp_results_regression_synthetic}: one-dimensional synthetic regression where the conditional mean is insufficient for uncertainty quantification.
	\item \cref{app:exp_results_symmetry_misspecification}: robustness of \gls{encp} under extrinsic and incorrect symmetry misspecification.
	\item \cref{app:exp_results_uncertainty_quantification_robot}: uncertainty quantification for quadruped locomotion observables.
	\item \cref{sec:train_inference_compt_cost}: training and inference costs of \gls{encp} relative to \gls{ncp}.
\end{itemize}

\paragraph{Sample Efficiency Experiments}
For both the conditional expectation operator approximation and the $\G$-equivariant regression experiments, we evaluate model performance by measuring sample efficiency/complexity. To do so, we partition the dataset
$
	\sD = \{(\vx_n, \vy_n)\}_{n=1}^{N}
$
into training, validation, and testing splits in proportions of 70\%, 15\%, and 15\%, respectively. With fixed validation and testing sets, we iteratively train the models on increasing portions of the training set and report the test performance for each size, for $3$ to $5$ different random seeds.

For each training set size, we select the model checkpoint with the best validation loss to compute the test performance. Thus, these experiments quantify the generalization error (or true risk) and its evolution as a function of the training set size.

\paragraph{\glspl{nn} Architectures and Hyperparameters}
To compare our equivariant representation learning framework with other contrastive and supervised methods, all (inference) models share a similar fixed architectural footprint. For the baseline models, the only hyperparameter tuned is the learning rate, whereas for the \gls{ncp} and \gls{encp} models we additionally tune the regularization weight $\gamma$ in \cref{eq:contrastive_loss,eq:disentangled_contrastive_loss}. Further details for each experiment are provided in the corresponding sections below.

\paragraph{Code Reproducibility}
All experiments, plots and examples are provided in the open-access repository and python package 
\href{https://github.com/Danfoa/symm_rep_learn}{\textit{github.com/Danfoa/symm\_rep\_learn}}.

\subsection{Conditional Expectation Operator Approximation}
\label{app:exp_results_conditional_expectation_operator}

\begin{figure}[!htbp]
	\includegraphics[width=\textwidth]{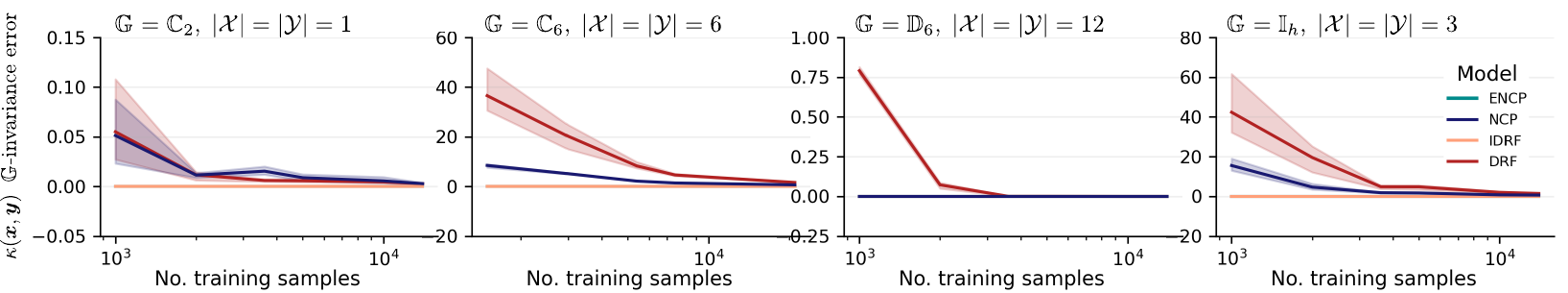}
	\caption{
		Sample-efficiency plots comparing the test-set \gls{pmd} $\G$-invariance constraint \eqref{eq:learning_pmd_symm} violation
		$
			\pmd_{\text{inv-err}} := \E_\rvx \E_\rvy \sum_{g \in \G} (\pmd_{\nnParams}(\rvx, \rvy) - \pmd_{\nnParams}(\g \Glact[\vsX] \rvx, \g \Glact[\vsY] \rvy))^2
		$
		versus the number of training samples, in log scales.
		Each column corresponds to a symmetric \gls{cgmm} with distinct symmetry groups and $(\rvx,\rvy)$ dimensionality. The tested groups are the cyclic groups $\CyclicGroup[2]$ and $\CyclicGroup[6]$, the Dihedral group $\DihedralGroup[6]$ (order 12), and the Icosahedral group $\sI_h$ (order 60).
	}
	\label{fig:G_invariance_kappa}
\end{figure}

In this experiment, we extend the conditional Gaussian Mixture Model (GMM) proposed by \citet{gilardi2002conditional} to parametrically construct symmetric random variables taking values in arbitrary data spaces $\vsX$ and $\vsY$ and with arbitrary finite symmetry groups $\G$. The GMM is defined by
\begin{equation*}
	\label{eq:symmGMM}
	\small
	\rvz := \rvx \oplus \rvy \sim \sum_{\g \in \G} \sum_{c=1}^{n_g} \mathcal{N}(
	\rep[\vsZ](g)\mu_{\rvz,c}
	\;,\;
	\rep[\vsZ](g)\Sigma_{\rvz,c}\rep[\vsZ](g)^\top
	),
\end{equation*}
where $\rep[\vsZ](g) := \rep[\vsX](g) \oplus \rep[\vsY](g)$ are arbitrary group representations of $\G$ and $n_g$ is the number of unique Gaussians with randomly sampled means $\mu_{\rvz} := \mu_{\rvx} \oplus \mu_{\rvy}$ and block-diagonal covariances $\Sigma_{\rvz} := \Sigma_{\rvx} \oplus \Sigma_{\rvy}$. Since every Gaussian appears in group orbits, this symmetric GMM has $\G$-invariant marginal distributions and an analytical expression for the conditional expectation operator kernel $\pmd(\vx, \vy) = \sfrac{p_{\rvx\rvy}(\vx,\vy)}{\mux(\vx)\muy(\vy)}$ (see 2D example in \cref{fig:symmetry_priors_long}). Consequently, we can directly estimate the approximation of the conditional expectation operator (\cref{eq:contrastive_loss}) as the mean squared error between the true and learned density ratios, i.e., $\pmd_{\text{mse}} := \E_\rvx \E_\rvy \left\| \pmd(\rvx, \rvy) - \pmd_{\nnParams}(\rvx, \rvy) \right\|^2$.

To the best of our knowledge, this is the first synthetic experiment that directly estimates the truncation error of the conditional expectation operator in an inference task-agnostic setting, serving as a benchmark for future work.

\cref{fig:ablation_results} compares sample efficiency using $\pmd_{\text{mse}}$, while \cref{fig:G_invariance_kappa} shows the error in the $\G$-invariant of the learned $\pmd$ ratio versus sample size, highlighting that symmetry-aware methods encode this property as an architectural constraint, ensuring a strictly $\G$-invariant learned ratio.

%%%%%%%%%%%%%%%%%%%%%%%%%%%%%%%%%%%%%%%%%%%%%%%%%%%%%%%%%%%%%%%%%%%%%%%%%%%%%%%
%%%%%%%%%%%%%%%%%%%%%%%%%%%%%%%%%%%%%%%%%%%%%%%%%%%%%%%%%%%%%%%%%%%%%%%%%%%%%%%
%%%%%%%%%%%%%%%%%%%%%%%%%%%%%%%%%%%%%%%%%%%%%%%%%%%%%%%%%%%%%%%%%%%%%%%%%%%%%%%
\subsection[G-Equivariant Regression of Robot's CoM Momenta]{$\G$-Equivariant Regression of Robot's \gls{com} Momenta}
\label{app:exp_results_regression}

In this experiment, we evaluate the quality of the learned representations using the contrastive loss \cref{eq:contrastive_loss,eq:disentangled_contrastive_loss} alongside supervised learning baselines trained with the standard \gls{mse} loss. The task is a $\G$-equivariant benchmark in robotics presented in \citep{ordonez2024morphological}, with the goal of predicting a quadruped robot’s \gls{com} linear $\vl \in \R^3$ and angular momenta $\vk \in \R^3$ from noisy observations of the robot's generalized positions $\vq \in \R^{12}$ and velocity coordinates $\dot{\vq} \in \R^{12}$. Consequently, the random variables are defined as $\vx = \vq + \epsilon_{\vq} \oplus \dot{\vq} + \epsilon_{\dot{\vq}}$ and $\vy = \vl \oplus \vk$, where $\epsilon_{\vq} \in \R^{12}$ and $\epsilon_{\dot{\vq}} \in \R^{12}$ are independent Gaussian noise terms that model sensor noise. The function computing the \gls{com} momenta from these proprioceptive observations is highly non-linear and $\G$-equivariant whenever $\G$ is a morphological symmetry group of the robot (see \cref{fig:example_group_action} and \citet{ordonez2024morphological} for details).

The robot considered is the quadruped robot Solo (\cref{fig:example_group_action}-right), which possesses a symmetry group of order 8: $\G = \KleinFourGroup \times \CyclicGroup[2]$, as depicted in this \href{https://danfoa.github.io/MorphoSymm/static/animations/solo-C2xC2xC2-symmetries_anim_static.gif}{animation} showing 8 symmetric robot configurations along with their corresponding linear and angular momenta vectors.

\begin{figure}[H]
	\includegraphics[width=\textwidth]{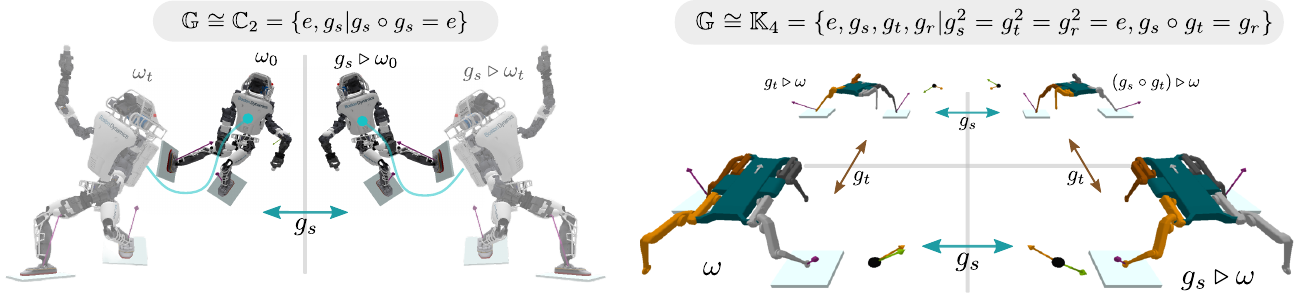}
	\caption{
		Example of morphological finite symmetry in robotics. \textbf{Left}: A humanoid robot with the reflectional symmetry group $\G \equiv \CyclicGroup[2]$. \textbf{Right}: The quadruped robot Solo with the symmetry group $\G = \KleinFourGroup \times \CyclicGroup[2]$ (only $\KleinFourGroup$ is shown for clarity). The robot's center of mass linear $\vl \in \R^3$ and angular $\vk \in \R^3$ momentum are depicted as orange and green vectors, respectively, for each symmetric configuration. Images adapted from \citet{ordonez2024morphological} with author approval.
	}
	\label{fig:example_group_action}
\end{figure}

\paragraph{NN Architectures}
We configure all models under consideration (\gls{encp}, \gls{ncp}, \gls{emlp}, and \gls{mlp}) to have an inference-time \gls{nn} architecture with a similar footprint. In particular, the encoder network for $\rvx$ in \gls{ncp} and \gls{encp} is designed similarly to the \gls{nn} used in \gls{mlp}/\gls{emlp}. The idea is to test how a model with the same capacity performs on the downstream task of regression when trained using either the representation learning loss or a supervised learning loss. The backbone of all architectures is a standard multilayer perceptron consisting of three hidden layers, each with 512 units, followed by a final hidden layer containing 128 units. This final layer encodes the feature vector $r$ for the \gls{ncp} and \gls{encp} models. Crucially, since $\G$-equivariance enforces weight sharing in the \gls{nn} architecture, the encoder \gls{nn} for \gls{encp} and \gls{emlp} comprises $\times 8$ fewer parameters than their symmetry-agnostic counterparts.

%%%%%%%%%%%%%%%%%%%%%%%%%%%%%%%%%%%%%%%%%%%%%%%%%%%%%%%%%%%%%%%%%%%%%%%%%%%%%%%
%%%%%%%%%%%%%%%%%%%%%%%%%%%%%%%%%%%%%%%%%%%%%%%%%%%%%%%%%%%%%%%%%%%%%%%%%%%%%%%
%%%%%%%%%%%%%%%%%%%%%%%%%%%%%%%%%%%%%%%%%%%%%%%%%%%%%%%%%%%%%%%%%%%%%%%%%%%%%%%
\subsection{Uncertainty Quantification via Conditional Quantile Regression}
\label{app:exp_results_uncertainty_quantification}

\begin{figure}[!htbp]
	\includegraphics[width=\textwidth]{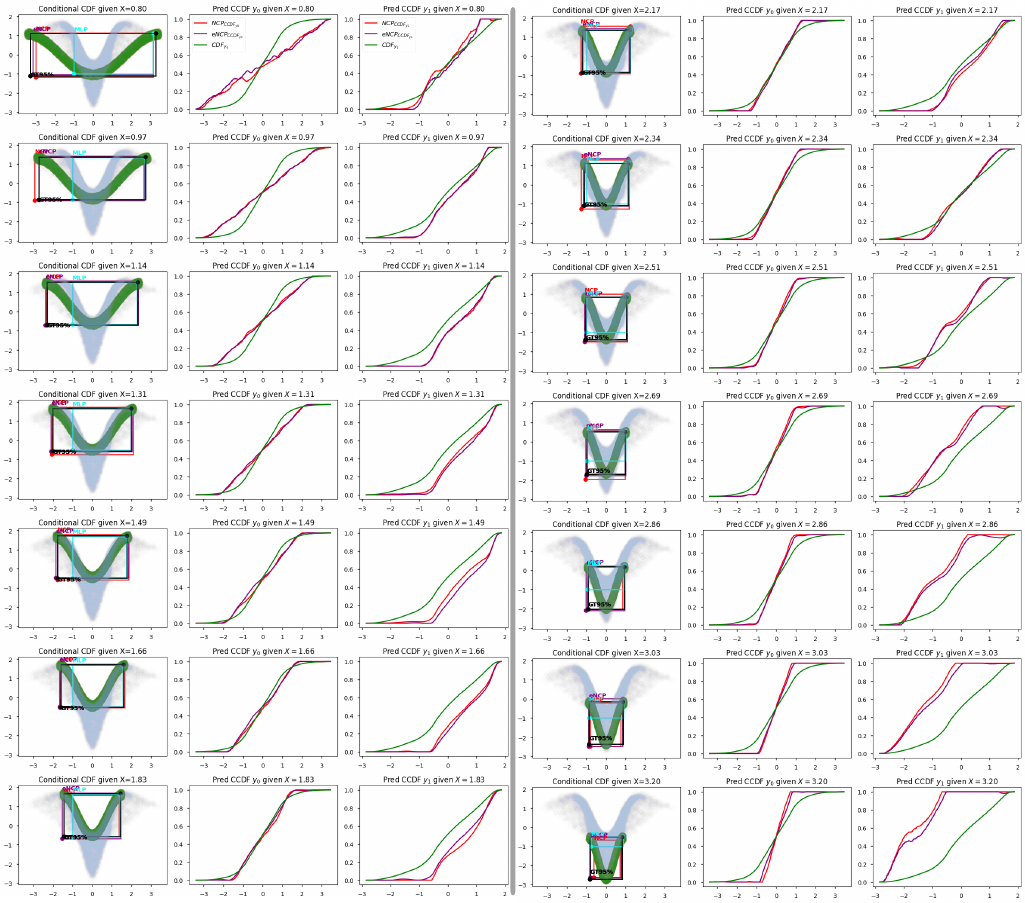}
	% \vspace*{-0.7cm}
	\caption{
		Results of a synthetic uncertainty quantification experiment comparing \gls{cqr}, \gls{ncp}, and \gls{encp} models for predicting $95\%$ \glspl{ci} of $\rvy \in \R^2$ given $\rvx \in \R$. Left and fourth columns show conditional distributions $\PP(\rvy \vert \rvx = \cdot)$ for different conditioning values. Second-third and fifth-sixth columns display \gls{ccdf} predictions by \gls{encp} and \gls{ncp} models. While \gls{cqr} directly regresses quantiles and requires retraining for different confidence levels, \gls{ncp} and \gls{encp} estimate the full \gls{ccdf}, enabling adaptation to any confidence interval without retraining.
	}
	% \vspace{-0.3cm}
	\label{fig:synthetic_uc_results}
\end{figure}

The goal of these experiments  benchmark is to learn the family of conditional distributions $\PP(\rvy\mid\rvx=\cdot)$ for a bivariate random variable $\rvy=[\ry_0,\ry_1]\in\R^2$ given a scalar covariate $\rvx\in\R$.  Once $\PP(\rvy\mid\rvx)$ is recovered, the practitioner can estimate \emph{conditional} $(1-\alpha)$--confidence regions by regressing the lower and upper conditional quantiles
\(
q_{\alpha/2}(\rvx),\,q_{1-\alpha/2}(\rvx)
\)
for any desired miscoverage level $\alpha\in(0,1)$.  In particular, a $95\%$ confidence region corresponds to $\alpha=0.05$, so the two quantiles of interest are $q_{0.025}(\rvx)$ and $q_{0.975}(\rvx)$. See \cref{fig:synthetic_uc_depiction} for a visual representation of the problem.

\paragraph{Conditional Quantile Regression Models}
\begin{figure}[!htbp]
	\centering
	\includegraphics[width=\textwidth]{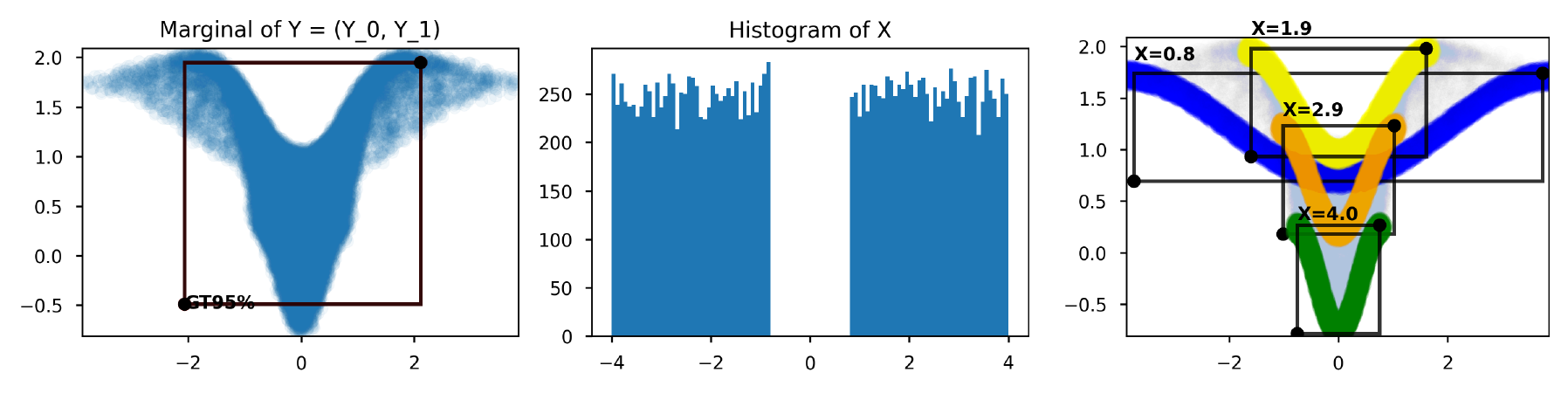}
	\caption{
		Uncertainty quantification experiment adapted from \citet{feldman2023calibrated}. Task: predict $95\%$ \glspl{ci} (black boxes) of $\rvy \in \R^2$ given $\rvx \in \R$.
		\textbf{Left:} Marginal $\PP(\rvy)$.
		\textbf{Middle:} Marginal $\PP(\rvx)$.
		\textbf{Right:} Conditional distributions $\PP(\rvy \vert \rvx = \cdot)$ for different values.
	}
	\label{fig:synthetic_uc_depiction}
\end{figure}
We compare the \gls{ncp} and proposed \gls{encp} models to a standard baseline for parametric \gls{nn} conditional quantile regression, namely \gls{cqr} \citet{feldman2023calibrated}, which uses two separate \glspl{nn} to predict lower and upper quantiles, trained with pinball loss. Both models use \gls{mlp} backbones with similar parameter counts, ensuring improvements are solely due to loss functions.

Furthermore, \gls{cqr} can only be trained for specific confidence intervals, requiring retraining for different quantiles. In contrast, the \gls{ncp} and \gls{encp} models, trained using the deep representation learning approach of \cref{sec:background,sec:learning_cond_exp_operator}, regress the \gls{ccdf} of each dimension of $\rvy$ given $\rvx$. Thus, they can estimate conditional quantiles for any confidence interval via the quantile estimation algorithm from the \gls{ccdf} described in \citet{kostic2024learning} without retraining. See details in \cref{fig:ccdf_regression}.

\begin{figure}[!htbp]
	\centering
	\includegraphics[width=\textwidth]{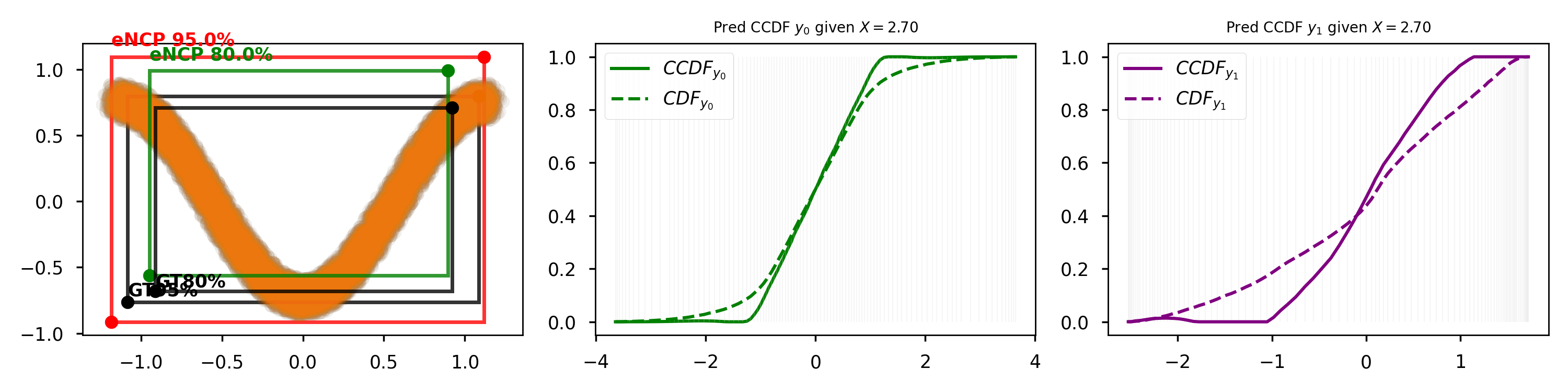}
	\caption{
	Prediction of the $80\%$ and $95\%$ \glspl{ci} for the random variable $\rvy$ in experiment \cref{app:exp_results_uncertainty_quantification} using the proposed \gls{encp} model. The model estimates the \gls{ccdf} by discretizing each dimension of $\rvy = [\ry_1, \ry_2]$ into $100$ bins and computing the conditional probabilities $\PP(\ry_i \in \sA_{n} \vert \rvx = \cdot) := [\CE \one_{\sA_{n}}](\cdot)$ for all $n \in [100]$ based on the learned conditional expectation operator $\pmd_{\nnParams}(\rvx, \rvy)$ (see \cref{sec:learning_guarantees}). Here, $\sA_{n}$ comprises the bins from the $0$-th to the $n$-th. This yields the estimated \gls{ccdf} for $\ry_1$ (center) and $\ry_2$ (right) at $\rvx = 2.7$. The \glspl{ccdf} can then be used to estimate upper and lower quantiles for any confidence interval \citep{kostic2024learning}.
	In practice, the \gls{encp} model regresses $2\times 100$ variables in a single forward pass. Thus, the final layer of the conditional quantile regression model is a linear layer of size $r \times (2 \times 100)$, where $r$ is the number of features in the $\rvy$ representation (see \cref{sec:background}). Note that the gray vertical lines in the \gls{ccdf} plots indicate the discretization bins, estimated using a `quantile\_transformer`, from \texttt{sklearn} \citep{pedregosa2011scikit}, fitted on the marginal distribution of each dimension of $\rvy$. This discretization procedure ensures high-resolution bins in high-density regions of $\rvy$ and coarser bins in low-density regions, while being able to cover the entire empirical support of $\rvy$ from the training set.
	}
	\label{fig:ccdf_regression}
\end{figure}

\paragraph{Evaluation Metrics: Coverage and Set Size}
\label{app:metrics_uc}

Let $\sC_{1-\alpha}(\rvx)\subseteq\R^d$ denote a \emph{prediction set} of nominal level $(1-\alpha)$ produced by a conditional quantile regression model for the response $\rvy\in\R^d$ given the covariate $\rvx\in\R^p$.  In all experiments we assess two complementary metrics.

\begin{itemize}
	\item \textbf{Coverage.}
	      The conditional \emph{coverage} of $\sC_{1-\alpha}$ is the probability that the true response is captured by the predicted region,
	      \begin{equation}
		      \label{eq:coverage}
			  \small
		      \mathrm{c}_{1-\alpha}(\rvx)
		      :=
		      \PP\bigl(\rvy\in\sC_{1-\alpha}(\rvx)\mid \rvx\bigr),
		      \qquad
		      \text{with the target }\;
		      \mathrm{c}_{1-\alpha}(\rvx)\approx 1-\alpha
		      \;\;\forall\,\rvx.
	      \end{equation}
	      In practice we report the \emph{marginal} coverage $\EE_{\rvx}[\mathrm{c}_{1-\alpha}(\rvx)]$, estimated on a large held-out sample; values above (resp.\ below) $1-\alpha$ indicate over- (resp.\ under-) coverage.

	\item \textbf{Relaxed Coverage (r-Coverage).}
	      The conditional \emph{relaxed coverage} of $\sC_{1-\alpha}$ is defined as the probability that each scalar component of the response lies within its corresponding predicted confidence interval. Formally, if $\rvy=[\ry_1,\ldots,\ry_d]$ and $\sC_{1-\alpha}(\rvx)$ has corresponding marginal intervals $\sC_{1-\alpha}^{(i)}(\rvx)$ for $i\in\{1,\ldots,d\}$, then
	      \begin{equation}
		      \label{eq:relaxed_coverage}
			  \small
		      \mathrm{rc}_{1-\alpha}(\rvx)
		      :=
		      \prod_{i=1}^d \PP\Bigl(\ry_i\in \sC_{1-\alpha}^{(i)}(\rvx) \,\Big|\, \rvx\Bigr),
	      \end{equation}
	      with the target $\mathrm{rc}_{1-\alpha}(\rvx)\approx 1-\alpha$ for all $\rvx$. As with coverage, we report the \emph{marginal} relaxed coverage $\EE_{\rvx}[\mathrm{rc}_{1-\alpha}(\rvx)]$.

	\item \textbf{Set size.}
	      To quantify how informative the region is, we measure its \emph{size} (volume) under the Lebesgue measure~$\lambda^d$:
	      \begin{equation}
		      \label{eq:setsize}
			  \small
		      \mathrm{Size}_{1-\alpha}(\rvx)
		      :=
		      \operatorname{vol}\bigl(\sC_{1-\alpha}(\rvx)\bigr).
	      \end{equation}
	      Smaller sets correspond to sharper uncertainty estimates, provided the required coverage is met. For multidimensional responses the volume is expressed in the natural units of $\R^d$; for $d=1$ it reduces to the interval length. As with coverage, we report the marginal expectation $\EE_{\rvx}[\mathrm{Size}_{1-\alpha}(\rvx)]$ so that models can be compared fairly across the entire input distribution.
\end{itemize}

\paragraph{Data Generation}
The data is generated from a conditional distribution
$
	\PP(\rvy\mid\rvx=\cdot)
$
for a bivariate random variable \(\rvy=[\ry_0,\ry_1]\in\R^2\) given a scalar covariate \(\rvx\in\R\).
Adapted from \citet{feldman2023calibrated}, the covariate is sampled uniformly:
$
	\rvx \sim \mathrm{Unif}\bigl( [0.8,\,4.0] \cup [-4.0, -0.8]\bigr),
$
and the response variable \(\rvy\) is produced by a non-linear transformation of auxiliary latent variables (see \cref{fig:synthetic_uc_depiction}):
\begin{equation*}\small
	\begin{aligned}
		\ry_0 & = \frac{\rz}{\beta\,|\rvx|}\;+\;\rr\cos\phi,                           \\
		\ry_1 & = \tfrac12\!\bigl(-\cos \rz + 1\bigr)\;+\; \rr\sin\phi\;+\;\sin|\rvx|,
	\end{aligned}
	\qquad
	\begin{aligned}
		\rz  & \sim \mathrm{Unif}(-\pi,\pi), \\
		\phi & \sim \mathrm{Unif}(0,2\pi),   \\
		\rr  & \sim \mathrm{Unif}(-0.1,0.1).
	\end{aligned}
\end{equation*}
Here, $\beta>0$ is a scaling constant. Both marginal distributions (see \cref{fig:synthetic_uc_depiction}) and the conditional probability distribution are invariant under the reflection symmetry group $\G=\CyclicGroup[2] = \{e, g_r \mid g_r \Gcomp g_r =e  \}$, acting on $\rvy$ as $g_r\Glact[\vsY]\rvy = [-\ry_0,\ry_1]$ and on $\rx$ as $g_r\Glact[\vsX]\rx = -\rx$.

\paragraph{Results}
The experimental results are shown in \cref{fig:synthetic_uc_results}, where the \gls{ncp} and \gls{encp} models outperform the baseline \gls{cqr} model in terms of both coverage and set size, for a diverse range of conditioning values. 

Moreover, \cref{fig:basis_functions} illustrates the basis functions learned by the \gls{ncp} and \gls{encp} models for the random variable $\rvy = [\ry_0, \ry_1]$. In contrast to the standard \gls{ncp} model, \gls{encp} incorporates symmetry priors, enabling a clean separation of its latent representation into two orthogonal subspaces: one corresponding to $\CyclicGroup[2]$-invariant functions and the other to functions that change sign under reflection. These correspond to the two isotypic subspaces of the representation space. Intuitively, by leveraging the orthogonality constraints between functions in these subspaces, \gls{encp} performs independent representation learning within each subspace, leading to improved performance and satisfaction of the $\G$-equivariance of the learned conditional expectation operator, as shown in \cref{fig:G_invariance_kappa}.

\begin{figure}[!htbp]
	\centering
	\includegraphics[width=.99\textwidth]{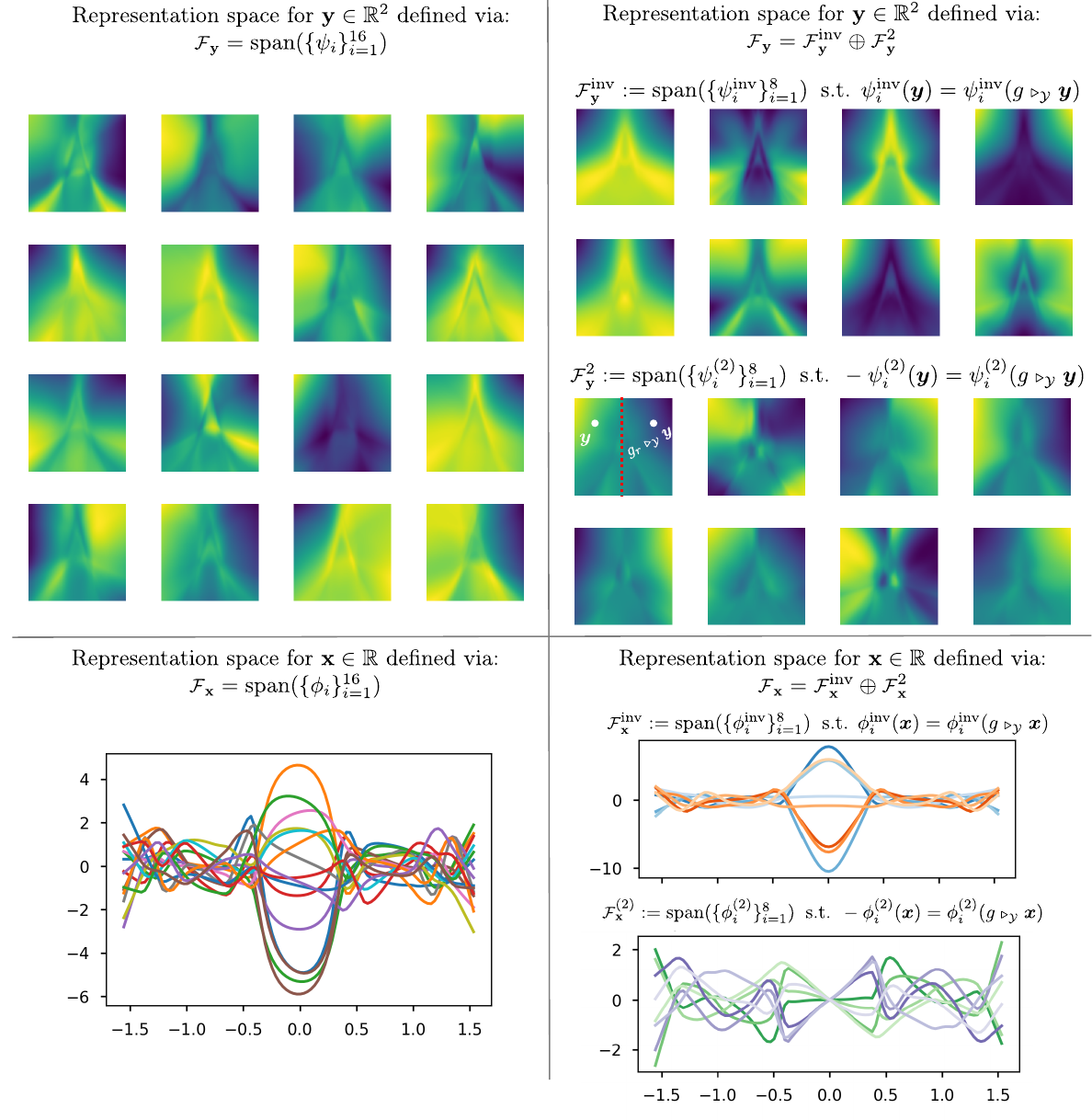}
	\caption{
	\textbf{Left:} Learned basis functions from the \gls{ncp} model for $\rvy=[\ry_0,\ry_1]$ and $\rvx$.
	\textbf{Right:} Learned basis functions from the \gls{encp} model for $\rvy$ and $\rvx$.
	The marginal distributions of $\rvy$ and $\rvx$ exhibit reflection symmetry (see \cref{fig:synthetic_uc_depiction}) defined by the symmetry group $\G=\CyclicGroup[2]\{e,\g_r \mid \g_r^2=e\}$, with group actions $\g_r\Glact[\vsY]\rvy = [-\ry_0,\ry_1]$ and $\g_r\Glact[\vsX]\rx = -\rx$. By incorporating this prior, the \gls{encp} model learns nonlinear representation spaces for $\rvx$ and $\rvy$ that decompose into two isotypic subspaces, $\Lpy = \LpyIso[\inv]\oplus\LpyIso[(2)]$ and $\Lpx = \LpxIso[\inv]\oplus\LpxIso[(2)]$. The invariant subspaces $(\inv)$ consist of $\CyclicGroup[2]$-invariant functions, while the complementary subspaces consist of functions that change sign under reflection.
	}
	\label{fig:basis_functions}
\end{figure}

%%%%%%%%%%%%%%%%%%%%%%%%%%%%%%%%%%%%%%%%%%%%%%%%%%%%%%%%%%%%%%%%%%%%%%%%%%%%%%%
%%%%%%%%%%%%%%%%%%%%%%%%%%%%%%%%%%%%%%%%%%%%%%%%%%%%%%%%%%%%%%%%%%%%%%%%%%%%%%%
%%%%%%%%%%%%%%%%%%%%%%%%%%%%%%%%%%%%%%%%%%%%%%%%%%%%%%%%%%%%%%%%%%%%%%%%%%%%%%%
\subsection[G-Equivariant Synthetic Regression with Uncertainty Quantification]{$\G$-Equivariant Synthetic Regression with Uncertainty Quantification}
\label{app:exp_results_regression_synthetic}

\begin{wrapfigure}{r}{0.25\textwidth}
	\vspace{-10pt}
	\centering
	\resizebox{\linewidth}{!}{%
		\begin{tikzpicture}[x=0.5cm, y=1.0cm]
			% Controllable vertical spacing parameter
			\def\rowspacing{1.1cm}

			% --- Row 1: Uniform distribution over small range ---
			% axes (no ticks)
			\draw[gray!50] (-3.2,0) -- (3.2,0);
			% uniform distribution from -1.5 to 1.5, height 0.9 (matching max of other rows)
			\draw[blue, thick] (-1.5,0) -- (-1.5,0.9) -- (1.5,0.9) -- (1.5,0);
			% expectation marker at center
			\draw[blue!60, dashed] (0,0) -- (0,0.9);
			% labels
			\node[anchor=west, scale=0.7] at (0.05,0.5) {$\E[\rvy \vert \vx]$};
			\node[anchor=west, scale=0.7] at (3.25,0) {$\rvy$};

			% --- Row 2: Bimodal mixture (low probability at mean) ---
			\begin{scope}[yshift=-\rowspacing]
				\draw[gray!50] (-3.2,0) -- (3.2,0);
				% mixture of two Gaussians (scaled to max height 0.9)
				\draw[blue, thick, domain=-3:3, samples=300]
				plot (\x, {0.54*exp(-((\x+1.4)^2)/(2*0.35^2)) + 0.36*exp(-((\x-1.6)^2)/(2*0.50^2))});
				% approximate mean (between modes but with low density)
				\draw[blue!60, dashed] (-0.2,0) -- (-0.2,0.8);
				\node[anchor=west, scale=0.7] at (-0.1,0.45) {$\E[\rvy\vert \vx]$};
				\node[anchor=west, scale=0.7] at (3.25,0) {$\rvy$};
			\end{scope}

			% --- Row 3: Power law distribution ---
			\begin{scope}[yshift=-2*\rowspacing]
				\draw[gray!50] (-3.2,0) -- (3.2,0);
				% power law starting from left edge, max height 0.9
				\draw[blue, thick, domain=-3.2:3.2, samples=200]
				plot (\x, {0.9*exp(-0.8*(\x+3.2))});
				% mean marker 
				\draw[blue!60, dashed] (-1.95,0) -- (-1.95,0.95);
				\node[anchor=west, scale=0.7] at (-1.9,0.8) {$\E[\rvy\vert \vx]$};
				\node[anchor=west, scale=0.7] at (3.25,0) {$\rvy$};
			\end{scope}

			% Legend - positioned in top right corner
			\begin{scope}
				\draw[blue, thick] (1.8,0.8) -- (2.3,0.8);
				\node[anchor=west, scale=0.7] at (2.4,0.8) {$\sP(\rvy \vert \vx)$};
			\end{scope}
		\end{tikzpicture}
	}
	\vspace{-15pt}
	\caption{
		\small
		Example uninformative expected values.
	}
	\label{fig:expectation-illus}
	\vspace{-4em}
\end{wrapfigure}
This synthetic experiment demonstrates how the \gls{ncp} and \gls{encp} frameworks can train \glspl{nn} for regression tasks under diverse noise conditions, particularly in applications where the conditional expectation, i.e., the regression function, may be uninformative due to skewed and asymmetric conditional distributions (see \cref{fig:expectation-illus}), and uncertainty quantification is therefore required. Crucially, in such settings, the representations learned by \gls{ncp} and \gls{encp} can be reused without retraining to predict conditional quantiles at arbitrary coverage levels (see \cref{eq:coverage}), which is not possible with standard supervised \gls{mse} regression training.

% \paragraph{Data generation.}
The dataset is composed of a $1$-dimensional samples $\rvx\in\R$ and $\rvy\in\R$ drawn from a piecewise conditional distribution with three distinct regimes (see \cref{fig:data_with_zones_and_true_expectation}):
\begin{itemize}
	\item Skewed distribution ($|x|\le 1$): This regime features a skewed conditional distribution with an exponential tail, with heteroscedastic noise.
	\item Symmetric distribution ($1<|x|\le 2$): This regime has a symmetric conditional distribution, with heteroscedastic noise.
	\item Bimodal distribution ($|x|>2$): This regime exhibits a bimodal conditional distribution, with heteroscedastic noise.
\end{itemize}

Formally the piecewise conditional distribution is defined as follows:
\begin{equation*}
	\PP(\rvy \mid \rvx) =
	\begin{cases}
		f_c(x) + \varepsilon_{\exp},    & |x|\le 1,   \\[0.2em]
		f_c(x) + \sigma_{h}(|x|) Z,     & 1<|x|\le 2, \\[0.2em]
		S\,a(|x|) + \sigma_{p}(|x|) Z', & |x|>2,
	\end{cases}
	,\quad f_c(x) = \tfrac{1}{2}\cos\!\bigl(\tfrac{2\pi}{3}x\bigr) + \tfrac{1}{5}\cos\!\bigl(\tfrac{8\pi}{3}x\bigr) + \tfrac{1}{4},
\end{equation*}
Where $f_c$ is a deterministic function  $\varepsilon_{\exp}\sim \mathrm{Exp}(\text{scale}=s_{\exp}(|x|))$, $Z,Z' \sim \mathcal{N}(0,1)$, $S\in\{-1,+1\}$ is equiprobable. Furthemore $s_{\exp}$, $\sigma_{p}$ and $\sigma_{h}$ introduce heteroscedasticity (variance is conditioned on values of $\rvx$). The function is design such that the conditional distribution is $\G$-invariant under the relfection group $\G=\CyclicGroup[2]$ acting as $g_r \Glact[\vsX] = -x$ and $g_r \Glact[\vsY] = y$. Consequently, the target $\G$-invariant function defined by the conditional expectation is given as (see \cref{fig:data_with_zones_and_true_expectation}):
\begin{equation*}
	z(x) = \E[\rvy\mid \rvx=x] =
	\begin{cases}
		f_c(x) + \E[\varepsilon_{\exp}], & |x|\le 1,   \\
		f_c(x),                          & 1<|x|\le 2, \\
		0,                               & |x|>2,
	\end{cases}
\end{equation*}

\begin{figure}[!htbp]
	\centering
	\includegraphics[width=.8\textwidth]{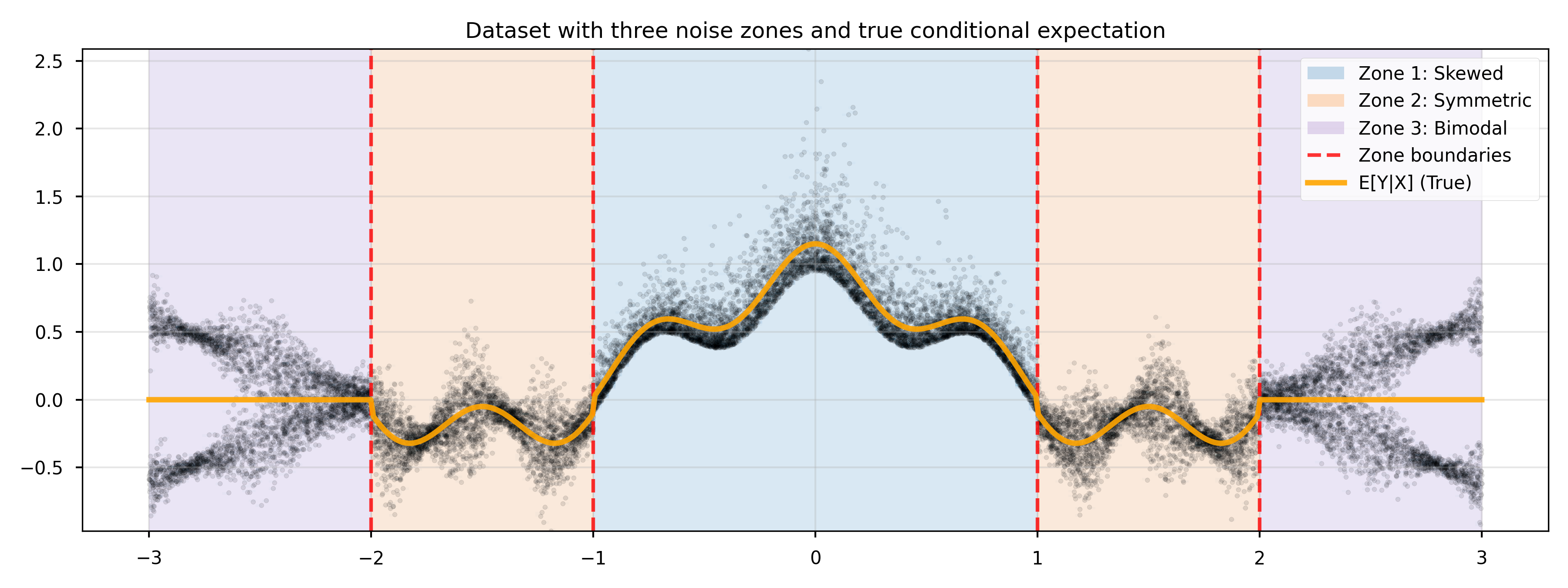}
	% \vspace*{-0.8cm}
	\caption{
		Synthetic regression experiment featuring three distinct zones with different families of conditional distributions for which the expected value is uninformative. Zone~1 ($|x|\le 1$) exhibits a exponential skewed distribution, Zone~2 ($1<|x|\le 2$) has symmetric, and Zone~3 ($|x|>2$) features a bimodal distributon. All zones are affected by heteroscedastic noise.
	}
	% \vspace*{-0.5cm}
	\label{fig:data_with_zones_and_true_expectation}
\end{figure}

All models in the comparison share the same standard three-layer \gls{mlp} backbone, with $64$ hidden neurons per layer and ELU activations. The first baseline is the \gls{mlp} architecture trained with the standard \gls{mse} regression loss. The \gls{ncp} and \gls{ncpaug} models train the same architecture using the contrastive loss in \cref{eq:contrastive_loss}, without and with data augmentation, respectively. Finally, the proposed \gls{encp} model trains the same architecture using the disentangled contrastive loss in \cref{eq:disentangled_contrastive_loss}, which incorporates symmetry priors and orthogonality regularization to enforce the $\G$-equivariance of the learned conditional expectation operator.

All \gls{ncp}-based models use orthogonality regularization with Lagrange multiplier $\lambda=10^{-2}$ and momentum $0.995$ to compute the orthonormality statistics through exponential moving-average estimates. All models are trained with early stopping on the validation objective, and the checkpoint with the best validation loss is used to compute test performance. The dataset is generated with $N=20{,}000$ sample pairs.

After training, for both contrastive models, we fit two linear decoders from the learned representations to predict the conditonal expectation $\hat{z}_{\nnParams}$, as described in \eqref{eq:regression_estimate}, and to recover the \gls{ccdf} with $200$ support points for quantile extraction, as described in \cref{fig:ccdf_regression}.

\paragraph{Results}
The results, shown in \cref{fig:data_with_zones_and_true_expectation}, illustrate that all models perform similarly in predicting the conditional expectation $\E[\rvy\mid \rvx]$. However, the key advantage of \gls{ncp} and \gls{encp} lies in their ability to model the conditional distribution, thereby enabling uncertainty quantification. Here, this is shown through the prediction of conditional lower and upper quantiles with coverage levels of $90\%$ and $70\%$. Recovering these quantiles is not possible when training the \gls{nn} backbone using standard supervised \gls{mse} regression.

While both \gls{ncp} and \gls{encp} perform comparatively well and predict well-calibrated confidence intervals, \gls{encp} achieves nearly perfect empirical coverage tracking for coverage levels ranging from $10\%$ to $90\%$, as shown in \cref{fig:synthetic_regression_results} (middle). Furthermore, \gls{encp} consistently predicts smaller \glspl{ci} (bottom-middle), indicating accurate and non-conservative uncertainty estimates.

Finally, \cref{fig:synthetic_regression_results_aug_training_curves} shows the validation-loss dynamics during training for the \gls{ncp}, \gls{ncpaug}, and \gls{encp} models. The results show that \gls{encp} converges faster and to a better optimum than the \gls{ncp}-based models, requiring approximately $2.5\times$ fewer epochs to reach the final loss values attained by \gls{ncp}. Although this difference in sample efficiency is mildly mitigated by data augmentation, \gls{encp} exhibits a clear optimality gap with respect to both \gls{ncp} and \gls{ncpaug}, achieving overall better validation performance. 

\begin{figure}[!htbp]
	\centering
	\includegraphics[width=\textwidth]{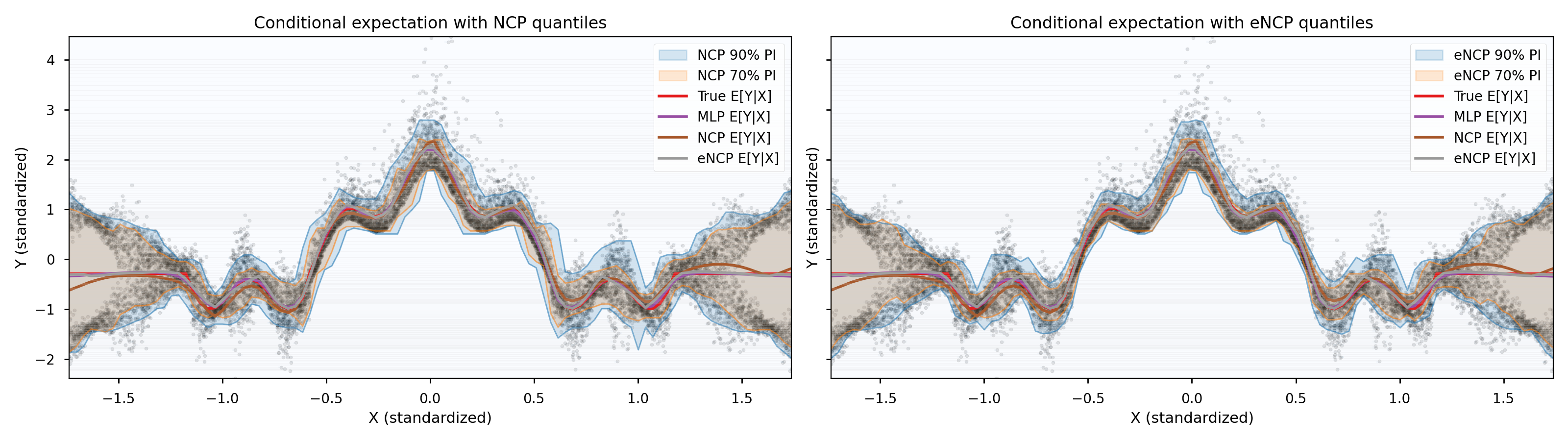}
	\includegraphics[width=\textwidth]{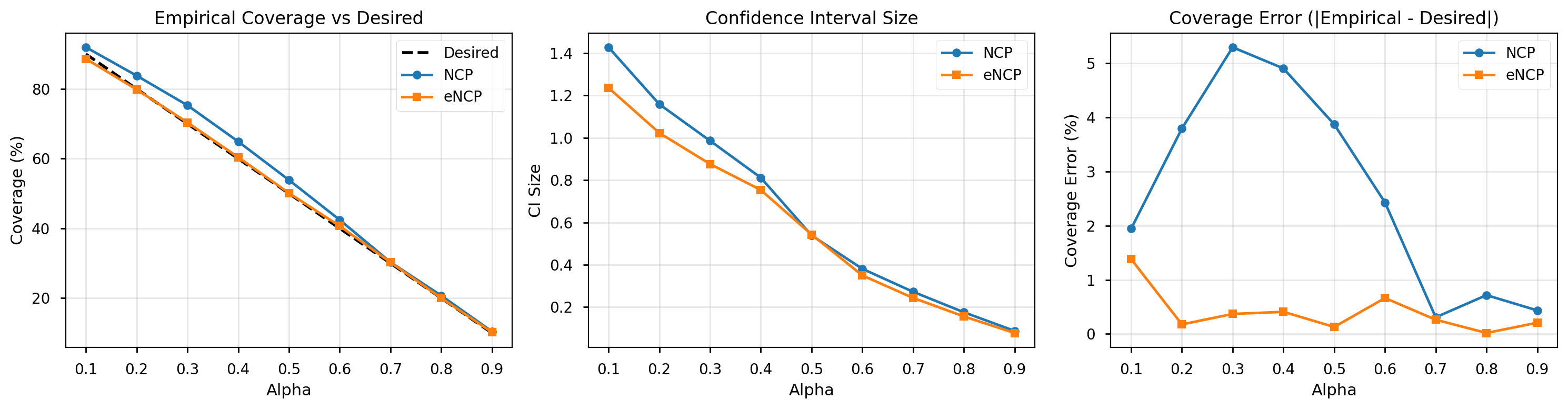}
	\includegraphics[width=.85\textwidth]{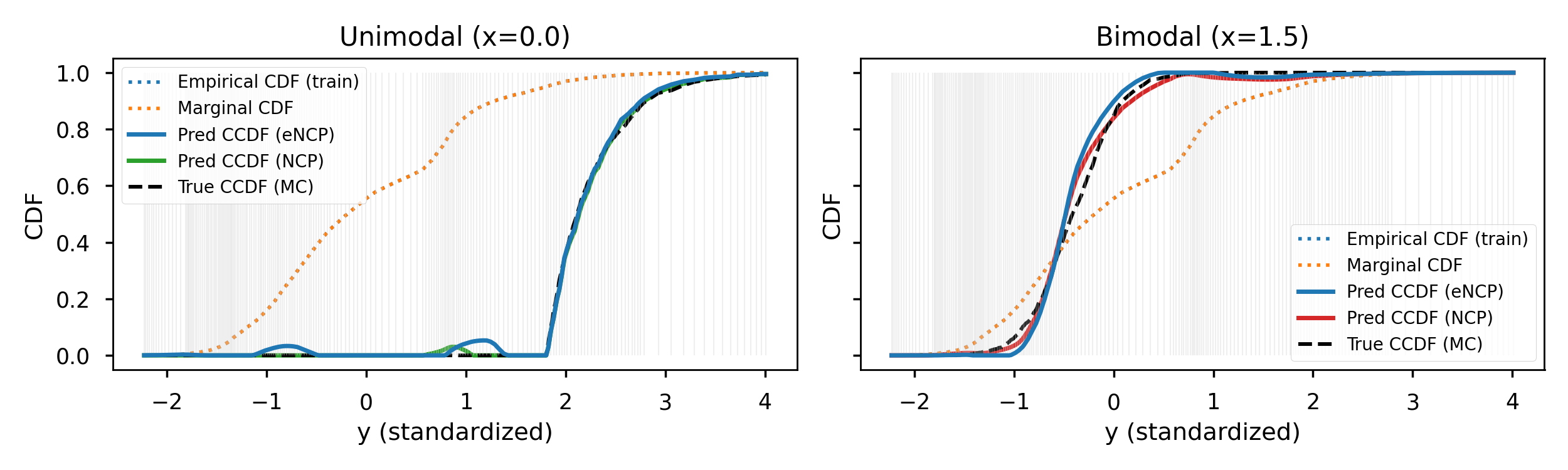}
	\vspace*{-0.3cm}
	\caption{
		\textbf{Top:} Comparison of conditional expectation predictions $\E[\rvy\mid \rvx]$ and uncertainty quantification via $90\%$ and $70\%$ \glspl{ci} for baseline \gls{mlp}, \gls{ncp}, and proposed \gls{encp} models.
		\textbf{Middle:} Empirical coverage \eqref{eq:coverage} tracking and \glspl{ci} set size \eqref{eq:setsize} for coverage levels from $10\%$ to $90\%$.
		\textbf{Bottom:} Example \glspl{ccdf} predicted by \gls{ncp} and \gls{encp} on the unimodal skewed distribution regime and the bimodal regime (see details in \cref{fig:ccdf_regression}).
	}
	\vspace*{-.5cm}
	\label{fig:synthetic_regression_results}
\end{figure}

\begin{figure}[!htbp]
	\includegraphics[width=\textwidth]{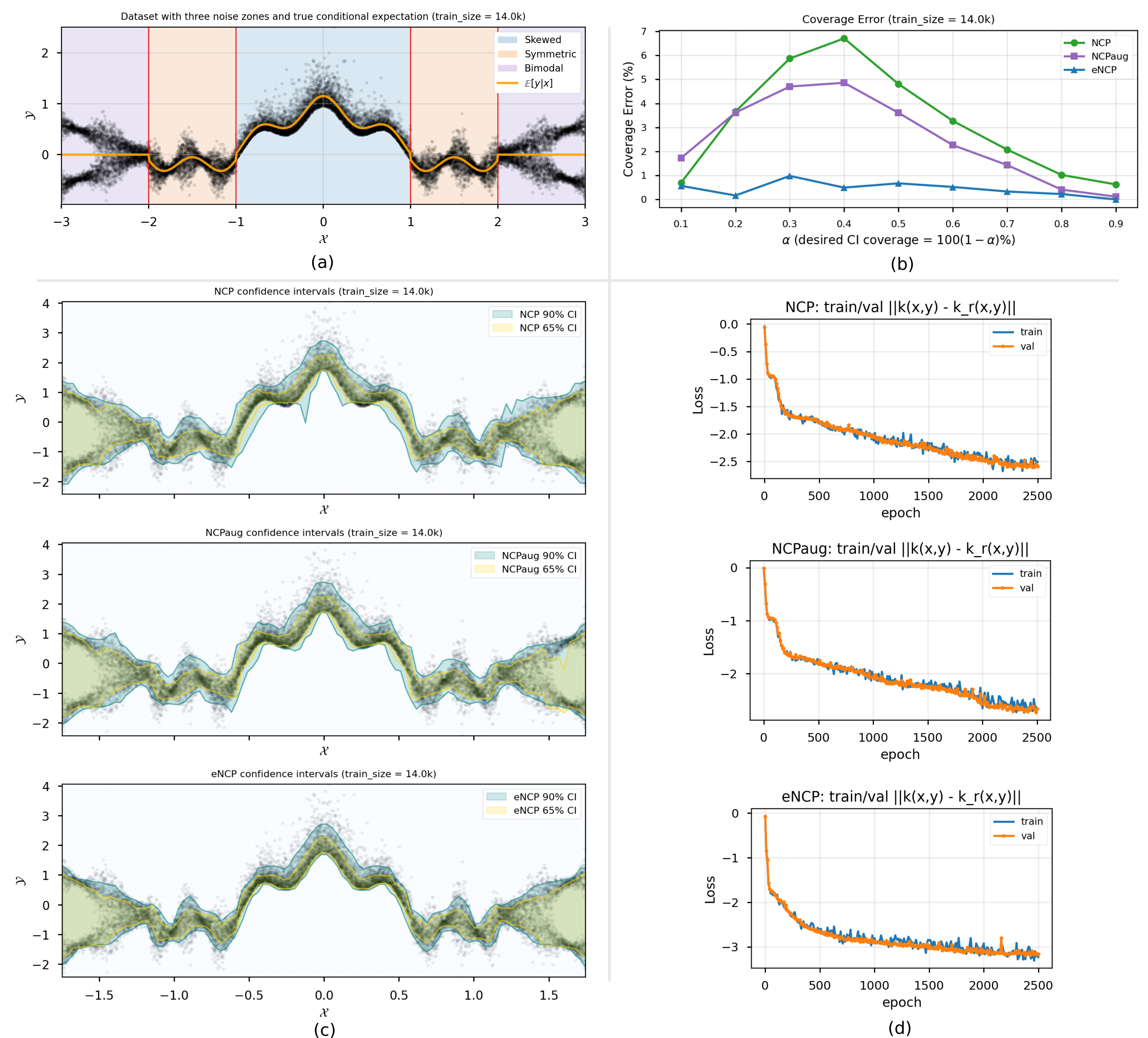}
	\caption{
		(a) Synthetic conditional distribution $\muycondx$ with three distinct zones whose conditional-distribution families share an uninformative conditional mean. Zone~1 ($|x|\le 1$) is exponentially skewed, zone~2 ($1<|x|\le 2$) is symmetric, and zone~3 ($|x|>2$) is bimodal. All zones include heteroscedastic noise.
		(b) Predicted confidence-interval coverage error for \gls{encp}, \gls{ncp}, and \gls{ncpaug} across desired coverage levels, defined as $100\% \cdot (1 - \alpha)$.
		(c) Predicted confidence intervals at desired coverages of $90\%$ and $65\%$ for \gls{ncp}, \gls{ncpaug}, and \gls{encp}. (d) Training curves of train and validation losses for \gls{ncp}, \gls{ncpaug}, and \gls{encp}. 
	}
	\label{fig:synthetic_regression_results_aug_training_curves}
\end{figure}

%%%%%%%%%%%%%%%%%%%%%%%%%%%%%%%%%%%%%%%%%%%%%%%%%%%%%%%%%%%%%%%%%%%%%%%%%%%%%%%
%%%%%%%%%%%%%%%%%%%%%%%%%%%%%%%%%%%%%%%%%%%%%%%%%%%%%%%%%%%%%%%%%%%%%%%%%%%%%%%
%%%%%%%%%%%%%%%%%%%%%%%%%%%%%%%%%%%%%%%%%%%%%%%%%%%%%%%%%%%%%%%%%%%%%%%%%%%%%%%
\subsection{Robustness to Symmetry Misspecification}
\label{app:exp_results_symmetry_misspecification}

The previous synthetic regression and uncertainty-quantification experiment assumes a correctly specified $\CyclicGroup[2]$ symmetry prior: the marginal distribution $\mux$ is invariant under $g_r\Glact[\vsX]x=-x$, and the conditional law $\muycondx$ is invariant under the same action on $\vsX$ with the trivial action on $\vsY$. Correctly specified symmetry priors is the core assumption of \gls{gdl}, and the main motivation for incorporating them is to improve generalization by leveraging the inductive bias of symmetry. 

However, symmetries can be used in practice also as \emph{local} or \emph{approximate} priors, used to regularize learning methods. In such cases, symmetry priors may hold only approximately or only on parts of the domain. Following the taxonomy of \citet{wang2023general}, we therefore test \gls{encp} under extrinsic and incorrect symmetry misspecification.

\begin{figure}[!htbp]
	\includegraphics[width=\textwidth]{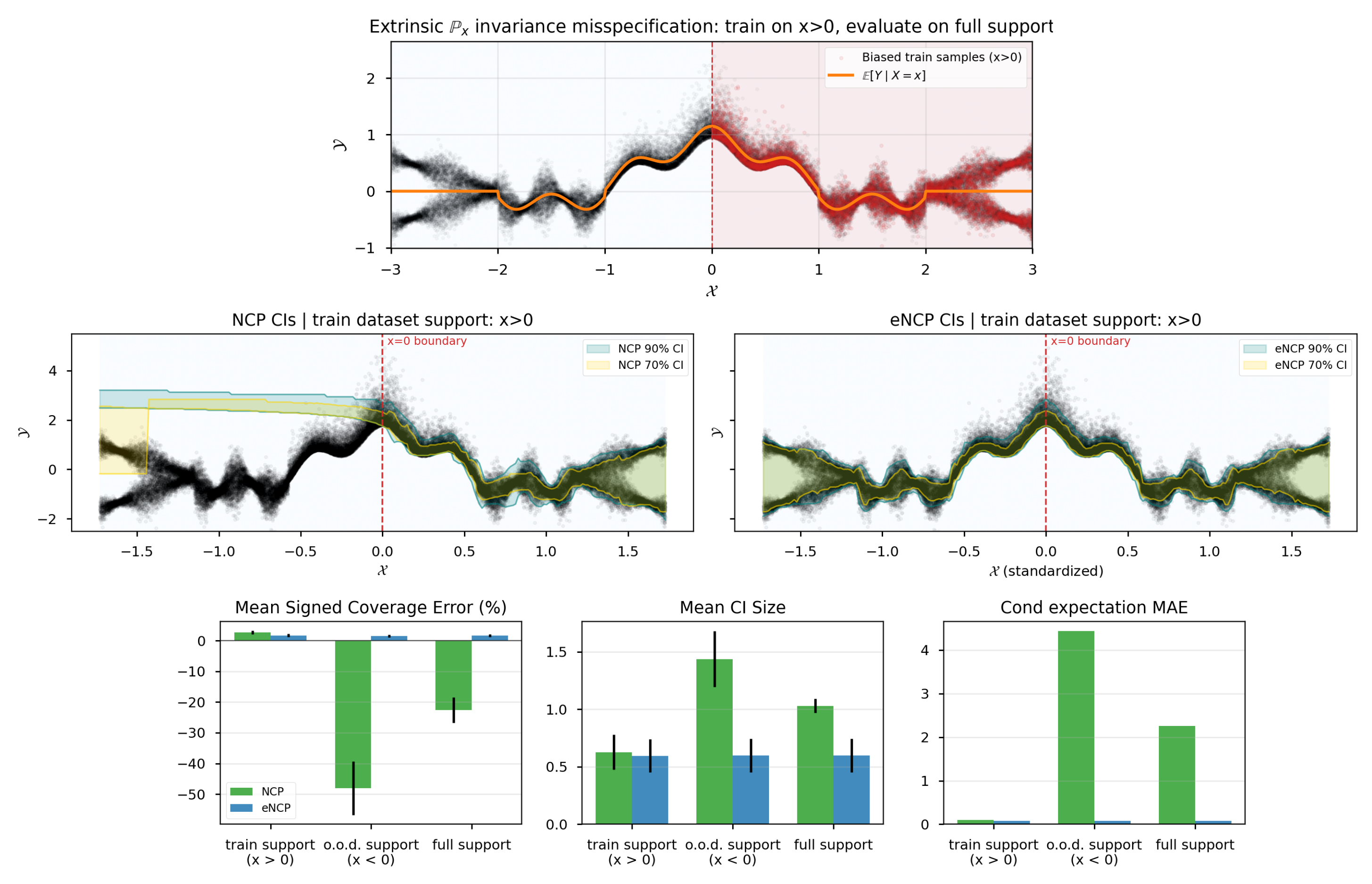}
	\caption[Extrinsic Symmetry Misspecification Summary]{
		Top: Illustration of extrinsic symmetry misspecification, where the train/validation/test support is biased toward the coset subspace $\rvx > 0$, violating the assumed $\CyclicGroup[2]$-invariance of $\mux$.
		Middle: Predicted confidence intervals at $90\%$ and $70\%$ coverage for \gls{ncp} and \gls{encp} under extrinsic symmetry misspecification.
		Bottom: Mean coverage error, confidence-interval volume, and conditional-expectation Maximum Absolute Error (MAE) for \gls{ncp} and \gls{encp}, evaluated on the biased support $\rvx > 0$, the out-of-distribution (o.o.d.) support $\rvx < 0$, and the full support $\vsX \equiv \R$.
	}
	\label{fig:misspec_summary}
\end{figure}

\paragraph{Misspecification Types}
We use the same one-dimensional regression and uncertainty-quantification setup as in \cref{app:exp_results_regression_synthetic}, and perturb only the symmetry assumptions. In the extrinsic case, the support of the train, validation, and test samples is biased toward the right coset subspace $\rvx>0$. The conditional law remains symmetric, but the marginal $\mux$ no longer satisfies the assumed $\CyclicGroup[2]$-invariance. In the incorrect case, $\mux$ remains symmetric, but $\muycondx$ violates the assumed symmetry locally on $|\rvx|>1$. We consider two variants: an unbiased violation that increases the heteroscedastic noise scale by a factor $C\in[1.0,6.0]$ only on $\rvx>1$, and a biased violation that introduces a linear conditional-mean shift with slope $b\in[0.0,0.75]$ only on $\rvx>1$. The former preserves the equivariance of the conditional expectation $\E[\rvy\mid\rvx=\cdot]$, whereas the latter violates it.

\begin{figure}[!htbp]
	\includegraphics[width=.95\textwidth]{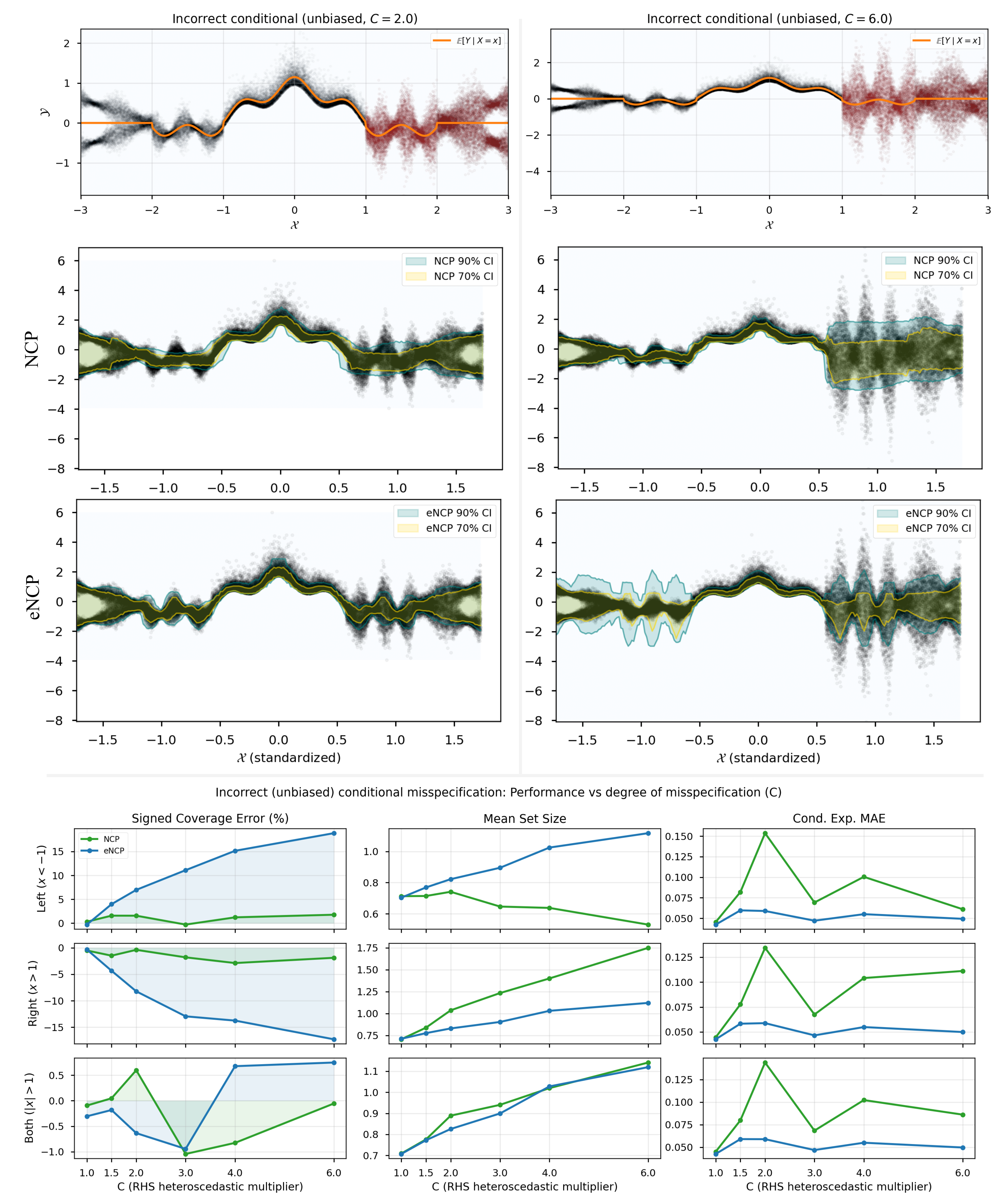}
	\caption[Incorrect Unbiased Conditional Misspecification Summary]{
		First row: Illustration of incorrect (unbiased) conditional misspecification. The $\CyclicGroup[2]$-invariance of $\muycondx$ is violated on $|\rvx| > 1$ by increasing the heteroscedastic noise scale with $C \in [1.0, 6.0]$ only on the right-coset subspace $\rvx > 1$. This misspecification is termed unbiased because the conditional expectation $\E[\rvy \mid \rvx = \cdot]$ remains $\CyclicGroup[2]$-equivariant.
		Second and third rows: Predicted confidence intervals at $90\%$ and $70\%$ coverage for \gls{ncp} and \gls{encp} with $C=2.0$ and $C=6.0$ under this incorrect unbiased conditional misspecification.
		Bottom: Mean (signed) coverage error, confidence-interval volume, and conditional-expectation Maximum Absolute Error (MAE) for \gls{ncp} and \gls{encp}, evaluated on the misspecified supports $\rvx > 1$ and $\rvx < -1$, and on the full support $\vsX \equiv \R$, across noise scales $C \in [1.0, 6.0]$.
	}
	\label{fig:incorrect_unbiased_summary}
\end{figure}

\paragraph{Metrics}
We compare \gls{ncp} and \gls{encp} using the same learned representations and conditional-quantile extraction procedure described in \cref{app:exp_results_regression_synthetic}. For each model, we report predicted $90\%$ and $70\%$ \glspl{ci}, mean signed coverage error, \gls{ci} volume, and the maximum absolute error of the conditional-expectation estimate. For extrinsic misspecification, the metrics are reported on the biased support $\rvx>0$, the out-of-distribution support $\rvx<0$, and the full support $\vsX\equiv\R$. For incorrect misspecification, the metrics are reported on the perturbed support $\rvx>1$, the symmetric counterpart $\rvx<-1$, and the full support.

\begin{figure}[!htbp]
	\includegraphics[width=.95\textwidth]{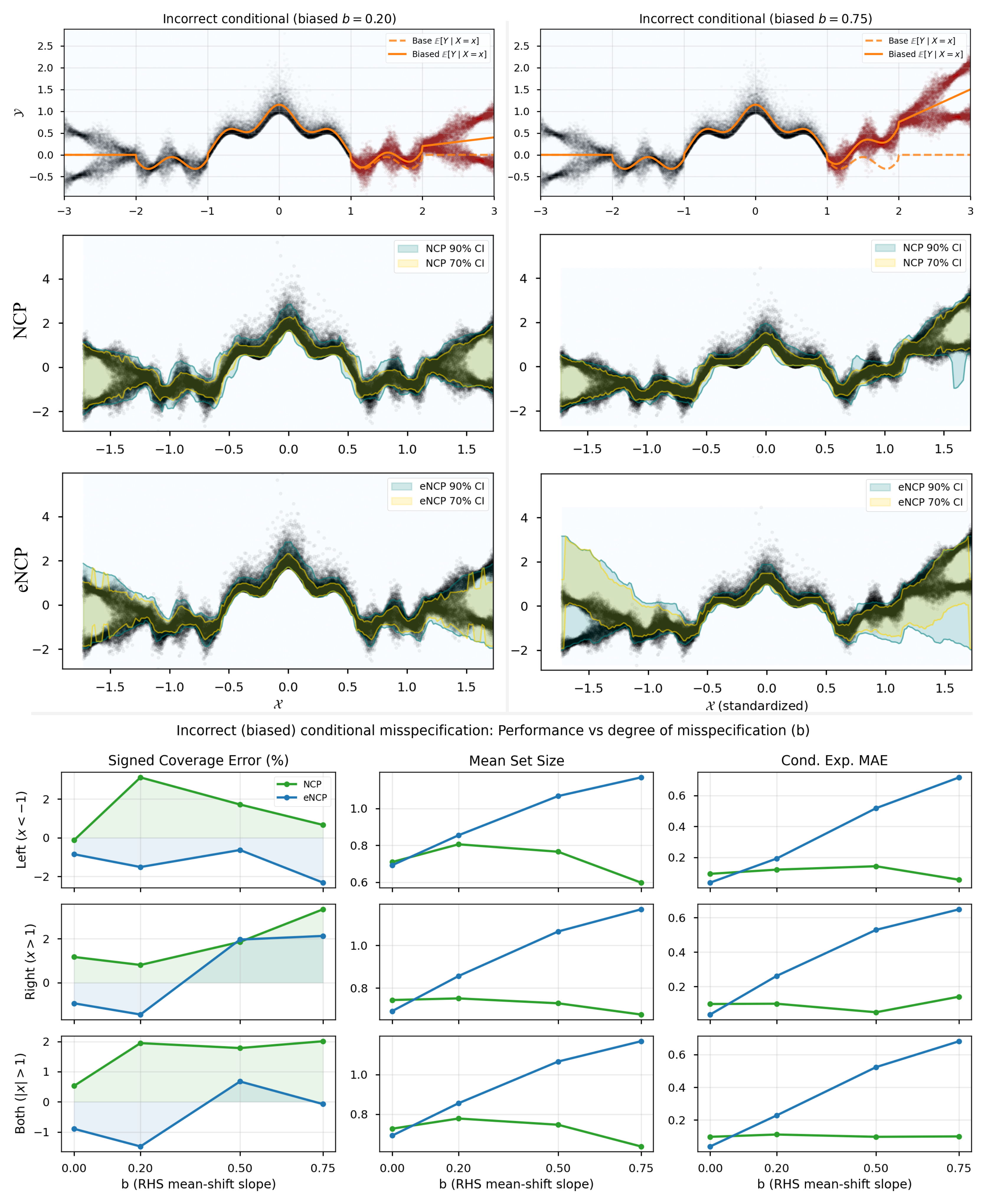}
	% \vspace*{-0.5cm}
	\caption[Incorrect Biased Conditional Misspecification Summary]{
		First row: Illustration of incorrect (biased) conditional misspecification. The $\CyclicGroup[2]$-invariance of $\muycondx$ is violated on $|\rvx| > 1$ by introducing a linear bias with slope $b \in [0.0, 0.75]$ on the right-coset subspace $\rvx > 1$. This misspecification is termed biased because the conditional expectation $\E[\rvy \mid \rvx = \cdot]$ is no longer $\CyclicGroup[2]$-equivariant.
		Second and third rows: Predicted confidence intervals at $90\%$ and $70\%$ coverage for \gls{ncp} and \gls{encp} with $b=0.2$ and $b=0.75$ under this incorrect biased conditional misspecification.
		Bottom: Mean (signed) coverage error, confidence-interval volume, and conditional-expectation Maximum Absolute Error (MAE) for \gls{ncp} and \gls{encp}, evaluated on the misspecified supports $\rvx > 1$ and $\rvx < -1$, and on the full support $\vsX \equiv \R$, across bias slopes $b \in [0.0, 0.75]$.
	}
	\label{fig:incorrect_biased_summary}
\end{figure}

\paragraph{Results}
\Cref{fig:misspec_summary} shows that, under extrinsic misspecification, enforcing the $\CyclicGroup[2]$ prior acts as a useful regularizer: \gls{encp} improves out-of-distribution generalization on $\rvx<0$ without degrading the in-support estimates on $\rvx>0$. This is the benign regime of misspecification in which the conditional relation remains symmetry-consistent and the violation comes from sampling support.

\Cref{fig:incorrect_unbiased_summary,fig:incorrect_biased_summary} show the two incorrect conditional-misspecification regimes. As the violation strength increases, performance degrades continuously rather than catastrophically, and the degradation remains concentrated on the regions where the symmetry assumption is violated. In the unbiased heteroscedastic case, the conditional expectation remains $\CyclicGroup[2]$-equivariant, so the main effect appears in coverage error and interval volume on the perturbed support. In the biased case, the conditional expectation itself is no longer equivariant, and the conditional-expectation error increases with the bias slope $b$. In both cases, the comparison between the misspecified and symmetric supports helps localize where the assumed symmetry is no longer compatible with the data.

%%%%%%%%%%%%%%%%%%%%%%%%%%%%%%%%%%%%%%%%%%%%%%%%%%%%%%%%%%%%%%%%%%%%%%%%%%%%%%%
%%%%%%%%%%%%%%%%%%%%%%%%%%%%%%%%%%%%%%%%%%%%%%%%%%%%%%%%%%%%%%%%%%%%%%%%%%%%%%%
%%%%%%%%%%%%%%%%%%%%%%%%%%%%%%%%%%%%%%%%%%%%%%%%%%%%%%%%%%%%%%%%%%%%%%%%%%%%%%%
\subsection{Uncertainty Quantification in Quadruped Legged Locomotion}
\label{app:exp_results_uncertainty_quantification_robot}
We test how well conditional-quantile models can recover the conditional $95\%$ confidence regions of three physically meaningful observables produced by a simulated AlienGo quadruped walking over rough terrain (see \cref{fig:teaser}) under varying friction coefficients. The dataset was collected using the \href{https://github.com/iit-DLSLab/Quadruped-PyMPC}{Quadruped-PyMPC} simulation framework and model predictive controller from \cite{turrisi2024sampling}.

The observables for which state‐dependent uncertainty estimates are desired are
$
	\vy_t \;=\; [U_t,\; T_t,\; \vtau^{\mathrm{grf}}_t]^{\!\top},
$
with each component defined as follows:
\begin{itemize}[leftmargin=*]
	\item \textbf{$\G$-invariant Kinetic Energy.}
	      $
		      T(\vq,\dot{\vq}) = \tfrac12\,\dot{\vq}^{\!\top}\,M(\vq)\,\dot{\vq} \in \R,
	      $
	      where $M(\vq)$ is the configuration-dependent inertia matrix. Noise is introduced through sensor measurement errors on the robot's \gls{dof} position $\vq \in \R^{12}$ and velocity $\dot{\vq} \in \R^{12}$.
	\item \textbf{$\G$-invariant Instantaneous Mechanical Work.}
	      $
		      U(\vq,\dot{\vq},\vtau) \in \R,
	      $
	      representing the instantaneous mechanical work exerted or absorbed by the robot. This quantity depends on the actuator torques (typically measured with noisy, biased sensors) as well as the external forces (e.g. gravity, contact forces) that are not reliably measurable due to unobserved terrain parameters.
	\item \textbf{$\G$-equivariant Ground-Reaction Forces}
	      $
		      \vtau_{\mathrm{grf}} \in \R^{12},
	      $
	      a fundamental quantity in quadruped control, whose reliable estimation and uncertainty quantification are critical for downstream tasks in robotics \citep{nistico2025muse,liu1994hierarchical}.
\end{itemize}

The observables of interest are predicted using a suit of onboard proprioceptive sensory signals available at time~$t$:
\[
	\vx_t \;=\;
	\Bigl[
	\vq_t,\;
	\dot{\vq}_t,\;
	\va_t,\;
	\vv_t,\;
	\vv_{t,\text{err}},\;
	\vomega_{t},\;
	\vomega_{t,\text{err}},\;
	\vg_t,\;
	\dot{\vp}_{t,\text{feet}},\;
	\vtau^{\mathrm{cmd}}_t
	\Bigr]^{\!\top}.
	\qquad
\]

Specifically, $\vq_t\in\R^{n_q}$ and $\dot{\vq}_t\in\R^{n_q}$ are the joint positions and velocities, respectively; $\va_t\in\R^{3}$ is the linear acceleration of the robot’s base frame measured by the IMU; $\vv_t\in\R^{3}$ is the base linear velocity, while $\vv_{t,\text{err}}\in\R^{3}$ the command error base linear velocity; $\vomega_t\in\R^{3}$ and $\vomega_{t,\text{err}}\in\R^{3}$ are the base angular velocity and its command error; $\vg_t\in\R^{3}$ is the gravity vector expressed in the base frame; $\dot{\vp}_{t,\text{feet}}\in\R^{12}$ stacks the linear velocities of the four feet (three components each); and $\vtau^{\mathrm{cmd}}_t\in\R^{n_q}$ contains the commanded joint torques.

Hence we design the experiments to compare models of similar footprint in number of parameters, while the loss used for training differs between the \gls{ncp} and \gls{encp} models w.r.t to the \gls{cqr} and \gls{ecqr} models.

% Together they define the \emph{Lagrangian}
% \(
% \mathcal{L}(\vq,\dot{\vq},\tau_{\mathrm{grf}})
% =
% T(\vq,\dot{\vq})
% -
% U(\vq,\dot{\vq},\tau_{\mathrm{grf}}),
% \)
% whose accurate, uncertainty-aware estimation benefits disturbance rejection, state estimation, and model-based control.

\begin{figure}[!htbp]
	\centering
	\includegraphics[width=\textwidth]{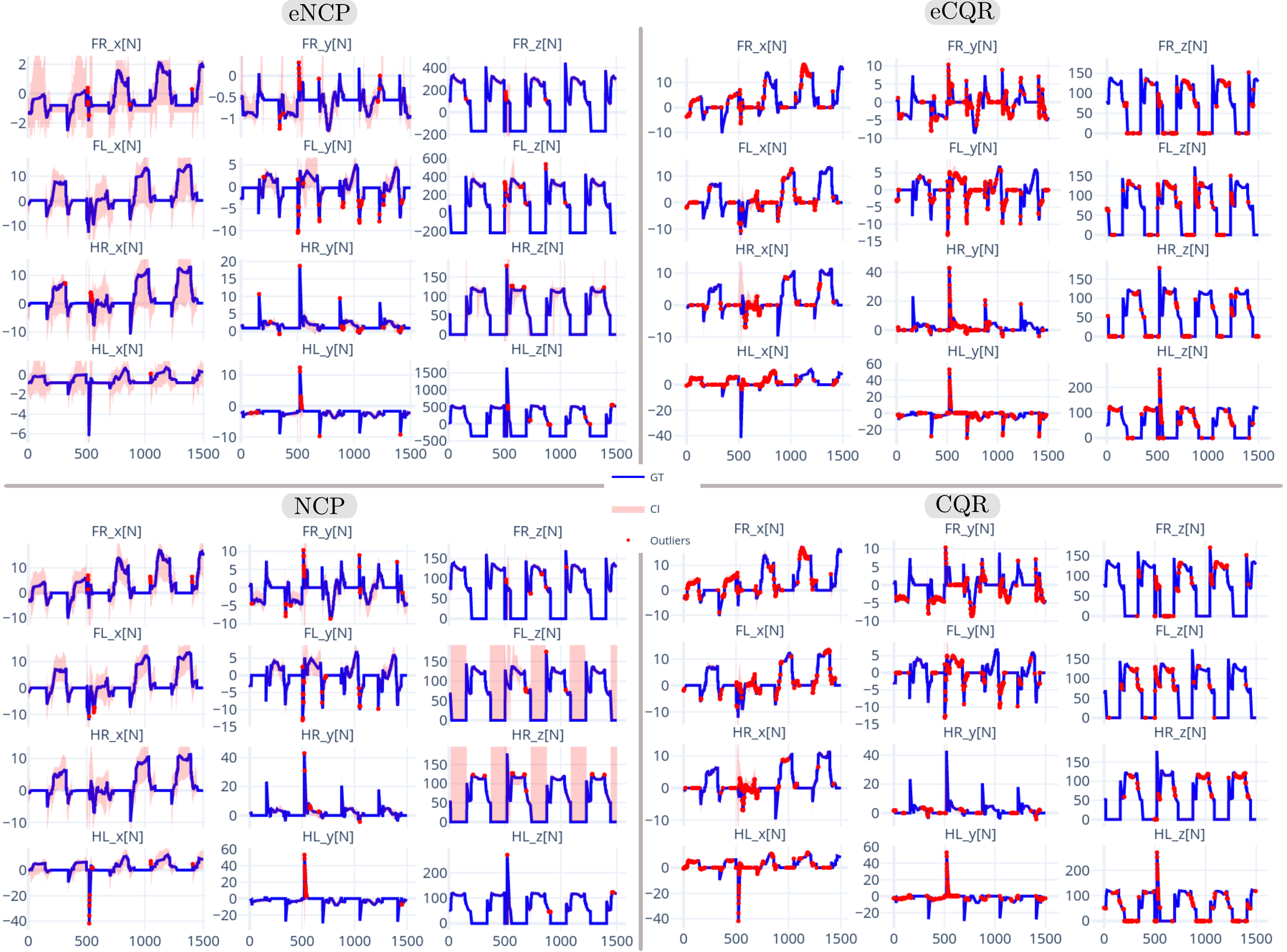}
	\vspace*{-0.5cm}
	\caption{
		Prediction of 90\% \glspl{ci} for ground-reaction forces $\vtau_{\text{grf}} \in \R^{12}$ of a quadruped robot on rough terrain with varying friction. We compare \gls{encp}, \gls{ncp}, \gls{ecqr}, and \gls{cqr} models. CIs are computed for each leg—front-right (FR), front-left (FL), hind-right (HR), hind-left (HL)—along $x$, $y$, $z$ axes. Forces outside CIs are red, within CIs are blue. Terrain variations cause significant variability in $x$/$y$ components due to surface orientation and friction differences, while $z$ components are influenced by local height changes affecting contact timing and producing short-duration high-impact forces.
	}
	\label{fig:quadruped_grf_uc}
\end{figure}

\paragraph{NN Architectures}
We configure all models (\gls{encp}, \gls{ncp}, \gls{ecqr}, and \gls{cqr}) with similar inference-time \gls{nn} architectures. The backbone consists of three hidden layers with 512 units each, followed by a final 128-unit layer that encodes the feature vector $r$ for \gls{ncp} and \gls{encp} models. Due to $\G$-equivariance weight sharing, \gls{encp} and \gls{ecqr} have $\times 2$ fewer trainbale parameters than their symmetry-agnostic counterparts.

\paragraph{Results.}
Given sensory input $\rvx$, the model predicts a set $\sC_{0.95}(\rvx)\subseteq\R^{14}$ satisfying $\PP(\rvy\in\sC_{0.95}(\rvx)\mid\rvx)\approx0.95$, while minimizing its volume $\EE_{\rvx}[\operatorname{vol}(\sC_{0.95}(\rvx))]$. High coverage implies that the true $\G$-invariant kinetic energy, instantaneous mechanical work, and $\G$-equivariant 12-dimensional ground-reaction forces lie within the predicted confidence set. Relaxed coverage (r-Coverage) quantifies the reliability of estimates on a per-dimension basis. \Cref{tab:uc_robotics_coverage_set_size_full} summarizes validation and test results for all models, and \cref{fig:quadruped_grf_uc} illustrates a trajectory of \gls{grf} with respective $90\%$ confidence intervals. Both \gls{cqr} and \gls{ecqr} produce smaller \glspl{ci} but fail to achieve desired coverage on the test set, implying unreliable \glspl{ci} requiring further calibration through retraining or conformal calibration \citep{feldman2023calibrated}. In contrast, the \gls{encp} model achieves desired coverage on the test set while producing larger confidence intervals, hence yielding reliable uncertainty estimates.

Note that we omitted the use of conformal calibration on the \gls{cqr} estimates to ensure a fair comparison between the parametric uncertainty estimates, considering that conformal calibration is model-agnostic and applicable to \gls{encp} as well.

\begin{table}[!htbp]
	\centering
	\small
	% \resizebox{\linewidth}{!}{
		\begin{tabular}{lcccccc}
			\small
			           & \multicolumn{3}{c}{Validation}       & \multicolumn{3}{c}{Test}                                                                                                                                                                                               \\
			\cmidrule(lr){2-4} \cmidrule(lr){5-7}
			           & r-Coverage $\uparrow$                & Coverage $\uparrow$                  & Set Size $\downarrow$                           & r-Coverage $\uparrow$                & Coverage $\uparrow$                  & Set Size $\downarrow$                           \\
			\midrule
			\gls{encp} & $\textcolor{blue}{99.3 {\pm} 0.0\%}$ & $\textcolor{blue}{94.1 {\pm} 0.4\%}$ & $2.4 {\pm} 0.4{\times}10^{10}$                  & $\textcolor{blue}{99.5 {\pm} 0.1\%}$ & $\textcolor{blue}{95.0 {\pm} 0.4\%}$ & $4.3 {\pm} 3.6{\times}10^{9}$                   \\
			\gls{ncp}  & $96.4 {\pm} 0.0\%$                   & $56.9 {\pm} 0.1\%$                   & $3.9 {\pm} 4.5{\times}10^{10}$                  & $\textcolor{blue}{99.5 {\pm} 0.0\%}$ & $56.9 {\pm} 0.3\%$                   & $2.6 {\pm} 1.4{\times}10^{10}$                  \\
			\gls{ecqr} & $70.7 {\pm} 0.6\%$                   & $7.3 {\pm} 1.7\%$                    & $\textcolor{blue}{3.7 {\pm} 2.6{\times}10^{8}}$ & $84.2 {\pm} 0.7\%$                   & $6.7 {\pm} 1.2\%$                    & $\textcolor{blue}{1.7 {\pm} 1.7{\times}10^{7}}$ \\
			\gls{cqr}  & $67.6 {\pm} 1.8\%$                   & $7.6 {\pm} 0.4\%$                    & $2.5 {\pm} 2.4{\times}10^{9}$                   & $80.5 {\pm} 3.7\%$                   & $8.5 {\pm} 0.9\%$                    & $1.4 {\pm} 0.1{\times}10^{8}$                   \\
		\end{tabular}
	% }
	\caption{
		Validation and test metrics for $95\%$ \glspl{ci} on quadruped robot observables traversing rough terrain (see \cref{app:exp_results_uncertainty_quantification_robot}). Metrics: \textbf{(i)} relaxed coverage (r-Coverage) \eqref{eq:relaxed_coverage}, \textbf{(ii)} coverage \eqref{eq:coverage}, and \textbf{(iii)} set size \eqref{eq:setsize}. Best results in blue. While \gls{ecqr} and \gls{cqr} produce smaller confidence intervals, they fail to achieve expected $95\%$ coverage on validation and test sets. The \gls{encp} model achieves best overall coverage for reliable uncertainty quantification. Importantly, \gls{encp} and \gls{ncp} models can provide \glspl{ci} at any coverage level \textbf{without retraining}, whereas \gls{cqr} and \gls{ecqr} require retraining for each new level.
	}
	\label{tab:uc_robotics_coverage_set_size_full}
\end{table}

\subsection{Training and Inference Computational Costs}
\label{sec:train_inference_compt_cost}

We compare the training and inference cost of \gls{encp} and \gls{ncp} across symmetry groups of increasing complexity. All timings are averaged over $200$ passes after warmup, using a fixed batch size, fixed data dimensionality $(\dim\vsX,\dim\vsY)$, and encoder networks $\vbfxp$ and $\vbfyp$ parameterized by three-hidden-layer \gls{mlp}/\gls{emlp} architectures.

\begin{table}[!htbp]
	\centering
	\small
	\resizebox{\linewidth}{!}{
		\begin{tabular}{lcccccc}
			\small
			Group [Order] & Iso. subspaces & Model & No. trainable params (k) $\downarrow$ & Train forward (ms) $\downarrow$ & Train backward (ms) $\downarrow$ & Val. forward (ms) $\downarrow$ \\
			\midrule
			$\sI_h$ [$60$]              & $5$ & \gls{encp} & $9.2~(60\times\ \text{fewer})$  & $3.067 \pm 0.105~(3.7\times)$ & $2.237 \pm 0.084~(2.5\times)$ & $0.741 \pm 0.091~(1.1\times)$ \\
			$\sI_h$ [$60$]              & $5$ & \gls{ncp}  & $550.8$                          & $0.825 \pm 0.032$             & $0.896 \pm 0.051$             & $0.653 \pm 0.080$             \\
			$\DihedralGroup[10]$ [$20$] & $8$ & \gls{encp} & $16.7~(20\times\ \text{fewer})$ & $3.761 \pm 0.128~(5.4\times)$ & $3.185 \pm 0.109~(4.1\times)$ & $0.487 \pm 0.030~(1.2\times)$ \\
			$\DihedralGroup[10]$ [$20$] & $8$ & \gls{ncp}  & $333.2$                          & $0.692 \pm 0.063$             & $0.773 \pm 0.046$             & $0.415 \pm 0.025$             \\
			$\CyclicGroup[10]$ [$10$]   & $6$ & \gls{encp} & $28.7~(10\times\ \text{fewer})$ & $4.250 \pm 0.492~(5.6\times)$ & $2.707 \pm 0.269~(3.4\times)$ & $0.450 \pm 0.033~(1.2\times)$ \\
			$\CyclicGroup[10]$ [$10$]   & $6$ & \gls{ncp}  & $287.3$                          & $0.760 \pm 0.298$             & $0.797 \pm 0.161$             & $0.382 \pm 0.028$             \\
			$\CyclicGroup[2]$ [$2$]     & $2$ & \gls{encp} & $139.3~(2\times\ \text{fewer})$ & $2.315 \pm 0.705~(3.2\times)$ & $1.540 \pm 0.552~(2.2\times)$ & $0.353 \pm 0.014~(1.2\times)$ \\
			$\CyclicGroup[2]$ [$2$]     & $2$ & \gls{ncp}  & $278.5$                          & $0.721 \pm 0.167$             & $0.710 \pm 0.114$             & $0.301 \pm 0.011$             \\
		\end{tabular}
	}
	\caption{
		Training and inference runtime comparison between \gls{encp} and \gls{ncp} across symmetry groups with different orders and isotypic structures. Values report the mean $\pm$ standard deviation over $200$ passes after warmup. Parenthesized factors in the \gls{encp} rows are measured relative to the corresponding \gls{ncp} runtime or parameter count; for example, $(3.7 x)$ indicates that the metric is $3.7$ times larger for \gls{encp} than for \gls{ncp}.
	}
	\label{tab:train_inference_computational_costs}
\end{table}

\paragraph{Training Time Overhead}
\Cref{tab:train_inference_computational_costs} shows that \gls{encp} incurs a per-batch training overhead of roughly $3\times$--$6\times$ in the forward pass and $2.2\times$--$4.1\times$ in the backward pass. This overhead arises from the basis expansions in the equivariant layers (common to all equivariant \gls{nn} implementations) and from the symmetry-aware estimation of means, covariances, and cross-covariances in the loss \eqref{eq:disentangled_contrastive_loss}. Importantly, the computational overhead is governed by the number of isotypic subspaces rather than the order of the group, given that several algebraic operations are performed sequentially per each isotypic subspace. 

\paragraph{Inference Time Overhead}
At inference time, the equivariant \gls{nn} can be run as a standard \gls{nn} resulting in a minimal inference overhead \citep{ordonez2026symmlearning}, around $1.1\times$--$1.2\times$ across all groups. This mild increase is due to the additional change-of-basis layer used to expose the isotypic decomposition in the equivariant encoder (see \cref{remark:isotypic_decomposition}). Despite the per-batch overhead, \gls{encp} converges substantially faster: as shown in \cref{fig:synthetic_regression_results_aug_training_curves}, it reaches a lower validation loss before epoch $500$ than \gls{ncp} achieves after $2500$ epochs, corresponding to roughly $5\times$ fewer training batches to reach a better solution. Combined with the $2\times$--$60\times$ reduction in trainable parameters, this yields a favorable overall computational trade-off.

\FloatBarrier

%% file: appendix/ncp_primer.tex
% __________________________________________________________________________________________________________________________
\section{Conditional Probability Modeling via the Conditional Expectation Operator}
\label{sec:conditional_expectation_operator}
% __________________________________________________________________________________________________________________________
This section introduces the modelling of conditional probabilities for two random variables via the \textbf{conditional expectation operator}. Our goal is to understand conditional expectation from an operator‑theoretic perspective. We begin by describing the marginal, joint, and conditional probabilities of the random variables within a measure‑theoretic framework. This discussion extends the exposition of \citet{kostic2024neural}.

Given two random variables $(\rvx, \rvy)$ taking values in the measure spaces $(\vsX, \Sigma_{\vsX}, \mux)$ and $(\vsY, \Sigma_{\vsY}, \muy)$, we have that the marginal probability of any set $\sA \in \Sigma_\vsX$ and $\sB \in \Sigma_\vsY$ are given by
\begin{equation}
	\small
	\PP(\rvx \in \sA) = \int_{\vsX} \one_{\sA}(\vx) \mux(d\vx) = \int_{\sA} \mux(d\vx)
	\quad \text{and} \quad
	\PP(\rvy \in \sB) = \int_{\vsY} \one_{\sB}(\vy) \muy(d\vy) = \int_{\sB} \muy(d\vy),
\end{equation}
where $\one_{\sA} \in \LpxFull$ and $\one_{\sB} \in \LpyFull$ denote the characteristic functions of sets $\sA$ and $\sB$, respectively.

Furthermore, under the reasonable assumption that the joint probability measure is absolutely continuous w.r.t to the product of the marginals $\muxy \ll \mux \times \muy$, we have that there exist a Radon-Nikodym derivative $\pmd:\vsX\times\vsY \to \R_+$ such that $\muxy(d\vx,d\vy) = \pmd(\vx, \vy) \mux(d\vx) \muy(d\vy)$. Note that $\pmd$ is a kernel function that pointwise deforms the product of the marginals to produce the joint distribution \cite{sugiyama2012density} (see \cref{fig:symmetry_priors_long}). This kernel function enable us to express the joint probability by:
\begin{equation}
	\label{eq:joint_probability_measure}
	\small
	\PP( \rvx \in \sA, \rvy \in \sB)
	% = \int_{\vsX\times\vsY} \one_{\sB}(\vx) \one_{\sA}(\vy) P_{\rvx\rvy}(d\vx,d\vy) 
	{=}
	\int_{\vsX\times\vsY} \one_{\sA}(\vx) \one_{\sB}(\vy)
	\underbrace{\pmd(\vx,\vy) \muy(d\vy) \mux(d\vx)}_{\muxy(d\vx,d\vy)}
	{=}
	\int_{\sA\times\sB} \pmd(\vx,\vy) \mux(d\vx) \muy(d\vy) .
	% \;\;  \text{s.t.}\; \muxy \ll \mux \times \muy.
\end{equation}
Furthermore, given that $\PP(\rvy {\in} \sB \vert \rvx {\in} \sA) = \sfrac{\PP(\rvx {\in} \sA, \rvy {\in} \sB)}{\PP(\rvx \in \sA)}$, the conditional probability of any set $\sB \in \Sigma_{\vsY}$ given a value of the random variable $\rvx {=} \vx$ is given by:
\begin{equation}
	\label{eq:conditional_probability_measure}
	\small
	\PP(\rvy {\in} \sB \vert \rvx {=} \vx)
	{=}  \int_{\vsY} \one_{\sB}(\vy) \muycondx(d\vy \vert \vx)
	% {=}  \int_{\vsY} \one_{\sA}(\vy) \tfrac{\muxy(d\vy,\vx)}{\mux(\vx)}
	{=} \int_{\vsY} \one_{\sB}(\vy) \pmd(\vx, \vy) \muy(d\vy) = \int_{\sB} \pmd(\vx, \vy) \muy(d\vy),
\end{equation}
where $\muycondx: \Sigma_\vsY \times \vsX \mapsto [0,1]$ denotes the \textbf{conditional probability measure}. This is a well-defined probability measure considering that:
% $\muxy(\sB, \vsY):= \mux(\sB)$ by \cref{eq:joint_probability_measure}.
\begin{equation*}
	\small
	\PP(\rvx \in \sA) := \PP(\rvx \in \sA, \rvy \in \vsY)
	=
	\int_{\sA} \
	\underbrace{
		\bigg(
		\int_{\vsY} \pmd(\vx,\vy) \muy(d\vy)
		\bigg)
	}_{\E \muycondx(d\rvy\vert\rvx{=}\vx) = 1 \;\; \forall \vx \in \vsX}
	\mux(d\vx)
	=
	\int_{\sA} 
	\mux(d\vx).
\end{equation*}

\paragraph{The Operator Perspective} Every measurable function $h \in \LpyFull$ can be approximated by simple functions—that is, as a combination of characteristic functions on measurable sets: $h(\cdot) \approx \sum_{i\in\N} \ebcoeffy_i \one_{\sA_i}(\cdot)$. Thus, \cref{eq:conditional_probability_measure} is a special case of the more general problem of approximating the conditional expectation of any function $h \in \LpyFull$ given $\rvx$. This conditional expectation is captured by the action of a linear integral operator:

\begin{definition}[Conditional expectation operator]
	\label{def:conditional_expectation_operator}
	Let $(\rvx,\rvy)$ be two random variables defined on the measure spaces $(\vsX,\Sigma_{\vsX},\mux)$ and $(\vsY,\Sigma_{\vsY},\muy)$, respectively, and let $\LpxFull$ and $\LpyFull$ denote the corresponding spaces of square-integrable functions. The conditional expectation operator $\oE_{\rvy|\rvx}: \LpyFull \to \LpxFull$ is the linear integral operator---defined via the \gls{pmd} Radon–Nikodym derivative
	$
		\pmd(\vx,\vy)=\sfrac{\muxy(d\vx,d\vy)}{\mux(d\vx)\,\muy(d\vy)}
	$
	---which acts on any function $h\in\LpyFull$ by computing its conditional expectation:
	\begin{equation}
		\nonumber
		{\small
		[\oE_{\rvy|\rvx} h](\vx)
		{=}
		\E[h(\rvy) \vert \rvx {=} \vx]
		{:=}
		\int_{\vsY} h(\vy) \muycondx(d\vy\vert\vx)
		{=}
		\int_{\vsY} h(\vy) \tfrac{\muxy(d\vy,\vx)}{\mux(d\vx)}
		{=}
		\int_{\vsY} h(\vy) \pmd(\vx, \vy) \muy(d\vy).
		}
	\end{equation}
\end{definition}

From a learning perspective, approximating the conditional expectation operator sufficiently well for a relevant set of functions in $\LpyFull$ implies that we can approximate the conditional probability measure of any set $\sA \in \Sigma_\vsY$. This enables both regression \textit{and} uncertainty quantification applications with a single model (see \cref{eq:uses_of_conditional_expectation_operator}).

%% file: appendix/group_rep_theory.tex
\section{Background on Group and Representation Theory}
\label{sec:group_theory}

\paragraph{Group Actions and Representations}
This section provides a concise overview of the fundamental concepts in group and representation theory, which are used to define the symmetries of the random variables we consider in this work. For a comprehensive background on these topics in finite-dimensional vector spaces, see \citet{weiler2023EquivariantAndCoordinateIndependentCNNs}; for the infinite-dimensional case, consult \citet{Knapp1986}. These concepts will be referenced as needed in the main text. To begin, we define a group as an abstract mathematical object.
\begin{definition}[Group] \label{def:group}
	A group is a set $\G$, endowed with a binary composition operator defined as:
	\begin{subequations}
		\begin{equation}
			\mapping
			{(\Gcomp)}
			{\G \times \G}{\G}
			{(\g_1, \g_2)}
			{\g_1 \Gcomp \g_2,}
			\label{eq:group_operation}
		\end{equation}
		such that the following axioms hold:
		\begin{align}
			\text{Associativity:} & \quad (\g_1 \Gcomp \g_2) \Gcomp \g_3 = \g_1 \Gcomp (\g_2 \Gcomp \g_3), \quad \forall\; \g_1,\g_2,\g_3 \in \G, \label{eq:group_associativity}      \\
			\text{Identity:}      & \quad \exists\; e \in \G \; \text{such that} \; e \Gcomp \g = \g = \g \Gcomp e, \quad \forall\; \g \in \G, \label{eq:group_identity}              \\
			\text{Inverses:}      & \quad \forall\; \g \in \G, \; \exists\; \g^{-1} \in \G \; \text{such that} \; \g \Gcomp \g^{-1} = e = \g^{-1} \Gcomp \g. \label{eq:group_inverse}
		\end{align}
	\end{subequations}
\end{definition}
We are primarily interested in symmetry groups, i.e., groups of transformations acting on a set $\vsX$. Each transformation is a bijection that leaves a fundamental property of the element of the set invariant. For example, if $\vsX$ represents states of a dynamical system, the invariant property is the state energy (see \cref{fig:example_group_action}); if $\vsX$ is a data space, the preserved quantity is typically the probability density/distribution (see \cref{fig:symmetry_priors_long}).
\begin{definition}[Group action on a set \citep{weiler2023EquivariantAndCoordinateIndependentCNNs}]
	\label[definition]{def:left_group_action}
	Let $\vsX$ be a set endowed with symmetry group $\G$. The (left) group action of the group $\G$ on the set $\vsX$ is a map:
	\begin{subequations}
		\begin{align}
			\mapping
			{(\Glact)}
			{ \G \times \vsX}{\vsX}
			{(\g,\vx)}
			{\g \Glact \vx}
		\end{align}
		that is compatible with the group composition and identity element $e \in \G$, in the sense that:
		\begin{align}
			\text{Identity:}      & \quad e \Glact \vx = \vx,                                            & \quad \forall\; \vx \in \vsX \label{eq:symmetry_action_on_set_identity}                                    \\
			\text{Associativity:} & \quad (\g_1 \Gcomp \g_2) \Glact \vx = \g_1 \Glact (\g_2 \Glact \vx), & \quad \forall\; \g_1, \g_2 \in \G, \forall\; \vx \in \vsX. \label{eq:symmetry_action_on_set_associativity}
		\end{align}
	\end{subequations}
\end{definition}

These sets of bijections describe structural properties of the set $\vsX$. To study these properties, we will frequently refer to the group of symmetry related elements of to a given element $\vx \in \vsX$ as its \emph{group orbit} of $\vx$:

\begin{definition}[Group orbit]
	\label{def:group_orbit}
	Let $\vx$ be an element of the set $\vsX$ endowed with symmetry group $\G$. The group orbit of $\vx$ is the the set of all symmetry related set elements, denoted by:
	\[
		\G \vx := \{ \g \Glact[\vsX] \vx \;|\; \g \in \G\} \equiv \{ \vx \Gract[\vsX] \g \;|\; \g \in \G\}.
	\]
\end{definition}

For our practical purposes, we focus on symmetry transformations acting on sets with a vector space structure. In most cases, the group action on a vector space is linear, which allows us to express symmetry transformations as (orthogonal) matrix-vector  operations, once a basis for the space is chosen.

\begin{definition}[Linear group representation]
	\label{def:group_representation}
	Let $\vsX$ be a vector space endowed with symmetry group $\G$. A \emph{linear representation} of  $\G$ on $\vsX$ is a map, denoted by $\rep[\vsX]$, between symmetry transformation and invertible linear maps on $\vsX$ (i.e., elements of the general linear group $\GLGroup(\vsX)$):
	\begin{subequations}
		\begin{equation}
			\mapping{\rep[\vsX]}{\G}{\GLGroup(\vsX)}{\g}{\rep[\vsX](\g),}
		\end{equation}
		such that the following properties hold:
		\begin{align}
			\text{composition} & : \rep[\vsX](\g_1 \Gcomp \g_2) = \rep[\vsX](\g_1) \rep[\vsX](\g_2),
			                   & \quad \forall\; \g_1, \g_2 \in \G,
			\label{eq:group_representation_composition}                                                                                                                   \\
			\text{inversion}   & : \rep[\vsX](\g^{-1}) = \rep[\vsX](\g)^{-1}, \quad                  & \forall\; \g \in \G. \label{eq:group_representation_invertibility} \\
			\text{identity}    & :  \rep[\vsX](g \Gcomp \g^{-1}) = \rep[\vsX](e) = \mI,              &
			\label{eq:group_representation_identity}
		\end{align}\label{eq:group_representation_properties}
		Whenever the vector space is of finite dimension $n < \infty$, linear maps admit a matrix form  $\rep[\vsX](\g) \in \R^{n \times n}$, once a basis set $\bSet{\vsX}$ for the vector space $\vsX$ is chosen. In this case, \cref{eq:group_representation_composition,eq:group_representation_invertibility,eq:group_representation_identity} show how the composition and inversion of symmetry transformations translate to matrix multiplication and inversion, respectively. Moreover, $\rep[\vsX]$ allows to express a (linear) group action (\cref{def:left_group_action}) as a matrix-vector multiplication:
		\begin{align}
			\mapping
			{(\Glact)}
			{ \G \times \vsX}{\vsX}
			{(\g,\vx)}
			{\g \Glact \vx := \rep[\vsX](\g) \vx.}
		\end{align}
	\end{subequations}
\end{definition}

We will often study linear maps that preserve a vector space's symmetry structure by commuting with the group action. These maps are known as \textbf{endomorphisms} and include all change of basis  and affine transformations that do not break the symmetry.

\begin{definition}[Endomorphism]
	\label{def:endomorphism}
	Let $\vsX$ be a vector space endowed with symmetry group $\G$, with the group action $\Glact[\vsX]: \G \times \vsX \mapsto \vsX$. A linear map $\mA: \vsX \mapsto \vsX$ is said to be an endomorphism if it commutes with the group action, such that:
	\begin{equation*}
		\vcenter{
			\hbox{
				$\rep[\vsX](\g) \mA = \mA \rep[\vsX](\g), \quad \forall\; \g \in \G$
			}
		} \qquad \iff \qquad
		\vcenter{
			\hbox{\homomorphismDiag
				{\vsX}
				{\vsX}
				{\Glact[\vsX]}
				{\Glact[\vsX]}
				{\mA}
			}}
	\end{equation*}
	We will denote the space of all endomorphisms of $\vsX$ as $\Endomorphism[\G]{\vsX}$, such that any $\mA \in \Endomorphism[\G]{\vsX}$ satisfies the above commutation property.
\end{definition}

So far, we have studied symmetric vector spaces $\vsX$ endowed with a group action $\Glact[\vsX]$, which can be represented in matrix form via a group representation $\rep[\vsX]$ once a basis for $\vsX$ is chosen. However, while the choice of basis alters the group representation $\rep[\vsX]$, the underlying group action $\Glact[\vsX]$ remains invariant. This observation leads us to the concept of equivalent group representations.

\begin{definition}[Equivalent group representations]
	\label{def:equivalent_group_representations}
	Let $\vsX$ be a vector space endowed with symmetry group $\G$, and let $\rep[\vsX]'$ and $\rep[\vsX]$ be two group representations of $\G$ on $\vsX$. They are said to be \emph{equivalent}, denoted by $\rep[\vsX]' \sim \rep[\vsX]$, if there exists an invertible change of basis $\mQ \in \Endomorphism[\G]{\vsX}$ such that
	\begin{equation}
		\rep[\vsX]'(\g) = \mQ \rep[\vsX](\g) \mQ^{-1}, \quad \forall\; \g \in \G.
	\end{equation}
	Equivalent representations arise when the same group action $(\Glact): \G \times \vsX \to \vsX$ is expressed in different coordinate frames or bases. For instance, let $\sA_{\vsX}$ and $\sB_{\vsX}$ be two bases for $\vsX = \text{span}(\sA_{\vsX}) = \text{span}(\sB_{\vsX})$, and let $\mQ_{\sA}^{\sB}:\vsX \to \vsX$ denote the change of basis from $\sA_{\vsX}$ to $\sB_{\vsX}$, so that $\vx^\sB = \mQ_{\sA}^{\sB} \vx^{\sA}$ for all $\vx^{\sA} \in \vsX$. Then the group action admits equivalent representations, $\rep[\vsX]^\sA \sim \rep[\vsX]^\sB$, since
	\begin{equation}
		\begin{aligned}
			\g \Glact \vx^\sB            & := \mQ_{\sA}^\sB (\g \Glact \vx^\sA), \qquad \forall \g\in\G,                                                                         \\
			\rep[\vsX]^{\sB}(\g) \vx^\sB & = \mQ_{\sA}^\sB \bigl(\rep[\vsX]^{\sA}(\g) \vx^\sA\bigr) = \Bigl(\mQ_{\sA}^\sB \rep[\vsX]^{\sA}(\g){\mQ_{\sA}^\sB}^{-1}\Bigr)\vx^\sB, \\
			\rep[\vsX]^{\sB}(\g)         & = \mQ_{\sA}^\sB \rep[\vsX]^{\sA}(\g){\mQ_{\sA}^\sB}^{-1}.
		\end{aligned}
	\end{equation}
\end{definition}

To reveal the modular structure of symmetric vector spaces, we often change bases to decompose them into subspaces stable under the action of the group $\G$, termed $\G$-stable subspaces. This decomposition mirrors how a symmetry group can be broken down into products and direct products of smaller groups and is essential for analyzing and simplifying group representations. We introduce the following definition.

\begin{definition}[$\G$-stable and irreducible subspaces]
	\label{def:G_stable_subspace}
	Let $\vsX$ be a vector space endowed with a group action $(\Glact)$ of the symmetry group $\G$. A subspace $\vsX' \subseteq \vsX$ is said to be \emph{$\G$-stable} if the action of any group element on any vector in the subspace remains within the subspace, that is,
	\[
		\g \Glact \vx \in \vsX', \quad \forall\; \vx \in \vsX' \subseteq \vsX , \forall\; \g \in \G.
	\]
	If the only $\G$-stable subspaces of $\vsX$ are $\{\bm{0}\}$ and $\vsX$ itself, then $\vsX$ is a \emph{irreducible $\G$-stable space}. We will denote irreducible $\G$-stable spaces with an over bar, e.g., $\bar{\vsV}$.
\end{definition}

\paragraph{Building Blocks of Symmetric Vector Spaces}
A recurrent theme in this work and in geometric deep learning is to study and process symmetric vector spaces by decomposing them in terms of irreducible $\G$-stable spaces. This process directly corresponds to the decomposition of the associated group representation $\rep[\vsX]$ into smaller representations acting on these $\G$-stable subspaces.

\begin{definition}[Decomposable representation]
	\label{def:decomposable_representation}
	Let $\vsX$ be a vector space with a group action $(\Glact)$ defined by the representation $\rep[\vsX]$ in a chosen basis $\sA_{\vsX}$. The representation is \emph{decomposable} if it is equivalent to a direct sum of two lower-dimensional representations, $\rep[\vsX] \sim \rep[\vsX_1] \oplus \rep[\vsX_2]$, where $\vsX_1$ and $\vsX_2$ are $\G$-stable subspaces of $\vsX$. Equivalently, there exists a change of basis $\mQ_{\sA}^{\sB}: \vsX \to \vsX$ such that
	\begin{equation}
		\nonumber
		\small
		\rep[\vsX]^{\sB} =
		\begin{bsmallmatrix}
			\rep[\vsX_1] & \bm{0} \\
			\bm{0}       & \rep[\vsX_2]
		\end{bsmallmatrix}
		=
		\mQ_{\sA}^{\sB} \rep[\vsX] {\mQ_{\sA}^{\sB}}^{-1},
		\; \text{and} \quad
		\g \Glact \vx^{\sB} := \rep[\vsX]^{\sB}(\g) \vx^{\sB} =
		\begin{bsmallmatrix}
			\rep[\vsX_1](\g) \vx^{\sB}_1 \\
			\rep[\vsX_2](\g) \vx^{\sB}_2
		\end{bsmallmatrix},
		\; \text{where } \mQ_{\sA}^{\sB} \vx =
		\begin{bsmallmatrix}
			\vx^{\sB}_1 \in \vsX_1 \\
			\vx^{\sB}_2 \in \vsX_2
		\end{bsmallmatrix}
	\end{equation}
	Hence, the representation's decomposition $\rep[\vsX] \sim \rep[\vsX_1] \oplus \rep[\vsX_2]$ corresponds to \textbf{decomposing the vector space} into $\G$-stable subspaces, $\vsX = \vsX_1 \oplus \vsX_2$.
	% Moreover, if the representation is block-diagonal in some basis set $\sB_{\vsX}$, then $\sB_{\vsX}$ is the union of disjoint basis sets $\sB_{\vsX_1}$ and $\sB_{\vsX_2}$ for $\vsX_1$ and $\vsX_2$, respectively.
\end{definition}
When iteratively applying the decomposition process, we eventually reach representations that cannot be further decomposed. These are known as irreducible representations, or \emph{irreps}, and they serve as the fundamental building blocks for all representations of a compact symmetry group $\G$. From a vector space perspective, the irreducible $\G$-stable subspaces (\cref{def:G_stable_subspace}) associated with these irreps are the elementary subspaces that comprise any symmetric vector space, analogous to how one-dimensional subspaces are the fundamental components of standard vector spaces.

\begin{definition}[Irreducible representation]
	\label{def:irreducible_representation}
	Let $\vsX$ be a vector space endowed with a group action $(\Glact)$ of a symmetry group $\G$. A representation $\rep[\vsX]$ of $\G$ on $\vsX$ is said to be \emph{irreducible} if it cannot be decomposed into smaller representations acting on proper $\G$-stable subspaces (\cref{def:G_stable_subspace}). That is, the only $\G$-stable subspaces $\vsX' \subseteq \vsX$ are $\vsX' = \{\bm{0}\}$ and $\vsX' = \vsX$ itself.

	To differentiate irreps from decomposable representations we will denote the formers and their associated irreducible $\G$-stable spaces with an over bar: $\irrep[\bar{\vsV}] : \G \to \GLGroup(\bar{\vsV})$.
\end{definition}

Crucially, any compact symmetry group $\G$ has a unique set of countably many irreps, denoted by $\{\irrep[\isoIdx]\}_{\isoIdx \in [1, \isoNum]}$, where $\isoIdx$ denotes the irrep type and $\isoNum \leq |\G|$ denotes the number of unique irreps of $\G$. A fundamental property of these irreps is that any two non-equivalent irreps act on vector spaces that are mutually \textbf{orthogonal}. This implies that whenever we decompose symmetric vector spaces into their irreducible subspaces we are inherently decomposing the space into orthogonal $\G$-stable subspaces, which will greatly simplify numerical and theoretical analyses. Formally, these orthogonality relations are a consequence of Schur's lemma, which we state below in its original form for the case of complex irreps and discuss its adaptation to real irreps.

\begin{lemma}[Schur's Lemma for unitary (complex) representations {\citep[Prop 1.5]{Knapp1986}}]
	\label{lemma:schursLemma}
	Consider two \emph{complex} Hilbert spaces, $\vsH$ and $\vsH'$, endowed with the (\emph{complex}) irreducible unitary representations $\irrepC{} : \G \mapsto \UG[\vsH]$ and $\irrepC'{} : \G \mapsto \UG[\vsH']$, respectively. Let $\mT: \vsH \mapsto \vsH'$ be a linear map commuting with the group actions, such that $\mT \in \homomorphism[\G]{\vsH}{\vsH'}$. Then, if the irreducible representations are not equivalent, i.e., $\irrepC{} \nsim \irrepC'{}$, then $\mT$ is the trivial (or zero) map. Conversely, if $\irrepC{} \sim \irrepC'{}$, then $\mT$ is a constant multiple of an isomorphism. Denoting $\mI$ as the identity operator, this can be expressed as:
	\begin{subequations}
		\begin{align}
			\irrepC{} \nsim \irrepC'{} & \iff &  & \bm{0}_{\vsH'} = \mT \vh \mid\forall \; \vh \in \vsH
			\label{eq:schursLemma_non_equivalent}
			\\
			\irrepC{} \sim \irrepC'{}  & \iff &  & \mT = \alpha \mU
			,
			\alpha \in \C,
			\mU \cdot \mU^{H} = \mI
			\label{eq:schursLemma_equivalent}
			\\
			\irrepC{} = \irrepC'{}     & \iff &  & \mT = \alpha \mI, \; \alpha \in \C
			\label{eq:schursLemma_identical_basis}
		\end{align}
	\end{subequations}
\end{lemma}

The most common interpretation of Schur's lemma is that whenever the irreps are equivalent, $\irrepC{} \sim \irrepC'{}$, their associated spaces are isomorphic, $\vsH \sim \vsH'$ and $\mT$ is an element of the endomorphism space $\EndomorphismC[\G]{\bar{\vsH}}$, with $\bar{\vsH} \sim \vsH \sim \vsH'$ (see \cref{def:endomorphism}). Consequently, \eqref{eq:schursLemma_equivalent} implies that the endomorphism space is one-dimensional, i.e., $\dim(\EndomorphismC[\G]{\bar{\vsH}}) = 1$, with $\alpha \in \C$ denoting the only degree of freedom, and \eqref{eq:schursLemma_identical_basis} denotes the scenario in which the basis sets for the two spaces are identical.

However, this result holds only for \emph{complex} irreducible representations and requires adaptation for the case of real irreducible representations. The main difference stems from the fact that if $\irrep: \G \to \GLGroup(\bar{\vsV})$ is a real irrep, then the space of (real) endomorphisms, $\Endomorphism[\G]{\bar{\vsV}}$, is no longer one-dimensional, but rather it can be $1$, $2$, or $4$ dimensional, depending on whether $\Endomorphism[\G]{\vsV}$ is isomorphic to the real algebra ($\R = \text{span}\{1\}$), complex algebra ($\sC = \text{span}\{1, i \mid i^2 = -1\}$), or quaternionic algebra ($\sH = \text{span}\{1, i, j, k \mid i^2=j^2=k^2=-1\}$), respectively. Denoting by $\Psi: \sK \to \Endomorphism[\G]{\bar{\vsV}}$ the isomorphism of basis elements of $\sK \in \{\R, \sC, \sH\}$ with the basis elements of $\Endomorphism[\G]{\bar{\vsV}}$, we can summarize the basis sets of the three cases as follows (see \citet[Appendix C]{cesa2022program} for details):

\begin{equation}
	\label{eq:group_representation_endomorphism_basis}
	\begin{aligned}
		\R  & \sim \Endomorphism[\G]{\vsV} = \text{span}\{ \Psi(1) = \Identity_{\dim{\irrepC}} \}
		\\
		\sC & \sim \Endomorphism[\G]{\vsV}
		=
		\text{span}\{
		\Psi(1) =
		\begin{bsmallmatrix}
			\Identity_{n} & \bm{0} \\
			\bm{0} & \Identity_{n}
		\end{bsmallmatrix}
		,
		\Psi(i) =
		\begin{bsmallmatrix}
			\bm{0} & -\Identity_{n}  \\
			\Identity_{n} & \bm{0}
		\end{bsmallmatrix} \}
		\\
		\sH & \sim \Endomorphism[\G]{\vsV}
		=
		\text{span}\left\{
		\begin{aligned}
			\Psi(1) & =
			\begin{bsmallmatrix}
				\Identity_{n} & \bm{0} & \bm{0} & \bm{0} \\
				\bm{0} & \Identity_{n} & \bm{0} & \bm{0} \\
				\bm{0} & \bm{0} & \Identity_{n} & \bm{0} \\
				\bm{0} & \bm{0} & \bm{0} & \Identity_{n}
			\end{bsmallmatrix},
			\Psi(i)  =
			\begin{bsmallmatrix}
				\bm{0} & \bm{0} & -\Identity_{n} & \bm{0}  \\
				\bm{0} & \bm{0} & \bm{0} & -\Identity_{n}  \\
				\Identity_{n} & \bm{0} & \bm{0} & \bm{0}  \\
				\bm{0} & \Identity_{n} & \bm{0} & \bm{0}
			\end{bsmallmatrix}
			\\
			\Psi(j) & =
			\begin{bsmallmatrix}
				\bm{0} & -\Identity_{n} & \bm{0} & \bm{0}  \\
				\Identity_{n} & \bm{0} & \bm{0} & \bm{0}  \\
				\bm{0} & \bm{0} & \bm{0} & \Identity_{n}  \\
				\bm{0} & \bm{0} & -\Identity_{n} & \bm{0}
			\end{bsmallmatrix},
			\Psi(k)  =
			\begin{bsmallmatrix}
				\bm{0} & \bm{0} & \bm{0} & -\Identity_{n}  \\
				\bm{0} & \bm{0} & \Identity_{n} & \bm{0}  \\
				\bm{0} & -\Identity_{n} & \bm{0} & \bm{0}  \\
				\Identity_{n} & \bm{0} & \bm{0} & \bm{0}
			\end{bsmallmatrix}
		\end{aligned}
		\right\}
	\end{aligned}
\end{equation}

While this result might appear complex, its interpretation is straightforward: given a $\G$-stable irreducible space $\bar{\vsV}$, the space of linear maps from the space to itself that preserve the symmetry structure consists of \emph{linear transformations that scale all dimensions of $\bar{\vsV}$ uniformly, and possibly rotate or reflect the space}. Algebraically, this implies that any element of the algebra has a unique singular space, with a single singular value determined by the element's coefficients in the basis of \eqref{eq:group_representation_endomorphism_basis}. We summarize this result in the following proposition.

\usetikzlibrary{calc}
\begin{figure}[tb!]
	\centering
	\resizebox{0.65\linewidth}{!}{%
		\begin{tikzpicture}[x=1.0cm, y=1.0cm]
			% Left plane (original)
			\begin{scope}[shift={(-3,0)}]
				% axes and smaller unit circle
				\draw[gray!40] (-1.5,0) -- (1.5,0);
				\draw[gray!40] (0,-1.5) -- (0,1.5);
				\draw[gray!20] (0,0) circle (0.5);

				% original orbit under 120° rotation: p0, p1 = R120 p0, p2 = R120^2 p0
				\coordinate (p0) at (0.5,0);
				\coordinate (p1) at (-0.25,0.433013);
				\coordinate (p2) at (-0.25,-0.433013);

				% Draw points with colored borders
				\fill[blue!20, draw=blue, thick] (p0) circle (2.5pt);
				\fill[green!20, draw=green!70!black, thick] (p1) circle (2.5pt);
				\fill[orange!20, draw=orange!80!black, thick] (p2) circle (2.5pt);

			\end{scope}

			% Arrow with transformation label
			\draw[->, thick, black] (-1.2,0) -- (1.2,0);
			\node[above, scale=0.9] at (0,0.3) {\small$\mA = a\,\Psi(1) + b\,\Psi(i)$};

			% Right plane (transformed)
			\begin{scope}[shift={(3,0)}]
				% axes and scaled circle
				\draw[gray!40] (-1.8,0) -- (1.8,0);
				\draw[gray!40] (0,-1.8) -- (0,1.8);

				% Ghost of original circle and points (faded/transparent)
				\draw[gray!10, opacity=0.7] (0,0) circle (0.5);
				\fill[gray!10, draw=gray!30, opacity=0.7] (0.5,0) circle (2.5pt);
				\fill[gray!10, draw=gray!30, opacity=0.7] (-0.25,0.433013) circle (2.5pt);
				\fill[gray!10, draw=gray!30, opacity=0.7] (-0.25,-0.433013) circle (2.5pt);

				% transformed circle and points by A = gamma R_phi (rotation + uniform scale)
				\def\phi{30}   % rotation angle (degrees)
				\def\s{2.5}    % uniform scaling
				\draw[gray!20] (0,0) circle ({0.5*\s});

				% Transform the original points directly
				\coordinate (q0) at ({0.5*\s*cos(\phi)},{0.5*\s*sin(\phi)});
				\coordinate (q1) at ({(-0.25*cos(\phi) - 0.433013*sin(\phi))*\s},{(-0.25*sin(\phi) + 0.433013*cos(\phi))*\s});
				\coordinate (q2) at ({(-0.25*cos(\phi) + 0.433013*sin(\phi))*\s},{(-0.25*sin(\phi) - 0.433013*cos(\phi))*\s});

				% Draw transformed points with same colors
				\fill[blue!20, draw=blue, thick] (q0) circle (2.5pt);
				\fill[green!20, draw=green!70!black, thick] (q1) circle (2.5pt);
				\fill[orange!20, draw=orange!80!black, thick] (q2) circle (2.5pt);
			\end{scope}
		\end{tikzpicture}
	}
	\caption{
		Example of an endomorphism acting on a $\CyclicGroup[3]$-stable irreducible $2$D space. The irreducible representation is of complex type, with endomorphism space $\Endomorphism[{\CyclicGroup[3]}]{\R^2} \sim \mathbb{C} = \text{span}\{1, i\}$, comprising all transformations that uniformly scale and rotate/reflect the plane.
	}
	\label{fig:ex_c3-endos}
\end{figure}

\begin{proposition}[A real endomorphism has a single singular space]
	\label{prop:endomorphism_uniform_singular_values}
	Let $\G$ be a compact symmetry group,  $(\irrep, \bar{\vsV})$ be an irreducible representation and its associated $\G$-stable space, and let
	$\Endomorphism[\G]{\bar{\vsV}}$ denote the space endomorphism algebra.
	Then every $\mA \in \Endomorphism[\G]{\bar{\vsV}}$ admits an \gls{svd}
	\[
		\mA = \mU \,\gamma\,\Identity_{\irrepDim[]} \mV^\top,
	\]
	where $\gamma \in \R_{\ge 0}$ is the \emph{single} singular value, repeated with multiplicity
	$\irrepDim[] = |\irrep[]| = |\vsV_{\isoIdx}|$.
	The right singular basis $\mV$ can, without loss of generality, be taken as the canonical
	orthonormal basis of $\bar{\vsV}$, while the left singular basis is then
	$\mU = \gamma^{-1}\mA \mV$, which is an orthogonal rotation/reflection of~$\mV$.

	\begin{enumerate}
		\item[\textbf{($\R$)}] \textbf{Real case.}
			Any $\mA \in \Endomorphism[\G]{\bar{\vsV}}$ is of the form
			\[
				\mA = a\,\Psi(1) = a\,\Identity_{\irrepDim[]},
				\qquad a\in\R.
			\]
			Hence
			\[
				\mA^\top \mA = a^2 \Identity_{\irrepDim[]}, \qquad
				\sigma(\mA) = \bigl\{ |a| \bigr\}^{\times \irrepDim[]}.
			\]

		\item[\textbf{($\C$)}] \textbf{Complex case.}
			Every element can be written as
			\[
				\mA = a\,\Psi(1) + b\,\Psi(i)
				=
				\begin{bsmallmatrix}
					a\Identity_{n} & -b\Identity_{n}\\
					b\Identity_{n} &  \;\,a\Identity_{n}
				\end{bsmallmatrix},
				\qquad a,b\in\R.
			\]
			Using $\Psi(i)^\top = -\Psi(i)$ and $\Psi(i)^\top\Psi(i) = \Identity$,
			\[
				\mA^\top \mA = (a^2 + b^2)\,\Identity_{\irrepDim[]}, \qquad
				\sigma(\mA) = \bigl\{ \sqrt{a^2 + b^2} \bigr\}^{\times \irrepDim[]}.
			\]

		\item[\textbf{($\sH$)}] \textbf{Quaternionic case.}
			Each element admits the expansion
			\[
				\mA = a\,\Psi(1) + b\,\Psi(i) + c\,\Psi(j) + d\,\Psi(k),
				\qquad a,b,c,d\in\R,
			\]
			where $\Psi(i),\Psi(j),\Psi(k)$ are the quaternionic structure matrices from
			\eqref{eq:group_representation_endomorphism_basis},
			satisfying $\Psi(\alpha)^\top = -\Psi(\alpha)$,
			$\Psi(\alpha)^2 = -\Identity$, and $\Psi(\alpha)^\top \Psi(\alpha) = \Identity$,
			with the usual anti-commutation rules.
			Consequently,
			\[
				\mA^\top \mA = (a^2 + b^2 + c^2 + d^2)\,\Identity_{\irrepDim[]}, \qquad
				\sigma(\mA) = \bigl\{ \sqrt{a^2 + b^2 + c^2 + d^2} \bigr\}^{\times \irrepDim[]}.
			\]
	\end{enumerate}
\end{proposition}

As an intuitive low-dimensional example, we can consider the case of a 2D rotational irrep of the cyclic group $\CyclicGroup[3]$. The dimension of the irreducible $\G$-stable subspace is $\|bar{vsV}| = 2$, and the irreducible representation is of complex type, $\irrep: \CyclicGroup[3] \to \GLGroup(\bar{\vsV})$, with $\Endomorphism[\G]{\bar{\vsV}} \sim \sC = \text{span}\{1, i\}$ denoting the space of all rotations/reflections and uniform scaling of the plane, see \cref{fig:ex_c3-endos}.

\subsection{Maps between Symmetric Vector Spaces} \label{sec:group_equivariant_maps}

We will frequently study and use linear and non-linear maps between symmetric vector spaces. Our focus is on maps that preserve entirely or partially the group structure of the vector spaces. These types of maps can be classified as $\G$-equivariant, $\G$-invariant maps:

\begin{definition}[$\G$-equivariant and $\G$-invariant maps]
	\label{def:equivariantMaps}
	Let $\vsX$ and $\vsY$ be two vector spaces endowed with the same symmetry group $\G$, with the respective group actions $\Glact_\vsX$ and $\Glact[\vsY]$. A map $f: \vsX \mapsto \vsY$ is said to be $\G$-equivariant if it commutes with the group action, such that:
	\begin{subequations}
		\begin{equation}
			\small
			\begin{split}
				\small
				\g \Glact[\vsY] \vy = \g \Glact[\vsY] f(\vx) &= f (\g \Glact[\vsX] \vx), \quad \forall \vx \in \vsX, \g \in \G. \\
				\rep[\vsY](\g) f(\vx) &= f(\rep[\vsX](\g) \vx)
			\end{split}
			\qquad \vcenter{\hbox{$\iff$}}
			\qquad
			\vcenter{\hbox{
					\homomorphismDiag
					{\vsX}
					{\vsY}
					{\Glact[\vsX]}
					{\Glact[\vsY]}
					{f}
				}}
		\end{equation}

		A specific case of $\G$-equivariant maps are the $\G$-invariant ones, which are maps that commute with the group action and have trivial output group actions $\Glact[\vsY]$ such that $\rep[\vsY](\g) = \mI$ for all $\g \in \G$. That is:
		\begin{equation}
			\label{eq:invariant_maps}
			\small
			\begin{split}
				\small
				\vy = \g \Glact[\vsY] f(\vx) &= f(\g \Glact[\vsX] \vx), \quad \forall \vx \in \vsX, \g \in \G. \\
				\vy = \rep[\vsY](\g) f(\vx) & = f(\rep[\vsX](\g) \vx)
			\end{split}
			\qquad \vcenter{\hbox{$\iff$}}
			\qquad
			\vcenter{\hbox{
					\invariantDiag
					{\vsX}
					{\vsY}
					{\Glact[\vsX]}
					{f}
					{\Glact[\vsY]}
				}}
		\end{equation}
	\end{subequations}
\end{definition}

\paragraph{Structure of $\G$-Equivariant Linear Maps} \label{sec:group_equivariant_linear_maps}

\begin{definition}[Homomorphism and Isomorphism]
		\label{def:homomorphism_isomorphism}
	Let $\vsX$ and $\vsY$ be two vector spaces endowed with the same symmetry group $\G$, with the respective group actions $\Glact[\vsX]: \G \times \vsX \mapsto \vsX$ and $\Glact[\vsY]: \G \times \vsY \mapsto \vsY$. The spaces are said to be $\G$-homomorphic if there exists a linear map $\sA: \vsX \mapsto \vsY$ that commutes with the group action, such that
	$\g \Glact[\vsY] (\mA \vx) = \mA (\g \Glact[\vsX] \vx)$
	for all $\vx \in \vsX$. They are said to be $\G$-isomorphic if the linear map is invertible. Graphically, $\vsX$ and $\vsY$ are $\G$-homomorphic or $\G$-isomorphic if the following diagrams commute:
	\begin{equation}
		\underbrace{
			\homomorphismDiag
			{\vsX}
			{\vsY}
			{\Glact[\vsX]}
			{\Glact[\vsY]}
			{\mA}
		}_{\text{Homomorphism}}
		\quad \mA \in \homomorphism[\G]{\vsX}{\vsY}
		\qquad \text{or} \qquad
		\underbrace{
			\isomorphismDiag
			{\vsX}
			{\vsY}
			{\Glact[\vsX]}
			{\Glact[\vsY]}
			{\mA}
		}_{\text{Isomorphism}} \quad \mA \in \isomorphism[\G]{\vsX}{\vsY}.
	\end{equation}
	Here, $\homomorphism[\G]{\vsX}{\vsY}$ denotes the space of $\G$-equivariant linear maps between $\vsX$ and $\vsY$, and $\isomorphism[\G]{\vsX}{\vsY}$ denotes the space of $\G$-equivariant invertible linear maps between $\vsX$ and $\vsY$.
\end{definition}

\begin{proposition}[Structure of $\G$-homomorphisms / interwiners / $\G$-equivariant linear maps]
		\label{def:equivLinearMapsLong}
	Let $\G$ be a compact group and $\mA \in \homomorphism[\G]{\vsX}{\vsY}$ be a $\G$-equivariant linear map between two (real) $\G$-symmetric vector spaces $\vsX$ and $\vsY$, with isotypic decompostions:
	\[
		\vsX = \Oplus_{\isoIdx=1}^{\isoNum} \vsX^{\isoIdxBrace} = \Oplus_{\isoIdx=1}^{\isoNum} \Oplus_{i=1}^{\irrepMultiplicity^{x}_{\isoIdx}} \vsX^{\isoIdxBrace}_i
		\quad \text{and} \quad
		\vsY = \Oplus_{\isoIdx=1}^{\isoNum} \vsY^{\isoIdxBrace} = \Oplus_{\isoIdx=1}^{\isoNum} \Oplus_{j=1}^{\irrepMultiplicity^{y}_{\isoIdx}} \vsY^{\isoIdxBrace}_j,
	\]
	where $\isoNum$ denotes the number of isotypic subspaces, and $\irrepMultiplicity^{x}_{\isoIdx}$ and $\irrepMultiplicity^{y}_{\isoIdx}$ denote the multiplicities of the irreducible representation $\irrep[\isoIdx]: \G \to \GLGroup(\bar{\vsV}_{\isoIdx})$ in $\vsX$ and $\vsY$, respectively. Each $\vsX^{\isoIdxBrace}_i$ and $\vsY^{\isoIdxBrace}_j$ is isometrically isomorphic to $\bar{\vsV}_{\isoIdx}$ (see \cref{thm:isotypic_decomposition}). Hence, in the isotypic bases, the map $\mA$ decomposes block-diagonally into $\isoNum$ blocks corresponding to homomorphisms between isotypic subspaces of the same type, that is:
	\begin{equation*}
			\begin{aligned}
			\mA & = \Oplus_{\isoIdx=1}^{\isoNum} \mA^{\isoIdxBrace} \qquad \text{where} \quad
			\mA^{\isoIdxBrace} \in \homomorphism[\G]{\vsX^{\isoIdxBrace}}{\vsY^{\isoIdxBrace}}.
		\end{aligned}
	\end{equation*}
	Furthermore, the map $\mA^{\isoIdxBrace}$ decomposes into $\irrepMultiplicity^{x}_{\isoIdx} \times \irrepMultiplicity^{y}_{\isoIdx}$ blocks of endomorphisms of the irreducible subspace $\bar{\vsV}_{\isoIdx}$. That is:
	\begin{equation*}
			\begin{aligned}
			\mA^{\isoIdxBrace} =
			\begin{bsmallmatrix}
				\mA^{\isoIdxBrace}_{1,1} & \cdots & \mA^{\isoIdxBrace}_{1,\irrepMultiplicity^{x}_{\isoIdx}} \\
				\vdots              & \ddots & \vdots                                    \\
				\mA^{\isoIdxBrace}_{\irrepMultiplicity^{y}_{\isoIdx},1} & \cdots & \mA^{\isoIdxBrace}_{\irrepMultiplicity^{y}_{\isoIdx},\irrepMultiplicity^{x}_{\isoIdx}}
			\end{bsmallmatrix}
			\quad \text{where} \quad
			\mA^{\isoIdxBrace}_{i,j} \in \Endomorphism[\G]{\bar{\vsV}_{\isoIdx}}, \forall\; i \in [1, \irrepMultiplicity^{y}_{\isoIdx}], j \in [1, \irrepMultiplicity^{x}_{\isoIdx}].
		\end{aligned}
	\end{equation*}
	Consequently, depending of the type of irreducible representation $\sK \in \{\sR, \sC, \sH\}$, each sub-block is constrained to be in the span of the corresponding basis elements in \cref{eq:group_representation_endomorphism_basis}. Consequenly, if we denote by $\sB$ the basis set of $\sK$, we have that the map $\mA^{\isoIdxBrace}$ can be expressed in tensor product form as:
	\begin{equation}
		\label{eq:interwiner_endomorphism_blocks_tensor_product}
		\begin{aligned}
			\mA^{\isoIdxBrace} = \sum_{b \in \sB} \bm{\Theta}_{b}^{\isoIdxBrace} \otimes \Psi_\isoIdx(b)
			, \quad \text{where} \quad
			\bm{\Theta}_{b}^{\isoIdxBrace} \in \R^{\irrepMultiplicity^{y}_{\isoIdx} \times \irrepMultiplicity^{x}_{\isoIdx}},
			\Psi_\isoIdx: \sB \to \Endomorphism[\G]{\bar{\vsV}_{\isoIdx}}
		\end{aligned}
	\end{equation}
	With $[\bm{\Theta}_{b}^{\isoIdxBrace}]_{i,j} = \innerprod{\mA^{\isoIdxBrace}_{i,j}, \Psi_\isoIdx(b)}$ denoting the basis expansion coefficient of the $i$-th, $j$-th endomorphism sub-block with the basis element $\Psi_\isoIdx(b)$.
\end{proposition}

\subsection{Isotypic Decomposition and Disentangled Representations}
\label{sec:isotypic_decomposition_and_disentangled_representations}

We now have all the necessary tools to decompose symmetric vector spaces into their fundamental building blocks: irreducible $\G$-stable subspaces (\cref{def:G_stable_subspace}). This decomposition forms the theoretical foundation for disentangled representations, a concept introduced by \citet{higgins2018towards} in the context of group theory and generalized here within the framework of representation theory.

By Maschke's theorem \citep{Knapp1986}, we have that irreducible representations are the fundamental building blocks of the representations of a compact symmetry group $\G$, given that any group representation $\rep[\vsX]: \G \to \GLGroup(\vsX)$ can be decomposed into a direct sum of irreducible representations, $\rep[\vsX] \sim \Oplus_{i=1}^{n} \rep[\vsX_i]$, where each $\rep[\vsX_i]$ is isomorphic to one of the group's $\isoNum \leq |\G|$ irreducible representations (\cref{def:irreducible_representation,def:decomposable_representation}).

This decomposition will play a crucial role in facilitating numerical and theoretical analysis of operations on symmetric vector spaces. Therefore, we will frequenlt choose a convenient basis of the symmetric vector space which readily exposes this decomposition, termed isotypic basis. For the sake of generality we consider below the more general case of (finite and infinite dimensional) separable Hilbert spaces, which will enable us to extend this results to function spaces.

\begin{definition}[Isotypic Basis]
	\label{def:isotypic_basis}
	Let $\rep[\vsH]: \G \to \UG[\vsH]$ be a unitary group representation of a compact group $\G$ on a separable Hilbert Space $\vsH$. The representation is said to be defined in an \emph{isotypic basis} if it is defined by a direct sum of irreducible representations grouped by their type, that is, if:
	\begin{equation}
		\rep[\vsH]
		=\Oplus_{\isoIdx=1}^{\isoNum} \Oplus_{\isoIrrepIdx=1}^{\irrepMultiplicity_{\isoIdx}} \irrep[\isoIdx],
	\end{equation}
	where $\{\irrep[\isoIdx]: \G \to \UG[\bar{\vsH}_\isoIdx]\}_{\isoIdx=1}^{\isoNum}$ are the $\isoNum \leq |\G|$ irreducible representations of $\G$, and $\irrepMultiplicity_{\isoIdx} \leq \infty$ is the multiplicity (i.e., number of copies) of the irrep type $\isoIdx$ in the representation $\rep[\vsH]$.
	\begin{remark}
		Note that multiple isotypic bases exist for a given representation $\rep[\vsH]$, as both the irrep ordering and each irrep's multiplicity ordering can be arbitrarily permuted.
	\end{remark}
\end{definition}

The utility of an isotypic basis stems from Schur's orthogonality relations (\cref{lemma:schursLemma}), which ensure that any symmetric vector space decomposes into at most $\isoNum$ orthogonal subspaces.

\begin{theorem}[Isotypic decomposition of symmetric Hilbert spaces \citep{Knapp1986}]
	\label{thm:isotypic_decomposition}
	Let $\G$ be a compact group and $\vsH$ a separable Hilbert space with a unitary group representation $\rep[\vsH]: \G \rightarrow \UG[\vsH]$. Then we can identify $\isoNum \leq \vert \G \vert$ irreducible representations $\irrep[\isoIdx]: \G \rightarrow \UG[\bar{\vsH}_\isoIdx]$ that allow us to decompose $\vsH$ into a sum of orthogonal subspaces, denoted \emph{isotypic subspaces}: $\vsH=\bigoplus_{1\leq \isoIdx\leq \isoNum}^{\perp} \vsH^{\isoIdxBrace}$ where each $\vsH^{\isoIdxBrace}=\bigoplus_{\isoIrrepIdx=1}^{m_\isoIdx}\vsH^{\isoIdxBrace}_\isoIrrepIdx$ is the sum of at most $\irrepMultiplicity_{\isoIdx} \leq \infty$ countably many subspaces isometrically isomorphic to $\bar{\vsH}_\isoIdx$.

	\begin{remark}
		\label{remark:isotypic_decomposition}
		In practice, the isotypic decomposition of any finite-dimensional symmetric Hilbert space is obtained through a unitary or orthogonal change of basis. This change of basis can be computed numerically using Dixon's reduction method for finite-dimensional unitary (complex) representations \citep{dixon1970computing}, with minor additional logic for the orthogonal (real) case. This method is implemented in \symmlearning \citep{ordonez2026symmlearning} through the \href{https://danfoa.github.io/symmetric_learning/representation_theory.html}{\texttt{cplx\_isotypic\_decomposition}} function.

		For $\G$-equivariant \gls{nn} backbone architecture, the torch module \href{https://danfoa.github.io/symmetric_learning/nn.html}{\texttt{symm\_learning.nn.Change2DisentangledBasis}} computes the required change of basis to an isotypic basis and applies it to the backbone output representations, thereby producing the disentangled representations used in this work and explain in detail next.
	\end{remark}
\end{theorem}

\paragraph{Disentangled Representations} The concept of isotypic decomposition is intricately linked to the idea of disentangled representations, introduced by \citet{higgins2018towards} in the representation learning literature,  restated below for completeness.

\begin{definition}[Disentangled representation (\citet{higgins2018towards})]
	\label{def:disentangled_representation}
	A vector representation is called a disentangled representation with respect to a particular decomposition of a symmetry group into subgroups, if it decomposes into independent subspaces, where each subspace is affected by the action of a single subgroup, and the actions of all other subgroups leave the subspace unaffected.
\end{definition}

Note that the independent subspaces of \cref{def:disentangled_representation} refer to the orthogonal isotypic subspaces $\{\vsH^{\isoIdxBrace}\}_{\isoIdx=1}^{\isoNum}$, each of which is acted upon by a unique \emph{quotient group}\footnote{The original definition of disentangled representations refers to subgroups, but in general the quotient grup $\G^\isoIdxBrace$ need not be a \textbf{subgroup} of $\G$.} defined by the kernel of the irrep acting on that isotypic subspace:
\begin{equation}
	\label{eq:quotient_group_irrep}
	\G^\isoIdxBrace = \G / \sN_{\isoIdx}, \quad \text{where} \quad \sN_{\isoIdx} := \text{ker}(\irrep[\isoIdx]) = \{\g \in \G \mid \irrep[\isoIdx](\g) = \Identity_{\irrepDim[\isoIdx]}\}.
\end{equation}
Where each $\G^\isoIdxBrace$ is a well defined group of cosets generated by the normal subgroup $\sN_{\isoIdx}$. In practice, each $\G^\isoIdxBrace$ is isomorphic to the effective (matrix) group encoded by each irreducible representation $\irrep[\isoIdx]: \G \mapsto \UG[\bar{\vsH}_\isoIdx]$.

%% file: appendix/rep_theory_fn_spaces.tex
\section{Representation Theory of Symmetric Function Spaces}
\label{sec:representation_theory_symmetric_function_spaces}
% %

\begin{figure}[!bt]
	\centering
	\includegraphics[width=\linewidth]{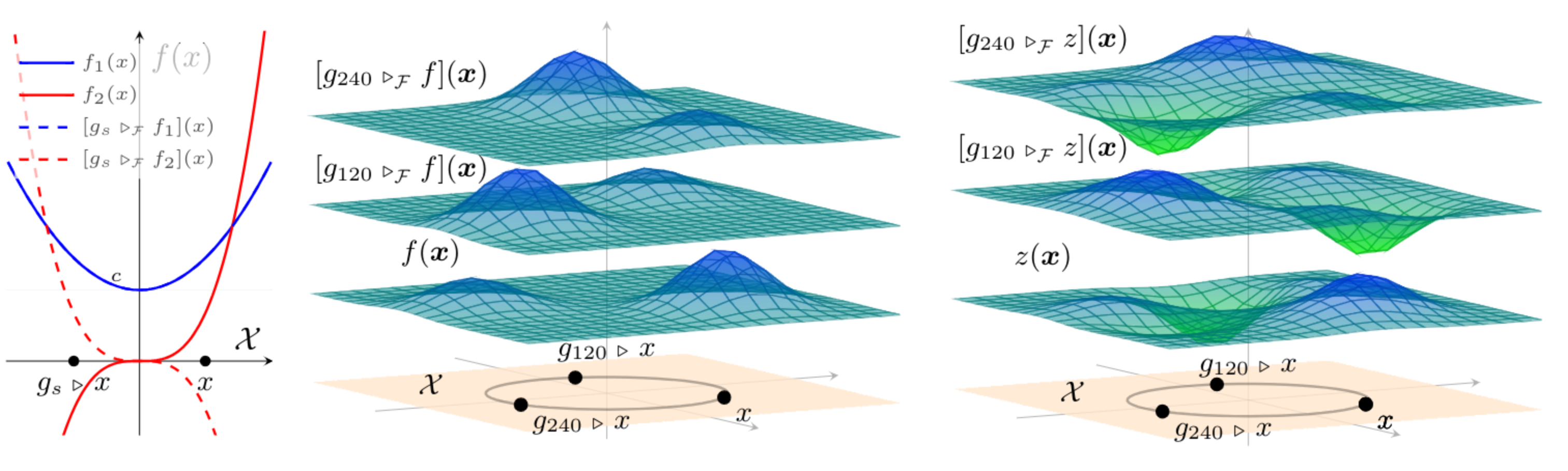}
	\caption{
		\textbf{Left:} Diagram of the group action $\Glact[\vsF]$ on functions $f_1(x) = x^2 + c$ and $f_2(x) = x^3$ defined on the domain $\vsX := \mathbb{R}$ endowed with the reflectional symmetry group $\G := \CyclicGroup[2] = \{e, \g_s\}$, with the reflection action acting on the domain by $\g_s \Glact x = -x$ and on the function space $\vsF := \{f \mid f: \vsX \mapsto \R\}$ by $[\g \Glact[\vsF] f](\vx) = f(\g \Glact[\vsX] \vx) = f(-\vx)$. Hence we have that $f_1$ is a $\G$-invariant function, $\g_s \Glact[\vsF] f_1(x) = f_1(x)$ and $f_2$ a $\G$-equivariant function $\g_s \Glact[\vsF] f_2(x) = -x^3$.
		\textbf{Center:} Diagram representing the action $\Glact[\vsF]$ on the (arbitrarily chosen) function $f(\vx) = \mathcal{N}( \vx ; \vc_1 , 2) + \mathcal{N}( \vx ; \vc_2 , 1)$ defined over the symmetric domain $\vsX = \mathbb{R}^2$ with the cyclic symmetry group $\G = \CyclicGroup[3] = \{e, \g_{120}, \g_{240}\}$ and group action $\g \Glact \vx = \rep[\vsX](\g) \vx = \mR_{\g} \vx$, where $\mR_\g$ is a rotation matrix in 2D. Here, $\g_{120} \Glact[\vsF] f$ is equivalent to evaluating $f$ on a domain rotated by $-120^\circ$. The same holds for $\g_{240} \Glact[\vsF] f$. Note that the $z$-offsets are added for visualization purposes.
		\textbf{Right:} Diagram representing the action $\Glact[\vsF]$ on the function $z \in \widehat{\vsF}$, defined to be a member of the finite-dimensional symmetric function space $\widehat{\vsF} := \text{span}(\bSet{\widehat{\vsF}})$, constructed from a basis set composed of the group orbit of the (arbitrarily chosen) function $f \in \vsF$, that is $\bSet{\widehat{\vsF}} := \G f = \{f, \g_{120} \Glact[\vsF] f, \g_{240} \Glact[\vsF] f\}$. This function space is $\G$-stable by construction, since $\G \bSet{\widehat{\vsF}} = \bSet{\widehat{\vsF}}$. Note that the $z$-offsets are added for visualization purposes.
	}
	\label{fig:symmetry_action_on_f_space}
\end{figure}

In this section, we study symmetry group actions on infinite-dimensional function spaces and specify the conditions needed to approximate these spaces in finite dimensions. Specifically, given a set $\vsX$ with a compact symmetry group $\G$ acting via $(\Glact)$ (\cref{def:left_group_action}), the space of scalar-valued functions on $\vsX$, $\vsF = \{ f \mid f: \vsX \mapsto \mathbb{R}\}$, becomes a symmetric function space. The action of a symmetry transformation on a function is defined as:
\begin{definition}[Group action on a function space]
	\label{def:action_funct_space}
	Let $\vsX$ be a set endowed with the symmetry group $\G$, and let $\vsF$ be the space of scalar-valued functions on $\vsX$. The (left) action of $\G$ on a function $\obsFn \in \vsF$ is defined as the composition of $\obsFn$ with the inverse of the group element $\g^{-1}$:
	\begin{subequations}
		\begin{align}
			\mapping
			{(\Glact[\vsF])}
			{ \G \times \vsF}{\vsF}
			{(\g,\obsFn)}
			{[\g \Glact[\vsF] \obsFn](\vx) := [\obsFn \Gcomp \g^{-1}](\vx) = \obsFn(\g^{-1} \Glact \vx),
			\quad \forall\; \vx \in \vsX.}
			\label{eq:symmetry_action_on_f_space}
		\end{align}
		In other words, the point-wise evaluation of $\obsFn$ on a $\g^{-1}$-transformed set $\vsX$ is equivalent to the evaluation of the transformed function $\g \Glact[\vsF] \obsFn \in \vsF$ on the original set $\vsX$ (see simple examples in \cref{fig:symmetry_action_on_f_space}). Any function space that is stable under the group action \cref{eq:symmetry_action_on_f_space} is refereed to as a \textit{symmetric function space}. Note that this action is compatible with the group composition and identity element $e \in \G$, such that the following properties hold:
		\begin{align}
			\text{Identity:}      & \quad e \Glact[\vsF] \obsFn(\cdot) = \obsFn(\cdot),                                                            &
			\label{eq:symmetry_action_on_f_space_identity}
			\\
			\text{Associativity:} & \quad [(\g_2 \Gcomp \g_1) \Glact[\vsF] \obsFn](\cdot) = [\g_2 \Glact[\vsF] [\g_1 \Glact[\vsF] \obsFn]](\cdot), & \quad \forall\; \g_1, \g_2 \in \G.
			\label{eq:symmetry_action_on_f_space_associativity}
		\end{align}
	\end{subequations}
\end{definition}

\begin{remark}
	From an algebraic perspective, the inversion $g^{-1}$ (\href{https://math.stackexchange.com/questions/387266/group-action-on-vector-space-of-all-functions-g-to-mathbbc}{\textit{contragredient representation}}) emerges to ensure that the associativity property of the group action (\cref{eq:symmetry_action_on_f_space_associativity}) holds:
	\begin{align}
		\small
		[(\g_2 \Gcomp \g_1) \Glact[\vsF] \obsFn](\vx) & = [\g_2 \Glact[\vsF] [\g_1 \Glact[\vsF] \obsFn]](\vx), \qquad \forall\; \vx \in \vsX
		\nonumber                                                                                                                                            \\
		\obsFn((\g_2 \Gcomp \g_1)^{-1} \Glact \vx)    & = [\g_1 \Glact[\vsF] \obsFn](\g_2^{-1} \Glact \vx) = \obsFn(\g_1^{-1} \Glact (\g_2^{-1} \Glact \vx))
		\nonumber                                                                                                                                            \\
		\obsFn((\g_2 \Gcomp \g_1)^{-1} \Glact \vx)    & = \obsFn((\g_1 \Gcomp \g_2)^{-1} \Glact \vx). \nonumber
	\end{align}
\end{remark}

In the context of this work, we will study the scenario where the function space $\vsF$ is a separable Hilbert space and the group action of $\G$ on $\vsF$ is unitary, i.e., it preserves the inner product of the function space. This setup is crucial to enable us to approximate $\vsF$ and the group action on $\vsF$ in finite dimensions.

% ______________________________________________________________________________________________
\subsection{Unitary Group Representation on Function Spaces}
\label{subsec:unitary_group_representation_on_f_space}
% ______________________________________________________________________________________________

Assume our symmetric set $\vsX$ is endowed with a measure space structure $(\vsX, \Sigma_{\vsX}, P_{\rvx})$, where $P_{\rvx} : \Sigma_{\vsX} \mapsto \R$ is the space measure. Then, consider a function space with a separable Hilbert space structure $\vsF := \Lp[2]{P_{\rvx}}{\vsX, \R}$, and inner product $\innerprod[P_{\rvx}]{\obsFn_1,\obsFn_2} = \int_{\vsX} \obsFn_1(\vx) \obsFn_2(\vx) P_{\rvx}(d\vx)$ for all $\obsFn_1, \obsFn_2 \in \vsF$. Then, the action $\Glact[\vsF]$ of the group $\G$ on the function space $\vsF$ is termed unitary if it preserves the inner product of the function space:
\begin{equation}
\small
	\begin{aligned}
		\innerprod[P_{\rvx}]{\obsFn_1,\obsFn_2}             & = \innerprod[P_{\rvx}]{\g \Glact[\vsF] \obsFn_1,\g \Glact[\vsF] \obsFn_2} \qquad \forall\; \obsFn_1, \obsFn_2 \in \vsF, \g \in \G
		\\
		\int_\vsX \obsFn_1(\vx) \obsFn_2(\vx) P_{\rvx}(d\vx) & = \int_\vsX (\g \Glact[\vsF] \obsFn_1)(\vx) (\g \Glact[\vsF] \obsFn_2)(\vx) P_{\rvx}(d\vx)
		\\
		                                                & = \int_\vsX \obsFn_1(\g^{-1} \Glact \vx) \obsFn_2(\g^{-1} \Glact \vx) P_{\rvx}(d\vx)
		\\
		                                                & = \int_{\g \Glact \vsX = \vsX} \obsFn_1(\vx) \obsFn_2(\vx) P_{\rvx}(\g \Glact d\vx).
		% \\
		                                                % & = \innerprod[P_{\rvx}]{\obsFn_1}{\obsFn_2} 
                                                        % \qquad \textit{\text{s.t. } P_{\rvx}(\g \Glact d\vx) = P_{\rvx}(d\vx).}
	\end{aligned}
	\label{eq:symmetry_action_on_f_space_unitarity}
\end{equation}
That is, the group action is unitary if $P_{\rvx}$ is a $\G$-invariant measure $P_{\rvx}(\g \Glact d\vx) = P_{\rvx}(d\vx),\; \forall \;\g \in \G, d\vx \subseteq \vsX$. Note that an $\G$-invariant measure (and inner product) exists whenever $\G$ is finite, because for any measure $\eta : \Sigma_{\Omega} \mapsto \R$, we can use the group-average trick to obtain one, given by $P_{\rvx}(\sX) = \Sigma_{\g \in \G} \eta(\g \Glact \sX)$.\footnote{Such a $\G$-invariant measure exists for any (finite or continuous) compact group. See \href{https://terrytao.wordpress.com/2011/01/23/the-peter-weyl-theorem-and-non-abelian-fourier-analysis-on-compact-groups/}{discussion}.}

The importance of the Hilbert space structure is that it enables the definition of a unitary group representation. Unitary representations have a well-studied modular structure that allows their decomposition (\cref{thm:isotypic_decomposition}) into $\G$-stable subspaces (\cref{def:G_stable_subspace}), which is crucial for approximating symmetric function spaces using a finite set of basis elements. Let $\bSet{\vsF} = \{ \bfx_i \mid \bfx_i \in \LpxFull \}_{i \in \N}$ be an orthogonal basis for the function space $\vsF = \text{span}(\bSet{\vsF})$, so that any function $\obsFn \in \vsF$ can be represented by its basis expansion coefficients $\bcoeff = [\innerprod[P_{\rvx}]{\bfx_i}{\obsFn}]_{i\in\N}$, since $\obsFn_{\bcoeff}(\vx) = \sum_{i\in \N} \innerprod[P_{\rvx}]{\bfx_i, \obsFn} \bfx_i(\vx)$. In this basis, the group action of $\G$ on $\vsF$ defines a unitary group representation mapping group elements to unitary linear integral operators on $\vsF$, which can be expressed in matrix form.

\begin{definition}[Unitary group representation on a function space]
	\label{def:group_representation_on_f_space}
	Let $\vsF = \Lp[2]{P_{\rvx}}{\vsX, \R}$ be a separable Hilbert space of scalar-valued functions on a set $\vsX$ endowed with the symmetry group $\G$. Let $\bSet{\vsF}$ be an orthogonal basis set spanning $\vsF$. Then, the group action of $\G$ on $\vsF$ (\cref{def:action_funct_space}) defines a unitary group representation mapping group elements to unitary linear integral operators on $\vsF$:
	\begin{equation}
    \small
		\mapping
		{\rep[\vsF]}
		{\G}{\UG[\vsF]}
		{\g}
		{\rep[\vsF](\g)}
		, \qquad \qquad \text{s.t. }\quad  \rep[\vsF](\g)^* = \rep[\vsF](\g^{-1}).
	\end{equation}
\end{definition}
Each unitary operator $\rep[\vsF](\g): \vsF \mapsto \vsF$ admits an infinite-dimensional matrix representation with entries $[\rep[\vsF](\g)]_{i,j} := \innerprod[P_{\rvx}]{\hat{f}_i, \g \Glact[\vsF] \hat{f}_j}$, which characterize how the group action transforms the chosen basis functions. Consequently, once the group representation for a chosen basis set is defined, the group action on a function $\obsFn_{\bcoeff} \in \vsF$ can be expressed as an (infinite-dimensional) matrix transformation of its basis expansion coefficients $\bcoeff$, given by:
\begin{equation}
	\small
	\begin{alignedat}{3}
		[\g \Glact[\vsF] \obsFn_{\bcoeff}](\cdot) &:= \sum_{i\in \N} \innerprod[P_{\rvx}]{\hat{f}_i, \g \Glact[\vsF] \obsFn_\bcoeff} \hat{f}_i(\cdot)
		=
		\sum_{i\in\N} \bigg(\sum_{j \in \N} \innerprod[P_{\rvx}]{\hat{f}_i, \g \Glact[\vsF] \hat{f}_j} \underbrace{\innerprod[P_{\rvx}]{\hat{f}_j, \obsFn}}_{\ebcoeff_j}\bigg) \hat{f}_i(\cdot).
		% \textit{
		% 	\equiv
		% 	(\rep[\vsF](\g) \bcoeff)^\top \hat{\vf}(\cdot),
		% 	\quad  \hat{\vf}(\cdot) = [\hat{f}_i(\cdot)]_{i\in\N}
		% }
	\end{alignedat}
	\label{eq:unitaty_group_representation_algebraic_form}
\end{equation}%

\begin{example}[Isotypic decomposition of symmetric function space]
	\label{ex:finite_dimensional_symmetric_function_space}
Let $(\vsX, \Sigma_\vsX, P_{\rvx})$ be a symmetric 2D measure space with domain $\vsX \sim \R^2$ and cyclic symmetry group $\G := \CyclicGroup[3] = \{e, \g_{120}, \g_{240}\}$, acting on the 2D plane by $120^\circ$ rotations (\cref{fig:iso_decomposition_example}). Define the finite-dimensional function space $\Lpx \subset \LpxFull$ with basis $\bSet{\Lpx} = \{ \bfx, \g_{120} \Glact \bfx, \g_{240} \Glact \bfx \}$, where $\bfx \in \Lpx$ is an arbitrary measurable function (\cref{fig:iso_decomposition_example}-left). In this basis, the group action $\Glact[\Lpx]$ for any function $z_\bcoeffx \in \Lpx$ is given by the regular representation $\rep[\Lpx] = \rep[\regular]$ acting on the coefficient vector $\bcoeffx \in \R^3$ (\cref{fig:example_group_action}-right).
	\begin{equation}
		\small
		[\g \Glact[\Lpx] z_\bcoeffx](\cdot) = \sum_{i=1}^3 \innerprod[P_{\rvx}]{\bfx_i, \g \Glact[\Lpx] z_\bcoeffx} \bfx_i(\cdot) \equiv (\rep[\regular](\g) \bcoeff)^\top
		\begin{bsmallmatrix}
			\bfx(\cdot) \\
			\g_{120} \Glact \bfx(\cdot) \\
			\g_{240} \Glact \bfx(\cdot)
		\end{bsmallmatrix}
		, \quad
		\rep[\regular](\g) =
		\begin{cases}
			\Identity_3,         & \text{if } \g = e       \\
			\begin{bsmallmatrix}
				0 & 1 & 0 \\
				0 & 0 & 1 \\
				1 & 0 & 0
			\end{bsmallmatrix}, & \text{if } \g = \g_{120} \\
			\begin{bsmallmatrix}
				0 & 0 & 1 \\
				1 & 0 & 0 \\
				0 & 1 & 0
			\end{bsmallmatrix}, & \text{if } \g = \g_{240}
		\end{cases}
	\end{equation}
	The group $\CyclicGroup[3]$ possesses two types of (real-valued) irreducible representations, $\isoNum = 2$: the trivial irreducible representation $\irrep[\inv]$ and a 2D rotation representation $\irrep[\sfrac{2\pi}{3}]$, defined by:
	\begin{equation}
		\small
		\irrep[\inv](\g) = \Identity_1, \forall\; \g \in \CyclicGroup[3],
		\quad \text{and} \quad
		\irrep[\sfrac{2\pi}{3}](\g) =
		\begin{bsmallmatrix}
			\cos(\theta) & -\sin(\theta) \\
			\sin(\theta) & \cos(\theta)
		\end{bsmallmatrix},
		\quad \text{s.t.  } \theta =
		\begin{cases}
			0^\circ,  & \text{if } \g = e        \\
			120^\circ & \text{if } \g = \g_{120} \\
			240^\circ & \text{if } \g = \g_{240}
		\end{cases}
	\end{equation}
Applying the appropriate change of basis, we decompose the regular representation into a direct sum of the group's irreducible representations: 
$
\rep[\regular] = \mQ (\irrep[\inv] \oplus \irrep[\sfrac{2\pi}{3}]) \mQ^{-1},
$
where $\mQ$ transitions from the regular basis to the isotypic basis of $\Lpx$. Since $\CyclicGroup[3]$ is abelian, $\mQ$ corresponds to the linear map defining the Fourier transform.

By \cref{thm:isotypic_decomposition}, this results in the orthogonal decomposition of the finite-dimensional function space into two orthogonal subspaces; $\Lpx = \Lpx^\inv \oplus^\perp \LpxIso[(2)]$, where $\Lpx^\inv$ denotes the 1-dimensional subspace of $\G$-invariant functions ($\G^{(0)} = \{e\}$), and $\LpxIso[(2)]$ is the 2-dimensional subspace with group actions defined by the 2D irreducible representation $\irrep[\sfrac{2\pi}{3}]$ ($\G^{(1)} = \CyclicGroup[3]$). We can construct the basis set in the isotypic basis given:
	\begin{equation}
		\label{eq:ex_change_to_isotypic_basis}
		\small
		\bSet{\Lpx}^\iso = \mQ \begin{bsmallmatrix}
			\bfx(\cdot) \\
			\g_{120} \Glact \bfx(\cdot) \\
			\g_{240} \Glact \bfx(\cdot)
		\end{bsmallmatrix}
		=
		\begin{bsmallmatrix}
			\lsfn^\inv(\cdot) \\
			\lsfn^{(2)}_1(\cdot) \\
			\lsfn^{(2)}_2(\cdot)
		\end{bsmallmatrix}
		\qquad \text{s.t. } \mQ =
		\begin{bsmallmatrix}
			\sfrac{1}{\sqrt{3}} & \sfrac{1}{\sqrt{3}} & \sfrac{1}{\sqrt{3}} \\
			\sfrac{2}{\sqrt{6}} & -\sfrac{1}{\sqrt{6}} & -\sfrac{1}{\sqrt{6}} \\
			0 & \sfrac{1}{\sqrt{2}} & -\sfrac{1}{\sqrt{2}}
		\end{bsmallmatrix}
	\end{equation}
	The new basis functions in the isotypic basis are depicted in \cref{fig:iso_decomposition_example}-right, and elucidate that the symmetry constraints on this $3$-dimensional function space, result in $m = 2$ unique functions, each associated with a unique irreducible representation.

Assuming $P_{\rvx}$ is a $\G$-invariant probability measure, we compute the expected value of each basis function. In the regular basis, functions related by a symmetry transformation share the same expected value, i.e., $\E_{\rvx} \bfx = \E_{\rvx} \g \Glact \bfx$ for all $\g \in \CyclicGroup[3]$. In the isotypic basis, functions lacking a $\G$-invariant component (i.e., $\lsfn^{(2)}_1, \lsfn^{(2)}_2$) are centered: $\E_{\rvx} \lsfn^{(2)}_1 = \E_{\rvx} \lsfn^{(2)}_2 = 0$. In our example this constraint becomes clear from the nature of the change of basis $\mQ$. 
\cref{eq:ex_change_to_isotypic_basis}.

\end{example}
	\begin{figure}[!bt]
		\centering
	       \includegraphics[width=.8\linewidth]{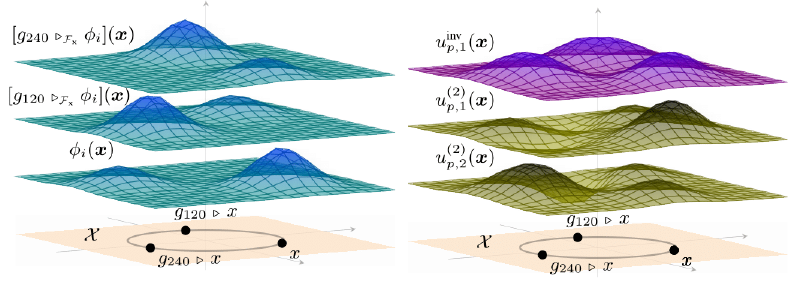}	
		\caption{
			Visualization of the basis functions in the finite-dimensional symmetric function space $\Lpx$ from \cref{ex:finite_dimensional_symmetric_function_space}.
			\textbf{Left}: Depiction of the basis functions in the regular basis $\bSet{\Lpx} = \{\bfx, \g_{120} \Glact \bfx, \g_{240} \Glact \bfx\}$, generated by the action of the cyclic group $\CyclicGroup[3]$ on an arbitrary function $\bfx \in \Lpx$.
			\textbf{Right}: Depiction of the basis functions in the isotypic basis $\bSet{\Lpx}^\iso = \{\lsfn^\inv, \lsfn^{(2)}_1, \lsfn^{(2)}_2\}$, obtained via the change of basis matrix $\mQ$.
			The first basis function $\lsfn^\inv$ corresponds to the $\G$-invariant subspace $\Lpx^\inv$ and is visually invariant under the action of $\CyclicGroup[3]$ on $\vsX$.
			The other two basis functions $\lsfn^{(2)}_1$, $\lsfn^{(2)}_2$ are constrained to span a $\G$-stable subspace of $\LpxFull$, denoted by $\LpxIso[(2)]$ that transform according to the irreducible representation $\irrep[\sfrac{2\pi}{3}]$. Meaning for any function $f \in \LpxIso[(2)]$, the group action $\g \Glact[\Lpx] f$ can be computed by a linear transformation of its basis expansion coefficients.
		}
		\label{fig:iso_decomposition_example}
	\end{figure}
% \end{example}

%% file: appendix/equivariant_operators.tex
\section[G-Equivariant Linear Integral Operators]{$\G$-Equivariant Linear Integral Operators}
\label{sec:equivariant_operators}

This section gives an overview of $\G$-equivariant linear integral operators between symmetric function spaces. We define these operators, discuss their properties, and specify conditions under which they commute with group actions. In \cref{sec:equiv_op_infinite_matrix_representation} we examine their infinite-dimensional matrix form and the resulting algebraic constraints from $\G$-equivariance. In \cref{sec:equiv_op_finite_rank_approximation} we then show how to exploit these constraints in a finite-rank approximation.

Let $\G$ be a compact group acting on two measure spaces $(\vsX, \Sigma_{\vsX}, \mux)$ and $(\vsY, \Sigma_{\vsY}, \muy)$ via the group actions $\Glact[{\vsX}]$ and $\Glact[{\vsY}]$ (see \cref{def:left_group_action}). Assume that the measures $\mux$ and $\muy$ are $\G$-invariant, i.e., $\mux(g \Glact[{\vsX}] \sB) = \mux(\sB)$ and $\muy(g \Glact[{\vsY}] \sA) = \muy(\sA)$ for all $g \in \G$, $\sB \in \Sigma_{\vsX}$, and $\sA \in \Sigma_{\vsY}$ (see \cref{def:equivariantMaps}).

Let
$
	\LpxFull = \{ f : \vsX \mapsto \R \mid \norm{f}_{\mux} < +\infty \} \quad \text{and} \quad \LpyFull = \{ h : \vsY \mapsto \R \mid \norm{h}_{\muy} < +\infty \}
$
be the Hilbert spaces of square-integrable functions with respect to $\mux$ and $\muy$, respectively. Since $\vsX$ and $\vsY$ have a $\G$-action, the spaces $\LpxFull$ and $\LpyFull$ inherit group actions defined by
$
	[\g \Glact[{\LpxFull}] f](\vx) = f(g^{-1} \Glact[{\vsX}] \vx), \quad [\g \Glact[{\LpyFull}] h](\vy) = h(g^{-1} \Glact[{\vsY}] \vy),
$
for all $f \in \LpxFull$ and $h \in \LpyFull$ (see \cref{def:action_funct_space}).

We consider linear integral operators $\oT: \LpxFull \mapsto \LpyFull$ defined by
\begin{equation}
	\small
	h(\vy) = [\oT f](\vy) = \int_{\vsX} \pmd(\vx, \vy) f(\vx) \, \mux(d\vx),
	\label{eq:linear_integral_operator}
\end{equation}
where $k: \vsX \times \vsY \mapsto \R$ is the kernel function of $\oT$. In this work we focus on those operators whose kernels are $\G$-invariant such operators are called $\G$-equivariant.

\begin{definition}[$\G$-equivariant linear intergral operators]
	\label{def:equiv_operator}
	Let $(\vsX, \Sigma_{\vsX}, \mux)$ and $(\vsY, \Sigma_{\vsY}, \muy)$ be two measure spaces endowed with group actions $\Glact[\vsX]$ and $\Glact[\vsY]$ and $\G$-invariant measures $\mux$ and $\muy$ for a given compact symmetry group $\G$. Let $\oT: \LpxFull \mapsto \LpyFull$ be a linear integral operator between the spaces of square-integrable functions defined on the two measure spaces. The operator $\oT$ is said to be $\G$-equivariant if it commutes with the group action, that is $\forall\; f \in \LpxFull, \g \in \G$ and $\vy \in \vsY$:
	\begin{subequations}
		\small
		\begin{align}
			[\oT [\g \Glact[{\LpxFull}] f]](\vy) & = [\g \Glact[{\LpyFull}] [\oT f]](\vy)
			%\qquad \forall\; f \in \LpxFull, \g \in \G, \vy \in \vsY
			\label{eq:equivariant_operator}
			\\
			\int_{\vsX} \pmd(\vx, \vy) f(g^{-1} \Glact[{\vsX}] \vx) \, \mux(d\vx)
			                                     & = \g \Glact[{\LpyFull}] \left(\int_{\vsX} \pmd(\vx, \vy) f(\vx) \, \mux(d\vx)\right)
			\nonumber
			\\
			\int_{\vsX} \pmd(\g \Glact[{\vsX}] \vx, \vy) f(\vx) \, \mux(\g \Glact[{\vsX}]d\vx)
			                                     & =
			\int_{\vsX} \pmd(\vx, \g^{-1} \Glact[{\vsY}] \vy) f(\vx) \, \mux(d\vx)
			\qquad
			\text{s.t. }
			\g \Glact[{\vsX}] \vsX := \vsX
			\nonumber
			\\
			\int_{\vsX} \pmd(\g \Glact[{\vsX}] \vx, \vy) f(\vx) \, \mux(d\vx)
			                                     & =
			\int_{\vsX} \pmd(\vx, \g^{-1} \Glact[{\vsY}] \vy) f(\vx) \, \mux(d\vx)
			\qquad
			\text{s.t. }
			\mux(\g \Glact[{\vsX}]d\vx) = \mux(d\vx)
			\nonumber                                                                                                                   \\
			k(\g \Glact[{\vsX}] \vx, \vy)        & = \pmd(\vx, \g^{-1} \Glact[{\vsY}] \vy)
			% \text{s.t. }
			\iff
			k(\g \Glact[{\vsX}] \vx, \g \Glact[{\vsY}] \vy) = \pmd(\vx, \vy).
			%\; \forall \g \in \G
			\label{eq:op_invariant_kernel}
		\end{align}
	\end{subequations}

	Notice that the $\G$-equivariance of the operator $\oT$ is linked to the $\G$-invariance of its kernel function, which is required to satisfy \cref{eq:op_invariant_kernel}.
	% \daniel{need to link the G-invariance of the kernel to the $\G$-equivariance of the operator. Vivien? }
\end{definition}

Multiple approaches exist to parameterize and approximate linear integral operators with finite
resources \citep[sec. 4]{kovachki2023neural}. Here, we assume that both the input and output function
spaces are separable Hilbert spaces, so that the operator can be represented as an infinite-dimensional
matrix once appropriate basis sets are chosen. Its finite-dimensional (truncated or finite-rank)
approximation is then obtained by selecting a finite number of basis functions in each space.

\subsection{Infinite-Dimensional Matrix Form of the Operator}
\label{sec:equiv_op_infinite_matrix_representation}
% \daniel{I am doing this step by step, because I want to retrace every step of the derivation when the orthogonality assumption is broken.}

Since $\LpxFull$ and $\LpyFull$ are Hilbert spaces with inner products $\innerprod[\mux]{\cdot, \cdot}$ and $\innerprod[\muy]{\cdot, \cdot}$ respectively, we can choose orthogonal bases for both spaces: $\bSet{\LpxFull} = \{\bfx_i \mid \bfx_i \in \LpxFull\}_{i\in\N}$ and $\bSet{\LpyFull} = \{\bfy_j \mid \bfy_j \in \LpyFull\}_{j\in\N}$. This choice allows any function $f\in\LpxFull$ and $h\in\LpyFull$ to be represented by their infinite-dimensional coefficient vectors $\bcoeffx=[\innerprod[\mux]{\bfx_i, f}]_{i\in\N}$ and $\bcoeffy=[\innerprod[\muy]{\bfy_j, h}]_{j\in\N}$, so that:
\begin{equation}
	\small
	f(\vx) := f_{\bcoeffx}(\vx)
	= \sum_{i=1}^{\infty} \innerprod[\mux]{\bfx_i, f}\,\bfx_i(\vx)
	\equiv \bcoeffx^T \vbfx(\vx)
	\qquad
	h(\vy) := h_{\bcoeffy}(\vy)
	= \sum_{j=1}^{\infty} \innerprod[\muy]{\bfy_j,h}\,\bfy_j(\vy)
	\equiv \bcoeffy^T \vbfy(\vy)
	\label{eq:fn_basis_expansion}
\end{equation}
Here, $\bcoeffx^T \vbfx(\vx)$ and $\bcoeffy^T \vbfy(\vy)$ represent the function as the dot product of its expansion coefficients with the basis evaluations $\vbfx(\vx)=[\bfx_i(\vx)]_{i\in\N}$ and $\vbfy(\vy)=[\bfy_j(\vy)]_{j\in\N}$. This notation is useful when we later select a finite number of basis functions to form a finite-dimensional approximation of $\oT$.

With the chosen bases, the action of a linear integral operator $\oT:\LpyFull\to\LpxFull$ on any $f\in\LpxFull$ is determined by its action on the basis functions:
\begin{equation}
	\small
	[\oT f_{\bcoeffx}](\vy)
	=\int_{\vsX}\pmd(\vx,\vy)\Bigl(\sum_{i\in\N}\ebcoeffx_i\, \bfx_i(\vx)\Bigr)\mux(d\vx)
	=\sum_{i\in\N}\ebcoeffx_i\int_{\vsX}\pmd(\vx,\vy)\,\bfx_i(\vx)\mux(d\vx)
	=\sum_{i\in\N}\ebcoeffx_i\,[\oT\,\bfx_i](\vy)
	\label{eq:lin_op_basis_expansion}
\end{equation}
Since $[\oT\bfx_i]\in \LpyFull$, each $[\oT\,\bfx_i](\vy)$ can be expanded using the output basis as
$
	[\oT\,\bfx_i](\vy)=\sum_{j\in\N}\innerprod[\muy]{\bfy_j,\,\oT\,\bfx_i}\,\bfy_j(\vy).
$
Thus, the operator $\oT$ can be represented by the infinite-dimensional matrix $\mT$ with entries
$
	\mT_{ij}=\innerprod[\muy]{\bfy_i,\,\oT\,\bfx_j}.
$
Therefore, the action of $\oT$ on any $f_{\bcoeffx}\in\LpxFull$ is given by the matrix–vector product $\bcoeffy=\mT\,\bcoeffx$, i.e.,
\begin{equation}
	\small
	\begin{split}
		[\oT f_{\bcoeffx}](\vy)
		&=\sum_{j\in\N}\ebcoeffx_j\,[\oT\,\bfx_j](\vy)
		=\sum_{j\in\N}\ebcoeffx_j\sum_{i\in\N}\innerprod[\muy]{\bfy_i,\,\oT\,\bfx_j}\,\bfy_i(\vy)\\[1mm]
		&=\sum_{i\in\N}\sum_{j\in\N}\mT_{ij}\,\ebcoeffx_j\,\bfy_i(\vy)
		\equiv (\mT\,\bcoeffx)^T\,\vbfy(\vy)
	\end{split}
	\label{eq:matrix_representation_operator}
\end{equation}
\Cref{eq:matrix_representation_operator} shows that knowing the action of $\oT$ on the bases $\bSet{\LpxFull}$ and $\bSet{\LpyFull}$ determines its action on any function in $\LpxFull$. In the sections that follow, we describe how symmetry constrains this action by requiring the bases to be $\G$-stable and by imposing $\G$-equivariance on $\mT$, thereby introducing exploitable algebraic constraints for improved finite-rank approximations.

\subsubsection[G-Equivariant Matrix Form of the Operator]{$\G$-Equivariant Matrix Form of the Operator}
\label{sec:op_equivariant_matrix_representation}
Whenever the function spaces carry a symmetry group $\G$, the group action on their bases $\bSet{\LpxFull}$ and $\bSet{\LpyFull}$ is defined by the unitary representations
$
	\rep[\LpxFull]: \G \to \UG{\LpxFull}
$ and $
	\rep[\LpyFull]: \G \to \UG{\LpyFull}
$
(see \cref{def:group_representation_on_f_space}). As in \cref{eq:matrix_representation_operator}, these representations can be expressed in (infinite-dimensional) matrix form so that the group action is given by a matrix-vector product:
\begin{equation}
	\label{eq:group_action_basis_xy}
	\small
	\begin{aligned}
		[\g \Glact[\LpxFull] f_\bcoeffx](\cdot ) & \equiv (\rep[\LpxFull](\g) \bcoeffx)^T \vbfx(\cdot), \quad \forall\; f_\bcoeffx \in \LpxFull, \g \in \G
		\\
		[\g \Glact[\LpyFull] h_\bcoeffy](\cdot ) & \equiv (\rep[\LpyFull](\g) \bcoeffy)^T \vbfy(\cdot), \quad \forall\; h_\bcoeffy \in \LpyFull, \g \in \G
	\end{aligned}
\end{equation}
Since the operator $\oT$ is $\G$-equivariant by construction (\cref{eq:equivariant_operator}), the matrix form $\mT$ of the operator must also be $\G$-equivariant with respect to the group representations $\rep[\LpxFull]$ and $\rep[\LpyFull]$:
\begin{equation}
	\label{eq:op_equivariance_matrix_representation_xy}
	\small
	\begin{aligned}
		[\oT [\g \Glact[{\LpxFull}] f_\bcoeffx]](\vy)
		                       & = [\g \Glact[{\LpyFull}] [\oT f_\bcoeffx]](\vy)
		                       &                                                     & \quad  \forall\; f_\bcoeffx \in \LpxFull, \g \in \G, \vy \in \vsY
		\\
		(\mT \rep[\LpxFull](\g) \bcoeffx)^\top \vbfy(\vy)
		                       & = (\rep[\LpyFull](\g) \mT \bcoeffx)^\top \vbfy(\vy)
		                       &                                                     & \quad  \text{s.t. \cref{eq:matrix_representation_operator,eq:group_action_basis_xy}}
		\\
		\mT \rep[\LpxFull](\g) & = \rep[\LpyFull](\g) \mT
	\end{aligned}
\end{equation}
With bases $\bSet{\LpxFull}$ and $\bSet{\LpyFull}$ for $\LpxFull$ and $\LpyFull$, the kernel (\cref{def:equiv_operator}) can be written as $\pmd(\vx,\vy)=\sum_{i,j\in\N}\mT_{i,j}\,\bfx_j(\vx)\,\bfy_i(\vy)$. Hence, the $\G$-invariance condition (\cref{eq:op_invariant_kernel}) on the kernel directly implies that the matrix $\mT$ is $\G$-equivariant, as stated in the following proposition:

\begin{proposition}[$\G$-invariant kernel implies $\G$-equivariant matrix form]
	\label{remark:invariant_kernel_implies_stable_basis}
	Let $\oT: \LpxFull \mapsto \LpyFull$ be a $\G$-equivariant operator between symmetric function spaces endowed with the group actions $\Glact[\LpxFull]$ and $\Glact[\LpyFull]$ of a compact symmetry group $\G$. Let $\rep[\LpxFull]$ and $\rep[\LpyFull]$ be the group representation of the on the input/output function spaces on the chosen basis sets $\bSet{\LpxFull}$ and $\bSet{\LpyFull}$. Then the $G$-invariance of the operator's kernel function (\cref{eq:op_invariant_kernel}) implies that the matrix form of the operator, in the chosen basis sets, is $\G$-equivariant w.r.t the group representations $\rep[\LpxFull]$ and $\rep[\LpyFull]$ (\cref{eq:op_equivariance_matrix_representation_xy}).

	\begin{proof}
		The proof follows by choosing appropriate $\G$-stable basis sets $\{\bfx_i\}\subset\LpxFull$ and $\{\bfy_j\}\subset\LpyFull$, so that for all $\g\in\G$ we have $\g\Glact[\LpxFull]\bfx_i=\bfx_{\g\Glact i}$ and $\g\Glact[\LpyFull]\bfy_j=\bfy_{\g\Glact j}$ with $\g\Glact i,\g\Glact j\in\N$ \footnote{By Cayley’s theorem \citep{cayley1854vii}, every finite symmetry group is isomorphic to a permutation group. From a representation-theoretic perspective, this guarantees the existence of a permutation (regular) basis for any finite group.}. This basis sets the $\G$-invariance of the kernel translates into algebraic constraints on the matrix form $\mT$.
		\begin{equation}
			\label{eq:invariant_kernel_expansion_perm_basis}
			\small
			\begin{aligned}
				k(\vx,\vy) & = \pmd(\g^{-1} \Glact[{\vsX}] \vx, \g^{-1} \Glact[{\vsY}] \vy)
				\qquad \forall\; \g \in \G, \vx \in \vsX, \vy \in \vsY
				\\
				\sum_{i\in\N} \sum_{j\in\N} \mT_{i,j} \bfx_i(\vx) \bfy_j(\vy)
				           & = \sum_{i\in\N} \sum_{j\in\N} \mT_{i,j} [\g \Glact[{\LpxFull}] \bfx_i](\vx) [\g \Glact[{\vsY}] \bfy_j](\vy)
				= \sum_{i\in\N} \sum_{j\in\N} \mT_{i,j} \bfx_{\g \Glact i}(\vx) \bfy_{\g \Glact j}(\vy)                                  \\
			\end{aligned}
		\end{equation}
		That is, the kernel is $\G$-equivariant if the operator's matrix satisfies
		$
			\mT_{i,j} = \mT_{\g \Glact i,\, \g \Glact j}
		$
		for all
		$\g\in\G,\; i,j\in\N.$
		This condition exactly characterizes the $\G$-equivariance of the matrix form.
		\begin{equation}
			\label{eq:invariant_kernel_implies_equivariant_matrix}
			\small
			\begin{aligned}
				\mT_{i,j}
				= \innerprod[\muy]{\bfy_i, \oT \bfx_j}
				 & = \innerprod[\muy]{\bfy_{\g \Glact i}, \oT \bfx_{\g \Glact j}}
				= \mT_{\g \Glact i, \g \Glact j}
				 &                                                                                      & \qquad \forall \g \in \G, i,j \in \N
				\\
				 & = \innerprod[\muy]{\g \Glact[{\LpyFull}] \bfy_i, \oT[\g \Glact[{\LpxFull}] \bfx_j]}
				\\
				 & = \innerprod[\muy]{\g \Glact[{\LpyFull}] \bfy_i, \g \Glact[{\LpyFull}] [\oT \bfx_j]}
				 &                                                                                      & \qquad \text{s.t. \cref{eq:equivariant_operator}}
				\\
				 & = \innerprod[\muy]{\bfy_i, \oT \bfx_j} = \mT_{i,j}
				 &                                                                                      & \qquad \text{s.t. \cref{eq:symmetry_action_on_f_space_unitarity}}
			\end{aligned}
		\end{equation}
	\end{proof}

\end{proposition}
% ==============================================================================
\subsubsection{Block-Diagonal Structure of the Operator Matrix Form}
\label{sec:op_block_diagonal_structure}
% ==============================================================================

According to \Cref{thm:isotypic_decomposition}, a Hilbert space with a compact symmetry group $\G$ decomposes into $\isoNum$ orthogonal subspaces—one for each irreducible representation type—yielding an orthogonal decomposition of the operator's input and output spaces:
\begin{equation}
	\label{eq:isotypic_decomposition_xy}
	\small
	\LpxFull := \Oplus^{\perp}_{1 \leq \isoIdx \leq \isoNum} \LpxFullIso,
	\qquad \text{and} \qquad
	\LpyFull := \Oplus^{\perp}_{1 \leq \isoIdx \leq \isoNum} \LpyFullIso,
\end{equation}
where $\LpxFullIso$ and $\LpyFullIso$ denote the $\isoIdx$-th isotypic subspaces of $\LpxFull$ and $\LpyFull$, respectively. Such that any function in these spaces can be decomposed into a sum of its  projections onto the isotypic subspaces:
\begin{equation}
	\small
	f(\vx) = \sum_{\isoIdx=1}^{\isoNum} f^{\isoIdxBrace}(\vx), \quad h(\vy) = \sum_{\isoIdx=1}^{\isoNum} h^{\isoIdxBrace}(\vy)
	\quad \text{with } \quad
	f^{\isoIdxBrace} \in \LpxFullIso, h^{\isoIdxBrace} \in \LpyFullIso
	\label{eq:fn_harmonic_analysis_xy}.
\end{equation}
The orthogonal decomposition of the function spaces implies there exist unitary operators $\oA: \LpxFull \rightarrow \LpxFull$ and $\oB: \LpyFull \rightarrow \LpyFull$ (with matrix forms $\mA$ and $\mB$), that describe a change of basis from the canonical basis to an \textit{isotypic basis}, $\bSet{\LpxFull}^{\iso} = \cup_{\isoIdx=1}^{\isoNum} \bSet{\LpxFullIso} = \oA \bSet{\LpxFull}$
and
$\bSet{\LpyFull}^{\iso} = \cup_{\isoIdx=1}^{\isoNum} \bSet{\LpyFullIso} = \oB \bSet{\LpyFull}$, where the group's representations decompose into a direct sum of representations per isotypic subspace (see \cref{def:equivalent_group_representations}):
\begin{equation}
	\label{eq:group_representations_isotypic_xy}
	\small
	\rep[\LpxFull]^{\iso}(\cdot)
	:= \mA \rep[\LpxFull](\cdot) \mA^*
	=  \Oplus_{\isoIdx=1}^{\isoNum} \rep[\LpxFullIso](\cdot)
	% = \begin{bsmallmatrix}
	% 	\rep[{\LpxFullIso[1]}](\cdot) &  & \\
	% 	& \ddots & \\
	% 	&  & \rep[{\LpxFullIso[\isoNum]}](\cdot)
	% \end{bsmallmatrix}
	\quad \text{and} \quad
	\rep[\LpyFull]^{\iso}(\cdot)
	:= \mB \rep[\LpyFull](\cdot) \mB^*
	= \Oplus_{\isoIdx=1}^{\isoNum} \rep[\LpyFullIso](\cdot).
	% =
	% \begin{bsmallmatrix}
	% 	\rep[{\LpyFullIso[1]}](\cdot) &  & \\
	% 	& \ddots & \\
	% 	&  & \rep[{\LpyFullIso[\isoNum]}](\cdot)
	% \end{bsmallmatrix}
\end{equation}
Then, denoting the matrix form of $\oT$ in the isotypic basis by $\mT^{\iso} = \mB^* \mT \mA$, the $\G$-equivariance of $\oT$ results in the matrix form of the operator in the isotypic basis being block-diagonal, with each block being $\G$-equivariant with respect to the group representations on the isotypic subspaces:
\begin{equation}
	\label{eq:op_equivariant_matrix_representation_block_structure}
	\small
	\begin{aligned}
		\small
		\mT^{\iso}         & = \rep[\LpyFull]^{\iso}(\g) \mT^{\iso} \rep[\LpxFull]^{\iso}(\g^{-1})
		\\     & = \Oplus_{\isoIdx=1}^{\isoNum} \rep[\LpyFullIso](\g) \mT^{\iso} \Oplus_{\isoIdx=1}^{\isoNum} \rep[\LpxFullIso](\g^{-1})
		                   &                                                                        & \text{s.t. \cref{eq:op_equivariance_matrix_representation_xy,eq:group_representations_isotypic_xy}}
		\\
		\mT^{\isoIdxBrace} & = \rep[\LpyFullIso](\g) \mT^{\isoIdxBrace} \rep[\LpxFullIso](\g^{-1}),
		\quad \forall\; \isoIdx = 1, \ldots, \isoNum
		                   &                                                                        &
		\mT^{\iso}
		= \Oplus_{\isoIdx=1}^{\isoNum} \mT^{\isoIdxBrace}
		= \begin{bsmallmatrix}
			  \mT^{(1)} &  & \\
			  & \ddots & \\
			  &  & \mT^{(\isoNum)}
		  \end{bsmallmatrix}.
		% \mT^{\isoIdxBrace} = \rep[\LpyFullIso](\g) \mT^{\isoIdxBrace} \rep[\LpxFullIso](\g^{-1})
	\end{aligned}
\end{equation}
Each $\mT^{\isoIdxBrace}$ represents the matrix form of the operator $\oTiso: \LpxFullIso \mapsto \LpyFullIso$ in the isotypic basis, acting on the isotypic subspaces of type $\isoIdx$ in the input and output spaces. This shows that $\G$-equivariant operators preserve the structure of isotypic subspaces without mixing functions from different types.

This property is crucial for the finite-rank approximation of the operator $\oT$, as it reduces the problem to approximating lower-rank operators $\oTiso: \LpxFullIso \mapsto \LpyFullIso$, for $\isoIdx \in [1, \isoNum]$. Moreover, the block diagonal structure of $\mT^\iso$ allows us to rewrite \cref{eq:matrix_representation_operator} in the isotypic basis in terms of the action of each $\oTiso$ on the projection $f^{\isoIdxBrace}$ of the function onto the $\isoIdx^{\text{th}}$ isotypic subspace, see  \eqref{eq:fn_harmonic_analysis_xy}, such that:
\begin{equation}
	\small
	\begin{aligned}
		[\oT f_{\bcoeffx}](\vy)
		= \sum_{\isoIdx=1}^{\isoNum}[\oTiso f^{\isoIdxBrace}](\vy)
		\equiv \sum_{\isoIdx=1}^{\isoNum} (\mT^{\isoIdxBrace} \bcoeffx^{\isoIdxBrace})^\top \vbfy^{\isoIdxBrace}(\vy).
		\qquad \vbfy^{\isoIdxBrace}(\cdot) = [\bfy_j^{\isoIdxBrace}(\cdot)]_{j \in \N}, \forall\; \bfy_j^{\isoIdxBrace} \in \bSet{\LpyFullIso}.
	\end{aligned}
\end{equation}
In the isotypic basis $\bSet{\LpxFull}^{\iso} = \cup_{\isoIdx=1}^{\isoNum} \bSet{\LpxFullIso}$, the expansion coefficient vector $\bcoeffx = \Oplus_{\isoIdx=1}^{\isoNum} \bcoeffx^{\isoIdxBrace}$ is formed from the projections of $f$ onto each isotypic subspace: $\bcoeffx^{\isoIdxBrace} = [\innerprod[\mux]{\bfx_i^{\isoIdxBrace}, f}]_{i \in \N}$. The block-diagonal structure of $\mT^{\iso}$ is only one of the algebraic constraints imposed on the matrix form of $\oT$ by the $\G$-equivariance condition. The next section describes the further structural constraints on each block.

% ==============================================================================
\subsubsection{Structure of Operators between Isotypic Subspaces}
\label{sec:structure_of_operators_between_isotypic_subspaces}
% ==============================================================================
In this section, we shift the focus from the input and output function spaces, $\LpxFull$ and $\LpyFull$; and the operator $\oT:\LpxFull \mapsto \LpyFull$, to their individual isotypic subspaces, $\LpxFullIso$ and $\LpyFullIso$ for $\isoIdx \in [1, \isoNum]$, and the operators $\oTiso: \LpxFullIso \mapsto \LpyFullIso$ (\cref{eq:isotypic_decomposition_xy}).

Recall from \cref{thm:isotypic_decomposition}, that each isotypic subspace possesses unitary group representations that decompose into direct sums of (infinitely many) multiplicities of the irreducible representation of type $\isoIdx$; that is:
\begin{equation}
	\small
	\rep[\LpxFullIso](\g) \sim \Oplus_{\isoIrrepIdx=1}^{\infty} \irrep[\isoIdx](\g)
	\qquad \text{and} \qquad
	\rep[\LpyFullIso](\g) \sim \Oplus_{\isoIrrepIdx=1}^{\infty} \irrep[\isoIdx](\g).
\end{equation}
This implies that each isotypic subspace further decomposes into (infinitely many) finite-dimensional $\G$-stable subspaces:
$\LpxFullIso := \Oplus_{\isoIrrepIdx=1}^{\infty} \LpxFullIso[\isoIdx,\isoIrrepIdx]$ and $\LpyFullIso := \Oplus_{\isoIrrepIdx=1}^{\infty} \LpyFullIso[\isoIdx,\isoIrrepIdx]$. Each subspace $\LpxFullIso[\isoIdx,\isoIrrepIdx]$ (and similarly $\LpyFullIso[\isoIdx,\isoIrrepIdx]$) has finite dimension ${\irrepDim[\isoIdx]}\leq \infty$ and its elements transform according to the irreducible representation $\irrep[\isoIdx]$ of the group $\G$.

The modular structure of the isotypic subspaces results in a similar structure of $\oTiso$ as it decomposes into $\G$-equivariant maps between the multiplicities of finite-dimensional, irreducible, $\G$-stable subspaces
$
	\oT^{\isoIdxBrace}_{i,j}: \LpxFullIso[\isoIdx,i] \mapsto \LpyFullIso[\isoIdx,j]
$
for $i,j \in \N$, such that, in such basis the infinite-dimensional matrix form of $\oTiso$ can be expressed as:
\begin{equation}
	\label{eq:op_kronecker_prod_structure_iso_subspace_structure_xy}
	\small
	\mT^{\isoIdxBrace} =
	\begin{bsmallmatrix}
		\mT^{\isoIdxBrace}_{1,1} & \mT^{\isoIdxBrace}_{1,2} & \cdots  \\
		\mT^{\isoIdxBrace}_{2,1} & \ddots &  \cdots \\
		\vdots & \vdots & \ddots
	\end{bsmallmatrix}
	\quad
	\text{ s.t.}\quad
	\mT^{\isoIdxBrace}_{i,j} \in \Endomorphism[\G]{\bar{\vsV}_{\isoIdx}}, \forall\; i,j \in \N,
\end{equation}
Where each $\mT^{\isoIdxBrace}_{i,j}$ is constrained to be an (real) endomorphism of the vector space $\bar{\vsV}_{\isoIdx}$ associated to the irreducible representation $\irrep[\isoIdx]: \G \to \UG[\bar{\vsV}_{\isoIdx}]$.

In practical terms, this implies that each operator $\oTiso$ is constructed from a combination of infinitetly many elements of the endomorphism space $\Endomorphism[\G]{\bar{\vsV}_{\isoIdx}}$ (\cref{def:endomorphism}). Crucially, for compact symmetry groups and real irreducible representations, the space of endomorphisms is $1$, $2$ or $4$ dimensional, and possess an analytical basis set, defined by the map $\Psi_{\isoIdx}: \sK \to \Endomorphism[\G]{\bar{\vsV}_{\isoIdx}}$, where $\sK \in \{\R,\C,\sH \}$ denotes the basis algebra associated to the irreducible representation $\irrep[\isoIdx]$ (see details in \cref{def:equivLinearMapsLong}).

Such a constraint implies that each $\mT^{\isoIdxBrace}_{i,j}$ can be expressed as a linear combination of endomorphism basis elements: $\mT^{\isoIdxBrace}_{i,j} = \sum_{b \in \sK} \alpha_{b,i,j} \Psi_{\isoIdx}(b)$. Since every element of the endormorphism basis $\{ \Psi_{\isoIdx}(b) \}_{b \in \sK}$ is a finite-dimensional orthogonal matrix, this structure can be interpreted as a constraint on the dimensionality of the singular spaces of the operator $\oTiso$ to be of dimension larger than ${\irrepDim[\isoIdx]} = |\irrep[\isoIdx]|$, as summarized in the following proposition:

\begin{proposition}[Minimum dimensionality of singular space of $\G$-equivariant operators between isotypic subspaces]
	\label{prop:op_isospaces_singular_space_dimensionality}
	Let $\oTiso: \LpxFullIso \mapsto \LpyFullIso$ be a $\G$-equivariant operator between isotypic subspaces $\LpxFullIso$ and $\LpyFullIso$ of type $\isoIdx$. Then,
	the minimum dimension of a singular space of the operator is ${\irrepDim[\isoIdx]}$.
	\begin{proof}
		The proof follows naturally from \cref{prop:endomorphism_uniform_singular_values}, which characterizes the singular spaces of each irrep endomorphism.
	\end{proof}
\end{proposition}

% ==============================================================================
\subsection[Finite-Rank Approximation of G-Equivariant Operators]{Finite-Rank Approximation of $\G$-Equivariant Operators}
\label{sec:equiv_op_finite_rank_approximation}
% ==============================================================================
In practical applications, infinite-dimensional operators are approximated by finite-dimensional ones to enable computation. For any linear integral operator $\oT: \LpxFull \mapsto \LpyFull$, the optimal rank-$r$ approximation in the Hilbert-Schmidt norm is obtained by truncating its \gls{svd} to the top $r$ singular values and associated left/right singular functions. Let $\{\sval_i\}_{i=1}^\infty$ be the singular values of $\oT$ in decreasing order and let $\{\lsfn_i\}_{i=1}^\infty \subset \LpxFull$, $\{\rsfn_i\}_{i=1}^\infty \subset \LpyFull$ be the corresponding singular functions satisfying $\innerprod[\muy]{\rsfn_i, \oT \lsfn_i} = \sval_i$ for each $i\in\N$ and $\innerprod[\muy]{\rsfn_i,\oT \lsfn_j} = 0$ when $i\neq j$. The best rank-$r$ approximation of $\oT$ is then given by \citep{eckart1936approximation}:
\begin{equation}
	\label{eq:best_rank_r_approximation}
	\small
	\lowrankOp{\oT} f = \sum_{i=1}^r \sval_i \innerprod[\mux]{\lsfn_i,f} \rsfn_i, \quad \forall f \in \LpxFull, \qquad \iff \qquad \pmd(\vx, \vy) \approx \sum_{i=1}^r \sval_i \lsfn_i(\vx) \rsfn_i(\vy).
\end{equation}
Since the left and right singular functions form orthonormal bases for $\LpyFull$ and $\LpxFull$, a rank-$r$ approximation reduces these infinite-dimensional spaces to the $r$-dimensional subspaces $\Lpx = \text{span}(\{\lsfn_i\}_{i=1}^r)$ and $\Lpy = \text{span}(\{\rsfn_i\}_{i=1}^r)$.

When $\LpxFull$ and $\LpyFull$ are symmetric function spaces with group actions $\Glact[\LpxFull]$ and $\Glact[\LpyFull]$ of a compact group $\G$, and $\oT$ is $\G$-equivariant, the finite-rank approximation $\lowrankOp{\oT}:\Lpx \to \Lpy$ must satisfy that for all $f \in \Lpx$, $h \in \Lpy$, and $\g \in \G$, both $\g \Glact[\LpxFull] f \in \Lpx$ and $\g \Glact[\LpyFull] h \in \Lpy$. This ensures that $\g \Glact[\LpyFull] [\lowrankOp{\oT} f] = \lowrankOp{\oT} [\g \Glact[\LpxFull] f]$ (see \cref{sec:representation_theory_symmetric_function_spaces}).

Moreover, since $\LpxFull$ and $\LpyFull$ decompose orthogonally into isotypic subspaces, $\LpxFull = \Oplus^{\perp}_{1 \leq \isoIdx \leq \isoNum} \LpxFullIso$ and $\LpyFull = \Oplus^{\perp}_{1 \leq \isoIdx \leq \isoNum} \LpyFullIso$, the operator $\oT$ is completely determined by the $\isoNum$ operators $\oTiso: \LpxFullIso \to \LpyFullIso$ (see \cref{sec:op_block_diagonal_structure}). Thus, the $\G$-equivariance of $\lowrankOp{\oT}$ depends on that of each finite-rank operator $\lowrankOp[r_{\isoIdx}]{\oTiso}: \LpxIso \to \LpyIso$, which requires the approximated subspaces $\LpxIso$ and $\LpyIso$ to be $\G$-stable. For simplicity, we assume $|\LpxIso| = |\LpyIso| = r_{\isoIdx}$, although this equality need not hold in general.

% ==============================================================================
\paragraph{Constraints on the Dimensionality of Truncation}
% ==============================================================================
Each approximation of an isotypic subspace $\LpxFullIso$ (and similarly $\LpyFullIso$) is $\G$-stable if the group representation is defined using a truncated multiplicity $\irrepMultiplicity_{\isoIdx} < \infty$ for the $\isoIdx^{\text{th}}$ irreducible representation, i.e. $\rep[\LpxIso] \sim \Oplus_{\isoIrrepIdx=1}^{\irrepMultiplicity_{\isoIdx}} \irrep[\isoIdx]$ and $\rep[\LpyFullIso] \sim \Oplus_{\isoIrrepIdx=1}^{\irrepMultiplicity_{\isoIdx}} \irrep[\isoIdx]$. Consequently, the dimension of the approximated subspaces is multiple of the irreducible representation's dimension: $r_{\isoIdx} = {\irrepDim[\isoIdx]}\,\irrepMultiplicity_{\isoIdx}$ (see \cref{sec:structure_of_operators_between_isotypic_subspaces}).

\paragraph{Spectral Decomposition}
\label{sec:low_rank_approximation_full_operator}
% ==============================================================================
Given the block-diagonal structure of the operator $\oT$ in the isotypic basis (\cref{eq:op_equivariant_matrix_representation_block_structure}), the truncated \gls{svd} of $\oT$ reduces to performing the truncated \gls{svd} of each per-isotypic operator $\oTiso$. Let $\Lpx \subset \LpxFull$ and $\Lpy \subset \LpyFull$ be the $\G$-stable finite-dimensional approximations of the input/output spaces of $\oT$, endowed with group representations
$\rep[\Lpx] = \Oplus_{\isoIdx=1}^{\isoNum} \rep[\LpxIso] = \Oplus_{\isoIdx=1}^{\isoNum} \Oplus_{\isoIrrepIdx=1}^{\irrepMultiplicity_{\isoIdx}} \irrep[\isoIdx]$
and
$\rep[\Lpy] = \Oplus_{\isoIdx=1}^{\isoNum} \rep[\LpyIso] = \Oplus_{\isoIdx=1}^{\isoNum} \Oplus_{\isoIrrepIdx=1}^{\irrepMultiplicity_{\isoIdx}} \irrep[\isoIdx]$. Here, $\irrepMultiplicity_{\isoIdx} \in \N$ denotes the multiplicity of the irreducible representation of type $\isoIdx$, and $\irrepDim[\isoIdx] := |\irrep[\isoIdx]|$ is its dimension. Then, the structural constraints on the \gls{svd} of the restriction of $\oT$ to these spaces are summarized in the following theorem:

\begin{theorem}[Isotypic-spectral basis]% [Structure of SVD of $\CE$]
	\label{thm:iso_svd}
	Let $\oT$ be a $\G$-equivariant operator and let $\oT_\star \colon \Lpy \to \Lpx$ be its $\G$-equivariant restriction in finite dimensions. Then, the singular value decomposition of the restricted operator matrix representation $\mT_{\star}$ reduces to:
	\begin{equation*}
		\label{eq:cond_exp_op_svd_isosingular_basis}
		\small
		\begin{aligned}
			\small
			\mT_{\star} = \Oplus_{\isoIdx=1}^{\isoNum} \mT^{\isoIdxBrace}_{\star}
			 & =
			\Oplus_{\isoIdx=1}^{\isoNum} \mW^\isoIdxBrace_{\star} \mS^\isoIdxBrace_{\star} \mM^{\isoIdxBrace\top}_{\star}
			=
			\Oplus_{\isoIdx=1}^{\isoNum} \mU^{\isoIdxBrace}_{\star} (\mSigma^{\isoIdxBrace}_{\star} \otimes \Identity_{\irrepDim[\isoIdx]}) \mV_{\star}^{\isoIdxBrace\top}
		\end{aligned}
	\end{equation*}
	Where $\Identity_{\irrepDim[\isoIdx]}$ denotes the identity matrix in $\irrepDim[\isoIdx]$-dimensions and
	$\mSigma^{\isoIdxBrace}_{\star}$ denotes the infinite-dimensional diagonal matrix singular values per irreducible $\G$-stable subspace.
\end{theorem}
\Cref{thm:iso_svd} shows that symmetries force each isotypic subspace’s singular space to have dimension at least $\irrepDim[\isoIdx]$, which is the minimum required for a faithful representation of $\G^{\isoIdxBrace}$ (see \cref{def:irreducible_representation}). Because in practice our goal is to approximate the top $r$ singular spaces of $\oT$, this result precisely characterizes the constraints imposed by $\G$-equivariance on the optimal rank-$r$ truncation’s spectral basis and corresponding kernel function in \cref{eq:pmd_model}, as summarized in the following corollary:
\begin{corollary}[Symmetry constraints on the spectral basis]
	\label{cor:iso_svd}
	Let $\oT$ be a $\G$-equivariant operator and let $\oT_\star \colon \Lpy \to \Lpx$ be its $\G$-equivariant restriction in $r$-dimensions. Then, the spectral basis of $\oT_\star$ is given by:
	\begin{equation}
		\label{eq:disentangled_spectral_basis}
		\small
		\pmd_\star(\vx,\vy) = \sum_{\isoIdx=1}^{\isoNum} \pmd^{\isoIdxBrace}_\star(\vx,\vy) = \sum_{\isoIdx=1}^{\isoNum} \sum_{s=1}^{r_\isoIdx} \sval^\isoIdxBrace_s \sum_{i=1}^{\irrepDim[\isoIdx]} \lsfn_{s,i}^\isoIdxBrace(\vx) \rsfn_{s,i}^\isoIdxBrace(\vy),
	\end{equation}
	where $\{\lsfn_{s,i}^\isoIdxBrace\}_{i\in[\irrepDim[\isoIdx]]}$ and $\{\rsfn_{s,i}^\isoIdxBrace\}_{i\in[\irrepDim[\isoIdx]]}$ are the left and right singular basis sets of the $s^{\text{th}}$ singular space of $\oTiso$. Note that the truncated dimension is restricted by the dimensionality and multiplicities of the individual irreducible representations $r = \sum_{\isoIdx=1}^{\isoDim} r_\isoIdx = \sum_{\isoIdx=1}^{\isoDim} \irrepDim[\isoIdx] \irrepMultiplicity_\isoIdx$.
\end{corollary}
\section[Relevant G-Equivariant Operators in Probability Theory]{Relevant $\G$-Equivariant Operators in Probability Theory}
\label{sec:statistical_properties_symmetric_function_spaces}
% ======================================================================================================================

In this section we study the properties of expectations and covariances of functions of symmetric random variables in the presence our assumed symmetry priors \cref{eq:main_assumtions}. In a nutshell, we characterize how expectations of observables of symmetric random variables are invariant to the group action, and that the covariance and cross-covariance matrices in these spaces are $\G$-equivariant and hence inherit rich structural constraints that can aid in empirical estimation.

Let $(\rvx, \rvy)$ be two vector-valued random variables over the probability spaces $(\vsX, \Sigma_{\vsX}, \mux)$ and $(\vsY, \Sigma_{\vsY}, \muy)$, with $\LpxFull$ and $\LpyFull$ being the corresponding square-integrable function spaces and $\one_{\mux} \in \LpxFull$, $\one_{\muy} \in \LpyFull$ the characteristic functions of sets with nonzero probability.

When $\LpxFull$ and $\LpyFull$ are symmetric function spaces (see \cref{sec:representation_theory_symmetric_function_spaces}), denote their orthogonal isotypic decompositions by $\LpxFull := \Oplus_{\isoIdx=1}^{\isoNum} \LpxFullIso$ and $\LpyFull := \Oplus_{\isoIdx=1}^{\isoNum} \LpyFullIso$ (cf. \cref{thm:isotypic_decomposition}). Any function $f \in \LpxFull$ or $h \in \LpyFull$ decomposes as $f = \sum_{\isoIdx=1}^{\isoNum} f^{\isoIdxBrace}$ and $h = \sum_{\isoIdx=1}^{\isoNum} h^{\isoIdxBrace}$ (see \cref{eq:fn_harmonic_analysis_xy}). By convention, the first isotypic subspace corresponds to the trivial group action. Thus, we write $\LpxFullIso[\inv] := \LpxFullIso[1] \subset \LpxFull$ and denote the $\G$-invariant component of $f$ by $f^{\inv} := f^{(1)}$ (and similarly for $\LpyFull$).

\subsection{The Expectation Operator}
\label{sec:expectation_operator}
% __________________________________________________________________________________________________________________________

The expected value of a function $f \in \Lpz:= \LpxFull$ can be interpreted as the result of applying a linear integral operator that projects each $f \in \Lpz$ to a constant function evaluating to the function's expected value $\E_{\mux}f$.

\begin{definition}[Expectation operator]
	\label{def:expectation_operator}
	Let $\Lpz \subseteq \LpxFull$ be a function space. The expectation operator $\oE_{\rvx}: \Lpz \mapsto \Lpz$ is a linear integral operator defined by a constant kernel function $k_{\E}(\vx, \vx') = \one_{P_{\rvx}}(\vx)\one_{P_{\rvx}}(\vx')$ for all $\vx,\vx' \in \vsX$, such that this operator maps any function $f$ to a constant function that evaluates to the function's expected value $\one_{P_{\rvx}}(\cdot) \E_{\mux} f$, that is:
	\begin{equation}
		\label{eq:expectation_op}
		\small
		[\oE_{\rvx} f](\vx') = \int_{\vsX} k_{\E}(\vx, \vx') f(\vx) \mu(d\vx)  =  \one_{P_{\rvx}}(\vx') \int_{\vsX} f(\vx) \mu(d\vx) \equiv \one_{P_{\rvx}}(\vx') \E_{\mux} f.
	\end{equation}
\end{definition}

Whenever $\Lpz$ is a symmetric function space, the operator $\oE_{\rvx}$ commutes with the group action and is $\G$-invariant (\cref{def:equivariantMaps}):

\begin{proposition}[$\G$-invariant expectation operator]
	\label{prop:G_inv_expectation_operator}
	Let $\Lpz$ be a symmetric function space with the action $\Glact[\Lpz]$ of a compact symmetry group $\G$. Then, the expectation operator commutes with the group action and is a $\G$-invariant operator $\oE_{\rvx}: \Lpz \mapsto \Lpz^\inv \subseteq \Lpz$:
	\begin{equation}
		\label{eq:G_inv_expectation_operator}
		\small
		\oE_{\rvx} [\g \Glact[\Lpz] f] = \g \Glact[\Lpz] [\oE_{\rvx} f]
		\qquad \text{and} \qquad
		\oE_{\rvx} f = \oE_{\rvx} [\g \Glact[\Lpz] f] \in \Lpz^\inv,
		\qquad \forall\; f \in \Lpz, \g \in \G.
	\end{equation}
	\begin{proof}
		The operator $\oE_{\rvx}$ commutes with the group action as its kernel function $k_{\E}$ is constant and therefore $\G$-invariant (\cref{def:equiv_operator}). Furthermore since the image of the expectation operator are constant functions, these functions belong to the subspace of $\G$-invariant functions, $\Lpz^\inv$.
	\end{proof}
\end{proposition}

As an operator that commutes with the group action, the expectation operator decomposes into $\oE_{\rvx} := \Oplus_{\isoIdx=1}^{\isoNum} \oE_{\rvx}^{\isoIdxBrace}$, where $\oE_{\rvx}^{\isoIdxBrace}: \Lpz^{\isoIdxBrace} \mapsto \Lpz^{\isoIdxBrace}$ denotes the restriction of $\oE_{\rvx}$ to the isotypic subspace $\Lpz^{\isoIdxBrace}$ (\cref{sec:op_block_diagonal_structure}). However, since the image of the operator lies in the subspace of $\G$-invariant functions, $\Im(\oE_{\rvx}) \subset \Lpz^\inv$, it follows that $\oE_{\rvx}^{\isoIdxBrace} = \bm{0}$ for every $\isoIdx \ne \inv$. Consequently, we obtain the following:

\begin{corollary}[Expectation of a function depends only on its $\G$-invariant component]
	\label{cor:expectation_invariant_component}
	For any function $f \in \Lpz$, the expectation depends only on its $\G$-invariant component:
	\begin{equation}
		\label{eq:expectation_invariant_component}
		[\oE_{\rvx} f](\cdot)
		= \sum_{\isoIdx=1}^{\isoNum} [\oE_{\rvx}^{\isoIdxBrace} f^{\isoIdxBrace}](\cdot)
		= [\oE_{\rvx}^{\inv} f^{\inv}](\cdot) := \one_\mu(\cdot) \E_\mu f^{\inv}.
	\end{equation}
\end{corollary}

\begin{corollary}[Functions without a $\G$-invariant component are centered]
	\label{cor:functions_without_invariant_component_centered}
	Any function $f = \sum_{\isoIdx=1}^{\isoNum} f^{\isoIdxBrace} \in \LpxFull$ without a $\G$-invariant compoment, i.e., $f^{\inv} = 0$, is centered:
	\begin{equation}
		\label{eq:non_invariant_components_centered}
		[\oE_{\rvx} f](\cdot) = \sum_{\isoIdx=2}^{\isoNum} [\oE_{\rvx}^{\isoIdxBrace} f^{\isoIdxBrace}](\cdot)
		=  \one_\mu(\cdot) 0,
		\qquad \iff \qquad
		\E_\mu f = 0, \quad \forall\; f \in {\LpxFullIso[\inv]}^\perp.
	\end{equation}
\end{corollary}

To better comprehend these concepts we refer the reader to \cref{ex:finite_dimensional_symmetric_function_space}.

% % __________________________________________________________________________________________________________________________
\subsection{The Cross-Covariance Operator}
\label{sec:cross_cov_operator}
% __________________________________________________________________________________________________________________________
Given two vector-valued random variables $(\rvx = [\rx_1,\dots,\rx_n], \rvy = [\ry_1,\dots,\ry_m])$ defined on the measure spaces $(\vsX, \Sigma_{\vsX}, \mux)$ and $(\vsY, \Sigma_{\vsY}, \muy)$, a key statistic assessing the linear relationship between scalar components is the covariance:
\begin{equation*}
	\small
	\Cov(\rx_i, \ry_j)
	= \E_{\muxy}[(\rx_i - \E_{\rvx}[\rx_i])(\ry_j - \E_{\rvy}[\ry_j])]
	= \E_{\muxy}[\rx_i \ry_j] - \E_{\rvx}[\rx_i]\E_{\rvy}[\ry_j].
\end{equation*}
For vector-valued random variables, the cross-covariance matrix $\Cov(\rvx, \rvy) \in \R^{n \times m}$ is defined entrywise by $\Cov(\rvx, \rvy)_{i,j} := \Cov(\rx_i, \ry_j)$. The cross-covariance operator is the extension of this concept to the Hilbert spaces of functions $\LpxFull$ and $\LpyFull$.

\begin{definition}[Cross-covariance operator \citep{fukumizu2004dimensionality}]
	\label{def:cross_covariance_operator}
	Let $\Lpx \subseteq \LpxFull$ and $\LpyFull \subseteq \LpyFull$ be two Hilbert spaces of functions defined on the random variables $\rvx$ and $\rvy$, which take values in the measure spaces $(\vsX, \Sigma_{\vsX}, \mux)$ and $(\vsY, \Sigma_{\vsY}, \muy)$, respectively. The cross-covariance operator $\oC_{\rvx\rvy}: \LpyFull \mapsto \LpxFull$ is a linear integral operator defined by
	\begin{equation}
		\label{eq:cross_covariance_operator_def}
		\small
		\innerprod[\mux]{f,\oC_{\rvx\rvy} h} := \Cov(f, h) = \E_{\muxy}[f(\rvx) h(\rvy)] - \E_{\rvx}[f(\rvx)] \E_{\rvy}[h(\rvy)], \qquad \forall\; f \in \LpxFull,\; h \in \LpyFull.
	\end{equation}
	Choosing separable basis sets for the two spaces, $\bSet{\LpxFull} = \{\bfx_i\}_{i\in\N}$ and $\bSet{\LpyFull} = \{\bfy_i\}_{i\in\N}$, the matrix representation of the cross-covariance operator has entries
	$
		[\mC_{\rvx,\rvy}]_{i,j} := \innerprod[\mux]{\bfx_i, \oC_{\rvx\rvy}\bfy_j} = \Cov(\bfx_i, \bfy_j),
	$
	where the covariance is computed with respect to the joint measure $\muxy$ and the marginals $\mux$ and $\muy$. Given a dataset of $N$ samples from the joint distribution $(\vx,\vy) \sim \muxy$, the empirical estimate of the matrix form of the cross-covariance operator is
	\begin{equation}
		\label{eq:empirical_cross_covariance_operator}
		\small
		\widehat{\mC}_{\rvx\rvy} = \frac{1}{N} \sum_{n=1}^{N} \vbfx(\vx_n) \vbfy(\vy_n)^\top -  \EE_\rvx[\vbfx(\vx_n)]\,  \EE_\rvy[\vbfy(\vy_n)]^\top, \quad \vbfx(\cdot) = [\bfx(\cdot)]_{i\in\N},\; \vbfy(\cdot) = [\bfy(\cdot)]_{i\in\N}.
	\end{equation}
\end{definition}
Note that the adjoint of the operator is defined by $\oC_{\rvx\rvy}^* = \oC_{\rvy\rvx} : \LpxFull \mapsto \LpyFull$. In the case $\LpxFull = \LpyFull$, the cross‐covariance operator reduces to the covariance operator, and has an analog definition to \cref{def:cross_covariance_operator}.
% \begin{definition}[Covariance operator \citep{fukumizu2004dimensionality}]
% 	\label{def:covariance_operator}
% 	Let $\LpxFull \subseteq \LpxFull$ be a Hilbert space of functions of the random variable $\rvx$ on the measure space $(\vsX, \Sigma_{\vsX}, \mux)$. The covariance operator $\oC_{\rvx}: \LpxFull \mapsto \LpxFull$ is a linear integral operator defined such that:
% 	\begin{equation}
% 		\label{eq:covariance_operator_def}
% 		\small
% 		\innerprod[\mux]{f,\oC_{\rvx} f'} := \Cov(f, f') = \E_{\rvx}[f(\rvx)f'(\rvx)] - \E_{\rvx}[f(\rvx)]\,\E_{\rvx}[f'(\rvx)], \qquad \forall\; f,f' \in \LpxFull.
% 	\end{equation}
% 	By choosing a separable basis for the function space, $\bSet{\LpxFull} = \{\bfx_i\}_{i\in\N}$, the matrix form of the covariance operator has entries 
% 	$[\mC_{\rvx}]_{i,j} := \innerprod[\mux]{\bfx_i, \oC_{\rvx}\bfx_j} = \Cov(\bfx_i, \bfx_j),$
% 	with respect to the measure $\mux$. The empirical estimate of the matrix form is obtained similarly to \cref{eq:empirical_cross_covariance_operator}.
% \end{definition}

\paragraph{Covariance and Cross-Covariance Operators of Symmetric Hilbert Spaces of Functions}
Whenever $\LpxFull$ and $\LpyFull$ are symmetric function spaces, and the joint probability measure is $\G$-invaraint, the cross-covariance operator $\oC_{\rvx\rvy}$ commute with the group action and is $\G$-equivariant (\cref{sec:group_equivariant_linear_maps}):

\begin{proposition}[$\G$-equivariant cross-covariance operator]
	\label{prop:G_equivariant_cross_covariance_operator}
	Let $\LpxFull \subseteq \LpxFull$ and $\LpyFull \subseteq \LpyFull$  be symmetric Hilbert spaces of functions endowed with the group actions $\Glact[\LpxFull]$ and $\Glact[\LpyFull]$ of a compact symmetry group $\G$. Then, whenever the joint probability measure is $\G$-invariant, i.e., $\muxy(\sB, \sA) = \muxy(\g \Glact[\vsX] \sB, \g \Glact[\vsY] \sA)$ for all $\g \in \G, \sB \in \Sigma_\vsX, \sA \in \Sigma_Y$, the cross-covariance operator $\oC_{\rvx\rvy}: \LpyFull \mapsto \LpxFull$ (\cref{def:cross_covariance_operator}) commutes with the group actions and is a $\G$-equivariant operator (\cref{def:equiv_operator}):
	\begin{equation}
		\label{eq:G_equivariant_cross_covariance_operator}
		\small
		\g \Glact[\LpxFull] [\oC_{\rvx\rvy} h] = \oC_{\rvx\rvy} [\g \Glact[\LpyFull] h], \qquad \forall\; h \in \LpyFull, \g \in \G.
	\end{equation}
	\begin{proof}
		To proof that the operator is $\G$-equivariant we must show its kernel function is $\G$-invariant (see \cref{def:equiv_operator}). The proof follows naturally in any regular basis of the input and output functions spaces $\bSet{\LpxFull} = \{\bfx_i\}_{i\in\N}$ and $\bSet{\LpyFull} = \{\bfy_i \}_{i\in \N}$, in which the group action on basis functions acts by permutations of basis functions, such that, $\g \Glact[\LpxFull] \bfx_i \equiv \bfx_{\g \Glact i} \in \bSet{\LpxFull}$ and $\g \Glact[\LpyFull] \bfy_j \equiv \bfy_{\g \Glact j} \in \bSet{\LpyFull}$, where $\g \Glact i, \g \Glact j \in \N$. Then we must show that that:
		\begin{equation}
			\small
			\begin{aligned}
				k(\vx,\vy) & = \pmd(\g^{-1} \Glact[{\vsX}] \vx, \g^{-1} \Glact[{\vsY}] \vy)
				\qquad     &                                                                                                                           & \forall\; \g \in \G, \vx \in \vsX, \vy \in \vsY
				\\
				\sum_{i\in\N} \sum_{j\in\N} [\mC_{\rvx,\rvy}]_{i,j} \bfx_i(\vx) \bfy_j(\vy)
				           & = \sum_{i\in\N} \sum_{j\in\N} [\mC_{\rvx,\rvy}]_{i,j} [\g \Glact[{\LpxFull}] \bfx_i](\vx) [\g \Glact[{\vsY}] \bfy_j](\vy)
				\quad      &                                                                                                                           & \text{s.t. \cref{def:action_funct_space,def:cross_covariance_operator}}
				\\
				\sum_{i\in\N} \sum_{j\in\N} \Cov(\bfx_i, \bfy_j) \bfx_i(\vx) \bfy_j(\vy)
				           & = \sum_{i\in\N} \sum_{j\in\N} \Cov(\bfx_i, \bfy_j) \bfx_{\g \Glact i}(\vx) \bfy_{\g \Glact j}(\vy).
				\quad      &                                                                                                                           &
				\\
			\end{aligned}
		\end{equation}
		Hence, the cross-covariance operator's kernel function is $\G$-invariant only if the covariance is $\G$-invariant:
		\begin{equation}
			\small
			\begin{aligned}
				\Cov(\bfx_i, \bfy_j)                                                     & = \Cov(\g \Glact[{\LpxFull}] \bfx_i, \g \Glact[{\vsY}] \bfy_j)
				                                                                         &                                                                                                                          & \forall\; \g \in \G, i,j \in \N
				\\
				\E_{\muxy}[ \bfx_i(\rvx) \bfy_j(\rvy)]
				                                                                         &
				= \E_{\muxy}[\bfx_i(\g^{-1} \Glact[\vsX] \rvx) \bfy_{j}(\g^{-1} \Glact[\vsY] \rvy)]
				\quad                                                                    &                                                                                                                          & \E_\mu f = \E_\mu \g \Glact f
				\\
				\int_{\vsX \times \vsY} \bfx_i(\rvx) \bfy_j(\rvy) \, \muxy(d\rvx, d\rvy) & = \int_{\vsX \times \vsY} \bfx_{i}(\g^{-1} \Glact[\vsX] \rvx) \bfy_{j}(\g^{-1} \Glact[\vsY] \rvy) \, \muxy(d\rvx, d\rvy)
				                                                                         &                                                                                                                          &
				\\
				                                                                         & = \int_{\vsX \times \vsY} \bfx_i(\rvx) \bfy_j(\rvy) \, \muxy(\g \Glact d\rvx, \g \Glact d\rvy)                                                             \\
				                                                                         & = \Cov(\bfx_i, \bfy_j).
				%  \text{s.t. $\muxy$ being $\G$-invariant}
			\end{aligned}
		\end{equation}
	\end{proof}
\end{proposition}
An equivalent result follows for covariance operators of symmetric Hilbert spaces.

%% file: appendix/bounds.tex
\section{Statistical Learning Theory}
\label{sec:statistical_learning_theory}

This section provides the development and proofs of the statistical learning guarantees in \cref{thm:main} for regression and conditional probability estimation using our proposed model.

Recall that regression and conditional probabilities can be expressed in terms of the conditional expectation operator $\CE \colon \LpyFull \to \LpxFull$ (see \cref{eq:op_form_cond_mean_reg,eq:uses_of_conditional_expectation_operator}). Given that the operator is compact \citep{kostic2024neural}, it admits a singular value decomposition. Hence, the kernel function defining the operator \cref{eq:op_form_cond_mean_reg} can be expanded in terms of the operator spectral basis:
\begin{equation}
	\label{eq:svd_CD}
	\small
	\pmd(\vx,\vy)
	:= \frac{ dP_{\rvx\rvy}(\vx,\vy)}{d(\mux(\vx) \times \muy(\vy))}
	%= 	\frac{P_{\rvx\rvy}(d\vx,d\vy)}{(\mux(d\vx) \times \muy(d\vy))}
	=\sum_{i=0}^\infty \sval_i \lsfn_i(\vx) \rsfn_i(\vy).
	% \quad\text{ and }\quad\DCE = \sum_{i=1}^\infty \sval_i \,\lsfun_i\otimes\rsfn_i.
\end{equation}
Where $(\sval_i)_{i\in\N}$ denotes the operator's singular values, and $(\lsfn_i)_{i\in\N}$ and $(\rsfn_i)_{i\in\N}$ denote the left and right singular functions, which form complete orthonormal basis sets for $\LpxFull$ and $\LpyFull$, respectively.
Given that the operator's first singular value is $\sval_0=1$, associated with the constant functions $\lsfn_0=\one_{\vsX}$, $\rsfn_0=\one_{\vsY}$, the conditional expectation operator can be defined as:
\begin{equation}
	\label{eq:def_CE-DCE}
	\CE
	=
	\sum_{i=1}^\infty \sval_i \lsfn_i \innerprod[\muy]{\rsfn_i, \cdot}
	% = 
	% \one_{\vsX}\innerprod[\muy]{\one_{\vsY}}{\cdot} + \DCE 
	=
	\one_{\vsX}\innerprod[\muy]{\one_{\vsY},\cdot} + \underbrace{\sum_{i=1}^\infty \sval_i \lsfn_i \innerprod[\muy]{\rsfn_i,\cdot}}_{\DCE}.
	% \quad\text{ and }\quad\DCE = \sum_{i=1}^\infty \sval_i \,\lsfun_i\otimes\rsfn_i,
\end{equation}
Where $\DCE$ denotes the \textit{deflated} operator, excluding the first eigen triplet $(\sval_0,\lsfn_0,\rsfn_0)$. Leveraging the SVD of $\CE$, we approximate the operator's action for any $h \in \LpyFull$ using a rank-$r$ ($1 < r < \infty$) operator given by:
\begin{equation}
	\label{eq:pointwisecondExp}
	\begin{aligned}
		\E[h(\rvy) \vert \rvx\!=\!\vx]
		 & =
		[\CE h](\vx)
		\approx
		\E[h(\rvy)]
		+
		\sum_{i=1}^r \sval_i \lsfnp_i(\vx) \E[\rsfnp_i(\rvy)h(\rvy)],
		\\
		 & \text{s.t.} \; \E[\lsfnp_i(\rvx)]=\E[\rsfnp_i(\rvy)]=0, \forall\;i \geq 1.
	\end{aligned}
\end{equation}
Where $(\lsfnp_i)_{i=1}^{r}$ and $(\rsfnp_i)_{i=1}^{r}$ denote parametrizations of the top-$r$ left and right singular functions. Given that the operator's kernel \cref{eq:svd_CD} preserves the probability mass, that is $\int_{\vsX\times\vsY} \pmd(\vx,\vy) d\mux(\vx) d\muy(\vy) = 1$, every non-constant singular function is constrained to be centered, as described in the r.h.s of \cref{eq:pointwisecondExp}.

In the context of symmetries, we note that $\DCE$ admits a block-diagonal structure w.r.t. to isotypic basis of associated $\Lp[2]{}{}$ spaces. Indeed we have the following from \cref{thm:iso_svd}.
\begin{equation}
	\label{eq:iso_svd}
	\oQ_{\rvx}^*\DCE\oQ_{\rvy} = \Oplus_{\isoIdx=1}^{\isoNum} \oQ_{\rvx}^{\isoIdxBrace*}\isoDCE\oQ_{\rvy}^{\isoIdxBrace}
	=
	\Oplus_{\isoIdx=1}^{\isoNum} \big[(\oU^\isoIdxBrace \oS^\isoIdxBrace \oV^{\isoIdxBrace*}) \otimes \Identity_{\irrepDim[\isoIdx]}\big].
\end{equation}
Where the unitary operators $\oQ_{\rvx} \colon \LpxFull \to \LpxFull$ and $\oQ_{\rvy} \colon \LpyFull \to \LpyFull$ change the basis to the isotypic decompositions $\bSet{\LpxFull}=\{\bfx^{\isoIdxBrace}_{i,j}\}_{\isoIdx\in[\isoNum],\, i\in[\irrepMultiplicity_\isoIdx],\, j\in[\irrepDim[\isoIdx]]}$ and $\bSet{\LpyFull}=\{\bfy^{\isoIdxBrace}_{i,j}\}_{\isoIdx\in[\isoNum],\, i\in[\irrepMultiplicity_\isoIdx],\, j\in[\irrepDim[\isoIdx]]}$, with $i$ indexing each irreducible $\G$-stable subspace and $j$ indexing the dimensions within that subspace (see \cref{sec:low_rank_approximation_full_operator}).

Further, by \cref{thm:iso_svd}, the \gls{svd} of $\DCE$ forces each isotypic subspace to have dimension at least $\irrepDim[\isoIdx]=\irrep[\isoIdx]$ for every $\isoIdx\in[\isoNum]$.
\begin{equation}
	\label{eq:iso_svd_k}
	\oQ_{\rvx}^{\isoIdxBrace*}\isoDCE\oQ_{\rvy}^{\isoIdxBrace}
	= \big[\oU^\isoIdxBrace \otimes \Identity_{\irrepDim[\isoIdx]}\big] \big[ \oS^\isoIdxBrace \otimes \Identity_{\irrepDim[\isoIdx]}\big] \big[ \oV^{\isoIdxBrace} \otimes \Identity_{\irrepDim[\isoIdx]}\big]^*,\,\isoIdx\in[\isoNum],
\end{equation}
where $\oQ_{\rvx}^{\isoIdxBrace}\oQ_{\rvx}^{\isoIdxBrace*}$ and $\oQ_{\rvy}^{\isoIdxBrace}\oQ_{\rvy}^{\isoIdxBrace*}$ are orthogonal projectors on $\isoIdx$-th isotypic subspace, and
\begin{equation}
	\label{eq:iso_svd_bis}
	\oQ_{\rvx}^*\DCE\oQ_{\rvy}
	= \big[\Identity_{\isoNum}\otimes\oU^\isoIdxBrace \otimes \Identity_{\irrepDim[\isoIdx]}\big] \big[ I_{\isoNum}\otimes \oS^\isoIdxBrace \otimes \Identity_{\irrepDim[\isoIdx]}\big] \big[ I_{\isoNum}\otimes \oV^{\isoIdxBrace} \otimes \Identity_{\irrepDim[\isoIdx]}\big]^*.
\end{equation}
Further, observe that the singular values of $\DCE$ are elements of positive diagonal operators $\oS^\isoIdxBrace$, denoted as $(\oS^\isoIdxBrace)_i = \sval_{i}^{\isoIdxBrace}$,
while the left and right singular functions are $\lsfnIso_{i} \otimes \ve^{\irrepDim[\isoIdx]}_{j}$ and $\rsfnIso_{i} \otimes \ve^{\irrepDim[\isoIdx]}_{j}$, respectively, for $i\in\N$, $j\in[\irrepDim[\isoIdx]]$ and $\isoIdx\in[\isoNum]$, where $\ve^{d}_{j}$ is $j$-th vector of standard basis of $\R^{d}$.

Given the constraints on the spectral basis of $\G$-equivariant operators (see \cref{cor:iso_svd}), our representation learning procedure approach results in feature maps:
\begin{equation}
	\label{eq:equi_fmaps}
	\begin{aligned}
		\small
		\vlsfnp (\cdot)
		 & =
		\sum_{\isoIdx\in[\isoNum], i\in[\irrepMultiplicity], j\in[\irrepDim[\isoIdx]]}[\ve^{\isoNum}_{\isoIdx}\otimes \ve^{\irrepMultiplicity}_{i} \otimes \ve^{\irrepDim[\isoIdx]}_{j}]\, \lsfnpIso_{i,j}(\cdot)\colon\vsX\to\R^{r_\irrepMultiplicity}
		\\
		\vrsfnp (\cdot)
		 & =
		\sum_{\isoIdx\in[\isoNum], i\in[\irrepMultiplicity], j\in[\irrepDim[\isoIdx]]} [\ve^{\isoNum}_{\isoIdx}\otimes \ve^{\irrepMultiplicity}_{i} \otimes \ve^{\irrepDim[\isoIdx]}_{j}]\,
		\rsfnpIso_{i,j}(\cdot) \colon\vsX\to\R^{r_\irrepMultiplicity},
	\end{aligned}
\end{equation}
% \daniel{Where do we use this tensor notation in the development}
% 
which can further be separated into $\isoNum$ orthogonal blocks $\vlsfnpIso=\sum_{i\in[\irrepMultiplicity],j\in[\irrepDim[\isoIdx]]}\bfx^{\nnParams\isoIdxBrace}_{i,j}$ and $\vbfy_{\nnParams}^{\isoIdxBrace}= \sum_{i\in[\irrepMultiplicity],j\in[\irrepDim[\isoIdx]]}\bfy^{\nnParams\isoIdxBrace}_{i,j}$ as
\begin{equation}\label{eq:equi_fmaps_iso}
	\vlsfnpIso
	=
	\sum_{i\in[\irrepMultiplicity],j\in[\irrepDim[\isoIdx]]}[
		\ve^{\irrepMultiplicity}_{i} \otimes \ve^{\irrepDim[\isoIdx]}_{j}
	]\,\lsfnpIso_{i,j}(\cdot)
	\quad\text{ and }\quad
	\vrsfnpIso
	=
	\sum_{i\in[\irrepMultiplicity],j\in[\irrepDim[\isoIdx]]}[\ve^{\irrepMultiplicity}_{i} \otimes \ve^{\irrepDim[\isoIdx]}_{j}]\,\rsfnpIso_{i,j}(\cdot).
\end{equation}
In addition, the singular value matrices have a tensor form
$\Sigmap=diag(\Sigmap^{(1)},\ldots,\Sigmap^{(\isoNum)})$,
where $\SigmapIso=diag(\svalpIso_1,\ldots,\svalpIso_{\irrepMultiplicity})\otimes\Identity_{\irrepDim[\isoIdx]}$
and
$\svalpIso_{i}\in[0,1]\;, i\in[\irrepMultiplicity],\isoIdx\in[\isoNum]$.
Thus, we obtain the operator $\PDCE = \CEp - \one_{\mux} \otimes \one_{\muy}$ in block form,
$\PDCE = \Oplus_{\isoIdx\in[\isoNum]}\PDCEIso$, where each $\PDCEIso$ acts on the $\isoIdx$-th isotypic subspace as
\begin{equation}\label{eq:dnn_model_iso}
	[\PDCEIso f](\vx):= \vlsfnpIso(\vx)^\top \SigmapIso\, \E_{\rvy}[\vrsfnpIso(\rvy) f^{\isoIdxBrace}(\rvy)], \quad f\in\LpyFull,
\end{equation}
and hence
\begin{equation}\label{eq:dnn_model}
	[\PDCE f](\vx):= \vlsfnp(\vx)^\top \Sigmap\, \E_{\rvy}[\vrsfnp(\rvy) f(\rvy)],\quad f\in\LpyFull.
\end{equation}
Finally, we extend the definition of $\PDCE$ to vector-valued observables $\vh \colon \vsY \to \vsZ$ via basis expansions.
\begin{equation}\label{eq:dnn_model_vec_valued}
	[\PDCE \vh](\vx):= \sum_{\ell}
	\vlsfnp(\vx)^\top \Sigmap\, \E_{\rvy}[\vrsfnp(\rvy) (\scalarp{\vh(\rvy),\vz_\ell}_{\vsZ} \vz_\ell)]
	,
	\quad
	\vh	\in \LpyFull(\vsY, \vsZ)
\end{equation}
where $(\vz_i)_{i\in[n_{\vsZ}]}$ is the orthonormal basis of $\vsZ$.

By doing so, we ensure that $\PDCE$ and, consequently, $\PCE$ are $\G$-equivariant operators for both the scalar map $\LpyFull\to\LpxFull$ and the vector-valued map $\LpyFull(\vsY,\vsZ)\to\LpxFull(\vsX,\vsZ)$. Moreover, a direct consequence of \eqref{eq:dnn_model_vec_valued} is as follows.

\begin{proposition}\label{prop:equivaraint_model}
	Let with $\vsZ$ being a real Euclidean space endowed with symmetry group $\G$, and let $
		\CEp \colon \Lp[2]{\muy}(\vsY,\vsZ) \mapsto \Lp[2]{\mux}(\vsX,\vsZ)
	$ be given by $\CEp \vf = \E_{\rvy}[\vf(\rvy)] + \PDCE \vf$.  Then for every $\G$-equivariant $\vf\in \Lp[2]{\muy}(\vsY,\vsZ)$
	%, it holds that 
	%\begin{equation}\label{eq:model_equivariance}
	%\PCE [g\Glact[{\LpyFull(\vsY,\vsZ)}] f] = \E_{\rvy}[f] + \PDCE [g\Glact[{\LpyFull(\vsY,\vsZ)}] f] = \E_{\rvy}[f] + g\Glact[{\Lp[2]{\mux}(\vsX,\vsZ)}] \PDCE f = g\Glact[{\Lp[2]{\mux}(\vsX,\vsZ)}] \PCE f,  
	%\end{equation}
	and every $\vx\in\vsX$
	\begin{equation}\label{eq:model_equivariance}
		\small
		[\PCE \vf](g \Glact[\vsX] \vx) = \E_{\rvy}[\vf(\rvy)] + [\PDCE \vf](g \Glact[\vsX] \vx) = \E_{\rvy}[\vf(\rvy)] +  g \Glact[\vsZ] [\PDCE \vf](\vx) = g \Glact[\vsZ] [\PCE \vf](\vx).
	\end{equation}
\end{proposition}
\begin{proof}
	Since $\PDCE$ is $\G$-equivaraint, for every $g\in\G$ we have that
	\[
		\small
		[\PDCE \vh](g^{{-}1}\Glact_{\vsX} \vx) = [\PDCE [\vh(g^{{-}1}\Glact_{\vsY}\cdot)]](\vx) = \sum_{i}\vlsfnp(\vx)^\top \Sigmap\, \E_{\rvy}[\vrsfnp(\rvy) \scalarp{\vh(g^{{-}1}\Glact_{\vsY} \rvy), \vz_i}_{\vsZ} \vz_i],
	\]
	which, using $\g$ instead of $g^{{-}1}$ and the assumption that $f$ is $\G$-equivariant, implies
	\[
		\small
		\begin{aligned}
			[\PDCE \vh](g\Glact_{\vsX} \vx)
			 & \!=\!
			\sum_{i}(\vlsfnp(\vx)^\top \Sigmap \E_{\rvy}[\vrsfnp(\rvy) \scalarp{g \Glact_{\vsZ} \vh(\rvy), \vz_i}_{\vsZ} \vz_i]
			\\
			 &
			\!=\!
			\sum_{i}(\vlsfnp(\rvx)^\top \Sigmap\, \E_{\rvy}[\vrsfnp(\rvy)
				\innerprod[\vsZ]{\vh(\rvy), g^{{-}1}\Glact_{\vsZ} \vz_i}) \vz_i.
		\end{aligned}
	\]
	Thus, changing the basis to $(g^{{-}1}\Glact_{\vsZ} \vz_i)_{i\in[n_\vsZ]}$ we obtain the result when $\E_{\rvy}[\vh(\rvy)]=0$. But since $\one_{\vsX}(g\Glact_{\vsX} \vx) = 1$ for every $\vx\in\vsX$ and $g\in\G$, the same holds for $\PCE$.
\end{proof}
Recall that for the effective latent dimension $\irrepMultiplicity$ the true latent dimension is constrained by the dimensionality of the singular spaces, i.e., $\dimtot = \sum_{\isoIdx\in[\isoNum]}r_k =\sum_{\isoIdx\in[\isoNum]}\irrepMultiplicity[\isoIdx] \irrepDim[\isoIdx]$.
Further, given a measurable set $\sA\subseteq\vsX$ and collection of group elements $\G'\subseteq\G$, let us define the following symmetry index of a set $\sA$ w.r.t. probability distribution of random variable $\rvx$
\begin{equation}\label{eq:setsym}
	\small
	\setsym{\G'}{\sA}=\frac{1}{|\G'|(|\G'|-1)}\sum_{\substack{g_1,g_2\in\G'\\g_1\neq g_2}}\frac{\sP[\rvx\in g_1\Glact \sA \cap g_2\Glact \sA]}{\sP[\rvx \in \sA]},
\end{equation}
which in the case when $\G'$ is a subgroup of $\G$ simplifies as
\begin{equation}\label{eq:setsym_bis}
	\small
	\setsym{\G'}{\sA}=\frac{1}{|\G'|-1}\sum_{\substack{g\in\G'\\g\neq e}}\frac{\sP[\rvx \in \sA \cap g\Glact \sA]}{\sP[\rvx \in \sA]}.
\end{equation}
Observe that always $\setsym{\G'}{\sA}\in[0,1]$, where extremes correspond to the cases $\setsym{\G'}{\sA}=1$ when set $\sA$ is $\G'$ invariant, and $\setsym{\G'}{\sA}=0$ when $\sA$ equals its coset w.r.t. $\G'$, that is $g\Glact \sA \cap \sA =\emptyset$ for all $g\in\G'$, meaning that the set is fully asymmetric w.r.t transformations $g\in\G'$.

We first generalize the approximation error bound in Lemma 1 from \cite{kostic2024neural} to the case of vector valued functions in the presence of symmetries.

\begin{theorem}[Approximation error]\label{thm:approx_error}
	Given a group of symmetries $\G$, let $\vsX$, $\vsY$ and $\vsZ$ be Hilbert spaces endowed with symmetry group $\G$, and let $\mux$, $\muy$ and $\muxy$ be $\G$-invariant probability distributions on $\vsX$, $\vsY$ and $\vsX\times\vsY$. Then, for every $\vh\in\LpyFull(\vsY,\vsZ)$ it holds that
	\begin{equation}\label{eq:approx_error}
		\norm{\E_{\rvy}[\vh(\rvy)\,\vert\, \rvx=\cdot] - \CEp \vh  }_{\Lp[2]{\mux}(\vsX,\vsZ)}\leq
		\left(
		\sval_{\dimtot+1}^{ \star} + \norm{\SVDdim{\DCE} - \PDCE} \right) \norm{\vh}_{\LpyFull(\vsY,\vsZ)}.
	\end{equation}
	Moreover, denoting
	\begin{equation}\label{eq:regression_set}
		\CEp[f(\rvy)\,\vert\, \rvx\in\sA] =\E_{\rvy}[f] + \frac{\E_{\rvx}[\one_{\sA}(\rvx)[\PDCE f](\rvx)] }{\sP[\rvx\in\sA]},
	\end{equation}
	if $\vh$ is either $\G'$-invariant or $\G'$-equivariant for some $\G'\subseteq\G$, then for every measurable set $\sA$
	\begin{equation}\label{eq:approx_error_set}
		\small
		\norm{\E[\vh(\rvy)\,\vert\, \rvx\in\sA]
			{-}
			\CEp[\vh(\rvy)\,\vert\, \rvx\in\sA] }_{\vsZ}
		{\leq}
		\left(\sval_{\dimtot+1}^{\star} + \norm{\SVDdim{\DCE}  - \PDCE} \right)
		\tfrac{\norm{h}_{\LpyFull(\vsY,\vsZ)}}{\sqrt{\sP[\rvx\in\sA]}}
		\sqrt{\tfrac{1+(|\G'|-1)\setsym{\G'}{\sA}}{|\G'|}}.
	\end{equation}
\end{theorem}
\begin{proof}
	Start by observing that
	\begin{align*}
		\small
		\norm{\E[\vh(\rvy)\,\vert\, \rvx=\cdot] - \CEp \vh  }_{\Lp[2]{\mux}(\vsX,\vsZ)} & \leq \norm{\DCE - \PDCE}_{\LpyFull(\vsY,\vsZ)\to \Lp[2]{\mux}(\vsX,\vsZ)} \norm{\vh}_{\LpyFull(\vsY,\vsZ)} \\
		                                                                                & = \norm{\DCE - \PDCE}_{\Lp[2]{\muy}(\vsY)\to \Lp[2]{\mux}(\vsX)} \norm{\vh}_{\LpyFull(\vsY,\vsZ)},
	\end{align*}
	where the equality holds since we extended operators $\DCE$ and $\PDCE$ to vector valued setting as integral operators with the same scalar kernel. Hence, \eqref{eq:approx_error} readily follows.

	To prove \eqref{eq:approx_error_set}, start with noting
	\[
		\small
		\E[\vh(\rvy)\,\vert\, \rvx\in\sA] -\CEp[\vh(\rvy)\,\vert\, \rvx\in\sA] = \frac{\E_{\rvx}[\one_{\sA}(\rvx)[(\DCE-\PDCE) \vh](\rvx)] }{\sP[\rvx\in\sA]}.
	\]
	Then, if $\vh$ is $\G$-equivariant, then, using that invariance of the probability distribution $\mux$, $\G$-equivariance of $\DCE$ and, due to Proposition \ref{prop:equivaraint_model}, of $\PDCE$, we have that for every $g\in\G'\subseteq\G$
	\begin{align*}
		\small
		\E_{\rvx}[\one_{\sA}(\rvx)[(\DCE-\PDCE) \vh](\rvx)] & = \E_{\rvx}[\one_{\sA}(g\Glact[\vsX] \rvx)[(\DCE-\PDCE) \vh](g\Glact[\vsX] \rvx)]              \\
		                                                    & = \E_{\rvx}[\one_{g^{{-}1}\Glact[\vsX] \sA}(\rvx)\, g\Glact[\vsZ] [(\DCE-\PDCE) \vh](\rvx)]
		\\
		                                                    & = \E_{\rvx}[\one_{g^{{-}1}\Glact[\vsX] \sA}(\rvx) \irrep[\vsZ](g)\, [(\DCE-\PDCE) \vh](\rvx)].
	\end{align*}
	Hence, averaging over $\G'$ we obtain
	\[
		\small
		\E[\vh(\rvy)\,\vert\, \rvx\in\sA] -\CEp[\vh(\rvy)\,\vert\, \rvx\in\sA] = \E_{\rvx} [\oH(\rvx) \overline{\vz}(\rvx)],
	\]
	where
	\[
		\small
		\oH(\vx) = \frac{1}{|\G'|\sP[\rvx\in\sA]}\sum_{g\in\G'} \one_{g^{{-}1}\Glact[\vsX] \sA}(\vx)\irrep[\vsZ](g)\quad\text{ and }\quad \vz(\vx)=[(\DCE-\PDCE) \vh](\vx).
	\]
	Since due to Cauchy-Schwartz inequality we have
	\[
		\small
		\norm{\E_{\rvx} [\oH(\rvx) \vz(\rvx)]}^2_{\vsZ} \leq [\E_{\rvx} \norm{\oH(\rvx)}^2_{\vsZ\to\vsZ}][\E_{\rvx}\norm{\vz(\rvx)}^2_{\vsZ}] = \norm{\vz}_{\Lp[2]{\mux}(\vsX,\vsZ)}^2\,[\E_{\rvx} \norm{\oH(\rvx)}^2_{\vsZ\to\vsZ}]
	\]
	and $\norm{\vz}_{\Lp[2]{\mux}(\vsX,\vsZ)}\leq \norm{\DCE-\PDCE}_{\LpyFull(\vsY,\vsZ)\to \Lp[2]{\mux}(\vsX,\vsZ)} \norm{\vh}_{\LpyFull(\vsY,\vsZ)}^2$, it remains to bound $\E_{\rvx} \norm{\oH(\rvx)}^2_{\vsZ\to\vsZ}$. But, the group actions in the vector spaces are unitary, so using the $\G$-invariance of the distribution of $\rvx$ we obtain
	\begin{align*}
		\small
		\E_{\rvx} \norm{\oH(\rvx)}^2_{\vsZ\to\vsZ} & \leq \E_{\rvx} \Big[\frac{1}{|\G'|\sP[\rvx\in\sA]}\sum_{g\in\G'} \one_{g^{{-}1}\Glact[\vsX] \sA}(\vx)\Big]^2
		\\
		                                           & =
		\frac{1}{|\G'|^2\sP[\rvx\in\sA]^2}\sum_{g,g'\in\G'} \E_{\rvx}[\one_{g^{{-}1}\Glact[\vsX] \sA}(\vx) \one_{g'^{-1}\Glact[\vsX] \sA}(\vx)]
		\\
		                                           & = \frac{1}{|\G'|^2\sP[\rvx\in\sA]^2}\sum_{g,g'\in\G'} \E_{\rvx}[\one_{g\Glact[\vsX]\sA \cap g'\Glact[\vsX] \sA}(\vx)]
		\\
		                                           & = \frac{1}{|\G'|^2\sP[\rvx\in\sA]}\sum_{g,g'\in\G'} \frac{\sP[\rvx \in g\Glact[\vsX]\sA \cap g'\Glact[\vsX] \sA]}{\sP[\rvx\in\sA]} \\
		                                           & = \frac{1}{\sP[\rvx\in\sA]} \frac{1+(|\G'|-1)\setsym{\G'}{\sA}}{|\G'|},
	\end{align*}
	which completes the proof of \eqref{eq:approx_error_set} for $\G'$-equivariant functions.  Finally, if $f$ is $\G'$-invariant, the proof follows the same lines by replacing group actions $(\Glact[\vsZ])$ by their respective group representation $\rep[\vsZ]$ (see \cref{def:group_representation}) with identity.
\end{proof}

Next we analyze the errors when, instead of applying learned operators $\PCE$, we apply their empirical counterparts in inference tasks. To that end, we define now estimators of $\E [\vh(\rvx)]$ and $\E [\vz(\rvy)]$  exploiting the $\G$-invariance of the distributions of $\rvx$ and $\rvy$. First, define the empirical $\G$-invariant distributions
$$
	\widehat{\PP}_{\rvx} := \frac{1}{|\G|\samplesize} \sum_{i=1}^{\samplesize}\sum_{g\in \g} \delta_{g\Glact \rvx_{i}}(\cdot),
	\quad
	\widehat{\PP}_{\rvy} := \frac{1}{|\G|\samplesize} \sum_{i=1}^{\samplesize}\sum_{g\in \g} \delta_{g\Glact \rvy_{i}}(\cdot).
$$
Hence we can define the equivariant empirical mean of any function $f\in \LpxFull$, $h\in \LpyFull$ as
\begin{equation}
	\label{eq:equiv_empirical_mean}
	\EEX[f] = \frac{1}{|\G|\samplesize} \sum_{i=1}^{\samplesize}\sum_{g\in \G} f(g\Glact[\vsX] \rvx_{i}),\quad \EEY[h] = \frac{1}{|\G|\samplesize} \sum_{i=1}^{\samplesize}\sum_{g\in \G} h(g\Glact[\vsY] \rvy_{i}).
\end{equation}

This extends naturally to operator on a function space $\LpyFull(\vsY,\vsZ)$ where $\vsZ$ is endowed with an inner product $\langle \cdot,\cdot \rangle = \langle \cdot,\cdot \rangle_{\vsZ}$. If the distribution of $\rvy$ is $\G'$-invariant, then for any $\vh\in \LpyFull(\vsY,\vsZ)$, we use the estimator $\EEY[\vh(\rvy)]$ in \eqref{eq:equiv_empirical_mean} as an estimator of $\E[\vh(\rvy)]$:
\begin{equation}
	\EEY[\vh] = \frac{1}{|\G|\samplesize} \sum_{i=1}^{\samplesize}\sum_{g\in \G} \vh(g\Glact[\vsY] \rvy_{i}).
\end{equation}

In this notation, we define our empirical estimators
\[
	\small
	[\PCE \vh](\vx)\approx[\ECE \vh](\vx)
	=
	\EEY[\vh(\rvy)] +
	\sum_{\isoIdx\in[\isoNum]}\sum_{i\in[\irrepMultiplicity]}\sum_{j\in[\irrepDim[\isoIdx]]}\svalpIso_{i}
	\lsfnpIso_{i,j}(\vx) \EEY[\rsfnpIso_{i,j} \vh]
\]
and
\[
	\small
	\PCE[\vh(\rvy)\,\vert\,\rvx\in\sA]\!\approx\! \ECE[\vh(\rvy)\,\vert\,\rvx\in\sA]
	\!=\!
	\EEY[\vh] + \sum_{\isoIdx\in[\isoNum]}\sum_{i\in[\irrepMultiplicity]}\sum_{j\in[\irrepDim[\isoIdx]]}\svalpIso_{i}
	\frac{
		\EEX[\lsfnpIso_{i,j} \one_\sA]}{\EEX[\one_\sA]} \EEY[\rsfnpIso_{i,j} \vh].
\]
and, by choosing $\vh=\one_\sB$,
\[
	\small
	P[\rvy \!\in\! \sB \mid \rvx \!\in\! \sA]\approx\widehat{P}_{\nnParams}[\rvy \!\in\! \sB \mid \rvx \!\in\! \sA]\!=\!\EEY[\one_\sB] + \sum_{\isoIdx\in[\isoNum]}\sum_{i\in[\irrepMultiplicity]}\sum_{j\in[\irrepDim[\isoIdx]]}\svalpIso_{i} \frac{\EEX[\lsfnpIso_{i,j} \one_\sA]}{\EEX[\one_\sA]} \EEY[\rsfnpIso_{i,j} \one_\sB].
\]
Direct consequence of the above construction which ensures that $\widehat{P}_{\rvx}$ and $\widehat{P}_{\rvy}$ are $\G$-invariant is the following result.

\begin{proposition}\label{prop:empirical_inner_product}
	Let $\mux$ and $\muy$ are $\G$-invariant, and $\PDCE$ from \eqref{eq:dnn_model} is $\G$-equivariant model, and let $z \in\LpxFull(\vsX,\R)$ and $\vh \in \Lp[2]{\mux}(\vsY,\vsZ)$ be arbitrary. If for every $\isoIdx\in[\isoNum]$
	\[
		\small
		\Big\{\norm{\isoDCE-\PDCEIso},
		\norm{\E_\rvx[\vlsfnpIso_{(1)}(\rvx)\vlsfnpIso_{(1)}(\rvx)^\top] -
			\Identity_{\irrepMultiplicity}},
		\norm{\E_\rvy[
				\vrsfnpIso_{(1)}(\rvy)
				\vrsfnpIso_{(1)}(\rvy)^\top]-
			\Identity_{\irrepMultiplicity}} \Big\} \leq \errorIso
	\]
	holds with $
		\vlsfnpIso_{(1)} = [\lsfnpIso_{1,1}\vert \ldots\vert \lsfnpIso_{\irrepMultiplicity,1}]^\top \in \R^{\irrepMultiplicity}
	$
	and
	$
		\vrsfnpIso_{(1)} = [
		\rsfnpIso_{1,1}\vert \ldots\vert \rsfnpIso_{\irrepMultiplicity,1}]^\top\in\R^{\irrepMultiplicity}
	$, and if
	\begin{equation}\label{eq:concentration_error}
		\begin{aligned}
			\small
			 & \frac{\norm{\EEX[\vlsfnpIso_{(1)} z_1^{\isoIdxBrace}] - \E_{\rvx}[\vlsfnpIso_{(1)}(\rvx) z_1^{\isoIdxBrace}(\rvx)]}}{\norm{z_1^{\isoIdxBrace}}_{\LpxFull}}\leq A(\vlsfnp,z)
			,
			\\
			 & \frac{\norm{\EEY[\vrsfnpIso_{(1)}\otimes \vh_{1}^{\isoIdxBrace}] - \E_{\rvy}[\vrsfnpIso_{(1)}(\rvy) \otimes \vh_{1}^{\isoIdxBrace}(\rvy)]}}{\norm{\vh_1^{\isoIdxBrace}}_{\LpyFull}}\leq A(\vrsfnp,\vh),
		\end{aligned}
	\end{equation}
	where $z=\sum_{\isoIdx\in[\isoNum]}\sum_{j\in[\irrepDim[\isoIdx]]}z^{\isoIdxBrace}_j$ and $\vh=\sum_{\isoIdx\in[\isoNum]}\sum_{j\in[\irrepDim[\isoIdx]]}\vh^{\isoIdxBrace}_j$ are isospectral decompositions, then
	\begin{equation}\label{eq:empirical_regression}
		\norm{\PCE \vh \!-\! \ECE \vh}_{\Lp[2]{\mux}(\vsX,\vsZ)}^2 \leq \norm{\E_{\rvy}[\vh(\rvy)-\EEY[\vh]]}_{\vsZ}^2+\Big[1+\max_{k\in[\isoNum]}\errorIso\Big]^{3} \norm{\vh}_{\LpyFull(\vsY,\vsZ)}^2\,[A(\vrsfnp,\vh)]^2.
	\end{equation}
	Moreover, the empirical estimation error
	%$\norm{\E_{\rvx}[h(\rvx) [\PDCE f](\rvx)]\!-\!\EEX[h[\EDCEIso f]]}_{\vsZ}$ 
	is upper bounded by
	\begin{equation}\label{eq:empirical_inner_product}
		%|\scalarp{f,\PDCE h}_{\Lp[2]{\mux\times\muy}{}}\!\!\!\!\!\!\!
		\begin{aligned}
			\small
			 & \norm{\E_{\rvx}[z(\rvx) [\PDCE \vh](\rvx)]\!-\!\EEX[z[\EDCEIso \vh]]}_{\vsZ}^2
			\leq
			\\
			 & (1+\error)^{3}
			\Big[
				A(\vlsfnp,z)
				+
				A(\vrsfnp,\vh) + A(\vlsfnp,z) A(\vrsfnp,\vh)
				\Big]^2 \norm{z}_{\Lp[2]{\mux}(\vsX)}^2 \norm{\vh}_{\LpyFull(\vsY,\vsZ)}^2.
		\end{aligned}
	\end{equation}
\end{proposition}
\begin{proof}
	First, observe that due to $\G$-invariance of distribution $\muxy$ and $\G$-equivaraince of $\PCE$ and $\PDCE$ we have that
	\begin{equation}\label{eq:prop_empirical_1}
		\PCE \vh = \E_{\rvy}[\vh^{(1)}(\rvy)] + \sum_{\isoIdx\in[\isoNum]} \PDCEIso \vh^{\isoIdxBrace},
	\end{equation}
	and
	\begin{equation}\label{eq:prop_empirical_2}
		\E_{\rvx}[z(\rvx)[\PCE \vh](\rvx)] = \E_{\rvx}[z^{(1)}(\rvx)]\E_{\rvy}[\vh^{(1)}(\rvy)] + \sum_{\isoIdx\in[\isoNum]} \E_{\rvx}[z^{\isoIdxBrace}(\rvx) [\PDCEIso \vh^{\isoIdxBrace}](\rvx)].
	\end{equation}
	In the same way, since the empirical distributions $\widehat{P}_{\rvx}$ and $\widehat{P}_{\rvy}$ are $\G$-invariant, we have that
	\begin{equation}\label{eq:prop_empirical_3}
		\ECE\vh = \EEY[\vh^{(1)}] + \sum_{\isoIdx\in[\isoNum]} \EDCEIso \vh^{\isoIdxBrace},
	\end{equation}
	and
	\begin{equation}\label{eq:prop_empirical_4}
		\EEX[z[\ECE \vh]] = \EEX[z^{(1)}]\EEY[\vh^{(1)}] + \sum_{\isoIdx\in[\isoNum]} \EEX[z^{\isoIdxBrace}[\EDCEIso \vh^{\isoIdxBrace}]],
	\end{equation}
	where
	\begin{equation}\label{eq:EDCEIso}
		[\EDCEIso \vh^{\isoIdxBrace}](\vx)
		=
		\vlsfnpIso(\vx)^\top\SigmapIso \EEY[\vrsfnpIso\otimes \vh^{\isoIdxBrace}].
	\end{equation}
	Therefore, combining \eqref{eq:prop_empirical_1} and \eqref{eq:prop_empirical_3}, we obtain that
	\[
		\small
		\begin{split}
			[\PCE \vh](\vx) - [\ECE \vh](\vx)
			&= \Big(\E_{\rvy}[\vh^{(1)}(\rvy)] - \EEY[\vh^{(1)}]\Big)\one_{\vsX}(\vx) \\
			&\quad + \sum_{\isoIdx\in[\isoNum]} \Big([\PDCEIso \vh^{\isoIdxBrace}](\vx) - [\EDCEIso \vh^{\isoIdxBrace}](\vx)\Big),
		\end{split}
	\]
	which after taking the norm in $\Lp[2]{\mux}(\vsX,\vsZ)$, due to orthonormality of isotypic subspaces gives
	\[
		\small
		\begin{split}
			\norm{\PCE \vh - \ECE[\vh]}_{\Lp[2]{\mux}(\vsX,\vsZ)}^2
			&= \norm{\E_{\rvy}[\vh^{(1)}(\rvy)] - \EEY[\vh^{(1)}]}_{\vsZ}^2 \\
			&\quad + \sum_{\isoIdx\in[\isoNum]} \norm{[\PDCEIso \vh^{\isoIdxBrace}] - [\EDCEIso \vh^{\isoIdxBrace}]}_{\Lp[2]{\mux}(\vsX,\vsZ)}^2.
		\end{split}
	\]
	Now, observe that, since
	\[
		\small
		[\PDCEIso \vh^{\isoIdxBrace}](\vx) - [\EDCEIso \vh^{\isoIdxBrace}](\vx) = \vlsfnpIso(\vx)^\top\SigmapIso\Big(\E_{\rvy}[\vrsfnpIso \otimes \vh^{\isoIdxBrace}] - \EEY[\vrsfnpIso \otimes \vh^{\isoIdxBrace}]\Big)
	\]
	applying the norm we have that $\norm{[\PDCEIso \vh^{\isoIdxBrace}] - [\EDCEIso \vh^{\isoIdxBrace}]}_{\Lp[2]{\mux}(\vsX,\vsZ)}^2$ equals
	\[
		\small
		\Big(\E_{\rvy}[\vrsfnpIso \otimes \vh^{\isoIdxBrace}] - \EEY[\vrsfnpIso \otimes \vh^{\isoIdxBrace}]\Big)^\top\SigmapIso\Big(\E_{\rvx}[\vlsfnpIso(\vx)\vlsfnpIso(\vx)^\top]\Big)\SigmapIso\Big(\E_{\rvy}[\vrsfnpIso \otimes \vh^{\isoIdxBrace}] - \EEY[\vrsfnpIso \otimes \vh^{\isoIdxBrace}]\Big)
	\]
	which using constraints within each isotypic block and
	\[
		\small
		\E_{\rvx}[\vlsfnpIso_{(\irrepMultiplicity)}(\vx)\vlsfnpIso_{(\irrepMultiplicity)}(\vx)^\top]\preceq \norm{\E_{\rvx}[\vlsfnpIso_{(\irrepMultiplicity)}(\vx)\vlsfnpIso_{(\irrepMultiplicity)}(\vx)^\top]}\Identity_{\irrepMultiplicity}\leq (1+\errorIso)\Identity_{\irrepMultiplicity},
	\]
	implies, due to \eqref{eq:concentration_error}, that
	\begin{align*}
		\small
		\norm{\PDCEIso \vh^{\isoIdxBrace} - \EDCEIso \vh^{\isoIdxBrace}}_{\Lp[2]{\mux}(\vsX,\vsZ)}^2
		 & \leq \irrepDim[\isoIdx] \,(1+\errorIso)\, (\svalpIso_1)^2                                 \\
		 & \quad
		\cdot
		\norm{\E_{\rvy}[\vrsfnpIso_{(1)}(\rvy)\otimes \vh^{\isoIdxBrace}_1(\rvy)]
		- \EEY[\vrsfnpIso_{(1)} \otimes \vh^{\isoIdxBrace}_1]}_{\R^{\irrepMultiplicity}\times\vsZ}^2 \\
		 & \leq \irrepDim[\isoIdx] \,(1+\errorIso)\, (\svalpIso_1)^2 \,[A(\vrsfnp,\vh)]^2 \,
		\norm{\vh^{\isoIdxBrace}_1}_{\Lp[2]{\mux}(\vsY,\vsZ)}^2.
	\end{align*}
	Therefore, bounding $
		\sigma_1^{\nnParams\isoIdxBrace} \leq \sigma_1^{\isoIdxBrace} + |\sigma_1^{\isoIdxBrace}-\sigma_1^{\nnParams\isoIdxBrace}|\leq 1+ \norm{\isoDCE-\PDCEIso}
	$ and summing over isotypic components, since $
		\norm{\vh}_{\Lp[2]{\mux}(\vsY,\vsZ)}^2 = \sum_{\isoIdx\in[\isoNum],j\in[\irrepDim[\isoIdx]]}\norm{\vh^{\isoIdxBrace}_j}_{\Lp[2]{\mux}(\vsY,\vsZ)}^2=\sum_{\isoIdx\in[\isoNum]}\irrepDim[\isoIdx]\norm{\vh^{\isoIdxBrace}_1}_{\Lp[2]{\mux}(\vsY,\vsZ)}^2$, we complete the proof of \eqref{eq:empirical_regression}.

	To show \eqref{eq:empirical_inner_product}, we combine \eqref{eq:prop_empirical_2} and  \eqref{eq:prop_empirical_4}, and obtain that $
		\E_{\rvx}[z(\rvx) [\PDCE \vh](\rvx)]
		\!-\!
		\EEX[z[\EDCE \vh]]
	$ can be written as
	\[
		\small
		\sum_{\isoIdx\in[\isoNum]}\irrepDim[\isoIdx]
		\Bigg[
		\E_{\rvx}[\vlsfnpIso_{(1)}(\rvx) z_1^{\isoIdxBrace}(\rvx)]^\top\SigmapIso\E_{\rvy}[\vrsfnpIso_{(1)}(\rvy)\!\otimes\! \vh_1^{\isoIdxBrace}(\rvy)]
		\!-\!
		\EEX[\vlsfnpIso z_1^{\isoIdxBrace}]^\top\SigmapIso \EEY[\vrsfnpIso_{(1)}\!\otimes\!\vh_{1}^{\isoIdxBrace}]
		\Bigg].
	\]
	Adding and subtracting mixed terms we then obtain for each isotypic component, $\frac{1}{\irrepDim[\isoIdx]}\E_{\rvx}[z(\rvx) [\PDCEIso \vh](\rvx)] - \EEX[z[\EDCEIso \vh]]$ can be expressed as
	\begin{align*}
		\small
		 & \E_{\rvx}[\vlsfnpIso_{(1)}(\rvx) z_1^{\isoIdxBrace}(\rvx)]^\top \mS^{\isoIdxBrace}\Bigg(\E_{\rvy}[\vrsfnpIso_{(1)}(\rvy) \!\otimes\! \vh_1^{\isoIdxBrace}(\rvy)] \!-\!\EEY[\vrsfnpIso_{(1)}\!\otimes\! \vh_{1}^{\isoIdxBrace}] \Bigg)                                                                \\
		 & \!+\! \Bigg(\E_{\rvx}[\vlsfnpIso_{(1)}(\rvx) z_1^{\isoIdxBrace}(\rvx)]\!-\! \EEX[\vlsfnpIso_1 z_1^{\isoIdxBrace}] \Bigg)^\top \mS^{\isoIdxBrace} \E_{\rvy}[\vrsfnpIso_{(1)}(\rvx)\!\otimes\! \vh_1^{\isoIdxBrace}(\rvy)]                                                                             \\
		 & \!+\! \Bigg(\E_{\rvx}[\vlsfnpIso_{(1)}(\rvx) z_1^{\isoIdxBrace}(\rvx)]\!-\! \EEX[\vlsfnpIso_1 z_1^{\isoIdxBrace}] \Bigg)^\top \mS^{\isoIdxBrace} \Bigg(\E_{\rvy}[\vrsfnpIso_{(1)}(\rvy)\!\otimes\! \vh_1^{\isoIdxBrace}(\rvy)] \!-\!\EEY[\vrsfnpIso_{(1)}\!\otimes\! \vh_{1}^{\isoIdxBrace}] \Bigg),
	\end{align*}
	and consequently bounded using \eqref{eq:concentration_error} as
	\[
		\small
		\begin{split}
			\small
			\norm{\E_{\rvx}[z(\rvx) [\PDCEIso \vh](\rvx)] - \EEX[z[\EDCEIso \vh]]}_{\vsZ}
			&\leq \irrepDim[\isoIdx]\svalpIso_1\, \Big[ A(\vlsfnp, z) + A(\vrsfnp, \vh)
				\\
				&\quad + A(\vlsfnp, z) A(\vrsfnp, \vh) \Big] \norm{z^{\isoIdxBrace}_1}_{\Lp[2]{\mux}(\vsX)} \norm{\vh^{\isoIdxBrace}_1}_{\LpyFull(\vsY,\vsZ)}.
		\end{split}
	\]
	Summing across isotypic components and bounding $\svalpIso_1$ as before, we complete the proof.
\end{proof}

First note that coupling \eqref{eq:empirical_regression} with \eqref{eq:approx_error} ensures that we can prove regression bound via concentration result ensuring \eqref{eq:concentration_error}. To obtain similar result for set-wise regression, we set $z=\one_\sA$ and use \eqref{eq:empirical_inner_product} to obtain the following.

\begin{proposition}\label{prop:empirical_setwise_regression}
	Under the assumptions of Proposition \ref{prop:empirical_inner_product}, let $A(\vlsfnp,\one_\sA)A(\vrsfnp,\vh) \leq A(\vlsfnp,\one_\sA) + A(\vrsfnp,\vh)$. If
	\begin{equation}\label{eq:indicators_concentration_error}
		|\E_{\rvx}[\one_\sA(\rvx)] - \EEX[\one_\sA]| / \E_{\rvx}[\one_\sA(\rvx)] \leq \eta_\sA
	\end{equation}
	and $\eta_\sA<1/2$, then
	\begin{equation}\label{eq:empirical_setwise_regression}
		\small
		\begin{split}
			\small
			\norm{\PCE[\vh(\rvy)\vert\rvx\in\sA] - \ECE[\vh(\rvy)\vert\rvx\in\sA]}_{\vsZ}
			&\leq \norm{\E_{\rvy}[\vh] - \EEY[\vh]}_{\vsZ}
			+ \frac{2\norm{\vh}_{\Lp[2]{\mux}(\vsY,\vsZ)}}{\sqrt{P[\rvx\in\sA]}} \\
			&\quad\quad \times \Big[ 2(1+\error)\Big(A(\vlsfnp,\one_\sA) + A(\vrsfnp,\vh)\Big) + \eta_\sA \Big],
		\end{split}
	\end{equation}
	%\daniel{$2\norm{\vh}_{\Lp[2]{\mux}(\vsX,\vsZ)}$ seems wrong as $\vh \in \LpyFull(\vsY,\vsZ)$}
	and for $\vh=\one_\sB$

	\begin{equation}\label{eq:empirical_probability_error}
		\small
		\begin{split}
			\small
			|P[\rvy \in \sB \mid \rvx \in \sA] {-} \widehat{P}_{\nnParams}[\rvy \in \sB \mid \rvx \in \sA]|
			&\leq \norm{\E_{\rvy}[\vh] {-} \EEY[\vh]}_{\vsZ} \\
			&\quad + \frac{2}{\EEX[\vh]}\sqrt{\frac{P[\rvy\in\sB]}{P[\rvx\in\sA]}}
			\Big[ 2(1+\error)[A(\vlsfnp,\one_\sA) {+} A(\vrsfnp,\one_\sB)] {+} \eta_\sA \Big].
		\end{split}
	\end{equation}
\end{proposition}
\begin{proof}
	Leveraging the representations in \eqref{eq:prop_empirical_2} and \eqref{eq:prop_empirical_4} with $z=\one_\sA$, we get
	\begin{align*}
		\small
		 & \PCE[\vh(\rvy)\vert\rvx\!\in\!\sA]  -\ECE[\vh(\rvy)\vert\rvx\!\in\!\sA]
		=                                                                                                                                                                                 \E_{\rvy}[\vh]-\EEY[\vh]\!+\! \tfrac{\E_{\rvx}[\one_\sA(\rvx)[\PDCE \vh](\rvx)] }{\E [\one_\sA]} - \tfrac{\EEX[\one_\sA[\EDCE \vh]] }{\EEX[\one_\sA]}
		=
		\\
		 & \E_{\rvy}[\vh]-\EEY[\vh]\!+\!\E_{\rvx}[\one_\sA(\rvx)[\DCE \vh](\rvx)] \left(\tfrac{1}{\E [\one_\sA(\rvx)]} -\tfrac{1}{\EEX[\one_\sA]} \right)
		+ \tfrac{\E_{\rvx}[\one_\sA(\rvx)[\PDCE \vh](\rvx)]- \EEX[\one_\sA[\EDCE \vh]] }{\EEX[\one_\sA]}.
	\end{align*}
	By triangular inequality applied to the norm in $\vsZ$, we get
	\begin{align*}
		\small
		 & \norm{\PCE[\vh(\rvy)\vert\rvx \in \sA] - \ECE[\vh(\rvy)\vert\rvx\!\in\!\sA]}_{\vsZ}
		\\
		 & \leq \norm{\E_{\rvy}[\vh]-\EEY[\vh]}_{\vsZ}\!+\! \norm{\E [\one_\sA(\rvx)[\PDCE f(\rvx)]]}_{\vsZ}
		\left\vert
		\tfrac{1}{\E [\one_\sA(\rvx)]}  {-} \tfrac{1}{\EEX[\one_\sA]}
		\right\vert
		{+}
		\tfrac{\norm{\E_{\rvx}[\one_\sA(\rvx)[\PDCE \vh](\rvx)]- \EEX[\one_\sA[\EDCE \vh]] }_{\vsZ}  }{\EEX[\one_\sA]}
		\\
		 & \leq\!\norm{\E_{\rvy}[\vh]-\EEY[\vh]}_{\vsZ}\!+\! \norm{\E [\one_\sA(\rvx)[\PDCE \vh](\rvx)]}_{\vsZ}  \tfrac{2\eta_\sA}{\PP[\rvx\in \sA]}  + \tfrac{\norm{\E_{\rvx}[\one_\sA(\rvx)[\PDCE \vh](\rvx)]- \EEX[\one_\sA[\EDCE \vh]] }_{\vsZ}  }{\EEX[\one_\sA]},
		% &\hspace{1cm} = \norm{\langle h, [\PCE f] \rangle_{\Lp[2]{\mux}(\vsX,\vsZ)}}_{\vsZ}  \frac{2\eta_\sA}{\PP[\rvx\in \sA]}  + \frac{\norm{ \langle h, [\PCE f] \rangle_{\Lp[2]{\mux}(\vsX,\vsZ)} - \E_{\rvx}[h(\rvx)[\PCE f](\rvx)]- \EEX[h[\ECE f]] }_{\vsZ}  }{\EEX[h]},
	\end{align*}
	where we have used Condition \eqref{eq:indicators_concentration_error} in the last line to get that
	$$
		\small
		\left\vert \frac{1}{\sP[\rvx\in\sA]} - \frac{1}{\EEX[\one_{\sA}]}\right\vert \leq \frac{\eta_\sA}{(1-\eta_\sA)\PP[\rvx\in \sA]} \leq \frac{2\eta_\sA}{\PP[\rvx\in \sA]}.
	$$
	From Proposition \ref{prop:empirical_inner_product} and Condition \eqref{eq:indicators_concentration_error} we get that
	\begin{align*}
		\small
		\tfrac{1}{\EEX[\one_{\sA}]} \norm{ \E_{\rvx}[\one_{\sA}(\rvx)[\PDCE \vh](\rvx)] \!-\!\EEX[\one_{\sA}(\rvx)[\EDCE \vh]] }_{\vsZ}
		\!\leq\!
		\tfrac{2(1\!+\!\error)\Big[ A(\vlsfnp,\one_{\sA}) A(\vrsfnp,\vh) \Big]}{\PP(\rvx\in \sA)}\norm{\one_{\sA}}_{\Lp[2]{\mux}{}}\norm{\vh}_{\Lp[2]{\muy}{(\vsY,\vsZ)}}
	\end{align*}
	Cauchy's Schwarz's inequality again and $\norm{\PDCE}\leq 1$ give
	$$
		\small
		\norm{\E [\one_\sA(\rvx)[\PDCE \vh](\rvx)]}_{\vsZ} \leq \norm{\one_\sA}_{\Lp[2]{\mux}{}} \norm{\PDCE} \norm{\vh}_{\Lp[2]{\muy}{}}\leq \norm{\one_\sA}_{\Lp[2]{\mux}{}}  \norm{\vh}_{\Lp[2]{\muy}{}} = \sqrt{\PP[\rvx\in \sA]} \norm{\vh}_{\Lp[2]{\muy}{(\vsY,\vsZ)}}.
	$$
	Combining the last four displays give the first result. The second result follows immediately for $\vh=\one_{\sB}$.
\end{proof}

Consequence of this result is that we can bound the error in probability as we can derive concentration inequalities on the terms in \eqref{eq:concentration_error} and \eqref{eq:indicators_concentration_error}. Then an union bound gives the estimation result for regression conditional on sets.

Next, we recall that $\CE$ being $(1/\rpar)$-Schatten class operator, implies:
\begin{assumption}
	\label{ass:smoothness-DCE}
	Let there exist some constant $c>0$ such that for $\rpar>0$, any $i\geq 1$ and any $k\in [\isoNum]$, we have $\sval_{i}^{\isoIdxBrace} \leq c\, i^{-\rpar}$.
\end{assumption}
Further, for any $\vh\in\LpyFull(\vsY,\vsZ)$, we define $\overline{\vh}(\rvy) = \vh(\rvy) - \E [\vh(\rvy)]$ and
\begin{equation}
	\label{eq:setsym_function}
	\setsym{\G'}{\vh}:=\frac{1}{|\G'|-1}\sum_{\substack{g\in\G'\\ \g\neq e}} \E [ \langle \overline{\vh}(\rvy),  \overline{\vh}(g\Glact[\vsY] \rvy) \rangle  ].
\end{equation}
In the following, we consider observables $h$ satisfying the following condition (that is clearly satisfied for an indicator of a set of positive measure)
\begin{assumption}
	\label{ass:gammaGh-C0}
	Let there exists an absolute constant $C_0\geq 1$ such that $(|\G'|-1)\setsym{\G'}{h} \leq C_0 \,\E [\norm{h(\rvy)}_{\vsZ}^2]$.
	%, for any $h\in \Lp[2]{\muy}{\vsX,\vsZ}$.
\end{assumption}
Define
\begin{align*}
	\small
	\eta_\sA = \eta_\sA(\delta):= \left(\frac{ 1 - \PP[\rvx\in \G \Glact \sA]}{\PP[\rvx\in \G \Glact \sA]}\right)
	{ \frac{  \log 2\delta^{-1}}{\samplesize} }  +   \sqrt{2  \frac{\log  2\delta^{-1}}{ \samplesize} } \sqrt{\frac{ 1 - \PP[\rvx\in \G \Glact \sA]}{\PP[\rvx\in \G \Glact \sA]}}.
\end{align*}
\begin{theorem}
	Let Assumptions \ref{ass:gammaGh-C0} and \ref{ass:smoothness-DCE} be satisfied. Let $\mux$ and $\muy$ are $\G$-invariant, and $\PDCE$ from \eqref{eq:dnn_model} is $\G$-equivariant model, and let $\vh \in \LpyFull(\vsY,\vsZ)$ and $\vf \in \LpxFull(\vsX,\vsZ)$ (with values in $\vsZ$) be subGaussian random variables. Assume in addition that the event $\sA$ is anti-symmetric for $\G$ and that $m_k = m$ for all $k\in [\isoNum]$. Assume that $N\geq |\G|$.
	Then for any $\delta\in (0,1)$, it holds w.p.a.l $1-\delta$
	\begin{equation*}
		\small
		\norm{\E[\vh(\rvy)\,\vert\, \rvx\in\sA] -\ECE[\vh(\rvy)\vert\rvx\!\in\!\sA]}_{\vsZ}
		\lesssim_{C_0}
		\tfrac{\norm{\vh}_{\LpyFull(\vsY,\vsZ)}}{\sqrt{\sP[\rvx\in \G\Glact[\vsX]\sA]}} \left(\error+
		\tfrac{\log (2\isoNum\delta^{-1})}{(\isoDim \samplesize)^{\frac{\rpar}{1+2\rpar}}}
		%+ \sqrt{|\G|}\eta_\sA  
		\right),
	\end{equation*}
	and
	\begin{equation*}
		\small
		|\PP(\rvy \!\in\! \sB \mid \rvx \!\in\! \sA) - \widehat{\PP}_{\nnParams}(\rvy \!\in\! \sB \mid \rvx \!\in\! \sA)|
		\lesssim_{C_0}
		\sqrt{\tfrac{\sP[\rvy\in \sB]}{\sP[\rvx\in \G\Glact[\vsX]\sA]}} \left(\error+
		\tfrac{\log (2\isoNum\delta^{-1})}{(\isoDim \samplesize)^{\frac{\rpar}{1+2\rpar}}}
		+ \sqrt{|\G|}\eta_\sA  \right).
	\end{equation*}
\end{theorem}
%\in\LpxFull, \Lp[2]{\mux}{\vsY,\vsZ}
\begin{proof}
	This result follows immediately from Propositions \ref{prop:empirical_inner_product} and \ref{prop:empirical_setwise_regression} combined with Lemmas \ref{lem:bernsteiniidMean-subgaussian-1} and Lemma \ref{lem:bernsteiniidProbX}. Set
	\begin{align*}
		\small
		A(\vlsfnp,\vf)
		 & :=
		C\, \sqrt{\tfrac{1}{|\G'|\,\samplesize}}  \sqrt{ C_0  \vee \tfrac{|\G'|}{\samplesize} }\, \log (2\isoNum\delta^{-1}),
		\\
		A(\vrsfnp, \vh)
		 & :=
		C\, \sqrt{\tfrac{\max_{k\in [\isoNum]}\{\irrepMultiplicity_{k}\}}{|\G'|\,\samplesize}}  \sqrt{ C_0  \vee \tfrac{|\G'|}{\samplesize} }\, \log (2\isoNum\delta^{-1}),
	\end{align*}
	for some large enough absolute constant $C>0$.

	Then an union bound based on Lemmas \ref{lem:bernsteiniidMean-subgaussian-1} and \ref{lem:bernsteiniidProbX} guarantees that \eqref{eq:empirical_setwise_regression} is satisfied w.p.a.l. $1-\delta$ (up to a rescaling of the constant $C$):
	\[
		\small
		\begin{split}
			&\norm{\PCE[\vh(\rvy)\vert\rvx\in\sA] - \ECE[\vh(\rvy)\vert\rvx\in\sA]}_{\vsZ}
			\leq
			\\
			&C \frac{\norm{\vh}_{\LpxFull(\vsX,\vsZ)}}{\sqrt{P[\rvx\in\sA]}}
			\Bigg[ 2(1+\error)\Bigg( \sqrt{\tfrac{\max_{k\in [\isoNum]}\{\irrepMultiplicity_{k}\}}{|\G'|\,\samplesize}}
				\sqrt{\, C_0 \vee \tfrac{|\G'|}{\samplesize} }\,\log (2\isoNum\delta^{-1})\Bigg) \Bigg].
		\end{split}
	\]
	%\daniel{$\norm{\vh}_{\Lp[2]{\muy}(\vsX,\vsZ)}$ seems wrong as $\vh \in \LpyFull(\vsY,\vsZ)$}
	Next we use our bound on the representation bias in \eqref{eq:approx_error_set}
	\begin{equation}
		\norm{\E[\vy(\rvy)\,\vert\, \rvx\in\sA] -\CEp[\vy(\rvy)\,\vert\, \rvx\in\sA] }_{\vsZ}\leq \left(\sval_{\dimtot+1}^{\star} + \error  \right) \frac{\norm{\vh}_{\LpyFull(\vsY,\vsZ)}}{\sqrt{\sP[\rvx\in\sA]}}\sqrt{\tfrac{1+(|\G'|-1)\setsym{\G}{\sA}}{|\G'|}}.
	\end{equation}
	Recall that $\error=\max_{k\in [\isoNum]}\{\errorIso\} $. Under Assumption \ref{ass:smoothness-DCE}, we have $\norm{\SVDdim{\DCE}  - \PDCE}\leq \frac{1}{(\isoDim m )^{\rpar}}$. In addition, $(|\G'|-1)\setsym{\G}{\sA}\leq C_0$ under Assumption \ref{ass:gammaGh-C0}.

	Combining the last two display gives w.p.a.l $1-\delta$
	\begin{equation*}
		\small
		\norm{\E[\vh(\rvy)\,\vert\, \rvx\in\sA] -\ECE[\vh(\rvy)\vert\rvx\!\in\!\sA]}_{\vsZ}
		\lesssim_{C_0}
		\tfrac{\norm{\vh}_{\LpyFull(\vsY,\vsZ)}}{\sqrt{\sP[\rvx\in \G\Glact[\vsX]\sA]}} \left(\error+ \tfrac{1}{(\isoDim m )^{\rpar}}   +  \sqrt{\tfrac{m}{\samplesize}} \, \log (2\isoNum\delta^{-1})
		%+ |\G|\eta_\sA 
		\right) .
	\end{equation*}
	Balancing the previous display w.r.t. dimension $m$, we get that $m\asymp (\isoDim^{-2\rpar} \samplesize)^{\frac{1}{1+2\rpar}}$ and the first result follows.

	The bound for the conditional probability follows by picking $\vy=\one_{\sB}$.
\end{proof}
\subsection{Quadratic Error Regression Bound}
Our goal is to estimate the conditional expectation function
$$
	\vz(\vx) = \E[\vh(\rvy) \vert \rvx {=} \vx] =  \E[\vh(\rvy)] + [\DCE \vh](\vx).
$$
Our estimator is
$$
	\widehat{\vz}_{\nnParams}(\cdot)  =   \EEY[\rvy] + [\EDCE \vh](\cdot)
	.
$$
\begin{theorem}
	Assume that $Y$ is a sub-Gaussian random vector. Let Assumption %\ref{ass:boundedtildesingularvector} and 
	\ref{ass:smoothness-DCE} be satisfied. Assume in addition that $\error\leq 1$, $m_k = m$ for all $k\in [\isoNum]$.
	%and that  for some $\delta\in (0,1)$. 
	Then for any $\delta\in (0,1)$ such that $\samplesize \geq \, (c_u\vee c_v)^2  m \log (e\delta^{-1} \isoNum)\vee |\G|$, it holds w.p.a.l. $1-\delta$
	\begin{align}
		\norm{\vz - \widehat{\vz}_{\nnParams}}^2_{\LpxFull(\vsX,\vsZ)} \lesssim  \mathrm{Tr}(\mathrm{Cov}(Y))\left(\error^2  +  (\isoDim |\G|\samplesize)^{\frac{-2\rpar}{1+2\rpar}} \log^2(\delta^{-1}\isoNum) \right).
	\end{align}
\end{theorem}

\customParagraph{Discussion} When the training of the \gls{nn} is successful, we expect the statistical rate to dominate the optimization error $\max_{k\in [\isoNum]}\{\errorIso\}$ for large enough sample size $N$. For distribution containing symmetry invariants with large isotopic components ($m$ is large), we observe that exploiting this information in the construction of the \gls{ncp} operator yields a substantial improvement in the statistical error rate as we go from a rate
$\samplesize^{-\frac{\rpar}{1+2\rpar}}$ for standard \gls{ncp} to $(\samplesize\, m)^{-\frac{\rpar}{1+2\rpar}}$ for \gls{encp}.
\begin{proof}
	Combining \eqref{eq:empirical_regression} with Lemma \ref{lem:bernsteiniidMean-subgaussian-1} gives w.p.a.l. $1-\delta$
	\begin{equation*}
		\small
		\begin{aligned}
			\norm{\PCE \; \rvy - \ECE\; \rvy}_{\Lp[2]{\mux}(\vsX)}^2 & \lesssim
			(1+\error)^3 \mathrm{Tr}(\mathrm{Cov}(\rvy)) \frac{
				m}{|\G|\samplesize} \log^2(2\isoNum \delta^{-1})
			\\
			                                                         &
			\lesssim
			\mathrm{Tr}(\mathrm{Cov}(\rvy)) \frac{
				m}{|\G|\samplesize} \log^2(\isoNum \delta^{-1}),
		\end{aligned}
	\end{equation*}
	provide that $\error\leq 1$.
	We derived in \eqref{eq:approx_error} an upper bound on the bias term
	\begin{equation}
		\norm{\E_{\rvy\vert\rvx}[\rvy\,\vert\, \rvx=\cdot] - \CEp \rvy  }^2_{\Lp[2]{\mux}(\vsX,\vsZ)}\leq \mathrm{Tr}(\mathrm{Cov}(\rvy)) \left(\frac{1}{(\isoDim m)^{2\rpar}} + \error^2 \right).
	\end{equation}
	Balancing the two bounds in the last two displays w.r.t. $m\asymp (|\G|\isoDim \samplesize)^{\frac{1}{1+2\rpar}}$, we get the result.
\end{proof}
\subsection{Auxiliary Results.}
% We define now estimators of $\E [f(\rvx)]$ and $\E [h(\rvy)]$  exploiting the $\G$-invariance of the distributions of $\rvx$ and $\rvy$. Let $\G'\leq\G$ and define the empirical $\G'$-invariant distributions
% $$
% 	\widehat{\PP}_{\rvx} := \frac{1}{|\G'|\samplesize} \sum_{i=1}^{\samplesize}\sum_{g\in \g} \delta_{g\Glact \rvx_{i}}(\cdot),\quad \widehat{\PP}_{\rvy} := \frac{1}{|\G'|\samplesize} \sum_{i=1}^{\samplesize}\sum_{g\in \g} \delta_{g\Glact \rvy_{i}}(\cdot).
% $$
% Hence we can define the equivariant empirical mean of any function $f\in \LpxFull$, $h\in \LpyFull$ as
% \begin{equation}
% 	\label{eq:equiv_empirical_mean}
% 	\EEX[f] = \frac{1}{|\G'|\samplesize} \sum_{i=1}^{\samplesize}\sum_{g\in \G'} f(g\Glact \rvx_{i}),\quad \EEY[h] = \frac{1}{|\G'|\samplesize} \sum_{i=1}^{\samplesize}\sum_{g\in \G'} h(g\Glact \rvy_{i}).
% \end{equation}
Consider the function space $\LpyFull(\vsY,\vsZ)$ where $\vsZ$ is endowed with an inner product $\langle \cdot,\cdot \rangle = \langle \cdot,\cdot \rangle_{\vsZ}$. If the distribution of $\rvy$ is $\G'$-invariant, then for any $h\in \LpyFull(\vsY,\vsZ)$, we use the estimator $\EEY[h]$ in \eqref{eq:equiv_empirical_mean} as an estimator of $\E [h(\rvy)]$.
% \begin{assumption}
%     \label{ass:gammaGh-C0}
%     Let there exists an absolute constant $C_0\geq 1$ such that $(|\G'|-1)\setsym{\G'}{h} \leq C_0 \,\E [\norm{h(\rvy)}_{\vsZ}^2]$.
%     %, for any $h\in \Lp[2]{\muy}{\vsX,\vsZ}$.
% \end{assumption}
\begin{lemma}\label{lem:bernsteiniidMean-subgaussian-1}
	%There exists an absolute constant $C>0$ such that,
	Assume that the distribution $\muy$ of $\rvy$ is $\G$-invariant and let $\G'\leq\G$. Let there exists a function $\vh\in \LpyFull(\vsY,\vsZ)$ such that $\vh(\rvy)$ is subGaussian. Then there exists an absolute constant $C>0$ such that for any $\delta \in (0,1)$, it holds w.p.a.l. $1-\delta$
	$$
		\small
		\norm{\EEY[\vh] - \E [\vh(\rvy)]}_{\vsZ} \leq C\,  \sqrt{\frac{ \log^2 2\delta^{-1}}{|\G'|\,\samplesize}}  \sqrt{ \E [\norm{\overline{\vh}(\rvy)}^2_{\vsZ}] +(|\G'|-1)\setsym{\G'}{\vh}  + \frac{|\G'| \E [\norm{\overline{\vh}(\rvy)}^2_{\vsZ}] }{\samplesize} }.
	$$
	Assume in addition that there exists an absolute constant $C_0\geq 1$ such that $(|\G'|-1)\setsym{\G'}{\overline{\vh}} \leq C_0 \,\E [\norm{\vh(\rvy)}_{\vsZ}^2]$. Then for any $\delta \in (0,1)$, it holds w.p.a.l. $1-\delta$
	$$
		\small
		\norm{\EEY[\vh] - \E [\vh(\rvy)]}_{\vsZ} \leq C\, \sqrt{\frac{\E [\norm{\overline{\vh}(\rvy)}^2_{\vsZ}] }{|\G'|\,\samplesize}}  \sqrt{ (1+C_0)   + \frac{|\G'|}{\samplesize} }\, \log 2\delta^{-1}.
	$$
	Note that similar bounds hold valid for the $\G$-invariant distribution $\mux$ and any function $\vf \in \Lp[2]{\mux}(\vsX,\vsZ)$ such that $\vf(\rvx)$ is subGaussian.
	%\mathrm{Tr}(\mathrm{Cov}(\vh(\rvy)))
\end{lemma}
\begin{proof}
	We note that
	\begin{align*}
		\small
		\EEY[\vh] - \E [\vh(\rvy)] & = \frac{1}{\samplesize}\sum_{i=1}^{\samplesize} Z_i\quad \text{with }\quad Z_i=\frac{1}{|\G'|}\sum_{g\in\G'} \vh(g\Glact[\vsY]  \ry_{i}) - \E_{\ry_{i}}[\vh(g\Glact[\vsY] \ry_{i})],\; \forall i \in [\samplesize
		].
	\end{align*}
	Define
	\begin{equation}
		Z:=\frac{1}{|\G'|}\sum_{g\in\G'} \vh(g\Glact[\vsY]  \rvy) - \E_{\rvy}[\vh(g\Glact[\vsY] \rvy)],
	\end{equation}
	%    Define the random variable $G$ admitting the uniform distribution $\mu_{G}$ on $\G'$. We note that $G$ and $\rvy$ are independent, meaning that $(G,\rvy)$ admits joint distribution $\mu_G \otimes \muy$. We denote by $\E_{G}$ the expectation w.r.t $G$ when $\rvy$ is frozen and by $\E_{G}$ the expectatation w.r.t. $G,\rvy$. Then we note that 
	% \begin{equation}
	% 	Z:= \E_{G}[h(G\Glact[\vsY]  \rvy)]  - \E_{G,\rvy}[h(G\Glact[\vsY]  \rvy)].
	% \end{equation}
	%Following a similar argument as in the proof of Lemma \ref{lem:bernsteiniidMean-subgaussian},
	%For brevity, we set $\norm{z} = \norm{z}_{\vsZ} = \sqrt{\langle z,z\rangle_{\vsZ}}$ for any $z\in\vsZ$. 
	and, for brevity, set $\norm{\vz} = \norm{\vz}_{\vsZ} = \sqrt{\langle \vz,\vz\rangle_{\vsZ}}$ for any $\vz\in\vsZ$.
	% and 
	% $\overline{h}(g \Glact[\vsY] \rvy) := h(g \Glact[\vsY] \rvy) - \E_{\rvy}[h(g\Glact[\vsY] \rvy)]$ for any $g\in \G'$.
	We apply Proposition \ref{prop:con_ineq_ps}, to get w.p.a.l. $1-\delta$
	% $$
	% 	\norm{\EEY[h] - \E [h(\rvy)]}_{\vsZ} \leq \frac{C}{\sqrt{\samplesize}}  \sqrt{ \frac{ \E [\norm{h(\rvy)}^2_{\vsZ}]}{|\G'|}  + \frac{\norm{ \norm{h(\rvy)}_{\vsZ}}_{\psi_2}^2}{\samplesize} } \log\frac{2}{\delta},
	% $$
	$$
		\norm{\EEY[\vh] - \E_{\rvy}[\vh(\rvy)]} \leq \frac{4\sqrt{2}}{\sqrt{\samplesize}}  \sqrt{ \mathrm{Var}_{\rvy}(\norm{Z}) + \frac{\norm{Z}_{\psi_2}^2}{\samplesize} } \log\frac{2}{\delta}.
	$$

	Using the triangular inequality successively on $\norm{\cdot}$ and $\norm{\cdot}_{\psi_2}$ and the $\G'$-invariance of $\muy$, $\norm{\overline{\vh}(g\Glact[\vsY]\rvy)}_{\psi_2} = \norm{\overline{\vh}(\rvy)}_{\psi_2} $ for any $g\in \G'$, we get that
	$$
		\norm{ \norm{Z}}_{\psi_2} \lesssim \norm{\norm{\overline{\vh}(\rvy)}}_{\psi_2}.
	$$
	We note next that $\norm{\overline{\vh}(\rvy)}$ is subGaussian. Consequently the well-known property of equivalence of moments for subGaussian distributions gives  $\norm{Z}_{\psi_2}\lesssim \norm{\norm{\overline{\vh}(\rvy)}}_{\psi_2} \lesssim \E [\norm{\overline{\vh}(\rvy)}^2]$.
	We derive now a control on $\mathrm{Var}_{\rvy}(\norm{Z})\leq \E [\norm{Z}^2]$.
	Using the $\G'$-invariance of $\muy$, we get

		% We focus next on $\E_{\rvy}\E_{G,G'}[ \langle  \overline{h}(G \Glact[\vsY] \rvy) ,  \overline{h}(G'\Glact[\vsY] \rvy) \rangle  ]$.
		% $$
		% \frac{1}{|\G'|^2} \sum_{\g,\g'\in\G'} \E_{\rvy}[ \langle  h(\g\Glact[\vsY]\rvy) - \E [ h(\rvy)]   ,  h(\g'\Glact[\vsY] \rvy) - \E [h(\rvy )] \rangle,
		% $$
		% where we have used that $\E_{\rvy}\E_{G,G'}[ \langle  \overline{h}(G \Glact[\vsY] \rvy) ,  \overline{h}(G'\Glact[\vsY] \rvy) \rangle  ] = 0$ by definition of $\overline{h}$.
		%	where we have used that $\E [\norm{h(\g\Glact[\vsY] \rvy ) - \E [h(\rvy)]}^2] = \mathrm{Tr}(\mathrm{Cov}(h(\rvy)))$ for any $\g\in\G'$ since the distribution of $\rvy$ is $\G'$-invariant. Using again this assumption, we also note that the distribution of $(\g\Glact[\vsY] \rvy, \g'\Glact[\vsY] \rvy) \stackrel{d}{=} (\g'' \Glact[\vsY]( \g\Glact[\vsY] \rvy), \g'' \Glact[\vsY](\g'\Glact[\vsY] \rvy))$ for any $\g''\in \G'$. Hence, taking $\g'' = \g^{{-}1}$, we get that
		% \begin{align*}
		% 		 & \sum_{\substack{\g,\g'\in\G' \\ \g\neq \g'}} \E [ \langle  h(\g\Glact[\vsY]\rvy) - \E [ h(\rvy)]   ,  h(\g'\Glact[\vsY] \rvy) - \E [h(\rvy )] \rangle = |\G'| \sum_{\substack{\g'\in\G'\\ \g'\neq e}} \E [ \langle  h(\rvy) - \E [ h(\rvy)]   ,  h(\g'\Glact[\vsY] \rvy) - \E [h(\rvy )] \rangle.
		% 	\end{align*}
		%Combining the last two displays, we get
		{
			\small
			\begin{align} \label{eq:varianceL2YZ-bound}
				\small
				\mathrm{Var}(\norm{Z}) & \leq \frac{\E [\norm{\overline{\vh}(\rvy)}^2]}{|\G'|} + \frac{1}{|\G'|}\sum_{\substack{\g\in\G'
				\\
				\g\neq e}} \E [ \langle  \vh(\rvy) - \E [ \vh(\rvy)]   ,  \vh(\g\Glact[\vsY] \rvy) - \E [\vh(\rvy )] \rangle \notag                                                                    \\
				                       & = \frac{\E [\norm{\overline{\vh}(\rvy)}^2]}{|\G'|} + \frac{(|\G'|-1)\setsym{\G'}{\vh} }{|\G'|}\leq (1+C_0)  \frac{\E [\norm{\overline{\vh}(\rvy)}^2]}{|\G'|}.
			\end{align}

		}

	Hence we get the result.

\end{proof}
We focus now on a concentration bound for indicator functions $z=\one_\sA$ for any event $\sA\in \Sigma_{\vsX}$. We define
\begin{align}
	\label{eq:ZA_equiv}
	\small
	Z_A & :=   \EEX[\one_{\sA}] -  \PP[\rvx \in \sA] =\frac{1}{|\G'|}\sum_{g\in \G'} (\one_{g^{{-}1}\Glact[\vsX] \sA}(\rvx) -  \E [\one_{g^{{-}1}\Glact[\vsX] \sA}(\rvx)]) \notag \\
	    & = \frac{1}{|\G'|}\sum_{g\in \G'} \big(\one_{g^{{-}1}\Glact[\vsX] \sA}(\rvx) -  \PP[
			\rvx \in \sA]\big)= \bigg(\frac{1}{|\G'|}\sum_{g\in \G'}\one_{g^{{-}1}\Glact[\vsX] \sA}(\rvx)\bigg) -  \PP[\rvx \in \sA].
\end{align}
Note that we always have $|Z_\sA|\leq 1$ but this bound can be quite conservative as we could get a much sharper bound for some events $\sA$. We denote by $\setsym{\G',\infty}{\sA}$ the smallest deterministic upper-bound on
$
	\frac{1}{|\G'|}\sum_{g\in \G'} \one_{g^{{-}1}\Glact[\vsX] \sA}(\rvx)
$
(For instance when $\sA$ is an antisymmetric event, then we have $\setsym{\G',\infty}{\sA} = 1/|\G'|$). Then we have
\begin{equation}
	\label{eq:hoeffding_bound_condition}
	- \PP[\rvx \in \sA]\leq Z_\sA \leq \setsym{\G',\infty}{\sA} - \PP[\rvx \in \sA].
\end{equation}
Define also
\begin{align}
	\small
	\Upsilon_{\G',X}(\sA)
	 & := \PP(\rvx\in \sA)(1-\PP(\rvx\in \sA)) + (|\G'|-1)\, \big(\setsym{\G'}{\sA}  - \PP[\rvx\in \sA]\big)\PP[\rvx\in \sA].
\end{align}
\begin{lemma}\label{lem:bernsteiniidProbX}
	Let the distribution of $\rvx$ be $\G'$-invariant. Then for any $\sA\in \Sigma_{\vsX}$ and any $\delta \in (0,1)$, it holds w.p.a.l. $1-\delta$
	\begin{equation*}
		\small
		\abs{\EEX[\one_\sA]-\E [\one_\sA(\rvx)]} \leq \big\vert\setsym{\G',\infty}{\sA} - \PP[\rvx \in \sA]\big\vert \frac{\log 2\delta^{-1}}{\samplesize}+ \sqrt{\frac{\Upsilon_{\G',\rvx}(\sA)}{|\G'|}}\sqrt{2  \frac{\log 2\delta^{-1}}{\samplesize} }.
	\end{equation*}
	% $$
	% \abs{\EEX[\one_\sA]-\E [\one_\sA(\rvx)]} \leq  \setsym{\G,\infty}{A} \frac{\log 2\delta^{-1}}{\samplesize}+ \sqrt{\Upsilon_{\G,X}(A)}\sqrt{2  \frac{\log 2\delta^{-1}}{ |\G|\samplesize} }.
	% $$  
	%$\PP(\rvx\in A)\geq \sqrt{ \frac{\log 2\delta^{-1}}{\samplesize} } $ and
	Assume in addition that $g\Glact \sA \cap \sA =\emptyset$ for all $g\in\G'\setminus \{e\}$. Then it holds w.p.a.l. $1-\delta$
	%  $$
	% \frac{\abs{\EEX[\one_\sA]-\E [\one_\sA(X)]}}{\E [\one_\sA(X)]}\leq  \frac{\log \delta^{-1}}{\samplesize \PP[X\in \G\Glact A]}
	% %\frac{\log 2\delta^{-1}}{\samplesize} 
	% \bigvee \sqrt{2  \frac{\log 2\delta^{-1}}{|\G|\samplesize} } \sqrt{\frac{ 1 - \PP[X\in \G \Glact A]}{\PP[X\in \G \Glact A]} }  .
	% $$
	\begin{equation*}
		\small
		\frac{\abs{\EEX[\one_\sA]-\E [\one_\sA(\rvx)]}}{\E [\one_\sA(\rvx)]}\leq
		\left(\frac{ 1 - \PP[\rvx\in \G' \Glact \sA]}{\PP[\rvx\in \G' \Glact \sA]}\right)
		%\left(\frac{1}{\PP[\rvx \in \G\Glact A ]} - 1\right) 
		{ \frac{  \log 2\delta^{-1}}{\samplesize} }  +   \sqrt{2  \frac{\log  2\delta^{-1}}{ \samplesize} } \sqrt{\frac{ 1 - \PP[\rvx\in \G' \Glact \sA]}{\PP[\rvx\in \G' \Glact \sA]}}  .
	\end{equation*}
	%Similar bounds hold valid for $Y$.
	% For any $B\in \Sigma_{\vsY}$, it holds w.p.a.l. $1-\delta$
	%    $$
	% \abs{\EEY[\one_B]-\E [\one_B(Y)]} \leq  2 \frac{\log 2\delta^{-1}}{n}+ \sqrt{\PP[Y\in B](1 - \PP[Y\in B])}\sqrt{2  \frac{\log 2\delta^{-1}}{n} }.
	% $$
	If the distribution of $\rvy$ is $\G'$-invariant, then an identical result is immediately available for $\rvy$ by the same proof argument.
\end{lemma}
%Note that if the distribution of $\rvy$ is $\G$-invariant, then an identical result is immediately available for $\rvy$. 
\begin{remark}
	Using the standard empirical mean estimator that does not take advantage of $\G$-invariance, we obtain a concentration bound with a slower rate. For example, for an antisymmetric event $A$, we would achieve, w.p.a.l. $1-\delta$, the following result:
	\begin{equation*}
		\small
		\frac{\abs{\EEX[\one_\sA]-\E [\one_\sA(\rvx)]}}{\E [\one_\sA(\rvx)]}\leq
		\left(\frac{ 1 - \PP[\rvx \in \sA]}{\PP[\rvx \in \sA]}\right)
		%\left(\frac{1}{\PP[\rvx \in \G\Glact A ]} - 1\right) 
		{ \frac{  \log 2\delta^{-1}}{\samplesize} }  +   \sqrt{2  \frac{\log  2\delta^{-1}}{ \samplesize} } \sqrt{\frac{ 1 - \PP[\rvx \in \sA]}{\PP[\rvx \in \sA]}}  .
	\end{equation*}
	Specifically, leveraging $\G'$-invariance allows us to replace $\PP[\rvx \in \sA]$ with $\PP[\rvx \in \G' \Glact[\vsX] \sA]$, which represents the probability of the entire orbit of $A$ under the action of $\G'$. This becomes particularly interesting when $\PP[\rvx \in \sA] \ll \PP[\rvx \in \G' \Glact[\vsX] \sA]$, especially in the case of rare events where $\PP[\rvx \in \sA] \approx 0$.
	%ADD Discussion comparison w.r.t standard estimator!!! we replaced rare event $A$ by the whole orbit $\G\Glact A$ by exploiting $\G$-invariance.}
\end{remark}
\begin{proof}
	Since $\mux$ is $\G'$-invariant, we have $\E [\one_{g^{{-}1} \Glact \sA}(\rvx)] = \PP[\rvx \in \sA]$ and $\mathrm{Var}(\one_{g^{{-}1} \Glact \sA}(\rvx)) = \mathrm{Var}(\one_{\sA}(\rvx)) = \PP[\rvx \in \sA](1-\PP[\rvx \in \sA])$, for any $g\in \G'$.
	Hence
	\begin{align*}
		\small
		\EEX[\one_\sA]- \E [\one_\sA(\rvx)] & = \frac{1}{\samplesize}\sum_{i=1}^{\samplesize} Z_i
		\;\; \text{with}\;\;
		Z_i=\frac{1}{|\G'|}\sum_{g\in\G'}\one_{g^{{-}1} \Glact \sA}(\rvx_i) - \E [\one_{g^{{-}1} \Glact \sA}(\rvx_i)],\; \forall i \in [\samplesize
		].
	\end{align*}
	The $Z_i$'s are i.i.d. copies of $Z= Z_\sA$. In view of \eqref{eq:hoeffding_bound_condition}, we can apply Hoeffding's inequality \citet[Theorem 2.16]{BercuRiu}. We get for any $\delta\in (0,1)$ w.p.a.l $1-\delta$
	\begin{equation}
		\label{eq:equiv_empmean_Hoeffding}
		\small
		\abs{\EEX[\one_\sA]-\E [\one_\sA(\rvx)]} \leq \setsym{\G',\infty}{\sA}\sqrt{\frac{\log 2\delta^{-1}}{2\samplesize}}.
	\end{equation}
	We propose to prove another bound based on application of Bernstein's inequality.
	We first prove an improved bound on $\mathrm{Var}(Z)$ as compared to the standard empirical mean estimator which does not exploit $\G$-invariance.  Indeed we have
	\begin{equation}
		\small
		\begin{aligned}
			\small
			\mathrm{Var}(Z) & = \frac{1}{|\G'|^2}\left(\sum_{g\in\G'} \mathrm{Var}(\one_{g^{{-}1}\Glact \sA}(\rvx)) + \sum_{g\neq g'}  \mathrm{Cov}\bigg(\one_{g^{{-}1}\Glact \sA}(\rvx),\one_{(g')^{-1}\Glact \sA}(\rvx)\bigg)   \right)\notag \\
			                & = \frac{\PP(\rvx \in \sA)(1-\PP(\rvx \in \sA))}{|\G'|} + \frac{1}{|\G'|^2} \sum_{g\neq g'}  \mathrm{Cov}\bigg(\one_{g^{{-}1}\Glact \sA}(\rvx),\one_{(g')^{-1}\Glact \sA}(\rvx)\bigg).
			% \notag\\
			 % &= \frac{\PP(\rvx\in A)(1-\PP(\rvx\in A)) + (|\G|-1)\, \setsym{\G'}{A} }{|\G|}.
		\end{aligned}
	\end{equation}

	Next, using again that $\PP_X$ is $\G$-invariant, we get for any $g,g'\in \G'$
	\begin{align}
		\small
		\mathrm{Cov}\bigg(\one_{g^{{-}1}\Glact \sA}(\rvx), & \one_{(g')^{-1}\Glact \sA}(\rvx)\bigg)  =                                           \\ &\PP[\rvx \in g^{{-}1}\Glact \sA \cap (g')^{-1}\Glact \sA ] - \PP[\rvx \in g^{{-}1}\Glact \sA ] \, \PP[\rvx \in (g')^{-1}\Glact \sA ]\notag \\
		                                                   & = \PP[\rvx \in g^{{-}1}\Glact \sA \cap (g')^{-1}\Glact \sA ]- \PP[\rvx \in \sA ]^2.
	\end{align}
	Using again the invariance assumption, we note that
	$$
		\small
		\begin{aligned}
			\small
			\sum_{g\neq g'}  \mathrm{Cov}\big(\one_{g^{{-}1}\Glact \sA}(\rvx),\one_{(g')^{-1}\Glact \sA}(\rvx)\big) =
			% \\ 
			|\G'|
			\big(
			\sum_{g\in\G',g\neq e} \PP[\rvx \in \sA \cap g\Glact \sA ]
			\big)
			-  |\G'| (|\G'|-1)\PP[\rvx \in \sA ]^2
		\end{aligned}
	$$
	Consequently by definition of $\setsym{\G'}{A}$ in \eqref{eq:setsym} and \eqref{eq:setsym_bis}, we get
	$$
		\small
		\sum_{g\in\G',g\neq e} \PP[\rvx \in \sA \cap g\Glact A ] = (|\G'| - 1)\, \setsym{\G'}{A}\, \PP(\rvx \in \sA).
	$$
	Combining the last four displays, we get
	\begin{align}
		\label{eq:VarZ_event_invariant}
		\small
		\mathrm{Var}(Z)
		 & = \tfrac{\PP(\rvx \in \sA)(1-\PP(\rvx \in \sA)) + (|\G'|-1)\, (\setsym{\G'}{A}  - \PP[\rvx \in \sA])\PP[\rvx \in \sA] }{|\G'|} = \frac{\Upsilon_{\G',\rvx}(A)}{|\G'|}.
	\end{align}
	We note that for any $p\geq 3$
	$$
		\sum_{i=1}^{\samplesize} \E [\big(\max(0,Z_i)\big)^p] \leq \frac{p!}{2} \max\big(0,  \setsym{\G',\infty}{\sA} - \PP[\rvx \in \sA]\big)^{p-2} \, N \, \mathrm{Var}(Z).
	$$
	Then \citet[Theorem 2.1]{BercuRiu} gives w.p.a.l. $1-\delta$
	\begin{equation*}
		\small
		\EEX[\one_\sA]-\E [\one_\sA(\rvx)] \leq \max\big(0,  \setsym{\G',\infty}{\sA} - \PP[\rvx \in \sA]\big) \frac{\log \delta^{-1}}{\samplesize}+ \sqrt{\mathrm{Var}(Z)}\sqrt{2  \frac{\log \delta^{-1}}{\samplesize} }.
	\end{equation*}
	Applying the same reasoning to variables $-Z_1,\ldots,-Z_{\samplesize}$ and an union bound gives gives w.p.a.l. $1-2\delta$
	\begin{equation}
		\small
		\abs{\EEX[\one_\sA]-\E [\one_\sA(\rvx)]} \leq \big\vert\setsym{\G',\infty}{\sA} - \PP[\rvx \in \sA]\big\vert \frac{\log \delta^{-1}}{\samplesize}+ \sqrt{\mathrm{Var}(Z)}\sqrt{2  \frac{\log \delta^{-1}}{\samplesize} }.
	\end{equation}
	Next, we note that when $g\Glact \sA \cap \sA =\emptyset$ for all $g\in\G'\setminus \{e\}$, then $\setsym{\G'}{\sA}=0$ and $\PP[\rvx \in \sA] = \PP[\rvx\in \G' \Glact \sA]/|\G'|$. Consequently we get
	$$
		\small
		\Upsilon_{\G',\rvx}(\sA) = \tfrac{\PP[\rvx\in \G' \Glact \sA] (1- \PP[\rvx\in \G' \Glact \sA])}{{\color{black} |\G'|}} \quad \text{and} \quad
		\tfrac{\setsym{\G',\infty}{\sA} - \PP[\rvx \in \sA]}{\PP[\rvx \in \sA]} = \frac{1}{\PP[\rvx \in \G'\Glact \sA ]} - 1.%{|\G|^2}.
	$$
	Hence under the additional assumptions, dividing by $\E [\one_\sA(\rvx)] = \PP[\rvx \in \sA]$ gives w.p.a.l. $1-2\delta$
	$$
		\small
		\frac{\abs{\EEX[\one_\sA]-\E [\one_\sA(\rvx)]}}{\E [\one_\sA(\rvx)]}\leq \left(\frac{1}{\PP[\rvx \in \G'\Glact \sA ]} - 1\right) {\color{black} \frac{\log \delta^{-1}}{\samplesize} }  +   \sqrt{2  \frac{\log  \delta^{-1}}{ \samplesize} } \sqrt{\frac{ 1 - \PP[\rvx\in \G' \Glact \sA]}{\PP[\rvx\in \G' \Glact \sA]}}  .
	$$
	Replacing $\delta$ by $\delta/2$ gives w.p.a.l. $1-\delta$
	\begin{equation}
		\small
		\frac{\abs{\EEX[\one_\sA]-\E [\one_\sA(\rvx)]}}{\E [\one_\sA(\rvx)]}\leq \left(\frac{1}{\PP[\rvx \in \G'\Glact \sA ]} - 1\right) {\color{black} \frac{  \log 2\delta^{-1}}{\samplesize} }  +   \sqrt{2  \frac{\log  2\delta^{-1}}{ \samplesize} } \sqrt{\frac{ 1 - \PP[\rvx\in \G' \Glact \sA]}{\PP[\rvx\in \G'\Glact \sA]}}  .
	\end{equation}
\end{proof}

\begin{proposition}\label{prop:con_ineq_ps}
	Let $A_i$, $i\in[\samplesize]$ be i.i.d copies of a random variable $A$ in a separable Hilbert space with norm $\norm{\cdot}$. If there exist constants $L>0$ and $\sigma>0$ such that for every $m\geq2$, $\E \norm{A}^m\leq \frac{1}{2} m! L^{m-2} \sigma^2$,
	then with probability at least $1-\delta$
	\begin{equation}\label{eq:con_ineq_ps}
		\left\|\frac{1}{\samplesize}\sum_{i\in[\samplesize]}A_i - \E A \right\|\leq \frac{4\sqrt{2}}{\sqrt{\samplesize}}  \sqrt{ \sigma^2 + \frac{L^2}{\samplesize} } \log\frac{2}{\delta}.
	\end{equation}
\end{proposition}

\begin{lemma}[(Sub-Gaussian random variable) Lemma 5.5. in \cite{vershynin2011introduction}]
	\label{lem:sub-gaussian-def}
	Let $Z$ be a random variable. Then, the following assertions are equivalent with parameters $K_i>0$ differing from each other by at most an absolute constant factor.
	\begin{enumerate}
		\item Tails:
		      $\sP \{ |Z| > t \} \le \exp(1-t^2/K_1^2)$ for all $t \ge 0$;
		\item Moments:
		      $(\E |Z|^p)^{1/p} \le K_2 \sqrt{p}$ for all $p \ge 1$;
		\item Super-exponential moment:
		      $\E \exp(Z^2/K_3^2) \le 2$.
	\end{enumerate}
	A random variable $Z$ satisfying any of the above assertions is called a sub-Gaussian random variable. We will denote by $K_3$ the sub-Gaussian norm.
	%  Moreover, if $\E X = 0$ then properties 1--3 are also equivalent 
	%  to the following one: 
	%  \begin{enumerate} \setcounter{enumi}{3}
	%    \item[(d)] Moment generating function:
	%      $\E \exp(tX) \le \exp(t^2 K_4^2)$ for all $t \in \R$.
	% \end{enumerate}
\end{lemma}
Consequently, a sub-Gaussian random variable satisfies the following equivalence of moments property. There exists an absolute constant $c>0$ such that for any $m\geq 2$,
$$
	\big(\E |Z|^m \big)^{1/m} \leq c K_3\sqrt{m} \big(\E |Z|^2 \big)^{1/2}.
$$

%\karim{in the current form of sub-gaussian condition, $K\propto \sqrt{d}$. If the components of $\embXp(\cdot)$ are uniformly bounded by $c$, then $K^2 = c^2 d$.}
\begin{lemma}\label{lem:bernsteiniidMean-subgaussian}
	Assume that $Y$ is sub-Gaussian with sub-Gaussian norm $K$.
	%Let there exists finite absolute constants $c_,c_Y>0$ such that $\norm{\embXp}_{\infty}<c_X$ and $\norm{\embYp}_{\infty}<c_Y$. 
	We set $\sigma^2_{\theta }(Y):= \mathrm{Var}(\norm{Y - \E[\rvy]})$. Then there exists an absolute constant $C>0$ such that for any $\delta \in (0,1)$, it holds w.p.a.l. $1-\delta$
	$$
		\norm{\EEY[\rvy]- \E[\rvy]} \leq  \frac{C}{\sqrt{\samplesize}} \sqrt{\sigma^2(\rvy) + \frac{K^2}{\samplesize}}  \log(2\delta^{-1}).
	$$
\end{lemma}

\begin{proof}
	Set $Z:=\norm{\rvy- \E \rvy }$ and we recall that $\sigma^2(\rvy):= \mathrm{Var}(\norm{\rvy - \E[\rvy]})$. We check that the moment condition,
	$$
		\E Z^m\leq \frac{1}{2} m! L^{m-2} \sigma^2(\rvy)^2,\quad \forall m\geq 2,
	$$
	for some constant $L>0$ to be specified.

	The condition is obviously satisfied for $m=2$. Next for any $m\geq 3$, the Cauchy-Schwarz inequality and the equivalence of moment property give
	$$
		\E Z^m\leq \left(\E Z^{2(m-2)} \right)^{1/2} \left(\E Z^{4} \right)^{1/2}\leq 4 K_3^2  \sigma^2_{\theta }(Y)^2  \left(\E Z^{2(m-2)} \right)^{1/2}.
	$$
	Next, by homogeneity, rescaling $Z$ to $Z/K_1$ we can assume that $K_1=1$ in Lemma \ref{lem:sub-gaussian-def}. We recall that if $Z$  is in addition non-negative random variable, then for every integer $p\geq 1$, we have
	$$
		\E Z^p
		= \int_0^\infty \sP \{ Z \ge t \} \, p t^{p-1} \, dt
		\le \int_0^\infty e^{1-t^2} p t^{p-1} \, dt
		= \big( \frac{ep}{2} \big) \Gamma \big( \frac{p}{2} \big).
	$$
	With $p=2(m-2)$, we get that $ \E Z^p \leq e(m-2) \Gamma \big( m-2 \big) = e (m-2)! = e m!/2.
	$
	Using again Lemma \ref{lem:sub-gaussian-def}, we can take $L=c K$ for some large enough absolute constant $c>0$. Then Proposition \ref{prop:con_ineq_ps} gives the result.

\end{proof}

%% file: main_vk.bbl
\begin{thebibliography}{93}
\providecommand{\natexlab}[1]{#1}
\providecommand{\url}[1]{\texttt{#1}}
\expandafter\ifx\csname urlstyle\endcsname\relax
  \providecommand{\doi}[1]{doi: #1}\else
  \providecommand{\doi}{doi: \begingroup \urlstyle{rm}\Url}\fi

\bibitem[Babai \& R{\'o}nyai(1990)Babai and R{\'o}nyai]{babai1990computing}
Babai, L. and R{\'o}nyai, L.
\newblock Computing irreducible representations of finite groups.
\newblock \emph{Mathematics of computation}, 55\penalty0 (192):\penalty0 705--722, 1990.

\bibitem[Baker(1973)]{Baker1973-ns}
Baker, C.~R.
\newblock Joint measures and cross-covariance operators.
\newblock \emph{Trans. Am. Math. Soc.}, 186:\penalty0 273--273, 1973.

\bibitem[Bao et~al.(2022)Bao, Nagano, and Nozawa]{bao2022surrogate}
Bao, H., Nagano, Y., and Nozawa, K.
\newblock On the surrogate gap between contrastive and supervised losses.
\newblock In \emph{International conference on machine learning}, pp.\  1585--1606. PMLR, 2022.

\bibitem[Behboodi et~al.(2022)Behboodi, Cesa, and Cohen]{behboodi2022pac}
Behboodi, A., Cesa, G., and Cohen, T.~S.
\newblock A pac-bayesian generalization bound for equivariant networks.
\newblock \emph{Advances in Neural Information Processing Systems}, 35:\penalty0 5654--5668, 2022.

\bibitem[Bengio et~al.(2013)Bengio, Courville, and Vincent]{bengio2013representation}
Bengio, Y., Courville, A., and Vincent, P.
\newblock Representation learning: A review and new perspectives.
\newblock \emph{IEEE transactions on pattern analysis and machine intelligence}, 35\penalty0 (8):\penalty0 1798--1828, 2013.

\bibitem[Bercu et~al.(2015)Bercu, Delyon, and Rio]{BercuRiu}
Bercu, B., Delyon, B., and Rio, E.
\newblock \emph{{Concentration Inequalities for Sums and Martingales}}.
\newblock SpringerBriefs in Mathematics. {Springer}, 2015.

\bibitem[Bietti et~al.(2021)Bietti, Venturi, and Bruna]{bietti2021sample}
Bietti, A., Venturi, L., and Bruna, J.
\newblock On the sample complexity of learning under geometric stability.
\newblock \emph{Advances in neural information processing systems}, 34:\penalty0 18673--18684, 2021.

\bibitem[Bledt et~al.(2018)Bledt, Wensing, Ingersoll, and Kim]{bledt2018contact}
Bledt, G., Wensing, P.~M., Ingersoll, S., and Kim, S.
\newblock Contact model fusion for event-based locomotion in unstructured terrains.
\newblock In \emph{2018 IEEE International Conference on Robotics and Automation (ICRA)}, pp.\  4399--4406. IEEE, 2018.

\bibitem[Brandstetter et~al.(2023)Brandstetter, Berg, Welling, and Gupta]{brandstetter2022clifford}
Brandstetter, J., Berg, R. v.~d., Welling, M., and Gupta, J.~K.
\newblock Clifford neural layers for pde modeling.
\newblock In \emph{International Conference on Learning Representations}, 2023.

\bibitem[Bronstein et~al.(2021)Bronstein, Bruna, Cohen, and Veli{\v{c}}kovi{\'c}]{bronstein2021geometric}
Bronstein, M.~M., Bruna, J., Cohen, T., and Veli{\v{c}}kovi{\'c}, P.
\newblock Geometric deep learning: Grids, groups, graphs, geodesics, and gauges.
\newblock \emph{arXiv preprint arXiv:2104.13478}, 2021.

\bibitem[Cartan(1952)]{cartan1952theorie}
Cartan, {\'E}.
\newblock \emph{La th{\'e}orie des groupes finis et continus et l'analysis situs}.
\newblock Number~42 in M\'emorial des sciences math\'ematiques. Gauthier-Villars, 1952.

\bibitem[Cayley(1854)]{cayley1854vii}
Cayley, A.
\newblock Vii. on the theory of groups, as depending on the symbolic equation $\theta$n= 1.
\newblock \emph{The London, Edinburgh, and Dublin Philosophical Magazine and Journal of Science}, 7\penalty0 (42):\penalty0 40--47, 1854.
\newblock \doi{10.1080/14786445408647421}.

\bibitem[Cesa et~al.(2022)Cesa, Lang, and Weiler]{cesa2022program}
Cesa, G., Lang, L., and Weiler, M.
\newblock A program to build e (n)-equivariant steerable cnns.
\newblock In \emph{International conference on learning representations}, 2022.

\bibitem[Chen et~al.(2020)Chen, Kornblith, Norouzi, and Hinton]{chen2020simple}
Chen, T., Kornblith, S., Norouzi, M., and Hinton, G.
\newblock A simple framework for contrastive learning of visual representations.
\newblock In \emph{International conference on machine learning}, pp.\  1597--1607. PmLR, 2020.

\bibitem[Cole et~al.(2022)Cole, Yang, Wilber, Mac~Aodha, and Belongie]{Cole_2022_CVPR}
Cole, E., Yang, X., Wilber, K., Mac~Aodha, O., and Belongie, S.
\newblock When does contrastive visual representation learning work?
\newblock In \emph{Proceedings of the IEEE/CVF Conference on Computer Vision and Pattern Recognition (CVPR)}, pp.\  14755--14764, June 2022.

\bibitem[Dangovski et~al.(2022)Dangovski, Jing, Loh, Han, Srivastava, Cheung, Agrawal, and Solja{\v{c}}i{\'c}]{dangovski2022equivariant}
Dangovski, R., Jing, L., Loh, C., Han, S., Srivastava, A., Cheung, B., Agrawal, P., and Solja{\v{c}}i{\'c}, M.
\newblock Equivariant contrastive learning.
\newblock In \emph{International Conference on Learning Representations}, 2022.

\bibitem[Dixon(1970)]{dixon1970computing}
Dixon, J.~D.
\newblock Computing irreducible representations of groups.
\newblock \emph{Mathematics of Computation}, 24\penalty0 (111):\penalty0 707--712, 1970.

\bibitem[Donhauser et~al.(2021)Donhauser, Wu, and Yang]{donhauser2021rotational}
Donhauser, K., Wu, M., and Yang, F.
\newblock How rotational invariance of common kernels prevents generalization in high dimensions.
\newblock In \emph{International Conference on Machine Learning}, pp.\  2804--2814. PMLR, 2021.

\bibitem[Dresselhaus et~al.(2007)Dresselhaus, Dresselhaus, and Jorio]{dresselhaus2007group_theory}
Dresselhaus, M.~S., Dresselhaus, G., and Jorio, A.
\newblock \emph{Group theory: application to the physics of condensed matter}.
\newblock Springer, 2007.
\newblock \doi{10.1007/978-3-540-32899-5}.

\bibitem[Eckart \& Young(1936)Eckart and Young]{eckart1936approximation}
Eckart, C. and Young, G.
\newblock The approximation of one matrix by another of lower rank.
\newblock \emph{Psychometrika}, 1\penalty0 (3):\penalty0 211--218, 1936.

\bibitem[Elesedy(2021)]{elesedy2021provablystrict}
Elesedy, B.
\newblock Provably strict generalisation benefit for invariance in kernel methods.
\newblock In \emph{Advances in Neural Information Processing Systems}, volume~34, pp.\  17273--17283. Curran Associates, Inc., 2021.
\newblock URL \url{https://proceedings.neurips.cc/paper_files/paper/2021/file/8fe04df45a22b63156ebabbb064fcd5e-Paper.pdf}.

\bibitem[Elesedy(2023)]{elesedy2025symmetrygeneralisationmachinelearning}
Elesedy, B.
\newblock \emph{Symmetry and generalisation in machine learning}.
\newblock PhD thesis, University of Oxford, 2023.

\bibitem[Elesedy \& Zaidi(2021)Elesedy and Zaidi]{elesedy2021provablystrict2}
Elesedy, B. and Zaidi, S.
\newblock Provably strict generalisation benefit for equivariant models.
\newblock In Meila, M. and Zhang, T. (eds.), \emph{Proceedings of the 38th International Conference on Machine Learning}, volume 139 of \emph{Proceedings of Machine Learning Research}, pp.\  2959--2969. PMLR, 18--24 Jul 2021.
\newblock URL \url{https://proceedings.mlr.press/v139/elesedy21a.html}.

\bibitem[Feldman et~al.(2023)Feldman, Bates, and Romano]{feldman2023calibrated}
Feldman, S., Bates, S., and Romano, Y.
\newblock Calibrated multiple-output quantile regression with representation learning.
\newblock \emph{Journal of Machine Learning Research}, 24\penalty0 (24):\penalty0 1--48, 2023.

\bibitem[Fukumizu et~al.(2004)Fukumizu, Bach, and Jordan]{fukumizu2004dimensionality}
Fukumizu, K., Bach, F.~R., and Jordan, M.~I.
\newblock Dimensionality reduction for supervised learning with reproducing kernel hilbert spaces.
\newblock \emph{Journal of Machine Learning Research}, 5\penalty0 (Jan):\penalty0 73--99, 2004.

\bibitem[Gautier(1997)]{gautier1997dynamic}
Gautier, M.
\newblock Dynamic identification of robots with power model.
\newblock In \emph{Proceedings of international conference on robotics and automation}, volume~3, pp.\  1922--1927. IEEE, 1997.

\bibitem[Gilardi et~al.(2002)Gilardi, Bengio, and Kanevski]{gilardi2002conditional}
Gilardi, N., Bengio, S., and Kanevski, M.
\newblock Conditional gaussian mixture models for environmental risk mapping.
\newblock In \emph{Proceedings of the 12th IEEE workshop on neural networks for signal processing}, pp.\  777--786. IEEE, 2002.

\bibitem[Golubitsky et~al.(2012)Golubitsky, Stewart, and Schaeffer]{golubitsky2012singularities_groups_bifurcation}
Golubitsky, M., Stewart, I., and Schaeffer, D.~G.
\newblock \emph{Singularities and Groups in Bifurcation Theory: Volume II}, volume~69.
\newblock Springer, 2012.
\newblock \doi{10.1007/978-1-4612-4574-2}.

\bibitem[Gupta et~al.(2023)Gupta, Robinson, Lim, Villar, and Jegelka]{gupta2023structuring}
Gupta, S., Robinson, J., Lim, D., Villar, S., and Jegelka, S.
\newblock Structuring representation geometry with rotationally equivariant contrastive learning.
\newblock In \emph{The Twelfth International Conference on Learning Representations}, 2023.

\bibitem[HaoChen et~al.(2021)HaoChen, Wei, Gaidon, and Ma]{chen2021provable}
HaoChen, J.~Z., Wei, C., Gaidon, A., and Ma, T.
\newblock Provable guarantees for self-supervised deep learning with spectral contrastive loss.
\newblock In \emph{Advances in Neural Information Processing Systems}, volume~34, pp.\  5000--5011. Curran Associates, Inc., 2021.
\newblock URL \url{https://proceedings.neurips.cc/paper_files/paper/2021/file/27debb435021eb68b3965290b5e24c49-Paper.pdf}.

\bibitem[HaoChen et~al.(2022)HaoChen, Wei, Kumar, and Ma]{haochen2022beyond}
HaoChen, J.~Z., Wei, C., Kumar, A., and Ma, T.
\newblock Beyond separability: Analyzing the linear transferability of contrastive representations to related subpopulations.
\newblock \emph{Advances in neural information processing systems}, 35:\penalty0 26889--26902, 2022.

\bibitem[Henaff(2020)]{henaff2020data}
Henaff, O.
\newblock Data-efficient image recognition with contrastive predictive coding.
\newblock In \emph{International conference on machine learning}, pp.\  4182--4192. PMLR, 2020.

\bibitem[Higgins et~al.(2018)Higgins, Amos, Pfau, Racaniere, Matthey, Rezende, and Lerchner]{higgins2018towards}
Higgins, I., Amos, D., Pfau, D., Racaniere, S., Matthey, L., Rezende, D., and Lerchner, A.
\newblock Towards a definition of disentangled representations.
\newblock \emph{arXiv preprint arXiv:1812.02230}, 2018.

\bibitem[Izbicki(2025)]{Izbicki2025}
Izbicki, R.
\newblock \emph{Machine Learning Beyond Point Predictions: Uncertainty Quantification}.
\newblock 1st edition, 2025.
\newblock ISBN 978-65-01-20272-3.

\bibitem[Izbicki \& B.~Lee(2017)Izbicki and B.~Lee]{izbicki2017converting}
Izbicki, R. and B.~Lee, A.
\newblock Converting high-dimensional regression to high-dimensional conditional density estimation.
\newblock \emph{Electronic Journal of Statistics}, 11\penalty0 (2):\penalty0 2800--2831, 2017.

\bibitem[Johnson et~al.(2023)Johnson, Hanchi, and Maddison]{johnson2022contrastive}
Johnson, D.~D., Hanchi, A.~E., and Maddison, C.~J.
\newblock Contrastive learning can find an optimal basis for approximately view-invariant functions.
\newblock In \emph{International Conference on Learning Representations}, 2023.

\bibitem[Kashinath et~al.(2021)Kashinath, Mustafa, Albert, Wu, Jiang, Esmaeilzadeh, Azizzadenesheli, Wang, Chattopadhyay, Singh, et~al.]{kashinath2021physics}
Kashinath, K., Mustafa, M., Albert, A., Wu, J., Jiang, C., Esmaeilzadeh, S., Azizzadenesheli, K., Wang, R., Chattopadhyay, A., Singh, A., et~al.
\newblock Physics-informed machine learning: case studies for weather and climate modelling.
\newblock \emph{Philosophical Transactions of the Royal Society A}, 379\penalty0 (2194):\penalty0 20200093, 2021.

\bibitem[Keurti et~al.(2023)Keurti, Pan, Besserve, Grewe, and Sch{\"o}lkopf]{keurti2023homomorphism}
Keurti, H., Pan, H.-R., Besserve, M., Grewe, B.~F., and Sch{\"o}lkopf, B.
\newblock Homomorphism autoencoder--learning group structured representations from observed transitions.
\newblock In \emph{International Conference on Machine Learning}, pp.\  16190--16215. PMLR, 2023.

\bibitem[Knapp(1986)]{Knapp1986}
Knapp, A.~W.
\newblock \emph{Representation Theory of Semisimple Groups, An Overview Based on Examples (PMS-36)}.
\newblock Princeton University Press, Princeton, 1986.

\bibitem[Kondor et~al.(2018)Kondor, Lin, and Trivedi]{kondor2018clebschgordannets}
Kondor, R., Lin, Z., and Trivedi, S.
\newblock Clebsch-gordan nets: a fully fourier space spherical convolutional neural network.
\newblock In \emph{Advances in Neural Information Processing Systems 31 (NeurIPS 2018)}, 2018.
\newblock URL \url{https://papers.nips.cc/paper_files/paper/2018/hash/a3fc981af450752046be179185ebc8b5-Abstract.html}.

\bibitem[Kostic et~al.(2024{\natexlab{a}})Kostic, Novelli, Grazzi, Lounici, and Pontil]{kostic2024learning}
Kostic, V.~R., Novelli, P., Grazzi, R., Lounici, K., and Pontil, M.
\newblock Learning invariant representations of time-homogeneous stochastic dynamical systems.
\newblock In \emph{The Twelfth International Conference on Learning Representations}, 2024{\natexlab{a}}.

\bibitem[Kostic et~al.(2024{\natexlab{b}})Kostic, Pacreau, Turri, Novelli, Lounici, and Pontil]{kostic2024neural}
Kostic, V.~R., Pacreau, G., Turri, G., Novelli, P., Lounici, K., and Pontil, M.
\newblock Neural conditional probability for uncertainty quantification.
\newblock In \emph{The Thirty-eighth Annual Conference on Neural Information Processing Systems}, 2024{\natexlab{b}}.

\bibitem[Kovachki et~al.(2023)Kovachki, Li, Liu, Azizzadenesheli, Bhattacharya, Stuart, and Anandkumar]{kovachki2023neural}
Kovachki, N., Li, Z., Liu, B., Azizzadenesheli, K., Bhattacharya, K., Stuart, A., and Anandkumar, A.
\newblock Neural operator: Learning maps between function spaces with applications to pdes.
\newblock \emph{Journal of Machine Learning Research}, 24\penalty0 (89):\penalty0 1--97, 2023.

\bibitem[Kovar et~al.(2002)Kovar, Gleicher, and Pighin]{kovar2002motion}
Kovar, L., Gleicher, M., and Pighin, F.
\newblock Motion graphs.
\newblock \emph{ACM Trans. Graph.}, 21\penalty0 (3):\penalty0 473--482, July 2002.
\newblock \doi{10.1145/566654.566605}.
\newblock URL \url{https://doi.org/10.1145/566654.566605}.

\bibitem[Le-Khac et~al.(2020)Le-Khac, Healy, and Smeaton]{le2020contrastive}
Le-Khac, P.~H., Healy, G., and Smeaton, A.~F.
\newblock Contrastive representation learning: A framework and review.
\newblock \emph{Ieee Access}, 8:\penalty0 193907--193934, 2020.

\bibitem[Lin et~al.(2024)Lin, Zhang, and Liu]{lin2024mutual_skeleton}
Lin, L., Zhang, J., and Liu, J.
\newblock Mutual information driven equivariant contrastive learning for 3d action representation learning.
\newblock \emph{IEEE Transactions on Image Processing}, 2024.

\bibitem[Liu et~al.(1994)Liu, Gortler, and Cohen]{liu1994hierarchical}
Liu, Z., Gortler, S.~J., and Cohen, M.~F.
\newblock Hierarchical spacetime control.
\newblock In \emph{Proceedings of the 21st annual conference on Computer graphics and interactive techniques}, pp.\  35--42. Association for Computing Machinery, 1994.
\newblock \doi{10.1145/192161.192169}.

\bibitem[Mackey(1980)]{mackey1980harmonic}
Mackey, G.~W.
\newblock {Harmonic analysis as the exploitation of symmetry–a historical survey}.
\newblock \emph{Bulletin (New Series) of the American Mathematical Society}, 3\penalty0 (1.P1):\penalty0 543 -- 698, 1980.

\bibitem[Maravgakis et~al.(2023)Maravgakis, Argiropoulos, Piperakis, and Trahanias]{maravgakis2023probabilistic}
Maravgakis, M., Argiropoulos, D.-E., Piperakis, S., and Trahanias, P.
\newblock Probabilistic contact state estimation for legged robots using inertial information.
\newblock In \emph{2023 IEEE International Conference on Robotics and Automation (ICRA)}, pp.\  12163--12169. IEEE, 2023.

\bibitem[Marchetti et~al.(2023)Marchetti, Tegn{\'e}r, Varava, and Kragic]{marchetti2023equivariant}
Marchetti, G.~L., Tegn{\'e}r, G., Varava, A., and Kragic, D.
\newblock Equivariant representation learning via class-pose decomposition.
\newblock In \emph{International Conference on Artificial Intelligence and Statistics}, pp.\  4745--4756. PMLR, 2023.

\bibitem[Mei et~al.(2021)Mei, Misiakiewicz, and Montanari]{mei2021learning}
Mei, S., Misiakiewicz, T., and Montanari, A.
\newblock Learning with invariances in random features and kernel models.
\newblock In \emph{Conference on Learning Theory}, pp.\  3351--3418. PMLR, 2021.

\bibitem[Mitra et~al.(2013)Mitra, Pauly, Wand, and Ceylan]{mitra2013symmetry}
Mitra, N.~J., Pauly, M., Wand, M., and Ceylan, D.
\newblock Symmetry in 3d geometry: Extraction and applications.
\newblock \emph{Computer Graphics Forum}, 32\penalty0 (6):\penalty0 1–23, February 2013.
\newblock URL \url{http://dx.doi.org/10.1111/cgf.12010}.

\bibitem[Mroueh et~al.(2015)Mroueh, Voinea, and Poggio]{mroueh2015learning}
Mroueh, Y., Voinea, S., and Poggio, T.~A.
\newblock Learning with group invariant features: A kernel perspective.
\newblock \emph{Advances in neural information processing systems}, 28, 2015.

\bibitem[Nagler \& Czado(2016)Nagler and Czado]{Nagler2016}
Nagler, T. and Czado, C.
\newblock Evading the curse of dimensionality in nonparametric density estimation with simplified vine copulas.
\newblock \emph{Journal of Multivariate Analysis}, 151:\penalty0 69--89, October 2016.
\newblock URL \url{http://dx.doi.org/10.1016/j.jmva.2016.07.003}.

\bibitem[Nistic{\`o} et~al.(2025)Nistic{\`o}, Soares, Amatucci, Fink, and Semini]{nistico2025muse}
Nistic{\`o}, Y., Soares, J. C.~V., Amatucci, L., Fink, G., and Semini, C.
\newblock Muse: A real-time multi-sensor state estimator for quadruped robots.
\newblock \emph{IEEE Robotics and Automation Letters}, 2025.

\bibitem[Olver(1993)]{olver1993applications}
Olver, P.~J.
\newblock \emph{Applications of Lie groups to differential equations}, volume 107.
\newblock Springer Science \& Business Media, 1993.

\bibitem[Oord et~al.(2018)Oord, Li, and Vinyals]{oord2018representation}
Oord, A. v.~d., Li, Y., and Vinyals, O.
\newblock Representation learning with contrastive predictive coding.
\newblock \emph{arXiv preprint arXiv:1807.03748}, 2018.

\bibitem[Ordo{\~n}ez-Apraez et~al.(2024)Ordo{\~n}ez-Apraez, Kostic, Turrisi, Novelli, Mastalli, Semini, and Pontil]{ordonez2024dynamics}
Ordo{\~n}ez-Apraez, D., Kostic, V., Turrisi, G., Novelli, P., Mastalli, C., Semini, C., and Pontil, M.
\newblock Dynamics harmonic analysis of robotic systems: Application in data-driven koopman modelling.
\newblock In \emph{6th Annual Learning for Dynamics \& Control Conference}, pp.\  1318--1329. PMLR, 2024.

\bibitem[Ordonez~Apraez(2026)]{ordonez2026symmlearning}
Ordonez~Apraez, D.~F.
\newblock Symmetric learning, 2026.
\newblock URL \url{https://github.com/Danfoa/symmetric_learning}.

\bibitem[Ordoñez-Apraez et~al.(2025)Ordoñez-Apraez, Turrisi, Kostic, Martin, Agudo, Moreno-Noguer, Pontil, Semini, and Mastalli]{ordonez2024morphological}
Ordoñez-Apraez, D., Turrisi, G., Kostic, V., Martin, M., Agudo, A., Moreno-Noguer, F., Pontil, M., Semini, C., and Mastalli, C.
\newblock Morphological symmetries in robotics.
\newblock \emph{The International Journal of Robotics Research}, 44\penalty0 (10--11):\penalty0 1743--1766, 2025.
\newblock \doi{10.1177/02783649241282422}.
\newblock URL \url{https://doi.org/10.1177/02783649241282422}.

\bibitem[Ozair et~al.(2019)Ozair, Lynch, Bengio, van~den Oord, Levine, and Sermanet]{ozair2019wasserstein}
Ozair, S., Lynch, C., Bengio, Y., van~den Oord, A., Levine, S., and Sermanet, P.
\newblock Wasserstein dependency measure for representation learning.
\newblock In Wallach, H., Larochelle, H., Beygelzimer, A., d\textquotesingle Alch\'{e}-Buc, F., Fox, E., and Garnett, R. (eds.), \emph{Advances in Neural Information Processing Systems}, volume~32. Curran Associates, Inc., 2019.
\newblock URL \url{https://proceedings.neurips.cc/paper_files/paper/2019/file/f9209b7866c9f69823201c1732cc8645-Paper.pdf}.

\bibitem[Pal et~al.(2017)Pal, Kannan, Arakalgud, and Savvides]{pal2017maxmargininvariantfeatures}
Pal, D., Kannan, A., Arakalgud, G., and Savvides, M.
\newblock Max-margin invariant features from transformed unlabelled data.
\newblock In \emph{Advances in Neural Information Processing Systems}, volume~30. Curran Associates, Inc., 2017.

\bibitem[Pedregosa et~al.(2011)Pedregosa, Varoquaux, Gramfort, Michel, Thirion, Grisel, Blondel, Prettenhofer, Weiss, Dubourg, et~al.]{pedregosa2011scikit}
Pedregosa, F., Varoquaux, G., Gramfort, A., Michel, V., Thirion, B., Grisel, O., Blondel, M., Prettenhofer, P., Weiss, R., Dubourg, V., et~al.
\newblock Scikit-learn: Machine learning in python.
\newblock \emph{the Journal of machine Learning research}, 12:\penalty0 2825--2830, 2011.

\bibitem[Ryu et~al.(2024)Ryu, Xu, Erol, Bu, Zheng, and Wornell]{ryu2024operator}
Ryu, J.~J., Xu, X., Erol, H. S.~M., Bu, Y., Zheng, L., and Wornell, G.~W.
\newblock Operator {SVD} with neural networks via nested low-rank approximation.
\newblock In \emph{Proceedings of the 41st International Conference on Machine Learning}, volume 235 of \emph{Proceedings of Machine Learning Research}, pp.\  42870--42905. PMLR, 2024.
\newblock URL \url{https://proceedings.mlr.press/v235/ryu24b.html}.

\bibitem[Salova et~al.(2019)Salova, Emenheiser, Rupe, Crutchfield, and D’Souza]{salova2019koopman}
Salova, A., Emenheiser, J., Rupe, A., Crutchfield, J.~P., and D’Souza, R.~M.
\newblock Koopman operator and its approximations for systems with symmetries.
\newblock \emph{Chaos: An Interdisciplinary Journal of Nonlinear Science}, 29\penalty0 (9), 2019.

\bibitem[Scott(1991)]{SCOTT1991}
Scott, D.~W.
\newblock Feasibility of multivariate density estimates.
\newblock \emph{Biometrika}, 78\penalty0 (1):\penalty0 197–205, 1991.
\newblock URL \url{http://dx.doi.org/10.1093/biomet/78.1.197}.

\bibitem[Shah \& Chandrasekaran(2012)Shah and Chandrasekaran]{shah2012group_cov_estimation}
Shah, P. and Chandrasekaran, V.
\newblock Group symmetry and covariance regularization.
\newblock In \emph{2012 46th Annual Conference on Information Sciences and Systems (CISS)}, pp.\  1--6. IEEE, 2012.

\bibitem[Smith(2013)]{Smith2013}
Smith, R.~C.
\newblock \emph{Uncertainty Quantification: Theory, Implementation, and Applications}.
\newblock Society for Industrial and Applied Mathematics, January 2013.
\newblock URL \url{http://dx.doi.org/10.1137/1.9781611973228}.

\bibitem[Song et~al.(2009)Song, Huang, Smola, and Fukumizu]{song2009hilbert}
Song, L., Huang, J., Smola, A., and Fukumizu, K.
\newblock Hilbert space embeddings of conditional distributions with applications to dynamical systems.
\newblock In \emph{Proceedings of the 26th Annual International Conference on Machine Learning}, pp.\  961--968, 2009.

\bibitem[Sugiyama et~al.(2012)Sugiyama, Suzuki, and Kanamori]{sugiyama2012density}
Sugiyama, M., Suzuki, T., and Kanamori, T.
\newblock \emph{Density ratio estimation in machine learning}.
\newblock Cambridge University Press, 2012.

\bibitem[Tahmasebi \& Jegelka(2023)Tahmasebi and Jegelka]{tahmasebi2023exact}
Tahmasebi, B. and Jegelka, S.
\newblock The exact sample complexity gain from invariances for kernel regression.
\newblock \emph{Advances in Neural Information Processing Systems}, 36, 2023.

\bibitem[Tosh et~al.(2021)Tosh, Krishnamurthy, and Hsu]{tosh2021contrastive}
Tosh, C., Krishnamurthy, A., and Hsu, D.
\newblock Contrastive learning, multi-view redundancy, and linear models.
\newblock In \emph{Proceedings of the 32nd International Conference on Algorithmic Learning Theory}, volume 132 of \emph{Proceedings of Machine Learning Research}, pp.\  1179--1206. PMLR, 16--19 Mar 2021.

\bibitem[Tsai et~al.(2020)Tsai, Zhao, Yamada, Morency, and Salakhutdinov]{tsai2020neuralPMD}
Tsai, Y.-H.~H., Zhao, H., Yamada, M., Morency, L.-P., and Salakhutdinov, R.~R.
\newblock Neural methods for point-wise dependency estimation.
\newblock In \emph{Advances in Neural Information Processing Systems}, volume~33, pp.\  62--72. Curran Associates, Inc., 2020.

\bibitem[Tsai et~al.(2021)Tsai, Ma, Yang, Zhao, Morency, and Salakhutdinov]{tsai2021self}
Tsai, Y.-H.~H., Ma, M.~Q., Yang, M., Zhao, H., Morency, L.-P., and Salakhutdinov, R.
\newblock Self-supervised representation learning with relative predictive coding.
\newblock In \emph{International Conference on Learning Representations}, 2021.

\bibitem[Turrisi et~al.(2024)Turrisi, Modugno, Amatucci, Kanoulas, and Semini]{turrisi2024sampling}
Turrisi, G., Modugno, V., Amatucci, L., Kanoulas, D., and Semini, C.
\newblock On the benefits of gpu sample-based stochastic predictive controllers for legged locomotion.
\newblock In \emph{2024 IEEE/RSJ International Conference on Intelligent Robots and Systems (IROS)}, pp.\  13757--13764, 2024.
\newblock \doi{10.1109/IROS58592.2024.10801698}.

\bibitem[Unger(2006)]{unger2006computing}
Unger, W.~R.
\newblock Computing the character table of a finite group.
\newblock \emph{Journal of Symbolic Computation}, 41\penalty0 (8):\penalty0 847--862, 2006.

\bibitem[van~der Pol et~al.(2020)van~der Pol, Worrall, van Hoof, Oliehoek, and Welling]{vanderpol2020mdp}
van~der Pol, E., Worrall, D., van Hoof, H., Oliehoek, F., and Welling, M.
\newblock Mdp homomorphic networks: Group symmetries in reinforcement learning.
\newblock In \emph{Advances in Neural Information Processing Systems}, volume~33, pp.\  4199--4210. Curran Associates, Inc., 2020.

\bibitem[Vershynin(2012)]{vershynin2011introduction}
Vershynin, R.
\newblock Introduction to the non-asymptotic analysis of random matrices.
\newblock In Eldar, Y.~C. and Kutyniok, G. (eds.), \emph{Compressed Sensing: Theory and Applications}, pp.\  210--268. Cambridge University Press, Cambridge, UK, 2012.

\bibitem[Waida et~al.(2023)Waida, Wada, And{\'e}ol, Nakagawa, Zhang, and Kanamori]{waida2023towards}
Waida, H., Wada, Y., And{\'e}ol, L., Nakagawa, T., Zhang, Y., and Kanamori, T.
\newblock Towards understanding the mechanism of contrastive learning via similarity structure: A theoretical analysis.
\newblock In \emph{Joint European Conference on Machine Learning and Knowledge Discovery in Databases}, pp.\  709--727. Springer, 2023.

\bibitem[Wang et~al.(2023)Wang, Zhu, Park, Jia, Su, Platt, and Walters]{wang2023general}
Wang, D., Zhu, X., Park, J.~Y., Jia, M., Su, G., Platt, R., and Walters, R.
\newblock A general theory of correct, incorrect, and extrinsic equivariance.
\newblock In \emph{Advances in Neural Information Processing Systems}, volume~36, 2023.

\bibitem[Wang(2021)]{wang2021incorporating}
Wang, R.
\newblock Incorporating symmetry into deep dynamics models for improved generalization.
\newblock In \emph{International Conference on Learning Representations (ICLR)}, 2021.

\bibitem[Wang et~al.(2021)Wang, Walters, and Yu]{wang2020incorporating}
Wang, R., Walters, R., and Yu, R.
\newblock Incorporating symmetry into deep dynamics models for improved generalization.
\newblock In \emph{International Conference on Learning Representations}, 2021.

\bibitem[Wang et~al.(2022{\natexlab{a}})Wang, Walters, and Yu]{wang2022approximately}
Wang, R., Walters, R., and Yu, R.
\newblock Approximately equivariant networks for imperfectly symmetric dynamics.
\newblock In \emph{International Conference on Machine Learning}, pp.\  23078--23091. PMLR, 2022{\natexlab{a}}.

\bibitem[Wang et~al.(2024{\natexlab{a}})Wang, Chen, Tang, Wu, and Zhu]{wang2024disentangled}
Wang, X., Chen, H., Tang, S., Wu, Z., and Zhu, W.
\newblock Disentangled representation learning.
\newblock \emph{IEEE Transactions on Pattern Analysis and Machine Intelligence}, 2024{\natexlab{a}}.
\newblock \doi{10.1109/TPAMI.2024.3420937}.

\bibitem[Wang et~al.(2024{\natexlab{b}})Wang, Hu, Gupta, Ye, Wang, and Jegelka]{wang2024understanding}
Wang, Y., Hu, K., Gupta, S., Ye, Z., Wang, Y., and Jegelka, S.
\newblock Understanding the role of equivariance in self-supervised learning.
\newblock In \emph{The Thirty-eighth Annual Conference on Neural Information Processing Systems}, 2024{\natexlab{b}}.

\bibitem[Wang et~al.(2022{\natexlab{b}})Wang, Luo, Li, Zhu, and Sch{\"o}lkopf]{wang2022spectral}
Wang, Z., Luo, Y., Li, Y., Zhu, J., and Sch{\"o}lkopf, B.
\newblock Spectral representation learning for conditional moment models.
\newblock \emph{arXiv preprint arXiv:2210.16525}, 2022{\natexlab{b}}.

\bibitem[Wasserman(2007)]{Wasserman2007-ek}
Wasserman, L.
\newblock \emph{All of Nonparametric Statistics}.
\newblock Springer Texts in Statistics. Springer, New York, NY, 1 edition, May 2007.

\bibitem[Weidmann(2012)]{weidmann2012linear}
Weidmann, J.
\newblock \emph{Linear operators in Hilbert spaces}, volume~68.
\newblock Springer Science \& Business Media, 2012.

\bibitem[Weiler et~al.(2023)Weiler, Forré, Verlinde, and Welling]{weiler2023EquivariantAndCoordinateIndependentCNNs}
Weiler, M., Forré, P., Verlinde, E., and Welling, M.
\newblock \emph{Equivariant and Coordinate Independent Convolutional Networks}.
\newblock World Scientific, 2023.
\newblock URL \url{https://maurice-weiler.gitlab.io/cnn_book/EquivariantAndCoordinateIndependentCNNs.pdf}.

\bibitem[Yang et~al.(2023)Yang, Dehmamy, Walters, and Yu]{yang2023latent}
Yang, J., Dehmamy, N., Walters, R., and Yu, R.
\newblock Latent space symmetry discovery.
\newblock In \emph{International Conference on Machine Learning}, 2023.

\bibitem[Yerxa et~al.(2024)Yerxa, Feather, Simoncelli, and Chung]{yerxa2024contrastive}
Yerxa, T., Feather, J., Simoncelli, E., and Chung, S.
\newblock Contrastive-equivariant self-supervised learning improves alignment with primate visual area it.
\newblock \emph{Advances in neural information processing systems}, 37:\penalty0 96045--96070, 2024.

\bibitem[Zhu et~al.(2022)Zhu, Wang, Biza, Su, Walters, and Platt]{zhu2022sample}
Zhu, X., Wang, D., Biza, O., Su, G., Walters, R., and Platt, R.
\newblock {Sample Efficient Grasp Learning Using Equivariant Models}.
\newblock In \emph{Proceedings of Robotics: Science and Systems}, New York City, NY, USA, June 2022.

\bibitem[Zou et~al.(2013)Zou, Hsu, Parkes, and Adams]{Zou2013}
Zou, J.~Y., Hsu, D.~J., Parkes, D.~C., and Adams, R.~P.
\newblock Contrastive learning using spectral methods.
\newblock In \emph{Advances in Neural Information Processing Systems}, volume~26, pp.\  2238--2246. Curran Associates, Inc., 2013.

\end{thebibliography}
